\documentclass{article}

\usepackage{times}
\usepackage{graphicx}
\usepackage{subfigure} 
\usepackage[round]{natbib}
\usepackage{algorithm}
\usepackage{algorithmic}
\usepackage{textcomp}
\usepackage{bm}

\usepackage{amsmath}
\usepackage{amssymb}
\usepackage{authblk}
\usepackage{lscape}
\usepackage{color}
\usepackage{amsthm}
\usepackage{multirow}
\usepackage{afterpage}
\usepackage[toc,page]{appendix}

\definecolor{orange}{rgb}{1,0.5,0}

\newcommand{\vect}[1]{\boldsymbol{\mathbf{#1}}}

\newtheorem{theorem}{Theorem}

\def\abovestrut#1{\rule[0in]{0in}{#1}\ignorespaces}
\def\belowstrut#1{\rule[-#1]{0in}{#1}\ignorespaces}

\def\abovespace{\abovestrut{0.20in}}

\def\belowspace{\belowstrut{0.10in}}

\usepackage[hidelinks]{hyperref}

\title{Hebbian-Descent}

\author[1]{Jan Melchior}
\author[1]{Laurenz Wiskott}
\affil[1]{\normalsize Institut f\"ur Neuroinformatik, Ruhr Universit\"at Bochum, 44780 Bochum, Germany
\texttt{<forename>.<surname>@ruhr-uni-bochum.de} }

\begin{document}
\date{}
\maketitle

\begin{abstract}  

In this work we propose Hebbian-descent as a biologically plausible learning rule for hetero-associative as well as auto-associative learning in single layer artificial neural networks.
It can be used as a replacement for gradient descent as well as Hebbian learning, in particular in online learning, as it inherits their advantages while not suffering from their disadvantages.
We discuss the drawbacks of Hebbian learning as having problems with correlated input data and not profiting from seeing training patterns several times. For gradient descent we identify the derivative of the activation function as problematic especially in online learning.
Hebbian-descent addresses these problems by getting rid of the activation function's derivative and by centering, \emph{i.e.} keeping the neural activities mean free, leading to a biologically plausible update rule that is provably convergent, does not suffer from the vanishing error term problem, can deal with correlated data, profits from seeing patterns several times, and enables successful online learning when centering is used.
We discuss its relationship to Hebbian learning, contrastive learning, and gradient decent and show that in case of a strictly positive derivative of the activation function Hebbian-descent leads to the same update rule as gradient descent but for a different loss function.
In this case Hebbian-descent inherits the convergence properties of gradient descent, but we also show empirically that it converges when the derivative of the activation function is only non-negative, such as for the step function for example.
Furthermore, in case of the mean squared error loss Hebbian-descent can be understood as the difference between two Hebb-learning steps, which in case of an invertible and integrable activation function actually optimizes a generalized linear model. 
We empirically show that Hebbian-descent outperforms Hebb's rule and the covariance rule in general, has a performance similar to gradient decent for several training epochs, and most importantly it has a much better performance than the other update rules in online / one-shot learning.
All update rules profit from centering, but only in combination with Hebbian-descent it does lead to an inherent gradual forgetting mechanism that prevents catastrophic 
interference.
In case of auto-associative learning we show that the Hebbian-descent update can be understood as a non-linear version of Oja's rule / Sanger's rule, which also corresponds to a one step mean field contrastive divergence update. 
We proof that this update rule is not the gradient of any objective function, which, however, does not mean that it does not converge.
We show empirically that it converges in terms of the reconstruction error with a performance similar to that of gradient descent.

\end{abstract}

\section*{Introduction}\label{sec:HD_introduction}

In machine learning various research fields are biologically inspired such as evolutionary computing, reinforcement learning, or artificial neural networks. 
Especially the recent success of the latter in the field of deep learning, by breaking one benchmark record after each other, shows how valuable computational neuroscience is for machine learning. 
On the one hand, machine learning often provides strong theoretical insights such as convergence guarantees or provable universality of a model, which would also be valuable for computational neuroscience. 
On the other hand, this usually compromises biologically plausibility such as locality or online capability of a learning rule, which makes it harder to transfer machine learning insights back to computational neuroscience. 

In neuroscience Hebbian learning can still be consider as the major learning principle since Donald Hebb postulated his theory in 1949~\citep{hebb1949organization}. 
It is still widely used in its canonical form generally known as Hebb's rule, where the synaptic weight changes are defined as the product of presynaptic and postsynaptic firing rates.
Since Hebb's rule cannot learn negative / inhibitory weights when assuming positive firing rates, it cannot model long- or short-term-depression. Therefore, \citet{sejnowski1989hebb} proposed the covariance rule as an alternative to overcome this limitation. 
An advantage of Hebb's rule and the covariance rule is that they are capable of one-shot learning allowing to store patterns instantaneously without delay.
A disadvantage is though that they do not take advantage of seeing input patterns several times and that they have problems with correlated patterns as has been stated by~\citet{MarrWillshawEtAl-1991} and was analyzed for auto-associative, and hetero-associative networks by \citet{Loeweothers-1998} and \citet{NeherChengEtAl-2015}, respectively.
Furthermore, Hebb's rule and the covariance rule are unstable learning rules, so that the weights are usually renormalized after each update or a weight decay term is added~\citep{BienenstockCooperEtAl-1982}. 
This prevents the weights from growing infinitely large but introduces an additional hyperparameter that controls the speed of forgetting. 
In unsupervised learning the stability problem was analytically addressed for a linear neuron by Oja's rule \citep{oja1982simplified} or for several linear neurons by Sanger's rule~\citep{sanger1989optimal}, which are convergent learning rules that drive the neurons to learn the principal components of the input patterns.
For auto-associative learning in Hopfield networks contrastive Hebbian learning~\citep{RumelhartMcClellandEtAl-1986} and for Boltzmann machines Contrastive Divergence~\citep{Hinton-2002a} and its variants~\citep{Tieleman-2008a, DesjardinsCourvilleEtAl-2010aa, cho:10} have been proposed as stable learning rules, respectively. 
While contrastive learning takes advantage of several sweeps through the data it is not capable of one-shot learning anymore. 

In machine learning, gradient descent is the most widely used optimization algorithm, which is indispensable in the field of artificial neural networks and deep learning.
Although gradient descent is theoretically sound and convergence can be proven analytically even for stochastic gradient descent~\citep{Robbins1985, saad1998online}, it can still get very slow or converge to a bad local optimum. 
One reason for slow convergence in deep and recurrent neural networks trained with gradient decent is known as the vanishing gradient problem~\citep{hochreiter1991untersuchungen}. 
In order to understand this problem, consider an artificial neural network with $N$ layers trained with gradient descent (back-propagation). 
Then the gradient of the weight matrix of layer $k$ is calculated using the chain rule so that it contains the product of the derivatives of the activation functions of the current and all $N-k$ higher layers. 
Now, if the derivatives of the activation functions take only values smaller than 1, such as the sigmoid's derivative for example, the gradient values decrease exponentially with $N-k+1$.
This problem made deep learning difficult until \citet{HochreiterSchmidhuber-1997} proposed long-short term memory cells, which help to overcome the vanishing gradient problem in recurrent neural networks, as they allow for a linear error flow back in time.
In deep neural-networks activation functions with a constant derivative of one in the positive domain, such as the rectifier~\citep{FukushimaMiyake-1982, hahnloser2000digital, hahnloser2001permitted} or exponential linear units~\citep{clevert2015fast} have been proposed, which help to overcome the vanishing gradient problem.  
In the context of shallow networks the potentially negative effect of the sigmoid's derivative has already been discussed by~\citet{Hinton-1989}.
The authors proposed to use the cross-entropy as a loss function instead of the squared error in which case the derivative of the sigmoids in the output layer cancels out in the gradient.
Due to the similarity to the sigmoid the same also holds for softmax as an activation function, which explains why the cross-entropy loss is the preferred loss when training deep networks with softmax outputs nowadays. Motivated by technical benefits \citet{HertzKroghEtAl-1997} proposed an algorithm that uses an approximation to get rid of the derivative of the activation function also in intermediate layers of deep networks. 
The authors also mentioned that for the hyperbolic tangent as an output activation function, a loss similar to the cross entropy exists that cancels out the derivative of the activation function.

A major problem of artificial neural networks in general is that they suffer from catastrophic interference~\citep{McCloskey1989, Ratcliff1990}, namely that networks immediately forget previously learned information once trained on new information. 
\citet{French-1991} analyzed the problem in small neural networks and found that a sparse hidden representation reduces catastrophic interference but comes at the cost of reduced network capacity. 
A popular approach to overcome catastrophic interference is rehearsal learning~\citep{Robins-1995} in which the network is simultaneously updated on a new and several randomly selected old patterns. This, however, has the disadvantage that all previously seen patterns need to be stored.
 In pseudo rehearsal learning~\citep{Robins-1995} random patterns are fed to the network and the corresponding output is calculated. 
 The network is then updated on a new pattern and several of these random input-output pairs, which regularizes the network to stay close to its previous representation. 
This also reduces catastrophic interference but has still an additional computational overhead in the number of rehearsal patterns.
Another popular approach is that of complementary learning systems~\citep{McClellandMcNaughtonEtAl-1995, AnsRousset-1997, French-1997}, which consists of a fast and a slow learning network. 
The fast learning network acts as a buffer that rapidly stores recent patterns, which are then carefully transfered to the slow learning network, which tries to store not only the latest but all patterns. 
A disadvantage of such systems is that we need two networks for a task that can potentially be solved by a single network and that the knowledge transfer from the fast to the slow learning network needs consolidation (rehearsal) again.
Recently, a regularization term named elastic weight consolidation has been proposed by~\citet{KirkpatrickPascanuEtAl-2017}, which actively regularizes the weights towards their previous values and thus allows for an optimal embedding of new information. This makes rehearsal / consolidation obsolete, but all old weights have to be stored for this regularization such that we effectively have two networks again.
\citet{SantoroBartunovEtAl-2016} used memory augmented neural networks to perform one shot learning in deep neural networks. The memory itself, however, is not a neural implementation such that the question remains how one-shot learning can be implemented using gradient based methods in a neural network.

In this work we propose Hebbian-descent, a learning rule for hetero-associative as well as auto-associative learning in artificial neural networks.
We begin by describing the considered type of artificial neurons and the concept of centering in Section~\ref{sec:HD_artificial_neuron}. 
Section~\ref{sec:HD_hebbdescent} introduces the Hebbian-descent update for hetero-associative learning, shows its relationship to Hebbian learning (Section~\ref{sec:HD_hebb}), contrastive Learning (Section~\ref{sec:HD_hetero_contrastive}), and gradient decent (Section~\ref{sec:HD_connection_to_gradient}), and discusses the individual advantages and disadvantaged of these learning rules.
In Section~\ref{sec:HD_Hebbian_descent_loss} we show that in case of a strictly positive derivative of the activation function Hebbian-descent 
leads to the same update rule as gradient descent but for a different loss function.
In this case Hebbian-descent is a particular form of gradient descent and thus inherits its convergence properties (Section~\ref{sec:HD_Hebbian_descent_covergence}).
In Section~\ref{sec:HD_glm} we show that in case of the mean squared error loss and an invertible and integrable activation function Hebbian-descent actually optimizes a generalized linear model.
As a consequence the Hebbian-descent loss can be seen as the general log-likelihood loss in this case (Section~\ref{sec:HD_HD_loss_generalized}).
The impact of Hebbian-descent on deep neural networks is discussed in Section~\ref{sec:HD_learning_rate}.
The auto-associative Hebbian-descent update is introduced in Section~\ref{sec:HD_auto_associative}, where we show its relationship to Oja's / Sangers's rule (Section~\ref{sec:HD_from_the_perspective_of_oja_sanger}), contrastive Learning (Section~\ref{sec:HD_auto_asso_from_the_perspective_of_contrastive_learning}), and gradient decent (Section~\ref{sec:HD_from_the_perspective_of_gradient_descent}).
Section~\ref{sec:HD_methods} describes the experiments performed to compare Hebbian-descent with gradient descent, Hebb's rule, and the covariance rule.
The results are described in Section~\ref{sec:HD_results} where we first show that all learning rules considered in this work benefit from centering (Section~\ref{sec:hetero_ssociative_online_learning_with_centering}).
We further show that Hebbian-descent outperforms Hebb's rule and the covariance rule in general (Section~\ref{sec:HD_hetero_associative_learning}), 
has a performance similar to gradient decent in batch or mini-batch learning for several epochs (Section~\ref{sec:HD_multi_epoch_experiments}), and most importantly has a much better performance than all the other update rules in online / one-shot learning (Section~\ref{sec:hetero_ssociative_online_learning_with_centering}).
We show empirically that Hebbian-descent even converges if the derivative of the activation function is non-negative and can even be used with non-linearities like the step function (Section~\ref{sec:HD_multi_epoch_experiments} and Section~\ref{sec:hetero_ssociative_online_learning_with_centering}).
Section~\ref{sec:HD_weight_decay_experiments} illustrates that only Hebbian-descent with centering shows an inherent and plausible curve of forgetting that follows a power law distribution and does not require an additional forgetting mechanism such as a weight decay.
We proof in Appendix~\ref{appendix:AA_HD_is_generally_not_the_gradient_of_any_objective_function} that the auto-associative Hebbian-descent update is not the gradient of any objective function. However, this does not imply that it does not converge and we show in Section~\ref{sec:HD_auto_associative_experiments} that it has a similar performance as the gradient descent update. 
Finally we show that the hidden units in a network trained with Hebbian-descent have approximately the same average activity, which implies a more efficient encoding scheme than that of gradient descent (Section~\ref{sec:HD_homo_hidden_experiments}).

\section{Artificial Neuron Model}\label{sec:HD_artificial_neuron}
In this work we consider artificial neural networks in which every neuron is a centered artificial neuron~\citep{LeCun1998, Melchior2016} given by
\begin{equation}
	h_j = \phi\left(a_j\right) := \phi \bigg(\sum_i^N w_{ij}\left( x_i-\mu_i \right)  + b_j \bigg),\label{eqn:neuron_element}
\end{equation}
with input $x_i$, offset value $\mu_i$ (usually the mean of the input dimension $i$), bias $b_j$, weight $w_{ij}$, activation function $\phi \left( \cdot \right)$,
and pre-threshold activity $a_j$.
%where first an offset value $\mu_i$ (\emph{i.e.}\ mean of the input dimension $i$) is subtracted from the corresponding input $x_i$ before being multiplied with the corresponding weight $w_{ij}$. Secondly, all weighted centered-inputs are summed up and a bias value $b_j$ is added forming the pre-threshold activity $a_j$ that is lastly being passed through an activation function $\phi \left( \cdot \right)$.
In this work, we consider various activation functions such as linear / identity, sigmoid, step, softmax, rectifier~\citep{hahnloser2000digital, hahnloser2001permitted}, as well as the recently proposed exponential linear units~\citep{clevert2015fast}, which altogether represent most of the frequently used activation functions in artificial neural networks.
The activity for all neurons $\vect h$ within one layer can also be written more compactly in matrix notation by
\begin{eqnarray}
\vect{h}\,\,\,\,=\,\,\,\, \vect \phi\left(\vect a\right)
		 &:=& \vect \phi\left(\vect{W}^T\left(\vect{x}-\vect\mu \right)+\vect{b}\right)\label{eqn:neuron_matrix}\\
         &=& \vect \phi \Big(\vect{W}^T\vect{x} + \underbrace{\vect{b} - \vect{W}^T\vect\mu}_{\vect b'}  \Big)
         \label{eqn:neuron_matrix_rep1}\\
         &=& \vect \phi \Bigg(\vect{W}^T\bigg(\vect{x} - \underbrace{\vect\mu  + \left(\vect{W}^{T}\right)^{-1}\vect{b}}_{\vect \mu'}\bigg)  \Bigg),\label{eqn:neuron_matrix_rep2}
\end{eqnarray}
with weight matrix $\vect W$, element-wise activation function $\vect \phi \left( \vect \cdot \right)$, as well as input, output, offset, bias, and pre-threshold activity vector $\vect x$, $\vect h$, $\vect \mu$, $\vect b$, and $\vect a$, respectively.

\subsection{The Importance of Centering}\label{sec:HD_importance_centering}

It has been shown that centering is useful for training artificial neural networks~\citep{LeCun1998}, in particular for training Boltzmann machines~\citep{MontavonMueller-2012, Melchior2016} and auto-encoder networks~\citep{Melchior2016}, since it makes the network independent of its first order statistics, \emph{i.e.}\ the mean value of each neuron.
Notice that Equation~\eqref{eqn:neuron_element} includes a standard-uncentered artificial neuron as a special case where all offset values $\mu_i$ are set to $0$. Furthermore, in case of a single layer neural network and constant offset values that are set to the data mean, centering is equivalent to training an uncentered network on mean-free input data. 
However, the mean for hidden units is usually not known in advance and changes during training, which is in particular the case in online learning where the mean can even be non-stationary. 
In this case the offsets are updated by an exponentially moving average 
of the form 
\begin{eqnarray}
\vect \mu(t+1) &=& (1-\nu) \vect \mu(t) + \nu \vect \mu_{batch}(t),
\end{eqnarray}
where $t$ denotes the update index, $0 < \nu < 1$ is a sliding factor controlling adaptation speed, and $\vect \mu_{batch}(t)$ is the mean of the current mini-batch \emph{i.e.}\ $\vect \mu_{batch}(t) = \vect x_i$ in online learning. 
An important property of centering is that it does not change the model class, \emph{i.e.}\ each centered artificial neural network can be reparameterized to an uncentered neural network and \emph{vice versa} ~\citep{Melchior2016}. It is therefore just a different parameterization for the same model, which can also be seen by comparing Equation~\eqref{eqn:neuron_matrix} and~\eqref{eqn:neuron_matrix_rep1}. 
Furthermore, Equation~\eqref{eqn:neuron_matrix_rep2} shows that in case of a freely choosable offset the network can also be reparameterized to a centered network with zero bias values.
Notice that centering is independent of the used learning rule, which is usually gradient decent / back-propagation~\citep{Kelley-1960, RumelhartWilliams-1986} or contrastive learning~\citep{RumelhartMcClellandEtAl-1986, Hinton-2002a} but can also be Hebbian learning~\citep{hebb1949organization}.

\section{Hebbian-Descent}\label{sec:HD_hebbdescent}

In this section we propose Hebbian-descent, a learning rule for hetero-associative / supervised as well as auto-associative / unsupervised learning in single layer neural networks. 
The name gives credit to Hebbian learning as well as gradient descent, both of which it is strongly connected to. 
Under certain conditions it is also related to generalized linear models, contrastive learning, and in case of supervised learning it can often even be reformulated as a gradient descent optimization problem.
Hebbian-descent thus provides a unified view on these algorithms rather than claiming to be an new algorithm itself.

\subsection{Hetero-Associative / Supervised Learning in Single Layer Networks with Hebbian-Descent}\label{sec:HD_heteroasso_HD}

In order to hetero-associate input vector $\vect x$ with output vector $\vect t$ by a centered single layer network (Equation~\eqref{eqn:neuron_matrix}) the Hebbian-descent updates for weight matrix $\vect W$ and bias vector $\vect b$ are given in its general form by
\begin{eqnarray}
\Delta_{_{HD}} \vect W &=& -\eta (\vect x- \vect \mu )\frac{\partial \mathcal{L}(\vect t, \vect h)}{ \partial \vect h}^T\label{eqn:update_general_hebbian_descent_sup_w_2}\\
&=& -\eta (\vect x- \vect \mu ) \,\bm{\mathcal{E}}(\vect t, \vect h)^T,\label{eqn:update_general_hebbian_descent_sup_w_1}\\
\Delta_{_{HD}} \vect b  &=& -\eta \frac{\partial \mathcal{L}(\vect t, \vect h)}{ \partial \vect h}
 \label{eqn:update_general_hebbian_descent_sup_b_2}\\
&=& -\eta  \, \bm{\mathcal{E}}(\vect t, \vect h),\label{eqn:update_general_hebbian_descent_sup_b_1}
\end{eqnarray}
with learning rate or step-size parameter $\eta$, and error signal $\bm{\mathcal{E}}(\vect t, \vect h) = \frac{\partial \mathcal{L}(\vect t, \vect h)}{ \partial \vect h}$ with corresponding loss function $\mathcal{L}(\vect t, \vect h)$. 
As discussed in the following sections choosing the activation function (should have a positive derivative)
 and error signal with corresponding loss function allows us to relate Hebbian-descent to other algorithms. In general a useful loss should define a unique optimum where the network output $\vect{h}$ matches the desired output $\vect{t}$, in which case the corresponding error term (\emph{i.e.}\ the partial derivative with respect to the input) vanishes. 
As a common standard consider the squared error loss $\mathcal{L}(\vect t, \vect h)=\frac{1}{2}\left(\vect h - \vect t\right)^T\left(\vect h - \vect t\right)$, which serves as an appropriate loss function for linear, sigmoid, softmax, rectifier as well as exponential linear activation functions. 
The corresponding error signal is $\bm{\mathcal{E}}(\vect t, \vect h) = \left(\vect h - \vect t\right)^T$ and the Hebbian-descent updates for weights $\vect W$ and bias $\vect b$ then become
\begin{eqnarray}
\Delta_{_{HD}} \vect W &\stackrel{(\ref{eqn:neuron_matrix},\ref{eqn:update_general_hebbian_descent_sup_w_1})}{=}&  -\eta \Big(\vect x- \vect \mu \Big)\Big(\vect \phi\big(\left(\vect{x}-\vect\mu \right)\vect{W}+\vect{b}\big) - \vect t\Big)^T\\
&\stackrel{\eqref{eqn:neuron_matrix}}{=}&  \underbrace{\eta\left(\vect x- \vect \mu \right)\vect t^T}_{Sup. Hebb}-\underbrace{\eta \left(\vect x- \vect \mu \right)\vect{h}^T}_{Unsp. Hebb},\label{eqn:update_hebbian_descent_sup_w}\\
\Delta_{_{HD}} \vect b &\stackrel{(\ref{eqn:neuron_matrix},\ref{eqn:update_general_hebbian_descent_sup_b_1})}{=}&   -\eta \Big(\vect \phi\big((\vect{x}-\vect\mu)\vect{W}+\vect{b}\big) - \vect t\Big)\\
&\stackrel{\eqref{eqn:neuron_matrix}}{=}& \underbrace{\eta  \vect t}_{Sup. Hebb}-\underbrace{\eta \vect{h}}_{Unsp. Hebb}.\label{eqn:update_hebbian_descent_sup_b}
\end{eqnarray}
Thus, with Hebbian-descent the network learns to produce a desired output  $\vect t$ given input $\vect x$ by comparing the current output $\vect h$ with the desired output $\vect t$.
This introduces stability since the updates become zero once the network output $\vect h$ is equivalent to the desired output $\vect t$. 
Equations~\eqref{eqn:update_hebbian_descent_sup_w} and~\eqref{eqn:update_hebbian_descent_sup_b} also show that in case of the squared error loss the update rules are the differences between a supervised and an unsupervised Hebb-learning step. 
While the supervised learning step measures the correlation between input and desired output, the unsupervised step measures the correlation between input and output that is already represented by the network and removes it from the current update step. 
This is a rather important property as it allows to learn only the missing information and thus to complement the representation that has already been learned by the network. 
Hebbian-descent with mean squared error loss can thus be understood as Hebb's rule with an adaptive weight decay, which depends on the network's representation of the current pattern.

\subsection{From the Perspective of Hebbian learning}\label{sec:HD_hebb}

As shown in Equations~\eqref{eqn:update_hebbian_descent_sup_w} and \eqref{eqn:update_hebbian_descent_sup_b}, in the case of the squared error loss Hebbian-descent can be considered as the difference between two Hebb-learning terms. 

\subsubsection{From the Perspective of Hebb's Rule}\label{sec:HD_hebb_rule}

By removing the second term in Equation~\eqref{eqn:update_hebbian_descent_sup_w} and fixing the bias term to zero without updating it we get the centered supervised Hebb-rule, which is given by
\begin{eqnarray}
\Delta_{_{H}} \vect W &=& \eta\left(\vect x- \vect \mu \right)\vect t ^T.\label{eqn:update_hebb_sup_w}
\end{eqnarray}
Hebb-learning~\citep{hebb1949organization} simply adds up the second cross-moments between input and output patterns and does therefore not account for learning only missing information such as Hebbian-descent does. As shown at the end of this section, it limits the model in associating arbitrary patterns with each other, which is in particular problematic if the patterns are correlated as has been analyzed for Hopfield networks by~\citep{Loeweothers-1998}. Furthermore, in contrast to Hebbian-descent and gradient descent, Hebb-rule is not a convergent learning rule since the weights continue to grow with each update step. While this is less of an issue for saturating activation functions such as the sigmoid, in case of non-saturating activation functions such as identity or rectifier one needs a mechanism that restricts the weights such that the network output stays in a reasonable range. Usually the weights are rescaled by the number of performed update steps, which has the disadvantage that we need to keep track of the number of updates, or a weight decay term~\citep{BienenstockCooperEtAl-1982} is added, which has the disadvantage of introducing an additional hyper-parameter that needs to be chosen in advance. 

\subsubsection{From the Perspective of  the Covariance Rule}\label{sec:HD_cov_rule}

A problem of the original Hebb-rule ($\vect\mu = \vect 0$, Equation~\eqref{eqn:update_hebb_sup_w}) is that it cannot learn negative / inhibitory weights when the activities are non-negative. In order to compensate this limitation \citet{sejnowski1989hebb} have proposed the covariance rule, which is given for the weights of a single layer neural network  by 
\begin{eqnarray}
\Delta_{_{C}} \vect W &=& \eta\left(\vect x- \vect \langle \vect x \rangle \right)\left(\vect t - \langle \vect t \rangle \right)^T,\label{eqn:update_cov_sup_w}
\end{eqnarray}
where $\langle \cdot \rangle$ is used to denote the argument averaged over the data.
The covariance rule, as the name suggests, models the covariance between input and output activities, whereas the uncentered Hebb-rule models the corresponding second moments. For positive input data the covariances can be positive or negative while the second moments are always positive, such that the covariance rule can learn inhibitory weights while the uncentered Hebb-rule cannot. 

Independently of whether centering is used or not, within the covariance rule the mean activities are subtracted from the current input activities and additionally also from the current output activities.
However, if the same number of update steps are performed for each data point as it is usually the case, it is sufficient to subtract either the input mean or the output mean since it leads to the same final weight matrix as the covariance rule. This can be shown as follows:
\begin{eqnarray}
\overbrace{\sum_{d=1}^D \eta\left(\vect x_d- \vect \langle \vect x \rangle \right)\left(\vect t_d - \langle \vect t \rangle \right)^T}^{Final\, weight\, matrix} 
&\overset{(\ref{eqn:update_cov_sup_w})}{=}& \eta D \langle \left(\vect x- \vect \langle \vect x \rangle \right)\left(\vect t - \langle \vect t \rangle \right)^T \rangle \\
	&=&\eta  D\langle  \vect x \vect t^T   - \vect x \langle \vect t^T \rangle - \langle \vect x \rangle \vect t^T + \langle \vect x \rangle  \langle \vect t^T \rangle \rangle \\
		&=&\eta  D\langle \vect x \vect t^T  -  \vect x  \langle \vect t^T \rangle \rangle - \langle \vect x \rangle \langle \vect t^T \rangle + \langle \vect x \rangle  \langle \vect t^T \rangle \,\,\,\,\,\,\,\\
&=&\eta  D\langle  \vect x \left( \vect t^T - \langle \vect t^T \rangle \right)  \rangle \\
	&=&\eta  D\langle  \left( \vect x - \langle \vect x \rangle \right) \vect t^T \rangle, \,\,\,\,\text{\small{(for symmetry reasons)}}\,\,\,\,\,\,\,\,\,\,\,\,\,\, \label{eqn:single_cov}
\end{eqnarray}
where $D$ is the number of data points, $\langle \vect x \rangle$ implicitly denotes the average of $d$, and we assume one update step with learning rate $\eta$ per data point.
In the centered case where the offsets are set to the average input activities ($\vect \mu =  \langle \vect x\rangle$) Hebb-rule becomes equivalent to the covariance rule as can be seen by comparing Equation~\eqref{eqn:update_hebb_sup_w} and Equation~\eqref{eqn:single_cov}. 
%But notice that when centering is used the average input activities are not only subtracted from the input within the update rule but also when the output activities are calculated (see Equation~\eqref{eqn:neuron_matrix}). 

An important property of centering shown in Equation~\eqref{eqn:neuron_matrix_rep1} is that it implicitly defines adaptive resting levels $\vect b' = \vect b-\vect W^T\vect \mu$ for all neurons that depend on the corresponding bias but also on the weights and offsets. 
As a consequence and although the bias values are set to zero when using the centered Hebb-rule this leads to adaptive resting levels $\vect b' = -\vect W^T\langle \vect x \rangle$ in this case.
In the context of the covariance rule the subtraction of the average firing rates was introduced as an implementation of Long-Term-Depression (LTD)~\citep{sejnowski1989hebb}. 
Thus centering can also be seen as an implementation of LTD through an adaptive resting level that depends on the average activities.
However, whether centered or not, the covariance rule is still a divergent learning rule, which does not account for the problem with correlated patterns. 

\subsubsection{Limitations of Hebb's Rule and the Covaricance Rule}\label{sec:HD_limit_cov_rule}

The limitations of Hebb's rule and the covariance rule in considering only missing information and in learning correlated pattern pairs can be illustrated on a simple 2D-example. Consider a dataset consisting of four input-output patterns, where two of the pattern pairs are the same and thus induce correlations. Figure~\ref{fig:hebb_diabvantage_example} shows the learned weights and the produced outputs of a centered single layer neural network with sigmoid units when trained with (a) Hebb-rule / covariance rule and (b) Hebbian-descent. 
\begin{figure}[t]
\centering
\begin{tabular}{lcc}
target & 
$\left\lbrace (0, 1), (1, 0), (1, 1), (1, 1) \right\rbrace$ & 
$\left\lbrace (0, 1), (1, 0), (1, 1), (1, 1) \right\rbrace$\\
output & 
$\lbrace (0, 0), (0, 0), (1, 1), (1, 1) \rbrace$ &
$\left\lbrace (0, 1), (1, 0), (1, 1), (1, 1) \right\rbrace$ \\
& \includegraphics[scale=0.75]{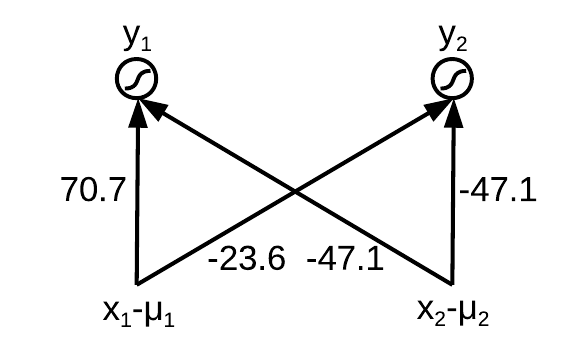}  &
\includegraphics[scale=0.75]{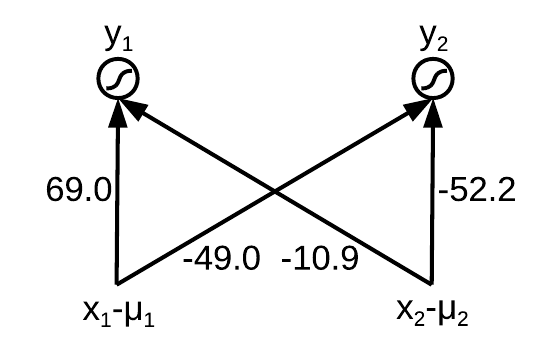} \\
input & 
$\left\lbrace (0,1), (1,1), (1,0), (1,0)  \right\rbrace$ & 
$\left\lbrace (0,1), (1,1), (1,0), (1,0)  \right\rbrace$\\
& \\
& (a) Hebb-rule / Covariance rule
 &  (b) Hebbian-descent \\
\end{tabular}
\caption{Illustration of the limitations of Hebbian learning. The \emph{input} patterns are associated with the \emph{target} patterns by using (a) the Hebb-rule / covariance rule, or (b)~Hebbian-descent, on a centered single layer neural network with sigmoid units. The update rules are repeated 300 times for each pattern with a learning rate of 10. For a fair comparison, although it does not alter the results, we fixed the bias value in Hebbian-descent to 0 and the resulting weight matrices are all normalized to have a norm of 100. The real-valued outputs of the networks are rounded to the second decimal place resulting in values of zero or one. While Hebbian-descent succeeds in learning the dataset the Hebb's / covariance rule gets the right output only for the doubled pattern.}
\label{fig:hebb_diabvantage_example}
\end{figure}
While Hebbian-descent succeeds in learning the dataset the Hebb's / covariance rule gets the correct output only for the doubled pattern. While doubling a pattern appears to be an extreme form of introducing correlation, quantitatively the same results can also be observed when correlation is induced without doubling patterns as in the following example: $input:\{(0,1,0),(0,1,1),(1,0,1),(1,1,0)\} \rightarrow target:\{(0,1,1),(1,0,0),$ $(1,1,0),(1,0,1)\}$. While Hebbian-descent succeeds again the Hebb's / covariance rule produces the correct output only for the third pattern. 

\subsection{From the Perspective of Contrastive Learning}\label{sec:HD_hetero_contrastive}

In case of the squared error loss Hebbian-descent is defined as the difference between two Hebb-learning terms, which is similar to contrastive learning rules such as Contrastive Hebbian learning~\citep{RumelhartMcClellandEtAl-1986} and Contrastive Divergence~\citep{Hinton-2002a}, which are used to train undirected auto-associative models. 

\subsubsection{Contrastive Learning Rule}\label{sec:HD_contrastive_learning_rule}
In contrastive learning the parameter updates are defined as the difference between neural activities of two different phases. 
In the so called positive phase, some or all neural activities are given from which possibly missing or incorrect activities are inferred. 
This defines the target distribution the network should ideally represent after learning. In the so called negative phase, no activities are provided, instead they are inferred from the network dynamics \emph{i.e.}\ through sampling, which defines the current model distribution.
The Contrastive Divergence algorithm~\citep{Hinton-2002a} and its variants~\citep{Tieleman-2008a, DesjardinsCourvilleEtAl-2010aa, cho:10} are the common choices when training stochastic Boltzmann machines, where the negative phase is estimated through sampling. In case of deterministic Boltzmann machines or Hopfield networks Contrastive Divergence mean field learning~\citep{WellingHinton-2002} or contrastive Hebbian learning~\citep{RumelhartMcClellandEtAl-1986} is commonly used, which estimate the values of the negative phase through a mean field estimation instead of sampling. 

In case of a centered single layer network where all input units $\vect{x}$ are connected to all output units $\vect{h}$ 
the contrastive learning updates for the weights $\vect W$, hidden bias $\vect b$, and visible bias $\vect c$ are given by
\begin{eqnarray}
\Delta_{_{CL}} \vect W &=& \eta\left(\vect{x}_{_{D}}-\vect\mu \right)\left(\vect{h}_{_{D}}-\vect\lambda \right)^T - \eta\left(\vect{x}_{_{M}}-\vect\mu \right)\left(\vect{h}_{_{M}}-\vect\lambda \right)^T \label{eqn:Contrastive_learning_W}\\
\Delta_{_{CL}} \vect b &=& \eta \vect{h}_{_{D}} - \eta \vect{h}_{_{M}} \label{eqn:Contrastive_learning_c}\\
\Delta_{_{CL}} \vect c &=& \eta  \vect{x}_{_{D}} - \eta \vect{x}_{_{M}} \label{eqn:Contrastive_learning_b}
\end{eqnarray}
where $\eta$ is the learning rate, $\vect \mu$ and $\vect \lambda$ are visible and hidden unit offsets, $\vect x_{_{D}}$ and $\vect h_{_{D}}$ are samples from the positive phase / data distribution, and $\vect x_{_{M}}$ and $\vect h_{_{M}}$ are samples from the negative phase / model distribution. Notice that in mini-batch learning on uses the average update over the individual  data-points.

\subsubsection{Contrastive Learning with Clamped Values}\label{sec:HD_contrastive_clamping}

In the supervised setup where input and output values are both known we have $\vect x_{_{D}}=\vect x$, and $\vect h_{_{D}}=\vect t$. If we additionally assume that $\vect \lambda = \vect 0$, the positive phase  becomes equivalent to the supervised Hebb-learning step in Hebbian-descent with squared error loss as given by Equation~\eqref{eqn:update_hebbian_descent_sup_w} and~\eqref{eqn:update_hebbian_descent_sup_b}. 
The negative phase, however, differs from the unsupervised Hebb-learning step since $\vect x_{_{M}}$ and $\vect h_{_{M}}$ are inferred from the model dynamics through sampling or mean field estimation. 
This leads to a different optimization objective in the two cases such that the same network trained with supervised Hebbian-descent (see Equation~\eqref{eqn:update_hebbian_descent_sup_w} and~\eqref{eqn:update_hebbian_descent_sup_b}) leads to a directed hetero-associative model, while trained with contrastive learning it leads to an undirected auto-associative model.

If we additionally fix $\vect x_{_{M}}$ to $\vect x$, also referred to as clamping the input, the update for the visible bias vanishes and the update for weights and hidden bias become equivalent to the Hebbian-decent update given by Equation~\eqref{eqn:update_hebbian_descent_sup_w} and Equation~\eqref{eqn:update_hebbian_descent_sup_b}. 
Thus, the Hebbian-descent update with squared error loss is equivalent to contrastive Hebbian learning update with clamped input units. 
\cite{Movellan-1991} argued that contrastive Hebbian learning with clamped input units is equivalent to gradient descent with squared error loss as both optimize the same error function. However, except for the identity as activation function this is not correct. 
As shown in the next section, the two algorithms optimize different loss functions and will thus converge to different solutions. 
This can also be concluded from the work of \cite{XieSeung-2003} who showed that in case of sigmoid output units the loss that is actually optimized by contrastive Hebbian learning with clamped inputs is the cross entropy loss~\citep{Hinton-1989} and not the squared error. Since clamping the input turns an auto-associative model into a hetero-associative model it is debatable if resulting updates should still be named contrastive learning.

From a probabilistic perspective (\emph{i.e.}\ Boltzmann machine) a network trained with contrastive learning models the joint distribution of $\vect x$ and $\vect t$ when the negative phase is sampled, whereas it models the conditional distribution of $\vect t$ given $\vect x$ if the input is clamped. Thus Hebbian-decent learns the maximum likelihood estimate of the output given the input, which is discussed in more detail in Section~\ref{sec:HD_glm}.

\subsubsection{Limitations of Contrastive Learning}\label{sec:HD_limit_contrastive_learning}

It is important to note that for contrastive learning to work properly (without clamping the input), the sampled negative phase needs to represent the current model distribution sufficiently well, which usually requires a rather big batch size, a small learning rate, and several sampling or mean-field estimation steps. As one can imagine it is thus rather difficult to perform rapid online learning using contrastive learning rules.

\subsection{From the Perspective of Gradient Descent}\label{sec:HD_connection_to_gradient}

Hebbian-descent has an even stronger connection to gradient descent. As shown in this section, in case of a strictly positive derivative of the activation function Hebbian-descent can be reformulated as gradient descent using a different loss function.

\subsubsection{Gradient Descent Update}\label{sec:HD_gradient_descent_update}

In gradient descent the parameters are updated by adding the negative partial derivatives of the loss function with respect to the parameters. For a single layer neural network as given by Equation~\eqref{eqn:neuron_matrix} the negative partial derivatives with respect to weight and bias values are given in matrix notation by
\begin{eqnarray}
\Delta_{_{GD}} \vect W = -\frac{\eta \partial \mathcal{L}(\vect t, \vect h)}{\partial \vect W} &=&-\eta \frac{\partial \vect a}{ \partial \vect W}\bigg(
\frac{\partial \mathcal{L}(\vect t, \vect h)}{ \partial \vect h} \odot \frac{\partial}{ \partial \vect a} \odot \vect h\bigg)^T\,\label{eqn:update_general_gradient_descent_sup_w_1}\\
&\stackrel{\eqref{eqn:neuron_matrix}}{=}&-\eta \big(\vect{x}-\vect\mu \big) \big(
\bm{\mathcal{E}}(\vect t, \vect h) \odot \vect \phi'(\vect a)\big)^T,\,\label{eqn:update_general_gradient_descent_sup_w_2}\\
\Delta_{_{GD}} \vect b = -\frac{\eta \partial \mathcal{L}(\vect t, \vect h)}{\partial \vect b} &=&-\eta \frac{\partial \vect a}{ \partial \vect b}\odot
\bigg(
\frac{\partial \mathcal{L}(\vect t, \vect h)}{ \partial \vect h} \odot \frac{\partial}{ \partial \vect a} \odot \vect h\bigg)\,\label{eqn:update_general_gradient_descent_sup_b_1}\\
&\stackrel{\eqref{eqn:neuron_matrix}}{=}&-\eta \big(
 \bm{\bm{\mathcal{E}}}(\vect t, \vect h) \odot \vect \phi'(\vect a)\big),\,\label{eqn:update_general_gradient_descent_sup_b_2}
\end{eqnarray}
where $\eta$ is the learning rate, $\odot$ denotes the Hadamard product (element-wise product), and $\vect \phi'(\vect a)$ denotes the derivative of the activation function applied element-wise to the input vector \emph{i.e.}\  
$\vect{\phi'}\left(\left[a_{0}, \cdots, a_M \right] \right) = \left[\phi' \left(a_{0}\right), \cdots, \phi' \left(a_{M}\right) \right]$.
It is obvious that when choosing the same loss and activation function the only difference to the Hebbian-descent learning rule (Equations~\eqref{eqn:update_general_hebbian_descent_sup_w_2} and \eqref{eqn:update_general_hebbian_descent_sup_b_2}) is the element-wise multiplication with the derivative of the activation function. 
As a direct consequence Hebbian-descent is equivalent to gradient descent in case of the identity as activation function. 
In this work, however, we explicitly consider nonlinear activation functions and want to motivate in the following why discarding the derivative of the activation function might be a good idea in general.

\subsubsection{Limitations of Gradient Descent}\label{sec:HD_problems_gradient_descent}

As discussed in the following, in gradient descent a problem can occur when the pre-threshold activities $\vect a$ take values that, passed through the derivative of the activation function, lead to values that are zero or close to zero ($\phi'(a_j)\approx 0$). 
In this case, the corresponding partial derivatives get also zero or close to zero, but independently of the actual error signal $\mathcal{E}( t_j,  h_j)$~(\emph{e.g.}\ $ h_j- t_j$) as can be seen from Equation~\eqref{eqn:update_general_gradient_descent_sup_w_2} and \eqref{eqn:update_general_gradient_descent_sup_b_2}. 
As a simple example consider a single layer neural network with a single sigmoid output unit that should be used to associate a pattern $\vect x$ with a binary output value $t_j$. Let us assume that $t_j = 0$ and that $\vect x$ leads to a pre-threshold activity $a_j = 4$, which corresponds to a post-threshold activity of $h_j\approx0.982$ as can be seen from Figure~\ref{fig:sigmoid_and_derivative}.
\begin{figure}[t]
\begin{center}
\subfigure[]{
\includegraphics[scale=0.34, trim=20 0 50 0, clip]{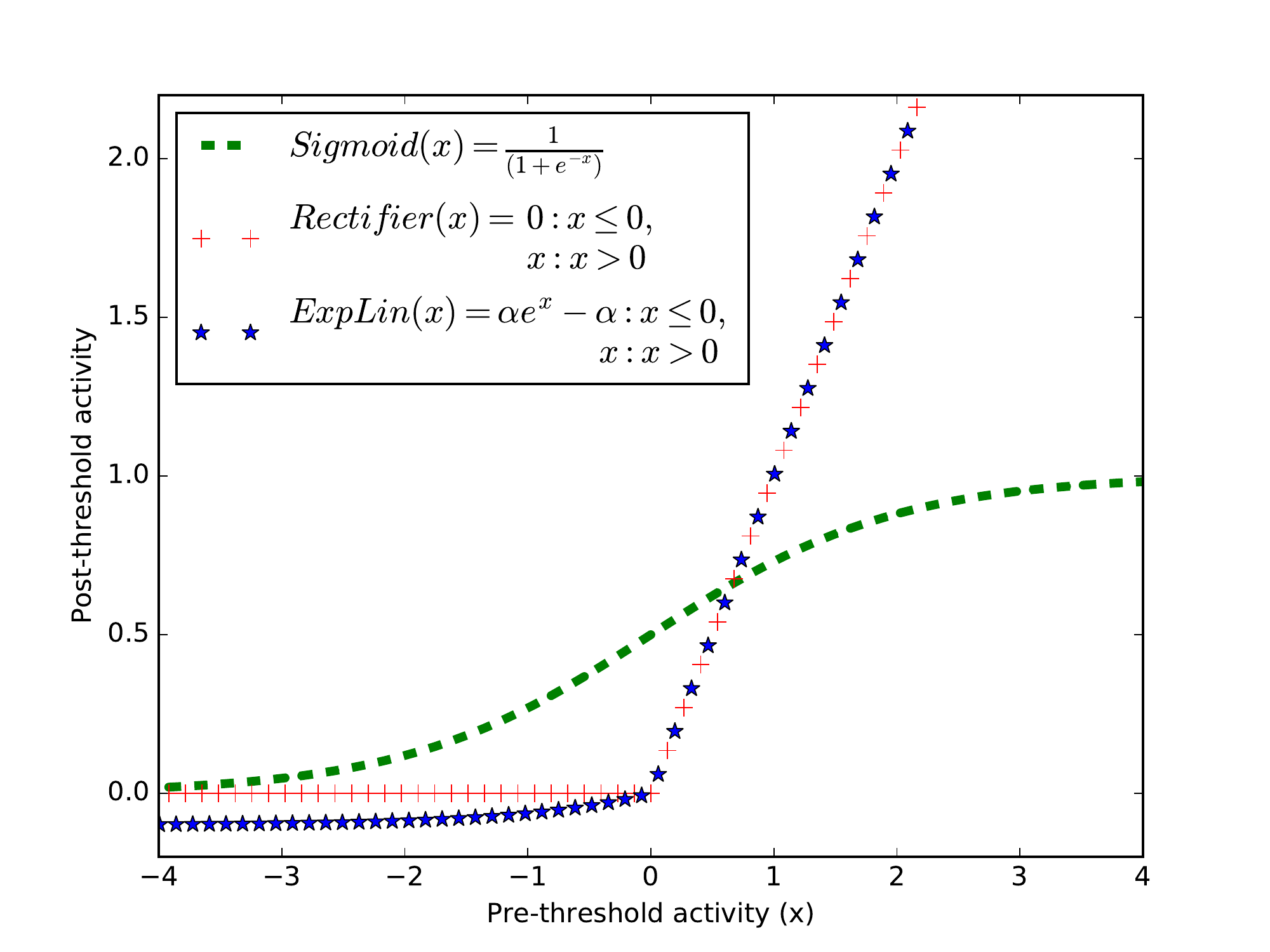}}
\subfigure[]{
\includegraphics[scale=0.34, trim=45 0 50 0, clip]{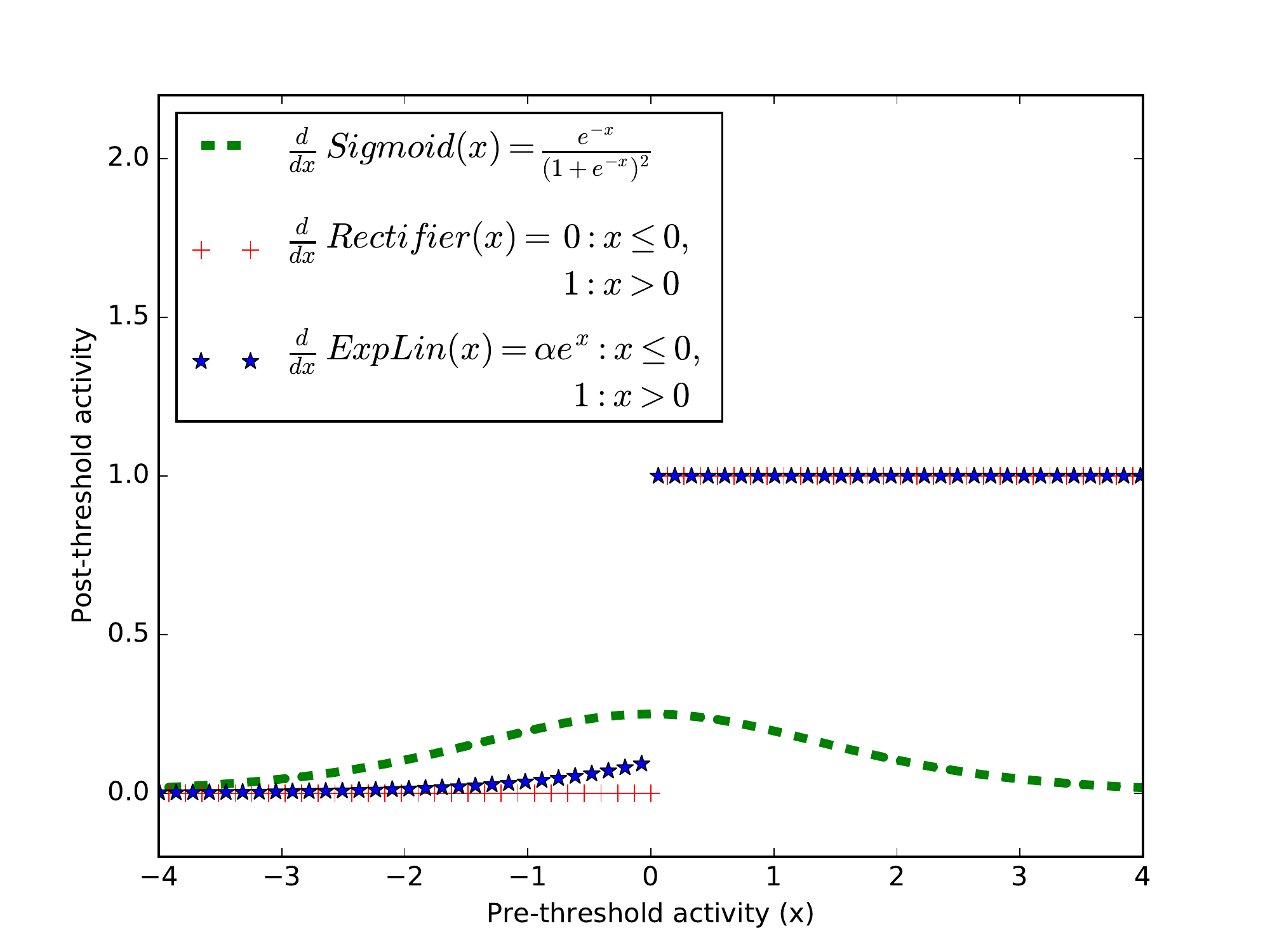}}
\caption{Plot of (a) sigmoid, rectifier, and exponential linear function (ExpLin) and (b) their corresponding derivatives.
}
\label{fig:sigmoid_and_derivative}
\end{center}
\end{figure} 
The error signal $h_j - t_j = 0.982$, which is almost maximal, gets multiplied with $\phi'(-4)=~6.14\times e^{-6}$ such that the partial derivative gets extremely small. We would thus need a lot of update steps with a big learning rate until the parameter updates are large enough to successfully learn the association between $\vect x$ and $t_j$. 

In case of batch or mini-batch learning with an appropriate parameter initialization and a sufficiently small learning rate, this effect, as shown in the experiments, is usually less of an issue since the network learns the associations of all patterns in the dataset slowly in parallel. In the example given above the pre-threshold activity $a_j$ will thus not get a very negative value for $\vect x$ in the first place. In online learning, however, where we want to store patterns more or less instantaneously or when learning non-stationary input distributions, this problem, as shown in the experiments, becomes more severe. 
For the example given above just think of the network already having learned to associate patterns of dogs with $t_j=1$ and that $\vect x$ corresponds to a new but similar pattern, let's say a pattern of a cat for which $t_j=0$. Since the network has not seen a pattern of a cat before it appears to be rather similar to a pattern of a dog such that the activity of $h_j$ given $\vect x$ as well as the error term is rather close to one (\emph{e.g.}\ $h_j-t_j = 0.982$).
Although there are some differences in the input that allow to differentiate cats from dogs, the partial derivatives of all weights $w_{ij}$ connecting the input with unit $h_j$ get scaled down extremely (\emph{e.g.}\ $\phi'(4)=~6.14\times e^{-6}$). It will thus take a lot of update steps on patterns of cats and also dogs before cats are associated correctly.  
Figure~\ref{fig:sigmoid_and_derivative} shows that this does not change for the negative regime when exponential linear units are used and in case of a rectifier it gets even impossible as its derivative is constant zero for $a_j \leq 0$.  

Figure~\ref{fig:update_rule_2D_example}~(a) illustrates the different speed of convergence for a network with sigmoid units trained on a 2D toy-example with squared error loss using either gradient descent or Hebbian-descent. 
\begin{figure}[t]
\begin{center}
\subfigure[]{
\includegraphics[scale=0.38, trim=25 10 140 40, clip]{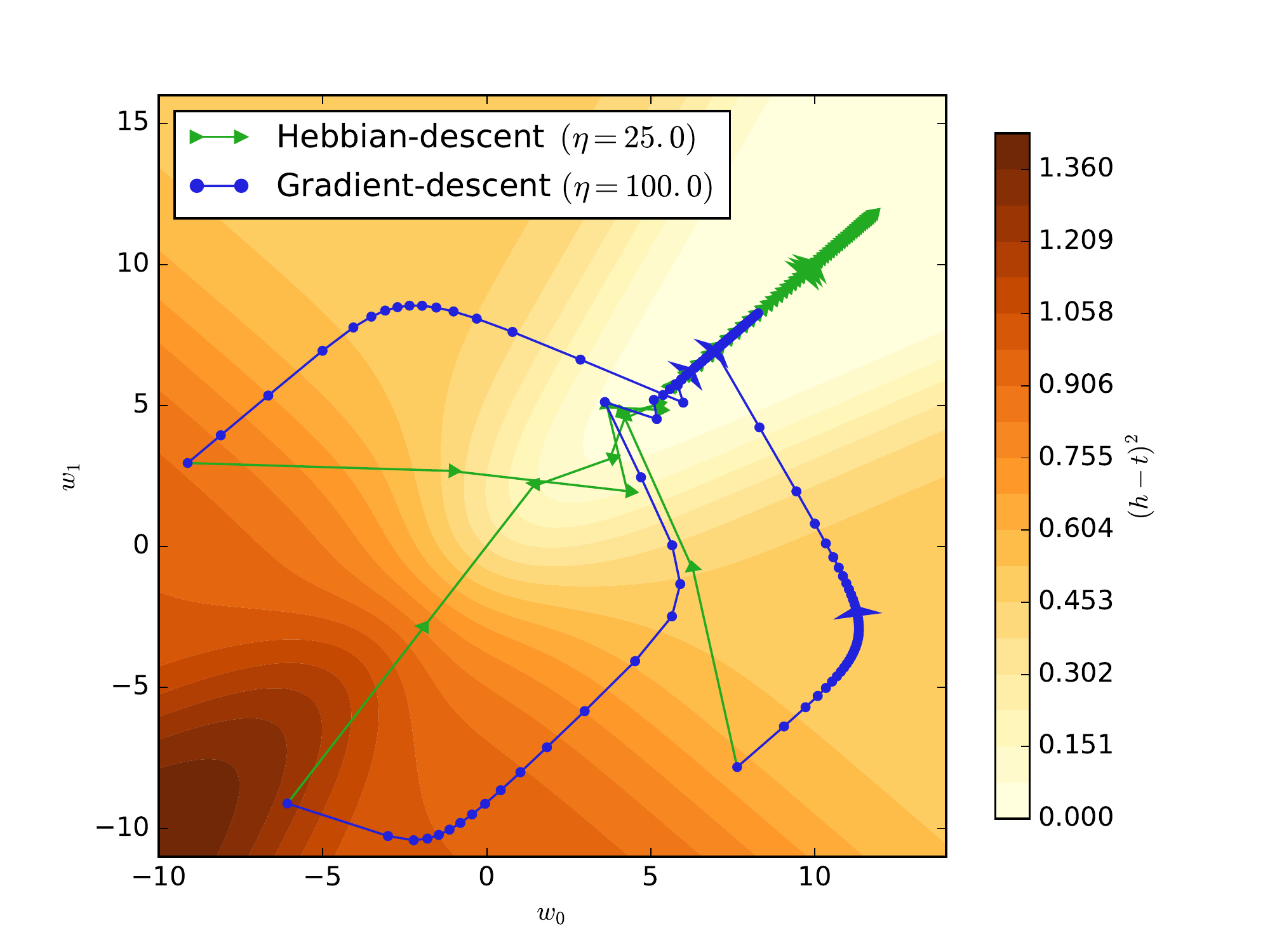}}
\subfigure[]{
\includegraphics[scale=0.38, trim=60 10 50 40, clip]{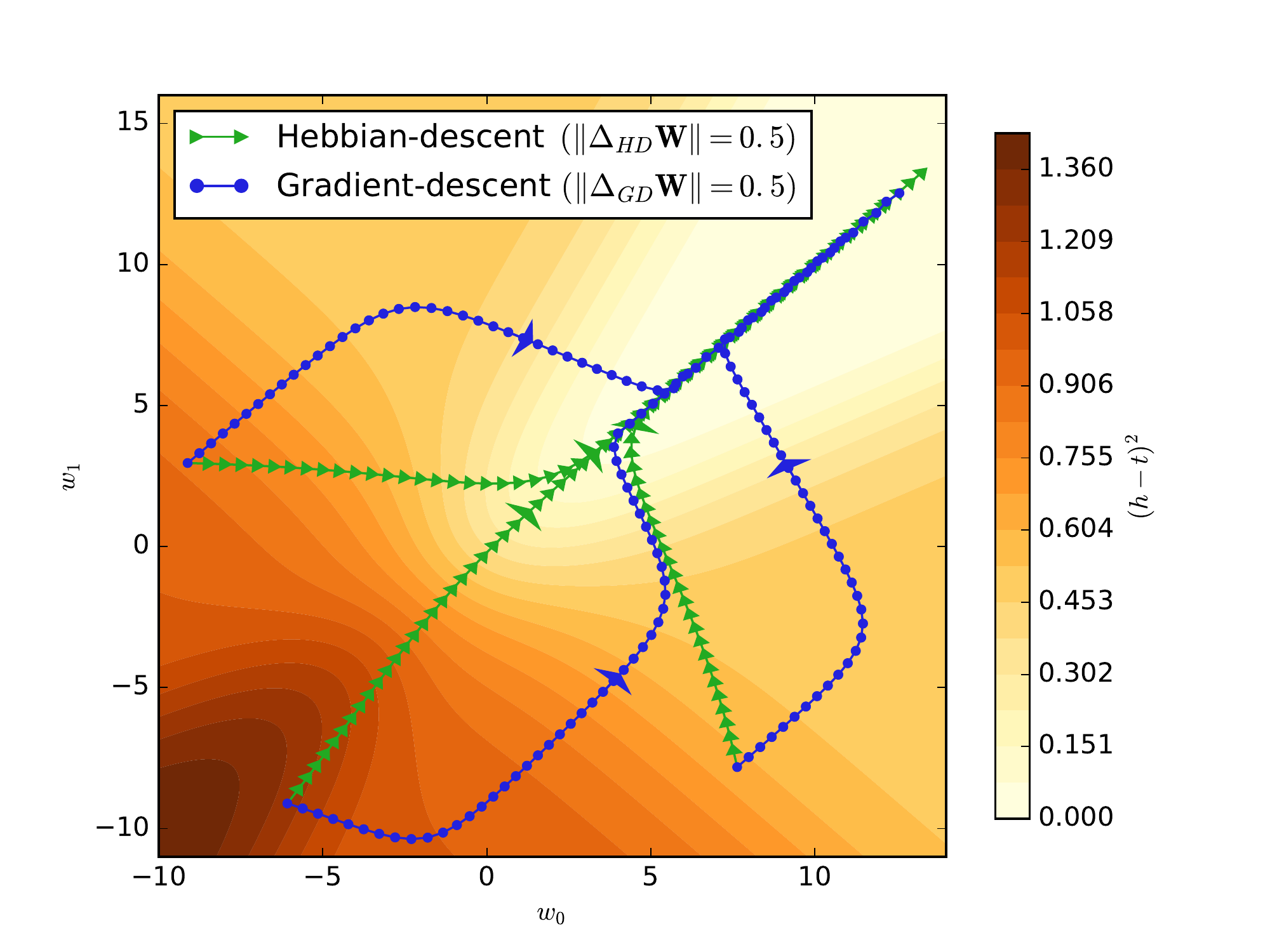}}
\caption{Comparison of gradient descent and Hebbian-descent.
Both figures show the loss landscape of a 2D toy-example and optimization paths for networks trained for 50 full-batch updates (equivalent results are obtained for online learning) using a squared error loss with either Hebbian-descent or gradient descent. 
The networks have two input units, a single sigmoid output unit and they are trained on the dataset 
$input:\{(1,0),(0,1),(1,1)\} \rightarrow target:\{(0),(0),(1)\}$ using a squared-error loss. The wider arrows indicate the optimization step after 25 update steps.
(a) The optimization paths for the two methods where the learning rates were chosen by hand such that both converged without much oscillation. Although Hebbian-descent has a smaller learning rate it converges faster than gradient descent as can be seen by comparing the wider arrows. (b) The same experiments were performed where all updates were normalized to have the same norm, illustrating that the direction of Hebbian-descent leads to a faster convergence than gradient descent. Notice that in this setting Hebbian-descent is equivalent to gradient descent with cross entropy loss (see Equation~\eqref{eqn:gradient_descent_loss_implementing_Hebbian_descent}). 
}
\label{fig:update_rule_2D_example}
\end{center}
\end{figure}
It shows how the norm of the gradient shrinks extremely in saturated regimes of the sigmoid where the derivative of the activation function take very small values. Figure~\ref{fig:update_rule_2D_example}~(b) illustrates that even if the norm of the update rules is normalized Hebbian-descent converges faster since it points almost directly towards the global minimum. 
However, when using gradient descent to 
train a network with sigmoid or softmax output units (denoted by $\sigma(\vect a)$) one  normally uses the cross-entropy loss  $\mathcal{L}(\vect t, \vect h) = -\vect t^T \ln ({\vect h})-(\vect 1-\vect t)^T\ln(1-{\vect h})$ instead of the squared error loss, since it usually leads to faster convergence.
The cross-entropy has the advantage that the partial derivative of the sigmoid or softmax units 
vanishes in the gradient~\citep{Hinton-1989, HertzKroghEtAl-1997}. 
This is exactly what we want to achieve through Hebbian-descent just for the special case  of sigmoid or softmax units in which the gradient descent update with cross-entropy loss is equivalent to the Hebbian-descent update with mean squared error loss.
Hebbian-descent can thus be understood as a generalization of the idea, \emph{i.e.}\ getting rid of the activation functions derivative of the output layer in gradient descent by using a different loss function, for arbitrary activation functions and error~terms.

\subsubsection{The Hebbian-Descent Loss}\label{sec:HD_Hebbian_descent_loss}

One of the key insights of this work is that under some mild conditions 
Hebbian-descent can be reformulate as gradient descent with an alternative loss function that is given by
\begin{eqnarray}
\mathcal{ L}_{_{HD}}(\vect t, \vect h) &=& \sum_j^M \int \left(\frac{\frac{\partial \mathcal{L}(\vect t, \vect h)}{ \partial \vect h}}{ \frac{\partial }{ \partial \vect a} \odot \vect h}\right) \,\mathrm{d} h_j 
%= \int \frac{ \bm{\mathcal{E}}(\vect t, \vect h)}{ \vect \phi'(\vect a)} \,d  h_j 
 = \sum_j^M \int \frac{ \mathcal{E}(t_j, h_j)}{\phi'(a_j)} \, \mathrm{d} h_j \, ,
\label{eqn:gradient_descent_loss_implementing_Hebbian_descent}
\end{eqnarray}
where fraction bars denote element-wise division (if not derivatives), and $\int$ denotes the general integral.
If we insert this loss into Equations~\eqref{eqn:update_general_gradient_descent_sup_w_1} and \eqref{eqn:update_general_gradient_descent_sup_b_1}, the integrals with respect to $\vect h$ cancel out with the partial derivative with respect to $\vect h$ and so do the partial derivatives of the activation functions in numerator and denominator\footnote{The equivalence of Hebbian-descent and gradient descent with Hebbian-descent loss for a single weight update can be shown as follows:
\begin{eqnarray}
\delta_{_{GD}} w_{ij} &\stackrel{\eqref{eqn:update_general_gradient_descent_sup_w_2}}{=}& -\eta \big(x_i-\mu_i \big) 
\mathcal{E}_{_{HD}}(t_j, h_j) \phi'(a_j)\\
&=& -\eta \big(x_i-\mu_i \big) 
\frac{\partial \mathcal{L}_{_{HD}}(t_j, h_j)}{ \partial h_j} \phi'(a_j)\\
&\stackrel{\eqref{eqn:gradient_descent_loss_implementing_Hebbian_descent}}{=}& 
 -\eta \big(x_i-\mu_i \big) 
\frac{\partial \int \frac{ \mathcal{E}(t_j, h_j)}{\phi'(a_j)} \,\mathrm{d} h_j}{ \partial h_j}  \phi'(a_j)\\
&\stackrel{\eqref{eqn:update_general_gradient_descent_sup_w_2}}{=}& -\eta \big(x_i-\mu_i \big) 
\frac{\mathcal{E}(t_j, h_j)}{\phi'(a_j)}\phi'(a_j)=-\eta \big(x_i-\mu_i \big)\mathcal{E}(t_j, h_j) \\ &\stackrel{\eqref{eqn:update_general_hebbian_descent_sup_w_1}}{=}& \delta_{_{HD}} w_{ij} 
\end{eqnarray}
}. 
The resulting equations are exactly the Hebbian-descent update given by Equation~\eqref{eqn:update_general_hebbian_descent_sup_w_1} and \eqref{eqn:update_general_hebbian_descent_sup_b_1}, showing that gradient descent with loss $\mathcal{L}_{_{HD}}(\vect t, \vect h)$ is equivalent to Hebbian-descent with loss $\mathcal{L}(\vect t, \vect h)$ and that Hebbian-descent actually optimizes the loss $\mathcal{L}_{_{HD}}(\vect t, \vect h)$. 
Notice, that if the $\phi'(a_j)$ contains $\phi(a_j)$ (and therefore $h_j$) you cannot simply factor out $\phi'(a_j)$ in front of the integral in Equation~\eqref{eqn:gradient_descent_loss_implementing_Hebbian_descent}. 
This is the case for the sigmoid function and most of the commonly used activation functions as they are based on the exponential function.
A list of some common activation functions with Hebbian-descent loss and corresponding gradient descent loss is given in Appendix~\ref{appendix:list_of_HD_losses}.

\subsubsection{Convergence and Expediency of Hebbian-descent}\label{sec:HD_Hebbian_descent_covergence}

First of all when the integral in Equation~\eqref{eqn:gradient_descent_loss_implementing_Hebbian_descent} exists, it is clear that Hebbian-descent naturally inherits the convergence properties of stochastic gradient descent~\citep{Robbins1985, saad1998online}. 
But the question remains whether the Hebbian-descent loss is actually useful in terms of reducing the error between $\vect h$ and $\vect t$.
Furthermore, it needs to be clarified under which conditions Hebbian-descent actually defines a valid loss and whether the integral in Equation~\eqref{eqn:gradient_descent_loss_implementing_Hebbian_descent} therefore needs to exist.

To answer these questions let us first consider the case of a single data-point, in which $\mathcal{L}(\vect t, \vect h)$ and $\mathcal{L}_{_{HD}}(\vect t, \vect h)$ have their optima at the same location since their partial derivatives only differ in the division by the derivative of the activation function. 
For a single data-point this only scales the loss function but does not shift the location of the optima. 
In case of negative values of the derivative of the activation function the two update rules can point in opposite directions in which case Hebbian-descent does not minimize but maximize the distance between $\vect h$ and $\vect t$. 
This can best be seen by comparing the entries of the gradient descent update (Equation~\eqref{eqn:update_general_gradient_descent_sup_w_1} and ~\eqref{eqn:update_general_gradient_descent_sup_b_1}) and the Hebbian-descent update (Equation~\eqref{eqn:update_general_hebbian_descent_sup_w_1} and ~\eqref{eqn:update_general_hebbian_descent_sup_b_1}).
In case of zero values of the derivative of the activation function the integral in Equation~\eqref{eqn:gradient_descent_loss_implementing_Hebbian_descent} does not exist and 
the gradient update vanishes while the Hebbian-descent update may be non-zero. In this case gradient descent has converged while Hebbian-descent has not, so that convergence for Hebbian-descent cannot be proven. 
However, as discussed at the end of this section this can be an advantage of Hebbian-descent over gradient descent in practice.
Only if the values of the derivative of the activation functions are strictly positive the two update rules share the same sign and thus share the same minima and maxima.
A strictly positive derivative thus guarantees that the angle between the two update vectors is less than 90 degrees as illustrated for the two dimensional case in Figure~\ref{fig:update_rule_compare}.
\begin{figure}[t]
\begin{center}
\centerline{\includegraphics[scale=0.5,trim=25 435 130 25, clip]{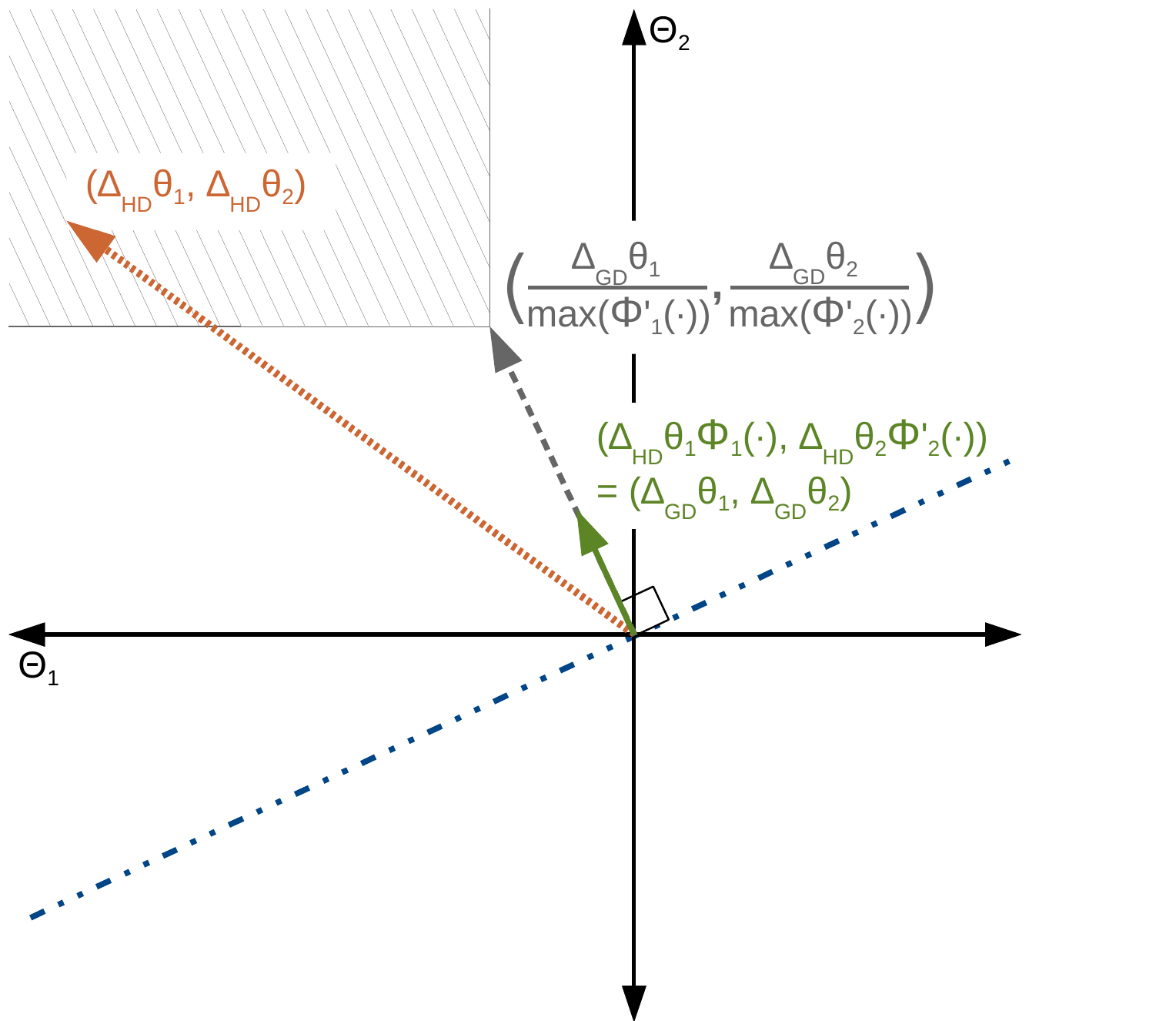}}
\caption{Visualization of the Hebbian-descent (orange fine-dashed arrow) and gradient-descent (green solid arrow) updates in the 2D case with a strictly positive derivative of the activation function. The loss equipotential-lines is given as blue dashed-dotted line.
The relationship between Hebbian-descent and gradient-descent is fully determined by the value of the derivative of the activation function. Thus for a given gradient descent update the Hebbian-descent update can only take values that fall in the gray hatched area. Furthermore, the minimal possible norm of the Hebbian-descent update vector is determined by the maximal possible value of the activation functions derivative as illustrated by the gray dashed arrow. In case of the sigmoid, the derivative has a maximal value of 0.25 such that the norm of the Hebbian-descent update vector is at least four times larger than the norm of the gradient descent update vector. 
}
\label{fig:update_rule_compare}
\end{center}
\end{figure}
Since the negative gradient is orthogonal to the loss equipotential-line and points locally in the direction of steepest descent in $\mathcal{L}(\vect t, \vect h)$, Hebbian-descent is also guaranteed to point locally in a direction of decreasing $\mathcal{L}(\vect t, \vect h)$, in case of an activation function with strictly positive derivative. 
This allows to prove convergence even if the integral in Equation~\eqref{eqn:gradient_descent_loss_implementing_Hebbian_descent} does not exist (see Appendix~\ref{appendix:convergence_of_stochastic_hebbian_descent} for a formal proof). 
Thus for a proper choice of $\mathcal{L}(\vect t, \vect h)$, which has its global minimum where $\vect t$ equals $\vect h$ such as the squared error loss, and an activation function with strictly positive derivative, gradient descent and Hebbian-descent both converge to the optimum of $\mathcal{L}(\vect t, \vect h)$. 
This is also the optimum of $\mathcal{L}_{_{HD}}(\vect t, \vect h)$, so that the Hebbian-descent loss is indeed useful in terms of reducing the error between $\vect h$ and $\vect t$.

Now for several data points the two update rules are simply the average over the individual updates. 
Unless the network is powerful enough to reach a zero loss for all data points (as it is the case in the example shown in Figure~\ref{fig:update_rule_2D_example}) the joint optimum over all data points of $\mathcal{L}(\vect t, \vect h)$ and $\mathcal{L}_{_{HD}}(\vect t, \vect h)$ will most likely differ. 
This reflects the fact that the two update rules find a different compromise between the individual errors as shown and discussed in the experiments. 

Finally, let us also consider the cases where the derivative of the activation function can additionally take zero values, which is problematic since it means that gradient descent can have converged while Hebbian-descent has not. 
One can often get around this problem in practice by modifying the chosen activation function such that it has little to no effect on the network performance but causes the integral to exist at least in a restricted domain. 
For the rectifier for example one could use a leaky-rectifier~\citep{MaasHannunEtAl-2013} instead, which uses a small positive slope $\epsilon$ for negative values instead of zero. 
However, even if we just use the rectifier in practice, Hebbian-descent still convergences empirically in terms of the specified loss $\mathcal{L}(\vect t, \vect h)$ as shown in the experiments. 
One can argue, that Hebbian-descent might still be preferable over gradient descent since a zero gradient caused by a zero derivative of the activation function (\emph{e.g.}\ rectifier) causes local plateau optima, which represent 'dead-ends' in the optimization process. 
As an example think of an initialization where any data point of a dataset causes a negative pre-threshold activity, which can be achieved by choosing a rather large negative initial bias vector for example. 
In this case the gradient is constant zero for the entire dataset right from the start and does not perform any updates at all. 
Despite this extreme example one can imagine that a gradient update that causes a unit `accidentally' to never be active again can also happen during training with a reasonable parameter initialization especially in case of online or mini-batch training with big learning rates or if the input domain is shifted. 
Furthermore, it is even problematic if a unit is updated such that it is active only on a few data points, since only for those data points the gradient is non-zero and thus allows corresponding changes. According to Equation~\eqref{eqn:update_hebbian_descent_sup_w} and \eqref{eqn:update_hebbian_descent_sup_b} Hebbian-descent still performs an update even if the gradient is zero and is thus advantageous over gradient descent as shown empirically.

\subsubsection{The Error Term Perspective}\label{sec:HD_the_error_term_perspective}

In this section we want to motivate why the error term perspective of Hebbian-descent is better when designing neural networks than the loss function perspective of gradient descent.
In practice two loss functions are mainly used, the cross entropy loss in case of sigmoid / softmax output units and otherwise the squared error loss.
But it makes a huge difference if we use the squared error loss with the sigmoid or the exponential linear units for example, since the two activation functions act in different output regions and thus in different regions of the parabola defined by the squared error loss. 
Furthermore, as has been discussed before, the derivative of the activation function might scale down the error propagated through the network in a nonlinear way, such that a huge error in the loss function does not necessarily lead to a huge change in the parameters of the network. 
It is therefore better to treat the activation function and the loss function as a unity as done by the Hebbian-descent loss.
By explicitly defining the error term, one gains full control over how the error in the output activities affects the parameters.
As an example we could define a nonlinear error term, \emph{e.g.}\ $\bm{\mathcal{E}}(\vect t, \vect h) =  \alpha\,\vect \tanh(\beta (\vect h-\vect t))$, which saturates the error at $\pm\alpha$ with $\beta$ defining the steepness around the origin. $\alpha$ and $\beta$ should thus be chosen according to the minimal and maximal output values of the network.
The corresponding loss can be calculated using Equation~\eqref{eqn:gradient_descent_loss_implementing_Hebbian_descent}, which in case of the identity as an activation function leads to $\mathcal{L}_{_{HD}}(\vect t, \vect h) = \frac{\alpha}{\beta}  \sum_j \ln(\cosh(\beta ( h_j- t_j)))$. 
This can be seen as a smoothed version of the Huber-loss~\citep{Huber-1964}, which is plotted in the two dimensional case for different $\alpha$ in Figure~\ref{fig:reg_loss_HD}. 
\begin{figure}[t]
\begin{center}
\centerline{\includegraphics[scale=0.5]{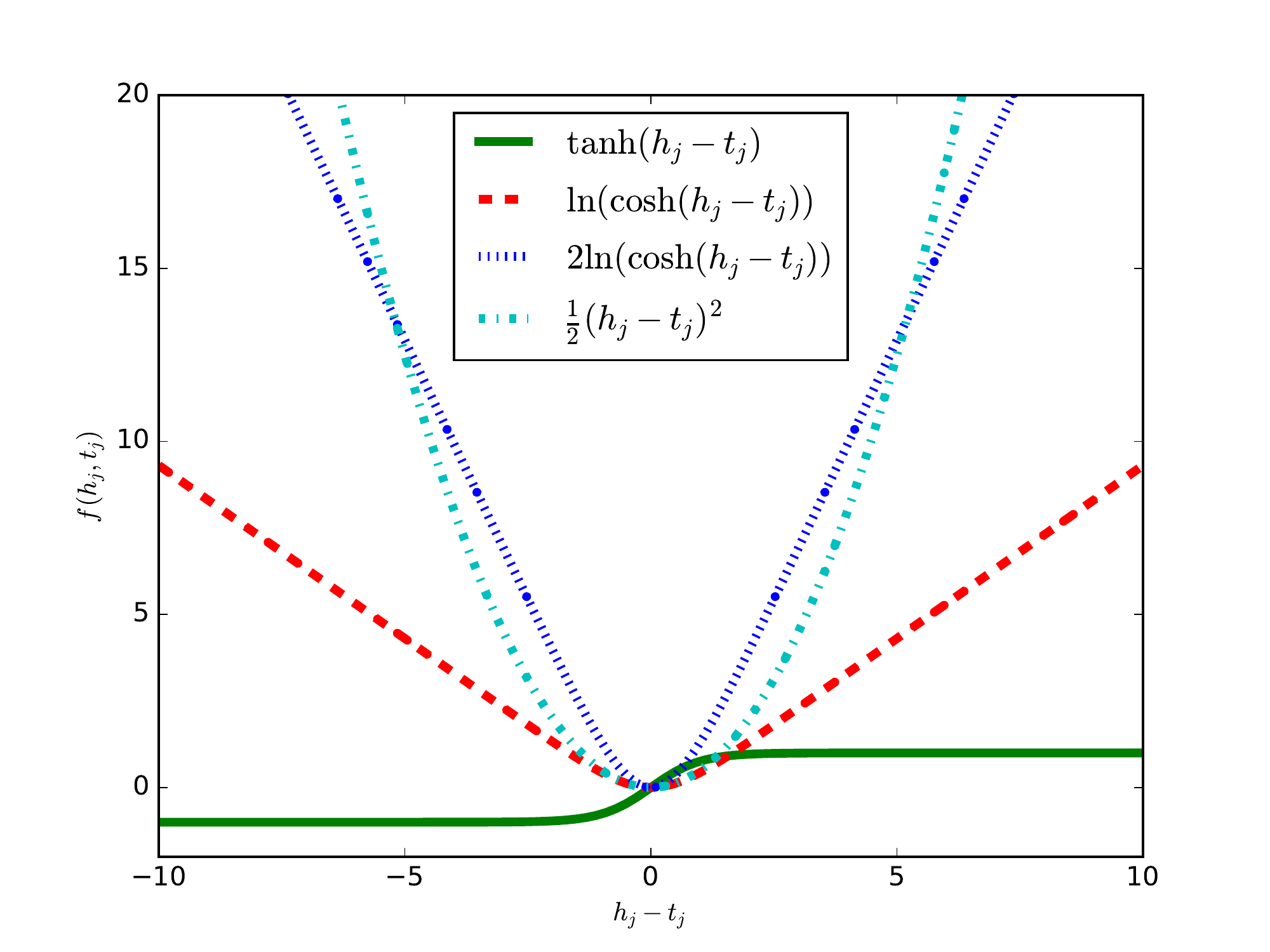}}
\caption{ Illustration of the smooth Huber loss $\mathcal{L}_{_{HD}}( t_j,  h_j) = \frac{\alpha}{\beta}  \ln(\cosh(\beta ( h_j- t_j)))$ and the corresponding error term $\mathcal{E}( t_j,  h_j) = \alpha \tanh(h_j-t_j)$, where $\beta=1$ and $\alpha = 1$ or $\alpha = 2$. 
See Equation~\eqref{eqn:gradient_descent_loss_implementing_Hebbian_descent} in combination with the identity activation function for the derivation of the smooth Huber loss. For comparison we also show the squared error loss, which has the error term $\mathcal{E}( t_j,  h_j) = h_j-t_j$.
}
\label{fig:reg_loss_HD}
\end{center}
\end{figure}
In case of the sigmoid as an activation function the integral in Equation~\eqref{eqn:gradient_descent_loss_implementing_Hebbian_descent} does not exist, but Hebbian-descent still converges as the angle between the gradient descent update and the Hebbian-descent update is still guaranteed to be less than 90 degrees (See Appendix~\ref{appendix:convergence_of_stochastic_hebbian_descent}).
A detailed analysis of this form of error terms is beyond the scope of this paper. 
But it is obvious that a bounded error term is especially useful if the activation function is unbounded such as the identity of rectifier, as in this case the error term stays in a reasonable range even for very large learning rates.

When solving classification problems with (deep) neural networks the combination softmax-output units and cross-entropy loss is the common choice. This is usually motivated by the fact that minimizing the cross-entropy loss on softmax-output units is equivalent to maximizing the likelihood of the network's softmax-output given the input~\citep{Hinton-1989}. In case of a single layer network this is also known as logistic regression, which belongs to a class of statistical models whose relation to Hebbian-descent is discussed in the next section. 

\subsection{From the Perspective of Generalized Linear Models}\label{sec:HD_glm}

In case of the squared error loss and an invertible activation function Hebbian-descent actually optimizes a generalized linear model~\citep{Nelder1972}, which models the probability distribution of some random output variables $\vect t$ given input $\vect x$. In the following we give a brief introduction to generalized linear models sufficient to show how they relate to Hebbian-descent. For a detailed introduction see \citet{McCullaghNelder-1989} for example.

\subsubsection{Generalized Linear Models}\label{sec:HD_generalized_linear_models}

As a generalization of linear least squares regression, which assumes Gaussian distributed output variables, generalized linear models allow the output variables to follow any distribution that is in the exponential family. 
In the context of generalized linear models it is convenient to consider distributions of the so called over-dispersed exponential family~\citep{Jorgensen1987}, which is a generalization of the natural exponential family. The corresponding probability density function for some output variables $\vect t$ parameterized by $\vect \theta$ and $\alpha$ is given by
\begin{eqnarray}
	p(\vect t | \vect \theta, \alpha ) &=&  b (\vect t, \alpha ) \exp\left(\frac{
		\vect \theta^T \vect  t -  \vect A (\vect \theta ) }{\alpha}\right),\label{eqn:GLM_PDF}
\end{eqnarray}
where $\vect \theta $ is usually known as the natural parameter, $\alpha$ as the dispersion parameter,  $\vect A (\vect \theta )$ as the log-partition function, and $ b (\vect t, \alpha )$ is a known function. 
Notice, that $\vect \theta $ can be matrix an that $\vect x$ enters the model through $\vect \theta$ as discussed below.
An important property of this kind distributions is that the expectation value of the output variable is known to be 
\begin{equation}
	\vect \gamma = \mathbf{E} [\vect t] =  \vect A'(\vect \theta)\label{eqn:GLM_Mean}, 
\end{equation} 
where $\mathbf{E} [\vect t]$ denotes the vector containing expectation values of the output variables \emph{i.e}\ $\mathbf{E} [\vect t] =  [\text{E} [t_0], \cdots, \text{E}[t_M]]$.
Thus, $\vect \gamma$ and $\vect \theta$ must be related, which is generally denoted by 
\begin{equation}
	\vect \theta = \vect \psi(\vect \gamma),\label{eqn:GLM_link_function}\,\,\,\,\,\,\,\,\,\,\,\,\,\,\,\,\,\,\,\,\,\text{\small{(implying $\vect A'$ to be invertable)}}
\end{equation}
where $\vect \psi(\cdot)$  is known as the link-function, which is an invertible mapping that relates the model mean $\vect \gamma$ to the natural parameter $\vect \theta$. Finally  the input $\vect x$ has to be incorporated into the model and has to be related to the model's mean. This is done by choosing a linear predictor $\vect a$ (\emph{i.e.}\ a linear combination of the input values and some parameters) that is passed through an invertible, integrable, and possibly nonlinear function $\vect\phi(\cdot)$. Such a mapping is provided by a centered single layer network (see Equation~\eqref{eqn:neuron_matrix}) when the activation function is chosen to be invertible and integrable, in which case the natural parameter becomes
\begin{equation}
	\vect \theta \overset{(\ref{eqn:neuron_matrix},\ref{eqn:GLM_link_function})}{=} \vect \psi\big(\vect\phi (\vect a)\big) = \vect\psi\Big(\vect\phi\big( \vect{W}^T\left(\vect{x}-\vect\mu \right)+\vect{b}\big)\Big)\label{eqn:GLM_natural_parameter}.
\end{equation}
The model is now fully specified and we can use maximum likelihood estimation to optimize the parameters of the model. Let us therefore consider the gradient of the model's log-likelihood w.r.t. $\vect W$ and $\vect b$, which is given by
\begin{eqnarray}
	\frac{\eta \partial \log p(\vect t| \vect x, \vect W, \vect b, \alpha )}{\partial \vect W} &\overset{(\ref{eqn:GLM_PDF},\ref{eqn:GLM_natural_parameter})}{=}& \eta \bigg(\vect x - \vect \mu \bigg)\bigg(\frac{\vect t - \vect A' (\vect \theta )}{\alpha } \odot \vect\psi'(\vect \gamma) \odot \vect\phi'(\vect a)\bigg)^T\label{eqn:GLM_log_grad_W_1}\\
	 &\overset{(\ref{eqn:GLM_Mean})}{=}& \eta \bigg(\vect x - \vect \mu \bigg)\left(\frac{\vect t - \vect \gamma }{\alpha } \odot \vect\psi'(\vect \gamma) \odot \vect\phi'(\vect a)\right)^T\label{eqn:GLM_log_grad_W_2} \\
	\frac{\eta \partial \log p(\vect t| \vect x, \vect W, \vect b, \alpha )}{\partial \vect b} &\overset{(\ref{eqn:GLM_PDF},\ref{eqn:GLM_log_grad_W_2})}{=}& \eta \left(\frac{\vect t - \vect \gamma }{\alpha } \odot \vect\psi'(\vect \gamma) \odot \vect\phi'(\vect a)\right).\label{eqn:GLM_log_grad_b}
\end{eqnarray}
In order to be able to calculate the gradient numerically we have to decide for an activation function, where we are free to choose any function as long as it is invertible and integrable. However, there is a striking choice namely to choose $\vect\phi (\cdot) = \vect\psi^{-1} (\cdot)$, in which case the model is said to be in canonical form since the natural parameter simplifies to $\vect \theta = \vect a$, in which case the gradient becomes    
\begin{eqnarray}
	\frac{\eta \partial \log p(\vect t| \vect x, \vect W, \vect b, \alpha )}{\partial \vect W} &\overset{(\ref{eqn:GLM_natural_parameter},\ref{eqn:GLM_log_grad_W_2})}{=}& \eta \bigg(\vect x - \vect \mu \bigg)\left(\frac{\vect t - \vect \gamma }{\alpha }\right)^T\\
	&\overset{(\ref{eqn:GLM_link_function})}{=}& \eta \bigg(\vect x - \vect \mu \bigg)\left(\frac{\vect t - \vect\psi^{-1}(\vect \theta) }{\alpha }\right)^T\\
&\overset{\text{\tiny{$\vect\phi (\cdot) = \vect\psi^{-1} (\cdot)$}}}{=}& \eta \bigg(\vect x - \vect \mu \bigg)\left(\frac{\vect t - \vect\phi(\vect a) }{\alpha }\right)^T,  \label{eqn:GLM_log_grad_W_canonical}\\
	\frac{\eta \partial \log p(\vect t| \vect x, \vect W, \vect b, \alpha )}{\partial \vect b} &\overset{(\ref{eqn:GLM_log_grad_b},\ref{eqn:GLM_log_grad_W_canonical})}{=}&  \eta
\left(\frac{\vect t - \vect\phi(\vect a) }{\alpha }\right) .\label{eqn:GLM_log_grad_b_canonical}
\end{eqnarray}
If we now set $\alpha=1$ the gradient of the log-likelihood becomes equivalent to the Hebbian-descent update for a squared error loss as given by Equations~\eqref{eqn:update_general_hebbian_descent_sup_w_1} and \eqref{eqn:update_general_hebbian_descent_sup_b_1}. 

\subsubsection{The Hebbian-Descent Loss is the General Log-Likelihood Loss}\label{sec:HD_HD_loss_generalized}

As shown above another important insight of this work is that for an invertible and integrable activation function and the squared error loss the Hebbian-descent update optimizes the log-likelihood of a generalized linear model in canonical form. 
As a consequence the Hebbian-descent loss given by Equation~\eqref{eqn:gradient_descent_loss_implementing_Hebbian_descent} with the mean squared error term is the general log-likelihood loss for any activation function as long as it is invertible and integrable.

We can now define the probability density function of the generalized linear model Hebbian-descent optimizes more specifically. Since the output variables $t_j$ are conditionally independent given the input $\vect x$ for a natural parameter as defined by Equation~\eqref{eqn:GLM_natural_parameter}, we can write the probability density function as 
\begin{eqnarray}
	%p(\vect t | \vect x, \vect W, \vect b) &\overset{(\ref{eqn:GLM_link_function},\ref{eqn:GLM_log_grad_W_2})}{=}&  b (\vect t ) \exp\Big(
	%	\vect a^T \vect  t -  \int \vect \phi (\vect a ) \,\mathrm{d} \vect a \Big), \\
	p(\vect t | \vect x, \vect W, \vect b) 
		&\overset{(\ref{eqn:GLM_natural_parameter})}{=}&  \prod_j^M p( t_j | \vect x, \vect W, \vect b) \\
		&\overset{(\ref{eqn:GLM_PDF},\ref{eqn:GLM_link_function})}{=}&  \prod_j^M b (t_j ) \exp\Big(
		\frac{a_j  t_j -  \int \phi ( a_j ) \,\mathrm{d} a_j}{\alpha} \Big),
		\label{eqn:GLM_PDF_HD}
\end{eqnarray}	
where $\int$ denotes the general integral emphasizing why $\phi (\cdot )$ needs to be integrable. 
For a generalized linear model one would usually define the model from the distribution's perspective rather than the activation function's perspective, \emph{i.e.}\ choose a certain distribution of the exponential family such as Gaussian, Binomial, Bernoulli, Poisson, ... for which the parameters are known.
As an example consider the Bernoulli distribution for which $b(t_j)= \alpha = 1$, and with link function $\psi (\gamma_j) =  \log(\gamma_j(1-\gamma_j))^{-1})$, which is known as the logit function that is the inverse of the sigmoid. Thus $\phi (a_j) = (1+\exp(-a_j))^{-1}$ and the corresponding anti-derivative becomes $\int \phi ( a_j ) \,\mathrm{d} a_j = \log (1+\exp(a_j) )$. 
It is, however, not obvious to see that these values lead to the Bernoulli distribution when inserted in Equation~\eqref{eqn:GLM_PDF_HD}, as it is usually defined with respect to the expectation value of random variable $\gamma_j$. The corresponding derivation is as follows:
\begin{eqnarray}
	p(t_j | \vect x, \vect W, \vect b) \!\!\!\!
		&\overset{(\ref{eqn:GLM_PDF_HD})}{=}& \!\!\!\!  \exp \Big(
		a_j  t_j -  \log \big(1+\exp(a_j) \big) \Big) \\
					&\overset{(\ref{eqn:GLM_link_function})}{=}& \!\!\!\!  
\exp \Big(
		\psi( \gamma_j) \Big)^{t_j}   \Big(1+\exp \big(\psi( \gamma_j) \big)\Big)^{-1} \\	
				&\overset{(\ref{eqn:GLM_link_function})}{=}&  \!\!\!\!
\exp \!\bigg(\!
		\log\big(\gamma_j(1-\gamma_j)^{-1} \big) \!\bigg)^{t_j}   \!\bigg(\!1+\exp \!\Big(\!\log\big(\gamma_j(1-\gamma_j)^{-1} \big)\!\Big)\!\bigg)^{-1}\,\,\,\,\,\, \\	
						&=&  \!\!\!\!
\big(\gamma_j(1-\gamma_j)^{-1} \big)^{t_j}   \big(1+\gamma_j(1-\gamma_j)^{-1} \big)^{-1}\,\,\,\,\,\, \\	
						&=&  \!\!\!\! 
\gamma_j^{t_j}(1-\gamma_j)^{-t_j}    \big((1-\gamma_j)(1-\gamma_j)^{-1}+\gamma_j(1-\gamma_j)^{-1} \big)^{-1} \\		
						&=&  \!\!\!\! 
\gamma_j^{t_j}(1-\gamma_j)^{-t_j}    \big(1-\gamma_j+\gamma_j\big)(1-\gamma_j)^{-1} \big)^{-1} \\	
						&=&  \!\!\!\! 
						\gamma_j^{t_j}(1-\gamma_j)^{-t_j}    (1-\gamma_j)  \\	
						&=&  \!\!\!\! 
\gamma_j^{t_j}(1-\gamma_j)^{1-t_j}	,
		\label{eqn:example_bernoulli}
\end{eqnarray}
which is obviously the Bernoulli-distribution for the output variable $t_j$ taking the value one with probability $\gamma_j$ and the value zero with probability $(1-\gamma_j)$.

However, one is often not really interested in modeling the probability density function, \emph{e.g.}\ when using logistic regression one is usually interested in the classification error rather than in accurate probabilistic modeling. 
In this case Hebbian-descent allows more flexibility since we are not as restricted in the choice of the activation function nor do we have to use the squared error loss. 
But even if the chosen activation function does not define a valid probability distribution as in case of a rectifier for example the relationship to maximum likelihood learning is still present and a useful insight as the  maximizing the likelihood is usually considered to be a good optimization objective.

\subsection{From the Perspective of an Adaptive Learning Rate for Deep Neural Networks}\label{sec:HD_learning_rate}

There exists another perspective on Hebbian-descent that is that of an adaptive learning rate for gradient descent. By defining an individual learning rate for each hidden unit separately as 
\begin{equation}
\eta_j = \frac{\kappa}{\phi'(\vect a_j)},
\end{equation}
where $\kappa > 0$ is a constant, the derivative of the activation function cancels out in the gradient descent update given by Equation~\eqref{eqn:update_general_gradient_descent_sup_w_1}, also turning it into the Hebbian-descent update given by Equations~\eqref{eqn:update_general_hebbian_descent_sup_w_1} and \eqref{eqn:update_general_hebbian_descent_sup_b_1}. 
A necessary requirement is obviously again a strictly positive derivative of the activation function, in which case we only scale the constant part of the learning rate $\kappa$ up or down. 
Even for a deep network with several hidden layers this learning rate only changes the gradient direction by less than 90 degrees, such that the update for a single data point converges to the minimum of the loss function as shown in Appendix~\ref{appendix:convergence_of_stochastic_hebbian_descent}. 
However, since we do not back-propagate the learning-rate through the network, the derivative only cancels out in the partial derivatives of the corresponding layer but is still present in the back-propagated delta values. One exception is a network with one hidden layer where we use the Hebbian-descent loss to get rid of the derivative in the top layer. In this case the back-propagated delta values do not contain any derivatives of activation functions and if we then use the adaptive learning rate only on the first layer no derivatives of activation functions appear in the gradient descent update at all.

Adaptive learning rates are very common when training deep neural networks, see \cite{ruder2016} for a good overview. However, all these methods use historical gradients to determine an adaptive learning rate, which is different from the approach given above where no historical gradient information is needed and an individual learning rate is given for each hidden unit and data point separately. A combination of the two methods would also be of interest, but the analysis of deep networks is beyond the scope of this work and left for future work.

\subsection{Auto-Associative Learning in Single Layer Networks}\label{sec:HD_auto_associative}

In case of auto-associative learning we consider a network structure of an auto-encoder with tied / symmetric weights as illustrated in Figure~\ref{fig:auto_ecoder_structure}, which can also be seen as a undirected model such as a Hopfield network with restricted connectivity or as a restricted Boltzmann machine. %TODO as given by Figure examplek RBM
\begin{figure}[t]
\begin{center}
\centerline{\includegraphics[scale=0.54, trim=0 100 0 10, clip]{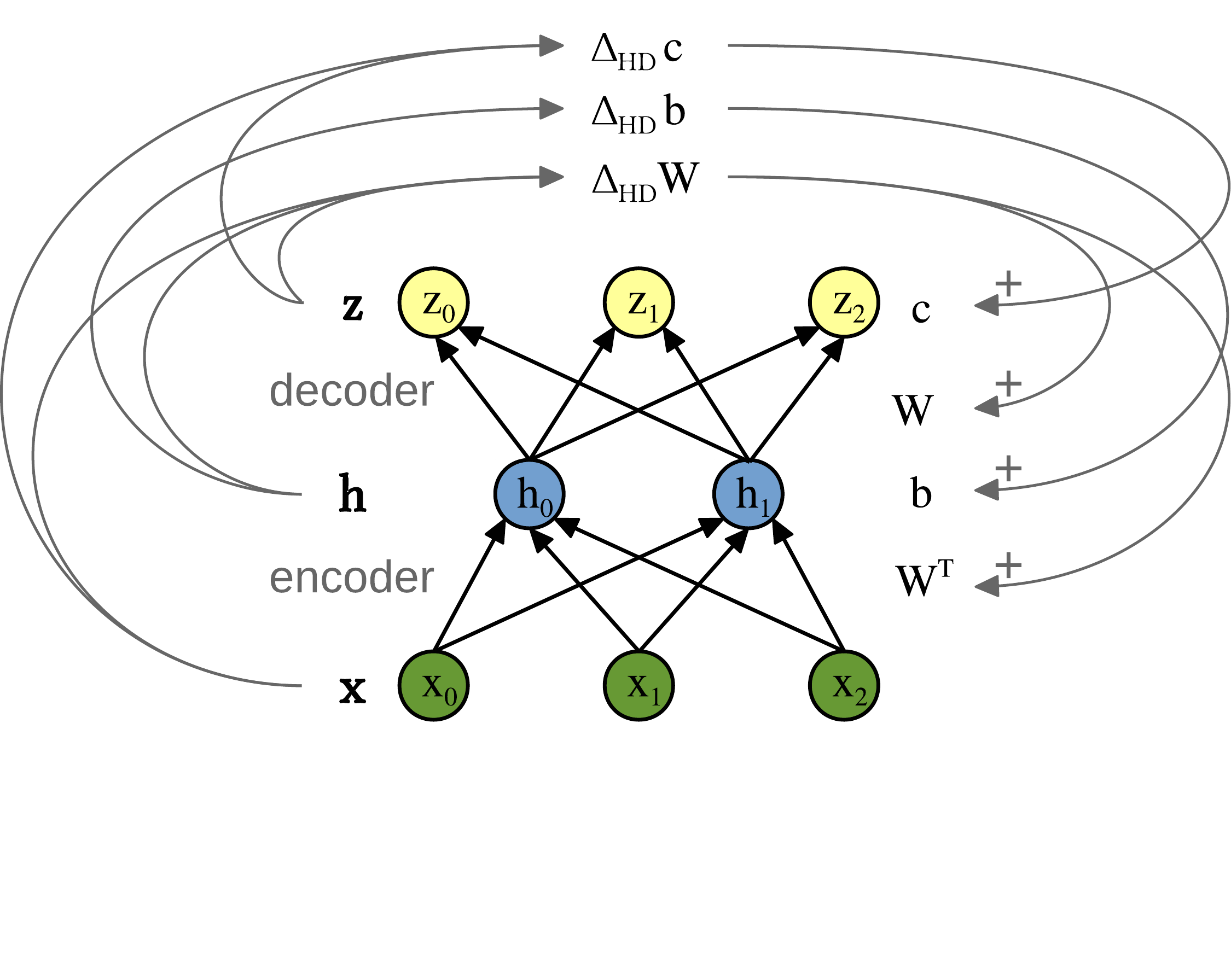}}
\caption{ Illustration of the auto-associative Hebbian-descent update for a single layer auto encoder network.
}
\label{fig:auto_ecoder_structure}
\end{center}
\end{figure}
Depending on the learning rule it is advantageous to consider one or the other perspective in the following.
An auto-encoder can generally be decomposed into an encoder that encodes the input $\vect x$ into a hidden representation $\vect h$ and a decoder that decodes the hidden representation  into a reconstructed input $\vect z$. For a centered single layer auto-encoder with tied weights we use Equation~\eqref{eqn:neuron_matrix} as encoder and define the decoder using the same weights by
\begin{eqnarray}
\vect{z} &=& \vect \phi_{_{dec}}\big(\vect a_{_{dec}}\big)
		 = \vect \phi_{_{dec}}\big(\vect{W}(\vect{h}-\vect\lambda)+\vect{c}\big)\label{eqn:neuron_second_layer1}\\
		 &=& \vect \phi_{_{dec}}\Big(\vect{W}\big(\vect \phi_{_{enc}}(\vect a_{_{enc}}) - \vect \lambda\big) + \vect c \Big)\label{eqn:neuron_second_layer2}\\
		 &\overset{\eqref{eqn:neuron_matrix}}{=}& \vect \phi_{_{dec}}\bigg(\vect{W}\Big(
		 \vect \phi_{_{enc}}\big(\vect{W}^T(\vect{x}-\vect\mu )+\vect{b}
		 \big)-\vect\lambda \Big)+\vect{c}\bigg),\label{eqn:neuron_second_layer3}
\end{eqnarray}
where $\vect W$ are the weights, $\vect b$, $\vect \mu$, $\vect a_{_{enc}}$, and $\vect \phi_{_{enc}}$ are the encoder's bias, offset, pre-threshold activity, and activation function, and $\vect c$, $\vect \lambda$, $\vect a_{_{dec}}$, and $\vect \phi_{_{dec}}$ are the decoder's bias, offset, pre-threshold activity, and activation function, respectively. 

\subsubsection{Auto-Associative Hebbian-Descent Update}\label{sec:HD_auto_associative_hebbian_descent_update}

The Hebbian-descent update we propose for these type of auto-encoder networks is basically the same as the supervised / hetero-associative update but in a reverse perspective, \emph{i.e.}\ we flip the role of input and target variables. The auto-associative Hebbian-descent updates for weight matrix $\vect W$, encoder bias $\vect b$, and decoder bias $\vect c$ can therefore be given by
\begin{eqnarray}
\Delta_{_{HD}} \vect W &=& -\eta \frac{\partial \mathcal{L}(\vect x, \vect z)}{ \partial  \vect z}(\vect h- \vect \lambda )^T \label{eqn:update_general_hebbian_descent_unsup_w_2}\\
&=&-\eta \bm{\mathcal{E}}(\vect x, \vect z)(\vect h- \vect \lambda )^T\label{eqn:update_auto_hebbian_descent_unsup_w_1},\\
\Delta_{_{HD}} \vect c  &=& -\eta \frac{\partial \mathcal{L}(\vect x, \vect z)}{ \partial \vect z}
 \label{eqn:update_auto_hebbian_descent_unsup_c_2}\\
&=& -\eta  \, \bm{\mathcal{E}}(\vect x,\vect{z})\label{eqn:update_auto_hebbian_descent_unsup_c_1},\\
\Delta_{_{HD}} \vect b &=& -\eta  \, (\vect h- \vect{\tilde \lambda}), \,\,\,\,\,\,\,\,\,\,\,\,\,\,\,\,\,\,\,\,\,\,\,\,\,\,\,\,\,\,\,\, \text{\small{(optional)}}\label{eqn:update_auto_hebbian_descent_unsup_b_1}
\end{eqnarray}
where the update for the encoder bias can be used to regularize the hidden units activity to be $\vect{\tilde \lambda}$ or set to zero otherwise.
The update process for the weights and bias terms is illustrated  in Figure~\ref{fig:auto_ecoder_structure}.
Notice, that analogously to contrastive learning this update rule is only local if we introduce time such that $\vect x$ and $\vect z$ are the post synaptic activities of the same neurons at two different time steps.
It is straight forward to see that, when assuming $\vect h$ to be given, the updates for weights and decoder bias are equivalent to the Hebbian-descent update for supervised / hetero-associative learning (Equation~\eqref{eqn:update_general_hebbian_descent_sup_w_1} and \eqref{eqn:update_general_hebbian_descent_sup_b_1}) applied on the decoder network (Equation~\eqref{eqn:neuron_second_layer1}). 
Strictly speaking, however, $\vect h$ is not given since it depends on the weights and input, such that for the same input the hidden representation will differ before and after updating the network's weights. 
This induces some kind of non-stationarity, which is problematic if we want to show (such as we did in the supervised case) that Hebbian-descent corresponds to the gradient of a specific loss function. In fact we prove in Appendix~\ref{appendix:AA_HD_is_generally_not_the_gradient_of_any_objective_function} that the auto-associative Hebbian-descent update is not only not the gradient of a corresponding loss function, it is not the gradient of any function.
Notice, however, that this does not necessarily mean that Hebbian-descent does not converge to a stable solution as it could systematically deviate less than 90 degrees from the corresponding gradient update, in which case it would still converge to the same loss value.
As shown empirically Hebbian-descent indeed converges to a reconstruction error that is comparable to when gradient decent / back-propagation is used to train the auto-encoder. 
This, can be argued, happens since the Hebbian-descent update is actually part of the gradient descent update and thus highly dependent on it.

\subsubsection{From the Perspective of Gradient Descent}\label{sec:HD_from_the_perspective_of_gradient_descent}

In gradient descent we get around the problem of a changing hidden representation by propagating the error from the reconstruction $\vect z$ through the hidden representation $\vect h$ back to the input $\vect x$, such that the hidden representation is automatically inferred alongside.
The gradient descent updates for weights $\vect W$, bias values $\vect b$ and 
$\vect c$ are given by
\begin{eqnarray}
\Delta_{_{GD}} \vect W &=& 
-\eta \underbrace{\Bigg(\vect x- \vect \mu \Bigg)
\Bigg(
\vect W^T\bigg(\frac{\partial \mathcal{L}(\vect x, \vect z)}{ \partial  \vect z} \odot \vect \phi'_{_{dec}}(\vect a_{_{dec}})\bigg)\odot \vect \phi'_{_{enc}}(\vect a_{_{enc}})\Bigg)^T}_{\text{encoder related part}},\,\,\,\,\,\,\,\,\,\,\,\\
&& -\eta \underbrace{\bigg(\frac{\partial \mathcal{L}(\vect x, \vect z)}{ \partial  \vect z} \odot \vect \phi'_{_{dec}}(\vect a_{_{dec}})\bigg)\bigg(\vect h- \vect \lambda \bigg)^T}_{\text{decoder related part}}
\label{eqn:update_auto_gradient_descent_unsup_w}\\
\Delta_{_{GD}} \vect c &=& -\eta \frac{\partial \mathcal{L}(\vect x, \vect z)}{ \partial  \vect zo} \odot \vect \phi'_{_{dec}}(\vect a_{_{dec}}),\label{eqn:update_auto_gradient_descent_unsup_c}\\
\Delta_{_{GD}} \vect b &=& -\eta 
\vect W^T\bigg(\frac{\partial \mathcal{L}(\vect x, \vect z)}{ \partial  \vect z} \odot \vect \phi'_{_{dec}}(\vect a_{_{dec}})\bigg)\odot \vect \phi'_{_{enc}}(\vect a_{_{enc}}),\label{eqn:update_auto_gradient_descent_unsup_b}
\end{eqnarray}
where we use 'encoder related part' and 'decoder related part' to denote the parts of the weight gradient that arise from the occurrence of the tied weights in encoder and decoder, respectively. 
Thus, if we would not use tied weights these two parts would be the individual gradients of an encoder and a decoder weight matrix. 
As shown below the performance of the auto-encoder is almost equivalent with and without updating the encoder bias, which can be argued is a result of the undetermined hidden representation that allows for a bias free distribution over $\vect h$. In case of a linear auto-encoder this can even be shown analytically as the encoder bias can be incorporated into the decoder bias\footnote{$\vect{W}\Big(\vect{W}^T(\vect{x}-\vect\mu )+\vect{b}-\vect\lambda \Big)+\vect{c} = \vect{W}\Big(\vect{W}^T(\vect{x}-\vect\mu )\Big)+\underbrace{\vect{W}\vect{b}-\vect{W}\vect\lambda +\vect{c}}_{\vect{\tilde c}}$}. 
This is also the reason why we do not update the encoder bias in Hebbian-descent unless we want to regularize the hidden units to match a desired average activity level of $\vect{\tilde \lambda}$. 
If we use the Hebbian-descent loss (Equation~\eqref{eqn:gradient_descent_loss_implementing_Hebbian_descent}) with the gradient descent updates, the partial derivative of the activation function disappears and the update for the decoder bias and the 'Decoder related part' of the weight gradient become equivalent to the Hebbian-descent update given by Equation~\eqref{eqn:update_auto_hebbian_descent_unsup_w_1} and~\eqref{eqn:update_auto_hebbian_descent_unsup_c_1}, respectively.
Thus, the remaining difference between the two update rules is the 'encoder related part' of the weight gradient, which is the encoder weight gradient if we would not use tied weights. 
Now, an important observation by~\citet{ImBelghaziEtAl-2016} is that even without tied weights auto-encoders tend to come up with approximately symmetric solution for the weights, giving evidence why Hebbian-descent and gradient descent have roughly the same performance.

\subsubsection{From the Perspective of Oja's / Sanger's Rule}\label{sec:HD_from_the_perspective_of_oja_sanger}

Another interesting observation is that in case of a squared error loss, zero offsets, and identity as activation functions the Hebbian-descent update for the weights become equivalent to Oja's rule~\citep{oja1982simplified} or Sanger's rule~\citep{sanger1989optimal} in case of several output units, which are known to perform principal component analysis.
Such as Sanger's Rule, Hebbian descent can thus be understood as an unsupervised learning rule for a directed single layer network. 
\citet{KarhunenJoutsensalo-1994} argued that the same update rule used with nonlinear hidden units approximates stochastic gradient descent for minimizing the square error between input and reconstruction and should thus perform nonlinear principal component analysis. The authors, however, did not show convergence and only argued that the update for the encoder becomes less important as we get closer to the optimum.

\subsubsection{From the Perspective of Contrastive Learning}\label{sec:HD_auto_asso_from_the_perspective_of_contrastive_learning}

Auto-associative Hebbian-descent cannot be reformulated as gradient descent, but such as for the hetero-associative Hebbian-descent update it can be relate to contrastive learning.
Since $\vect h$ depends on the same weights as $\vect z$ it does not change 
arbitrarily, such that it is more appropriate to view the auto-encoder as a dynamical system, \emph{i.e.}\ restricted Boltzmann machine or Hopfield network.
\citet{Kamyshanska2015} showed that an auto-encoder with tied weights is
guaranteed to have a well-defined energy function, which can be optimized using contrastive learning rules based on Equations~\eqref{eqn:Contrastive_learning_W}-\eqref{eqn:Contrastive_learning_b}. 
As discussed in Section~\ref{sec:HD_hetero_contrastive} in Contrastive Divergence~\citep{Hinton-2002a} we derive the values for the negative phase by sampling from a Markov chain that is initialized with the same data points as used in the positive phase. 
This is often done just for a single sampling step leading to a so called CD-1 update. 
Instead of sampling we can also use a mean field approach leading to the Contrastive Divergence mean field update~\citep{WellingHinton-2002}. 
Now for a single mean field step the states for the negative phase are calculate by Equation~\eqref{eqn:neuron_second_layer3}, showing that the Hebbian-descent
update with squared error loss corresponds to a CD-1 mean field update for a restricted Boltzmann machine. 
As discussed in Section~\ref{sec:HD_contrastive_clamping} this can also be understood as performing 
contrastive Hebbian learning \citep{RumelhartMcClellandEtAl-1986} where the hidden units are clamped to the same value in positive and negative phase.
Notice that, although contrastive divergence is used extensively and successfully in practice it is also not the gradient of any function~\citep{SutskeverTieleman-2010}.
It is, however, an rough approximation to stochastic gradient decent learning in Boltzmann machines, also known as persistent contrastive divergence in this context, that is sufficiently good in practice.
As a consequence, also Hebbian-descent can also be understood as an approximation to stochastic gradient decent training in restricted Boltzmann machines.
 
\subsubsection{Outlook on Deep Auto Encoder Networks}\label{sec:HD_outlook_deep_auto_encoder}

One could also use only the 'decoder related part' of the gradient for a deep auto encoder, but similar to the deep hetero-associative learning the derivative of the activation function is still present in the update for the lower layers of the network. 
Furthermore, one can also show that the 'decoder related part' of the gradient in deep auto encoders is also not the gradient of any objective function, which naturally extend from the proof for single layer auto encoders (see Appendix~\ref{appendix:AA_HD_is_generally_not_the_gradient_of_any_objective_function}).
While the analysis of deep networks is beyond the scope of this publication in our experience the performance of Hebbian-descent is also very similar to gradient descent for deep auto encoders, so that future research in this direction would be promising. 

\section{Methods}\label{sec:HD_methods}

In this section the benchmark datasets and the experimental setup is described.

\subsection{Benchmark Datasets}\label{sec:HD_benchmark_datasets}

We consider four real-world datasets from various domains as well as binary and normal distributed random patterns in our experiments. 

The {\bf{\emph{MNIST}}}~\citep{LeCun1998} dataset consists of 70,000 gray scale images of handwritten digits divided into training and test set of 60,000 and 10,000 patterns, respectively. The images have a size of $28 \times 28$ pixels, where all pixel values are normalized to lie in a range of $[0,1]$. The dataset is not binary, but the values tend to be close to zero or one. Each pattern is assigned to one out of ten classes representing the digits 0 to 9.

The {\bf{\emph{CONNECT}}}~\citep{LarochelleBengioEtAl-2010} dataset consists of 67,587 game-state patterns from the game Connect-4. The dataset is divided into training, validation, and test set with 16,000, 4,000, and 47,557 patterns, respectively. The binary patterns are 126 dimensional and each pattern is assigned to one out of three classes representing the game results win, lose, or draw.

The {\bf{\emph{ADULT}}}~\citep{LarochelleBengioEtAl-2010} dataset consists of 32,561 binary patterns of census data to predict whether a person's income exceeds 50,000 dollar per year.
The dataset is divided into training, validation, and test set with 5,000, 1,414, and 26,147 patterns, respectively. 
The binary patterns are $123$ dimensional and each pattern is assigned to one out of two classes representing whether the income level was exceeded or not.

The {\bf{\emph{CIFAR}}}~\citep{Krizhevsky-2009a} dataset consists of 60,000 color images of various objects divided into training, validation, and test set with 40,000, 10,000, and 10,000 patterns, respectively. The images have a size of $32 \times 32$ pixels that are converted to gray scale and rescaled to lie in a range of $[0,1]$ such that the dataset  has a non-zero mean, and can be represented by most of the activation functions. Each pattern is assigned to one out of ten classes representing trucks, cats, or dogs for example. 

The {\bf{\emph{RAND}}} and {\bf{\emph{RANDN}}} datasets serve as baseline datasets each consisting of $100$ random patterns with a size of 200 pixels. The pixels in the \emph{RAND} dataset take the value one with a probability of 0.5 and zero otherwise. The pixels in the \emph{RANDN} dataset are drawn from an isotropic Gaussian distribution with zero mean and unit variance. 
The dataset is rescaled to lie in a range of $[0,1]$ such that it has a non-zero mean, and can be represented by most of the activation functions. Notice that the data used did not contain extreme outliers such that the variance is still sufficiently large. Both datasets do not have label information.

\subsection{Network Structure and Learning Setup}\label{sec:HD_network_structure_and_learning_setup}

In our experiments we consider various activation functions such as linear, sigmoid, step, softmax, rectifier, and exponential linear units. 
The bias values were initialized to zero and according to \citet{Bengio-2010} we initialized the weights to $w_{ij} \sim U\Big(-\frac{\sqrt{6}}{\sqrt{N+M}},$ $ +\frac{\sqrt{6}}{\sqrt{N+M}} \Big)$, where $N$ is the number of input units, $M$ is the number of output units and ~$U(a, b)$ is the uniform distribution in the interval [a,~b]. 
When centering was used and if not mentioned otherwise the input offsets were fixed to the corresponding data mean, and the hidden offsets were initialized to $\lambda_j(t=0)= 0.5$ and updated with an exponential moving average of $\lambda_j(t+1) = 0.99 \lambda_j(t) + 0.01 h_j(t)$. 
For a fair comparison of the methods we had to fix input offsets to the data mean because changing offsets during training require a bias parameter for the reparameterization that is not available when using Hebb's rule and the covariance rule. 
However, slowly updated input offsets converge to the data mean leading to a very similar performance as when initially fixing them to the data mean. This has been shown for mini-batch learning in restricted Boltzmann machines by \citet{Melchior2016} and is shown for online hetero-association in the following.
Without centering the offsets were all fixed to zero.
Each experiment was repeated 10 times, where the initial weight matrices were the same among the methods but different in each trial. 
Depending on the dataset and activation function the optimal learning rate varied a lot such that we performed a grid search over 35 different learning rates\footnote{Learning rates: 100,  80,  60,  40,  20,  10,  8,  6,  4,  2,  1,  0.8,  0.6,  0.4,  0.2,  0.1,  0.08,  0.06,  0.04,  0.02,  0.01,  0.008,  0.006,  0.004,  0.002,  0.001,  0.0008,  0.0006,  0.0004,  0.0002,  0.0001,  0.00008,  0.00006,  0.00004, 0.00002} ranging from 100 to 0.00002. 
When a weight decay was used we additionally performed a grid search over 20 different weight decay values\footnote{Weight decay values: 2.0, 1.0, 0.8, 0.6, 0.4, 0.2, 0.1, 0.08, 0.06, 0.04, 0.02, 0.01, 0.008, 0.006, 0.004, 0.002, 0.001, 0.0005, 0.0001, 0.0} ranging from 2 to 0
leading to a total search space of $20\times 35 = 700$ hyper parameter combinations. 
The default batch size was one in case of online learning and 100 in case of mini-batch learning, and when training involved several sweeps through the data the models were trained for 100 epochs. 

\subsection{Performance Measurement}\label{sec:HD_performance_measurement}

As the different methods and activation functions effectively optimize different loss functions we decided for a coherent performance measure that represents the performed task best and is not in a favor for one or the other method.
For the classification experiments we therefore evaluated the average miss-classification rate, while for the other experiments we used the Mean Absolute Error (MAE) as an intuitive performance measure. 
Since the output patterns lie in a range between zero and one it is clear that a mean absolute error of $0.1$ for example means that on average each neuron deviates 10\% from the true value and that an error above $0.5$, is worse than for random output (chance level). 
Alternatively we could have used the mean squared error that would have been in favor for gradient descent, while the log-likelihood would have been in favor for Hebbian-descent, thus both not allowing for an unbiased performance measure. Furthermore, both overestimate outliers and underestimate small deviations thus leading to a less interpretable performance measure. 
Another choice that is often used in neuroscience is the Pearson correlation between target and output pattern. 
However, it is scale invariant and can thus lead to a wrong impression of the network's performance. 
For some experiments we also evaluated these other performance measurements and found qualitatively the same results, \emph{i.e.}\ if a method performed significantly better in terms of the MAE it was also the best w.r.t. the other measures.

\section{Results}\label{sec:HD_results}

% BASELINES
%RAND 0.4946
%ADULT 0.1229
%CONNECT 0.1683
%MNIST 0.1473
%CIFAR 0.1670

Here we compare the performance of Hebbian-descent, gradient descent, Hebb's rule, and the covariance rule for centered and uncentered networks. 
We focus on online hetero-association, \emph{i.e.}\ associating $N$ input patterns with $N$ output patterns one after the other, but we also compare the methods with respect to their classification performance, \emph{i.e.}\ associating all input patterns with the corresponding labels, and performed experiments for auto-associative learning, \emph{i.e.}\ associating all input patterns with themselves in an auto-encoding network.

\subsection{Hetero-Associative Learning}\label{sec:HD_hetero_associative_learning}

In this section we investigate the performance of Hebb's rule, the covariance rule, gradient descent, and Hebbian-descent in hetero-associative learning with focus on online learning.

\subsubsection{Hetero-Associative Online Learning without Centering}\label{sec:hetero_associative_online_learning_without_centering}

In a first set of experiments we analyze how well 100 patterns of one dataset can be associated with 100 patterns of another dataset one pair at a time without presenting a pattern twice.
Figure~\ref{fig:rand_rand_sigmoid_false_online} (a) shows the performance of uncentered networks with sigmoid units when the four different methods have been used to associate
100 binary random patterns (\emph{RAND}) with 100 binary random patterns (\emph{RAND}). 
\begin{figure}[t]
\begin{center}
\subfigure[]{
\includegraphics[scale=0.425, trim=21 0 37 0, clip]{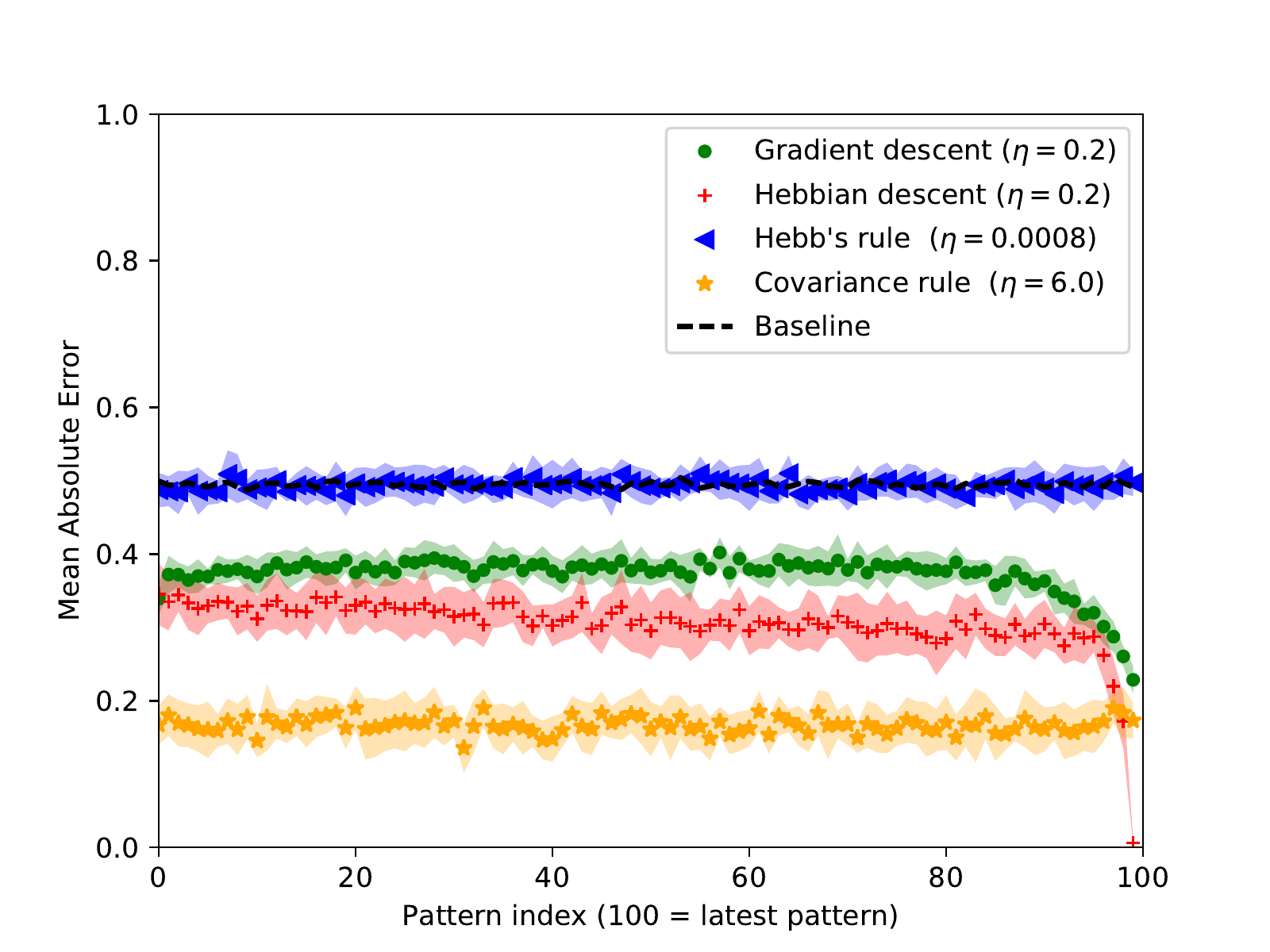}}
\subfigure[]{
\includegraphics[scale=0.425, trim=35 0 37 0, clip]{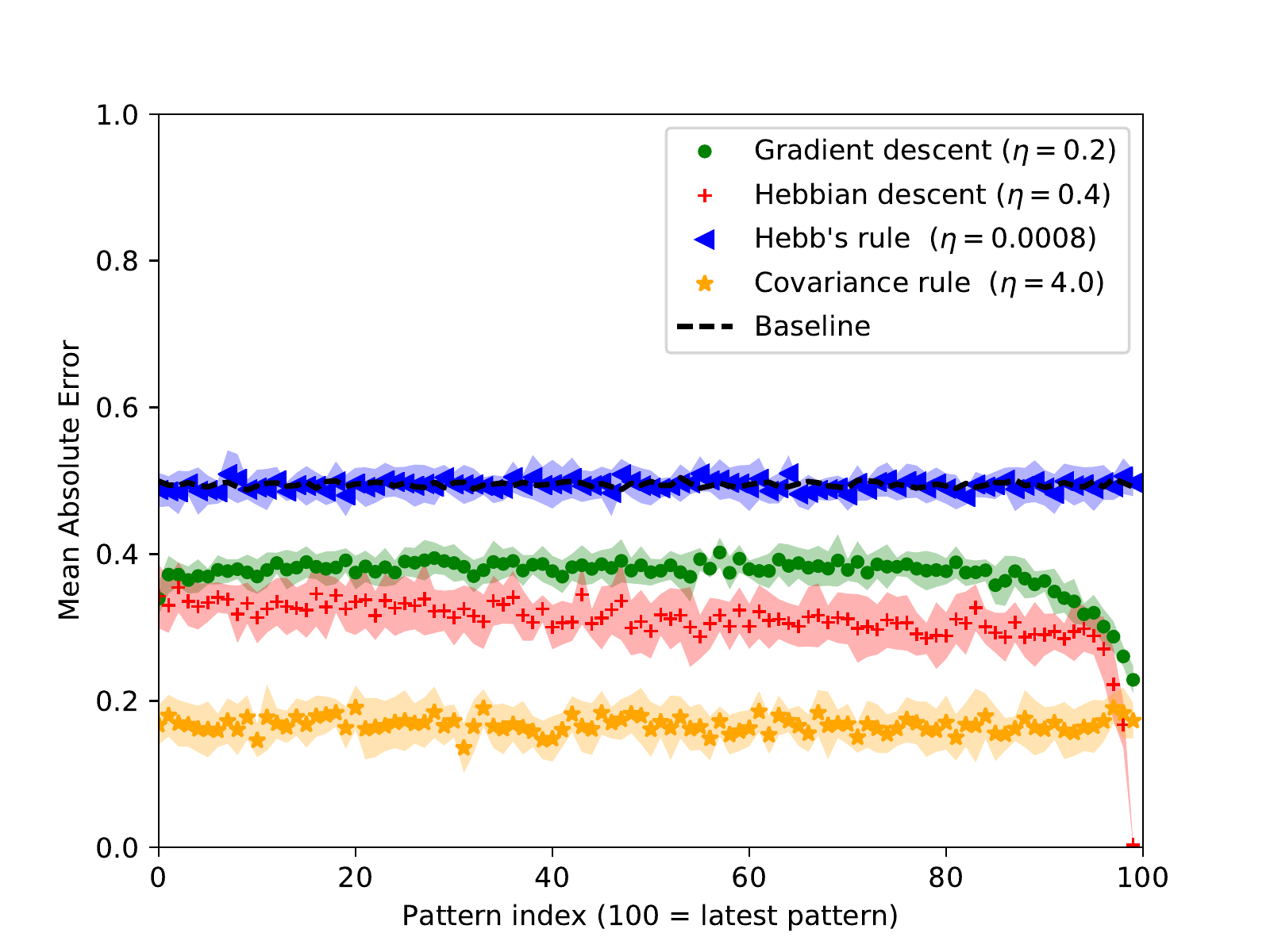}}
\caption{Online learning performance of the four different update rules without centering and sigmoid units, when 100 binary random patterns (\emph{RAND}) have been associated with another 100 binary random patterns (\emph{RAND}) one pattern pair at a time. 
The Mean absolute error and the corresponding standard deviation over 10 trials is plotted for each pattern separately. 
The learning rate $\eta$ was chosen for each method individually such that (a) the performance over all 100 patterns is best and (b) the performance for the last pattern is best. The baseline (overlaid with the curve for Hebb's rule) represents the performance of a network that independently of the input always returns the mean of the output patterns. }
\label{fig:rand_rand_sigmoid_false_online}
\end{center}
\end{figure}
The optimal learning rate was determined for each method individually such that the average performance over all patterns is best. 
For Hebb's rule the error is close to 0.5, which means that the network has not learned to associate the patterns at all. 
It is the same performance as baseline, which is the error between the output patterns and their mean value. 
This corresponds to the performance of a network that independently of the input always returns the mean output pattern and thus represents the most trivial solution. 
The bad performance of Hebb's rule is a direct consequence of not being able to learn negative or zero correlations and can be explained as follows. 
Equation~\eqref{eqn:update_hebb_sup_w} indicates that in the uncentered case and for binary patterns each weight $w_{ij}$ is updated by a value of $\eta$ when the corresponding input and target value is one, or by zero otherwise. 
Since the patterns are drawn uniformly at random, the weights will thus increase  with a probability of 0.25 or stay the same otherwise. Consequently, all neurons will sooner or later and independently of the input produce a constant output of one resulting in an error of 0.5. 
The covariance rule solves this problem by allowing negative weight changes leading to a much better performance, which is even better than that of Hebbian-descent and gradient descent. 
Interestingly, the error for all patterns is almost the same in case of Hebb's rule and the covariance rule, whereas for gradient descent and Hebbian-descent more recent patterns are represented better than older ones. 
But only Hebbian-descent allows to store the latest pattern almost perfectly. 
To show that this does not depend on the learning rate we performed the same experiments but selected the optimal learning rate, so that the performance of only the last pattern (index 100) is best.  
The results are shown in Figure~\ref{fig:rand_rand_sigmoid_false_online} (b), which are  almost equivalent to the results shown in Figure~\ref{fig:rand_rand_sigmoid_false_online} (a).
\begin{figure}[t]
\begin{center}
\subfigure[]{
\includegraphics[scale=0.425, trim=21 0 37 0, clip]{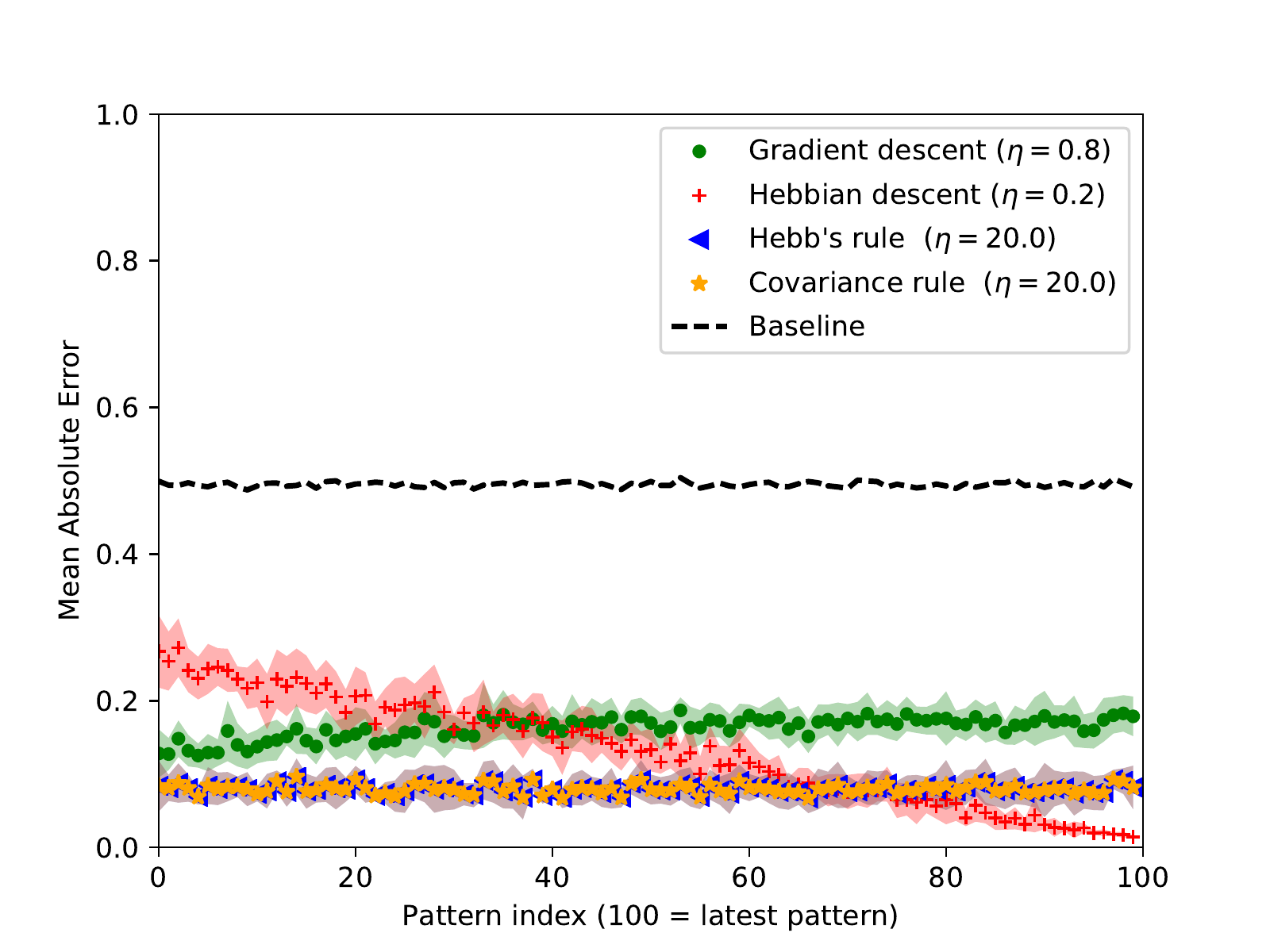}}
\subfigure[]{
\includegraphics[scale=0.425, trim=35 0 37 0, clip]{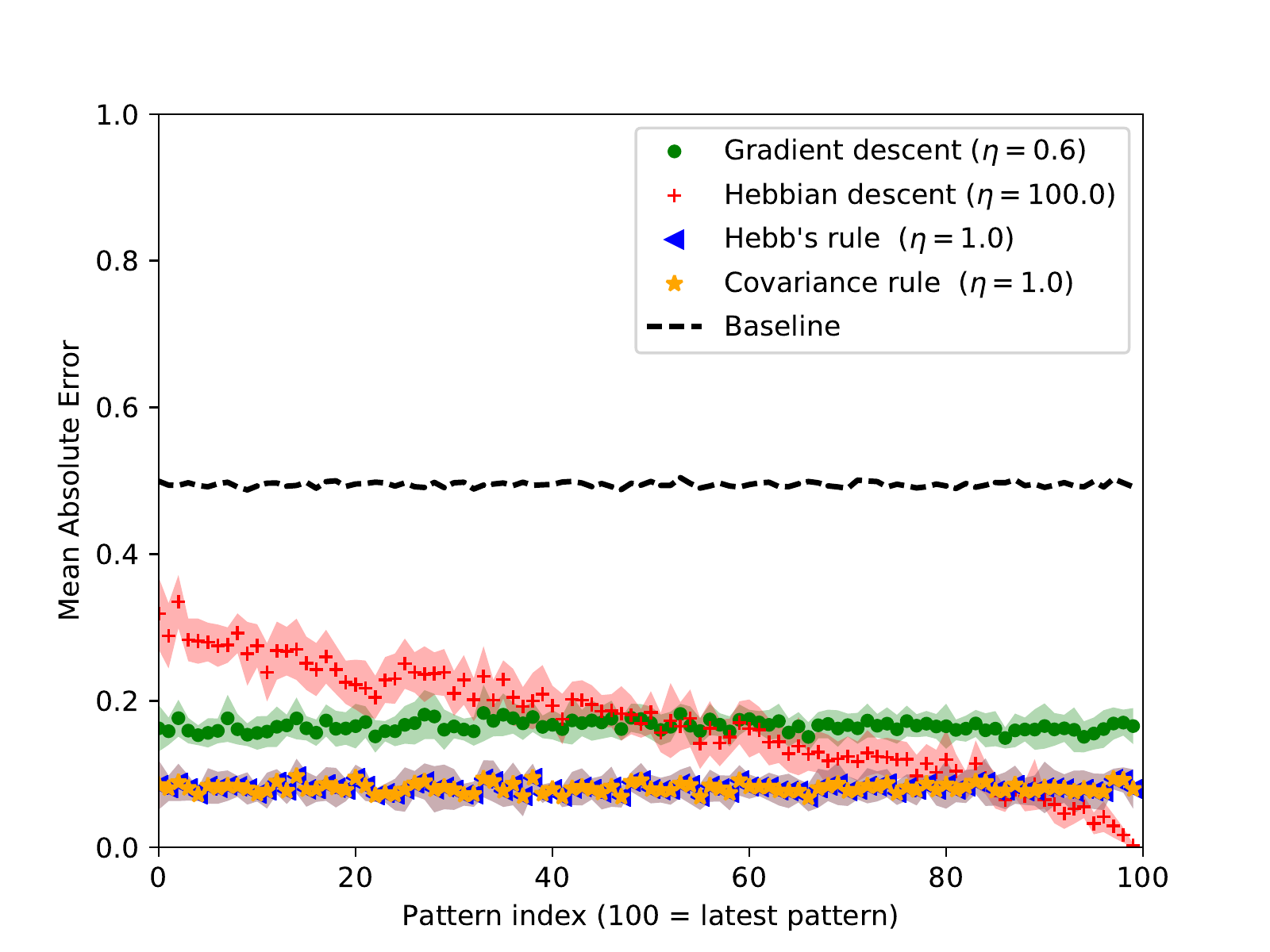}}
\caption{The same online learning experiments as in Figure~\ref{fig:rand_rand_sigmoid_false_online} except that now centering is used.
}
\label{fig:rand_rand_sigmoid_true_online}
\end{center}
\end{figure} 
Since the optimal learning rates in both experiments are almost equivalent, a larger learning rate does not allow the methods to represent only the latest pattern better.
However, the better storage of the last pattern when using gradient descent or Hebbian-descent comes at the cost of catastrophic interference, \emph{i.e.}\ abrupt decrease in performance of previously stored patterns.

\subsubsection{Hetero-Associative Online Learning with Centering}\label{sec:hetero_ssociative_online_learning_with_centering}

To show the importance of centering and its ability to prevent catastrophic interference we performed the same experiments as before but with centered networks. 
The results are shown in Figure~\ref{fig:rand_rand_sigmoid_true_online} illustrating that all methods profit significantly from centering and that the covariance rule and Hebb's rule become equivalent in case of centering as shown analytically in Equation~\eqref{eqn:single_cov}.
All methods except for Hebbian-descent have a homogeneous error distribution over the single patterns, which does also not change when the learning rate is selected such that the last pattern is represented best as shown in Figure~\ref{fig:rand_rand_sigmoid_true_online} (b). In comparison, Hebbian-descent allows for a linear slope of forgetting and does thus not suffer from catastrophic interference anymore. 
It is remarkable how a 500 times larger learning rate has little effect on the performance of Hebbian-descent in this case. 

Since natural data is usually not uncorrelated we performed the same experiments as before but with correlated real-world datasets. 
In particular, we associated 100 binary random patterns (\emph{RAND}) with 100 patterns of the \emph{ADULT} dataset. 
The results for centered networks are shown in Figure~\ref{fig:adult_rand_sigmoid_true_online} (a), where Hebb's rule and the covariance rule perform significantly worse than gradient descent, Hebbian-descent, and also as baseline.
\begin{figure}[t]
\begin{center}
\subfigure[]{
\includegraphics[scale=0.425, trim=21 0 37 0, clip]{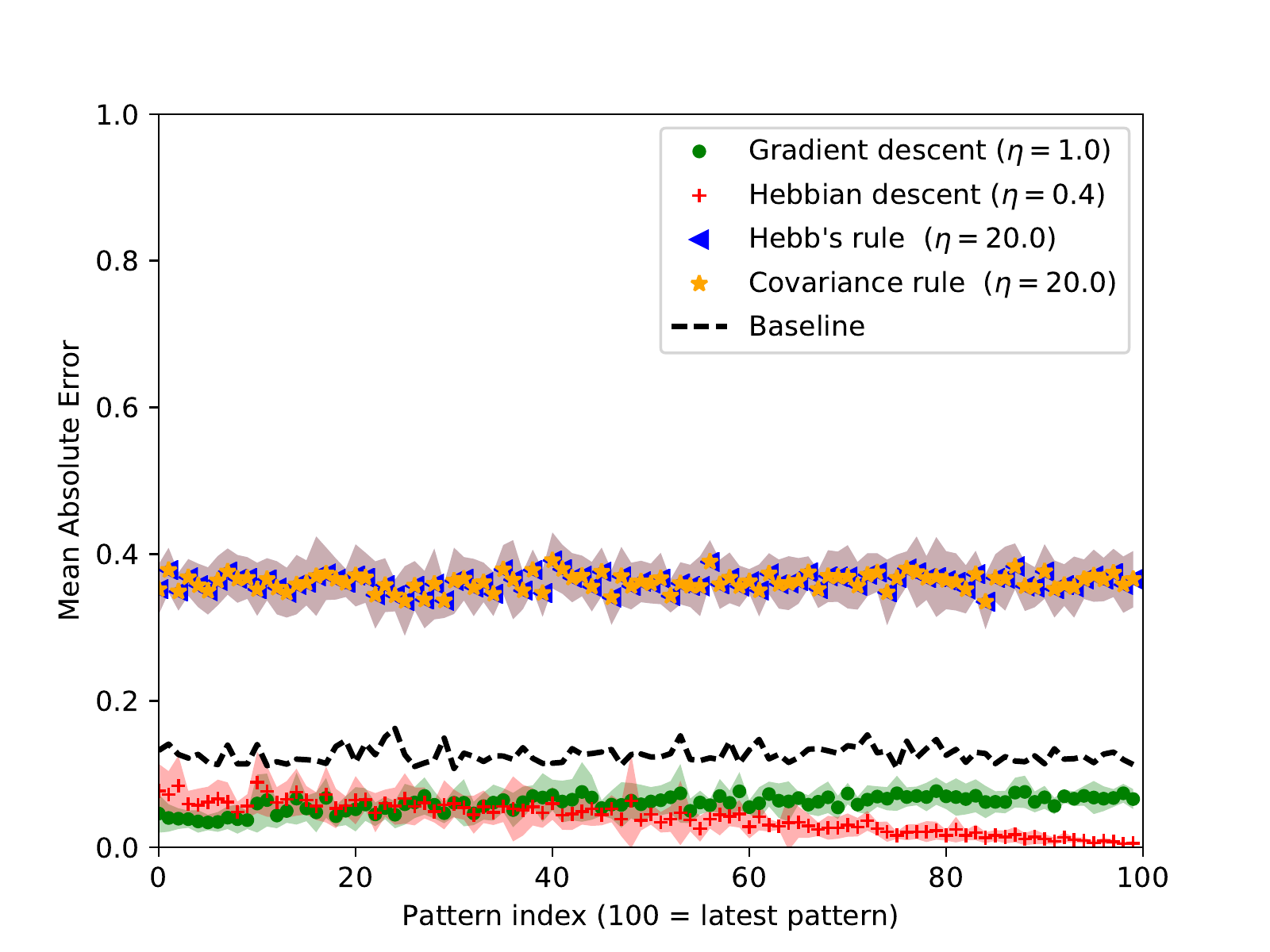}}
\subfigure[]{
\includegraphics[scale=0.425, trim=35 0 37 0, clip]{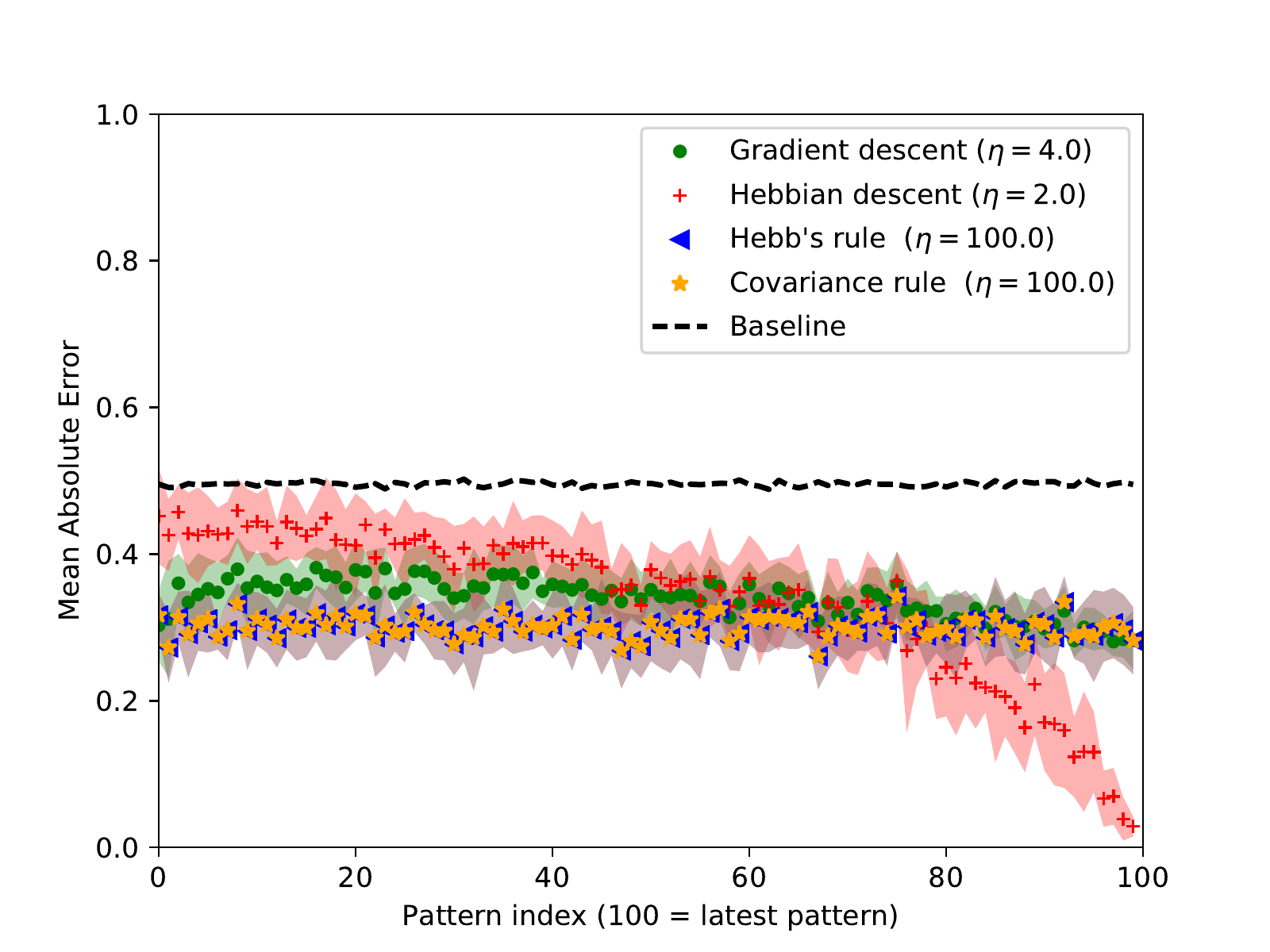}}
\caption{Online learning performance of the four different update rules with centering and sigmoid units, when (a) 100 binary random patterns (\emph{RAND}) have been associated with 100 patterns of the \emph{ADULT} dataset and (b) 100 patterns of the \emph{ADULT} dataset have been associated with 100 binary random patterns (\emph{RAND}). 
The Mean absolute error and the corresponding standard deviation over 10 trials is plotted for each pattern separately. 
The learning rate $\eta$ was chosen for each method individually, so that the performance over all 100 patterns was best.
The baseline represents the performance of a network that independently of the input always returns the mean of the output patterns.}
\label{fig:adult_rand_sigmoid_true_online}
\end{center}
\end{figure} 
Again Hebbian-descent has a linear slope of forgetting while all the other methods each have a homogeneous error distribution. 
The plot for associating a set of patterns of the \emph{ADULT} dataset with another set of patterns of the \emph{ADULT} dataset is qualitatively similar and is thus not shown.
Figure~\ref{fig:adult_rand_sigmoid_true_online} (b), shows the reverse experiment where 100 patterns of the \emph{ADULT} dataset in the input have been associated with 100 binary random patterns (\emph{RAND}) in the output.
The error for all methods is rather large, which arises from a rather high pairwise correlation of the patterns in the \emph{ADULT} dataset. 
Associating two very similar patterns with two completely random output patterns is a rather difficult task in online learning as the chance of `overwriting' of associations is extremely high.
Again all methods have roughly the same error on the individual patterns except for Hebbian-descent, which allows to store at least the most recent patterns with high accuracy. 
This comes at the cost of older pattern's accuracy leading to a power-law like gradual forgetting. 
The optimal learning rate in both experiments was chosen, so that the performance is best over all patterns, but consistent with the previous experiments the results are almost the same when choosing it to be best for only the last pattern. Furthermore, without centering all methods perform significantly worse (data not shown).

To confirm empirically that Hebbian-descent is generally better on the most recent patterns we performed several experiments with various datasets and activation functions. The results for centered networks are shown in Table~\ref{tab:Hetero_online_centered_1} as well as in Table~\ref{tab:Hetero_online_centered_2} and \ref{tab:Hetero_online_centered_3} in Appendix~\ref{appendix:hebbian_descent_additional_results}. 
The numbers represent the average absolute error over the last 20 patterns followed by the performance over all patterns given in parenthesis, both also averaged over 10 trials.
The optimal learning rate was chosen, so that the performance on the last 20 patterns was best and among the four update rules the best result is highlighted in bold face. 
Across all combinations of datasets and activation functions Hebbian-descent and gradient descent perform significantly better than Hebb's rule and the covariance rule and the latter two are both not even better than baseline performance.
Furthermore, Hebbian-descent performs significantly better than gradient descent in most cases, which gets more significant as the activation function gets more nonlinear. 
As an extreme case we used the step-function, which guarantees a binary output of the network but is incompatible with gradient descent as the gradient is constant zero for such networks. 
Hebbian-descent however can deal with this type of activation functions and reaches good performances.
The linear networks, for which gradient descent and Hebbian-descent are equivalent, perform always worse than or similar to corresponding nonlinear networks.  
A nonlinearity is thus always beneficial and the sigmoid function has the best performance among the different activation functions even when the output data comes from a continuous domain such as \emph{ADULT$\,\,\rightarrow\,\,$RANDN} or \emph{CIFAR$\,\,\rightarrow\,\,$RANDN} for example (See Table~\ref{tab:Hetero_online_centered_3}).  
To emphasize the superiority of Hebbian-descent over gradient descent more clearly we plotted the results of Table~\ref{tab:Hetero_online_centered_1}, Table~\ref{tab:Hetero_online_centered_2}, and Table~\ref{tab:Hetero_online_centered_3} in a scatter plot shown in Figure~\ref{fig:heteroassociative_online_centered}. 
All points lie either roughly on the diagonal or clearly below it showing that if there is a significant difference between the two methods Hebbian-descent performs significantly better than gradient descent.
\begin{figure}[t]
\begin{center}
\subfigure[]{
\includegraphics[scale=0.415, trim=21 0 37 0, clip]{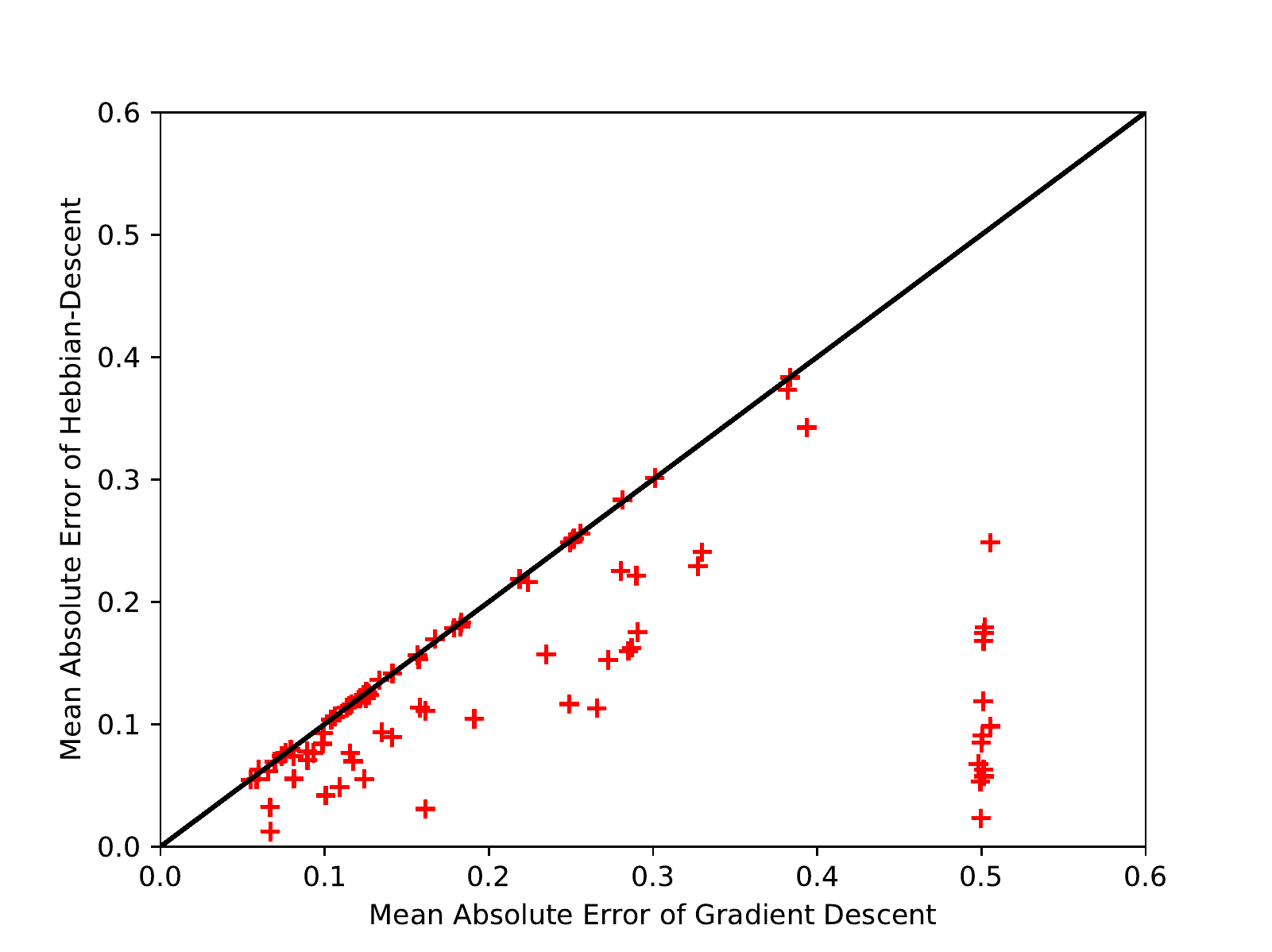}\label{fig:heteroassociative_online_centered}}
\subfigure[]{
\includegraphics[scale=0.415, trim=21 0 37 0, clip]{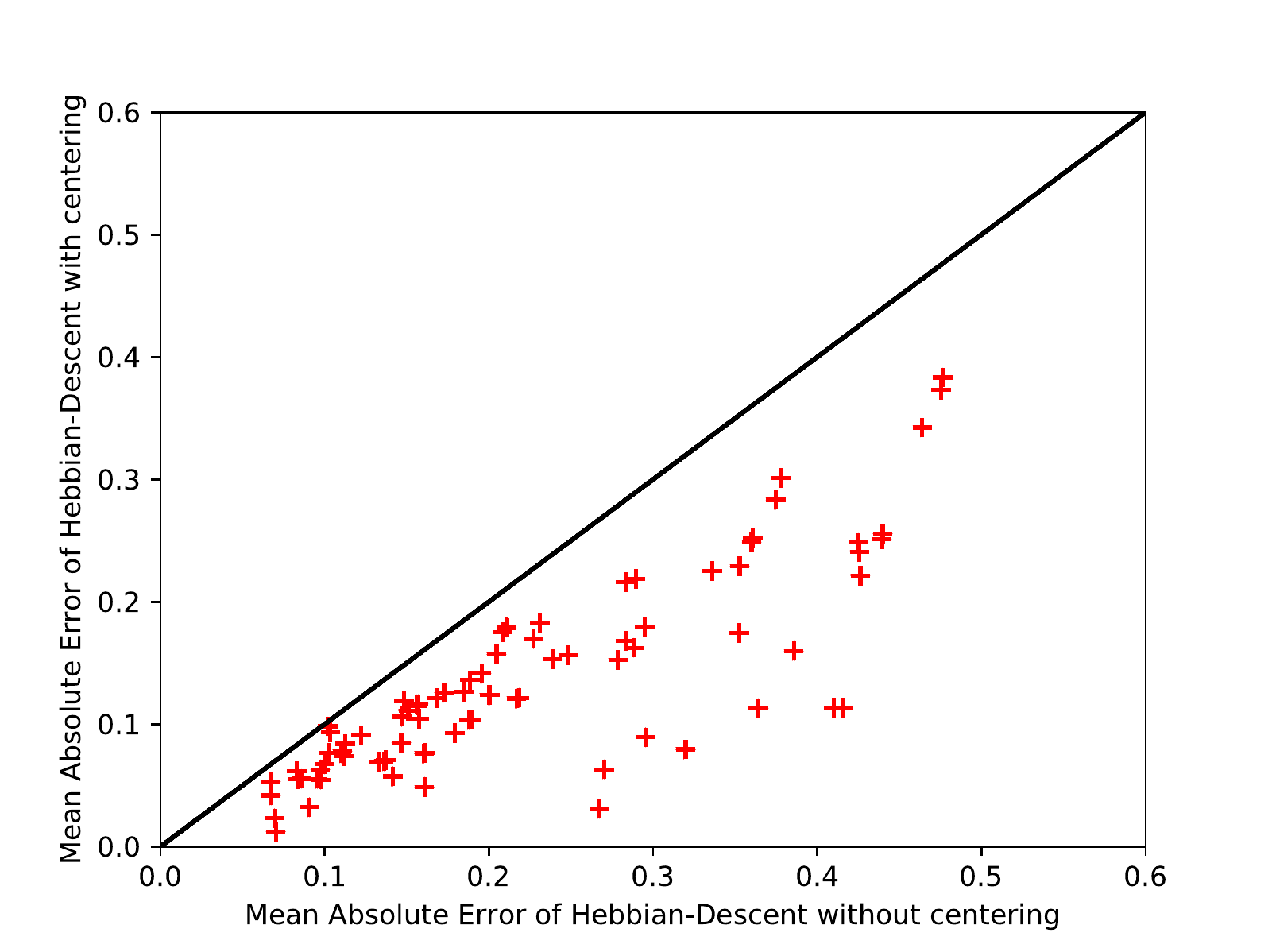}\label{fig:heteroassociative_online_hebbian_centered_vs_uncentered}}
%\subfigure[]{
%\includegraphics[scale=0.4, trim=21 0 37 0, clip]{figures/heteroassociative_online_uncentered}\label{tab:heteroassociative_online_uncentered}}
%\subfigure[]{
%\includegraphics[scale=0.4, trim=21 0 37 0, clip]{figures/heteroassociative_online_gradient_centered_vs_uncentered}\label{tab:heteroassociative_online_gradient_centered_vs_uncentered}}
\caption{Comparison of online hetero-association of Hebbian-descent and gradient descent. Each cross represents the mean absolute error of the last 20 patterns averaged over 10 trials for one experiment. 
(a) Hebbian-descent versus gradient descent with centering. 
%(b) Hebbian-descent versus gradient descent without centering.
(b) Hebbian-descent with centering versus Hebbian-descent without centering.
%(d) gradient descent with centering versus gradient descent without centering. 
For detailed results compare Table~\ref{tab:Hetero_online_centered_1}, \ref{tab:Hetero_online_uncentered_1}, \ref{tab:Hetero_online_centered_2}, \ref{tab:Hetero_online_centered_3}, \ref{tab:Hetero_online_uncentered_2}, and \ref{tab:Hetero_online_uncentered_3}.}
\label{fig:heteroassociative_online_hebbian_scatter}
\end{center}
\end{figure} 

For comparison we performed the same experiments as before but without centering.
The results are shown in Table~\ref{tab:Hetero_online_uncentered_1} as well as in Table~\ref{tab:Hetero_online_uncentered_2} and \ref{tab:Hetero_online_uncentered_3} in Appendix~\ref{appendix:hebbian_descent_additional_results}, which clearly support our statement that centering is valuable for all methods.
Without centering Hebbian-descent looses its ability to store recent patterns significantly better than older ones, but in most cases it still performs significantly better than the other methods.  
We also emphasize the superiority of centered over uncentered networks by plotting the results for centered and uncentered Hebbian-descent given in Table~\ref{tab:Hetero_online_centered_1}, \ref{tab:Hetero_online_uncentered_1}, \ref{tab:Hetero_online_centered_2}, \ref{tab:Hetero_online_centered_3}, \ref{tab:Hetero_online_uncentered_2}, and \ref{tab:Hetero_online_uncentered_3} in a scatter plot shown in Figure~\ref{fig:heteroassociative_online_hebbian_centered_vs_uncentered}.
All points lie clearly below the diagonal showing that centering is always beneficial.

\begin{table}[htbp]
\setlength{\tabcolsep}{2pt}
\begin{center}
\begin{small}
\begin{sc}
\begin{tabular}{l@{\hskip 0.2in} r@{\hskip 0.02in} r@{\hskip 0.14in} r@{\hskip 0.02in} r@{\hskip 0.14in} r@{\hskip 0.02in} r@{\hskip 0.14in} r@{\hskip 0.02in} r }
\hline
\abovespace\belowspace
  $\vect \phi$ & \multicolumn{2}{c}{\hskip -0.16in Grad. Descent} &   \multicolumn{2}{c}{\hskip -0.16in Hebb. Descent} &  \multicolumn{2}{c}{\hskip -0.16in Hebb rule} &  \multicolumn{2}{c}{\hskip -0.16in Cov. Rule} \\ 
\hline & & & & & & & &\vspace{-0.3cm}\\ 
\multicolumn{3}{l}{{\emph{RAND$\,\,\rightarrow\,\,$ADULT (0.1229)}}}\\ 
 & & & & & & & &  \vspace{-0.35cm}\\ 
Linear  & \textbf{0.0793}  & (0.1829)  & \textbf{0.0793}  & (0.1829)  & 0.4856  & (0.4871)  & 0.4856  & (0.4871)  \\ 
ExpLin  & 0.1411  & (0.1979)  & \textbf{0.0892}  & (0.1692)  & 0.4179  & (0.4196)  & 0.4179  & (0.4196)  \\ 
Rectifier  & 0.0667  & (0.0760)  & \textbf{0.0321}  & (0.0553)  & 0.2401  & (0.2422)  & 0.2401  & (0.2422)  \\ 
Sigmoid  & 0.0668  & (0.0651)  & \textbf{0.0120}  & (0.0459)  & 0.3627  & (0.3622)  & 0.3627  & (0.3622)  \\ 
Step  & 0.4997  & (0.5008)  & \textbf{0.0230}  & (0.0680)  & 0.3624  & (0.3623)  & 0.3624  & (0.3623)  \\ 
\hline & & & & & & & &\vspace{-0.3cm}\\ 
\multicolumn{3}{l}{{\emph{RANDN$\,\,\rightarrow\,\,$CONNECT (0.1683)}}}\\ 
 & & & & & & & &  \vspace{-0.35cm}\\ 
Linear  & \textbf{0.1215}  & (0.1479)  & \textbf{0.1215}  & (0.1479)  & 0.3906  & (0.3910)  & 0.3906  & (0.3910)  \\ 
ExpLin  & 0.1218  & (0.1466)  & \textbf{0.1209}  & (0.1466)  & 0.3791  & (0.3792)  & 0.3791  & (0.3792)  \\ 
Rectifier  & 0.0992  & (0.1138)  & \textbf{0.0928}  & (0.1135)  & 0.2724  & (0.2710)  & 0.2724  & (0.2710)  \\ 
Sigmoid  & 0.1091  & (0.1099)  & \textbf{0.0487}  & (0.0728)  & 0.3199  & (0.3171)  & 0.3199  & (0.3171)  \\ 
Step  & 0.5014  & (0.5015)  & \textbf{0.0573}  & (0.0839)  & 0.3176  & (0.3173)  & 0.3176  & (0.3173)  \\ 
\hline & & & & & & & &\vspace{-0.3cm}\\ 
\multicolumn{3}{l}{{\emph{ADULT$\,\,\rightarrow\,\,$MNIST (0.1473)}}}\\ 
 & & & & & & & &  \vspace{-0.35cm}\\ 
Linear  & \textbf{0.1063}  & (0.1531)  & \textbf{0.1063}  & (0.1531)  & 0.2215  & (0.2190)  & 0.2215  & (0.2190)  \\ 
ExpLin  & 0.1061  & (0.1594)  & \textbf{0.1060}  & (0.1521)  & 0.2172  & (0.2147)  & 0.2172  & (0.2147)  \\ 
Rectifier  & 0.0891  & (0.1134)  & \textbf{0.0778}  & (0.1217)  & 0.1531  & (0.1519)  & 0.1531  & (0.1519)  \\ 
Sigmoid  & 0.1155  & (0.1265)  & \textbf{0.0762}  & (0.1247)  & 0.4252  & (0.4262)  & 0.4252  & (0.4262)  \\ 
Step  & 0.5006  & (0.4998)  & \textbf{0.0907}  & (0.1325)  & 0.4254  & (0.4259)  & 0.4254  & (0.4259)  \\ 
\hline & & & & & & & &\vspace{-0.3cm}\\ 
\multicolumn{3}{l}{{\emph{CONNECT$\,\,\rightarrow\,\,$CIFAR (0.167)}}}\\ 
 & & & & & & & &  \vspace{-0.35cm}\\ 
Linear  & \textbf{0.1238}  & (0.1803)  & \textbf{0.1238}  & (0.1803)  & 0.4981  & (0.5019)  & 0.4981  & (0.5019)  \\ 
ExpLin  & \textbf{0.1238}  & (0.1803)  & \textbf{0.1238}  & (0.1802)  & 0.4887  & (0.4913)  & 0.4887  & (0.4913)  \\ 
Rectifier  & 0.1270  & (0.1886)  & \textbf{0.1237}  & (0.1802)  & 0.3590  & (0.3622)  & 0.3590  & (0.3622)  \\ 
Sigmoid  & \textbf{0.1253}  & (0.1682)  & 0.1264  & (0.1974)  & 0.1770  & (0.1749)  & 0.1770  & (0.1749)  \\ 
\hline & & & & & & & &\vspace{-0.3cm}\\ 
\multicolumn{3}{l}{{\emph{MNIST$\,\,\rightarrow\,\,$CONNECT (0.1683)}}}\\ 
 & & & & & & & &  \vspace{-0.35cm}\\ 
Linear  & \textbf{0.1830}  & (0.2659)  & \textbf{0.1830}  & (0.2659)  & 0.5071  & (0.5070)  & 0.5071  & (0.5070)  \\ 
ExpLin  & \textbf{0.1671}  & (0.2435)  & 0.1695  & (0.2507)  & 0.4788  & (0.4785)  & 0.4788  & (0.4785)  \\ 
Rectifier  & 0.1911  & (0.2339)  & \textbf{0.1042}  & (0.1668)  & 0.3484  & (0.3479)  & 0.3484  & (0.3479)  \\ 
Sigmoid  & 0.1242  & (0.1434)  & \textbf{0.0552}  & (0.1218)  & 0.3925  & (0.3937)  & 0.3925  & (0.3937)  \\ 
Step  & 0.4982  & (0.4992)  & \textbf{0.0672}  & (0.1329)  & 0.3926  & (0.3932)  & 0.3926  & (0.3932)  \\ 
\hline & & & & & & & &\vspace{-0.3cm}\\ 
\multicolumn{3}{l}{{\emph{CIFAR$\,\,\rightarrow\,\,$CONNECT (0.1683)}}}\\ 
 & & & & & & & &  \vspace{-0.35cm}\\ 
Linear  & \textbf{0.2187}  & (0.2576)  & \textbf{0.2187}  & (0.2576)  & 0.4978  & (0.4939)  & 0.4978  & (0.4939)  \\ 
ExpLin  & 0.2239  & (0.2620)  & \textbf{0.2161}  & (0.2536)  & 0.4757  & (0.4724)  & 0.4757  & (0.4724)  \\ 
Rectifier  & 0.2348  & (0.2636)  & \textbf{0.1570}  & (0.1964)  & 0.3663  & (0.3628)  & 0.3663  & (0.3628)  \\ 
Sigmoid  & 0.1614  & (0.1852)  & \textbf{0.1108}  & (0.1633)  & 0.4431  & (0.4377)  & 0.4431  & (0.4377)  \\ 
Step  & 0.5009  & (0.4998)  & \textbf{0.1186}  & (0.1707)  & 0.4436  & (0.4379)  & 0.4436  & (0.4379)  \\ 
\hline & & & & & & & &  \vspace{-0.85cm}
\end{tabular}
\end{sc}
\end{small}
\end{center}
\caption{Online learning performance of hetero-association with centering for various activation functions and datasets. 
The task is to associate 100 patterns of one dataset with 100 patterns of another dataset one after the other. 
Each experiment was repeated 10 times and the optimal learning rate was determined via grid search for each method separately, so that the performance on the last 20 patterns was best.
The first number in each entry shows the MAE for the last 20 patterns followed by the MAE for all patterns in parenthesis averaged over 10 trials. The best result is indicated in bold.
The name of the output dataset is followed by the corresponding baseline performance in parenthesis. See Table~\ref{tab:Hetero_online_centered_2} and \ref{tab:Hetero_online_centered_3} in Appendix~\ref{appendix:hebbian_descent_additional_results} for more results. 
} 
\label{tab:Hetero_online_centered_1}
\end{table}
\begin{table}[htbp]
\setlength{\tabcolsep}{2pt}
\begin{center}
\begin{small}
\begin{sc}
\begin{tabular}{l@{\hskip 0.2in} r@{\hskip 0.02in} r@{\hskip 0.14in} r@{\hskip 0.02in} r@{\hskip 0.14in} r@{\hskip 0.02in} r@{\hskip 0.14in} r@{\hskip 0.02in} r }
\hline
\abovespace\belowspace
  $\vect \phi$ & \multicolumn{2}{c}{\hskip -0.16in Grad. Descent} &   \multicolumn{2}{c}{\hskip -0.16in Hebb. Descent} &  \multicolumn{2}{c}{\hskip -0.16in Hebb rule} &  \multicolumn{2}{c}{\hskip -0.16in Cov. Rule} \\ 
\hline & & & & & & & &\vspace{-0.3cm}\\ 
\multicolumn{3}{l}{{\emph{RAND$\,\,\rightarrow\,\,$ADULT (0.1229)}}}\\ 
 & & & & & & & &  \vspace{-0.35cm}\\ 
Linear  & \textbf{0.3198}  & (0.3493)  & \textbf{0.3198}  & (0.3493)  & 0.6626  & (0.6630)  & 0.6701  & (0.6692)  \\ 
ExpLin  & 0.3216  & (0.3489)  & \textbf{0.2955}  & (0.3244)  & 0.5718  & (0.5723)  & 0.5593  & (0.5594)  \\ 
Rectifier  & 0.0920  & (0.0955)  & \textbf{0.0907}  & (0.1044)  & 0.3765  & (0.3785)  & 0.3473  & (0.3488)  \\ 
Sigmoid  & 0.0909  & (0.0947)  & \textbf{0.0702}  & (0.0836)  & 0.5044  & (0.5053)  & 0.4061  & (0.4048)  \\ 
Step  & 0.5112  & (0.5128)  & \textbf{0.0696}  & (0.0846)  & 0.5120  & (0.5136)  & 0.4076  & (0.4082)  \\ 
\hline & & & & & & & &\vspace{-0.3cm}\\ 
\multicolumn{3}{l}{{\emph{RANDN$\,\,\rightarrow\,\,$CONNECT (0.1683)}}}\\ 
 & & & & & & & &  \vspace{-0.35cm}\\ 
Linear  & \textbf{0.2184}  & (0.2250)  & \textbf{0.2184}  & (0.2250)  & 0.5023  & (0.5012)  & 0.6292  & (0.6282)  \\ 
ExpLin  & \textbf{0.2162}  & (0.2226)  & 0.2168  & (0.2233)  & 0.4709  & (0.4698)  & 0.5728  & (0.5716)  \\ 
Rectifier  & 0.2030  & (0.2064)  & \textbf{0.1792}  & (0.1856)  & 0.3728  & (0.3721)  & 0.4070  & (0.4054)  \\ 
Sigmoid  & \textbf{0.1359}  & (0.1417)  & 0.1607  & (0.1658)  & 0.4325  & (0.4324)  & 0.4456  & (0.4422)  \\ 
Step  & 0.5044  & (0.5057)  & \textbf{0.1415}  & (0.1487)  & 0.4361  & (0.4360)  & 0.4393  & (0.4379)  \\ 
\hline & & & & & & & &\vspace{-0.3cm}\\ 
\multicolumn{3}{l}{{\emph{ADULT$\,\,\rightarrow\,\,$MNIST (0.1473)}}}\\ 
 & & & & & & & &  \vspace{-0.35cm}\\ 
Linear  & \textbf{0.1474}  & (0.1705)  & \textbf{0.1474}  & (0.1705)  & 0.2336  & (0.2321)  & 0.2525  & (0.2509)  \\ 
ExpLin  & 0.1473  & (0.1698)  & \textbf{0.1469}  & (0.1698)  & 0.2293  & (0.2278)  & 0.2450  & (0.2434)  \\ 
Rectifier  & 0.1117  & (0.1237)  & \textbf{0.1107}  & (0.1354)  & 0.1857  & (0.1836)  & 0.1677  & (0.1665)  \\ 
Sigmoid  & 0.1278  & (0.1380)  & \textbf{0.1101}  & (0.1346)  & 0.4725  & (0.4729)  & 0.4245  & (0.4253)  \\ 
Step  & 0.5047  & (0.5046)  & \textbf{0.1221}  & (0.1462)  & 0.5056  & (0.5056)  & 0.4272  & (0.4274)  \\ 
\hline & & & & & & & &\vspace{-0.3cm}\\ 
\multicolumn{3}{l}{{\emph{CONNECT$\,\,\rightarrow\,\,$CIFAR (0.167)}}}\\ 
 & & & & & & & &  \vspace{-0.35cm}\\ 
Linear  & \textbf{0.2003}  & (0.2126)  & \textbf{0.2003}  & (0.2126)  & 0.3065  & (0.3038)  & 0.5149  & (0.5183)  \\ 
ExpLin  & 0.2004  & (0.2127)  & \textbf{0.2003}  & (0.2126)  & 0.3064  & (0.3038)  & 0.4980  & (0.5010)  \\ 
Rectifier  & 0.2279  & (0.2403)  & \textbf{0.2004}  & (0.2127)  & 0.3039  & (0.3048)  & 0.3655  & (0.3689)  \\ 
Sigmoid  & 0.1848  & (0.1942)  & 0.1851  & (0.1934)  & 0.2054  & (0.2032)  & \textbf{0.1831}  & (0.1810)  \\ 
\hline & & & & & & & &\vspace{-0.3cm}\\ 
\multicolumn{3}{l}{{\emph{MNIST$\,\,\rightarrow\,\,$CONNECT (0.1683)}}}\\ 
 & & & & & & & &  \vspace{-0.35cm}\\ 
Linear  & \textbf{0.2311}  & (0.2627)  & \textbf{0.2311}  & (0.2627)  & 0.4390  & (0.4400)  & 0.5650  & (0.5660)  \\ 
ExpLin  & 0.2301  & (0.2651)  & \textbf{0.2272}  & (0.2649)  & 0.4162  & (0.4168)  & 0.5241  & (0.5243)  \\ 
Rectifier  & 0.1580  & (0.1749)  & \textbf{0.1574}  & (0.1843)  & 0.3345  & (0.3349)  & 0.3794  & (0.3789)  \\ 
Sigmoid  & 0.1163  & (0.1217)  & \textbf{0.0954}  & (0.1305)  & 0.4333  & (0.4333)  & 0.4205  & (0.4215)  \\ 
Step  & 0.4986  & (0.4979)  & \textbf{0.0999}  & (0.1305)  & 0.4305  & (0.4292)  & 0.4192  & (0.4194)  \\ 
\hline & & & & & & & &\vspace{-0.3cm}\\ 
\multicolumn{3}{l}{{\emph{CIFAR$\,\,\rightarrow\,\,$CONNECT (0.1683)}}}\\ 
 & & & & & & & &  \vspace{-0.35cm}\\ 
Linear  & \textbf{0.2894}  & (0.3001)  & \textbf{0.2894}  & (0.3001)  & 0.6556  & (0.6441)  & 0.7530  & (0.7444)  \\ 
ExpLin  & 0.2835  & (0.2941)  & \textbf{0.2832}  & (0.2935)  & 0.5953  & (0.5850)  & 0.6538  & (0.6472)  \\ 
Rectifier  & \textbf{0.2033}  & (0.2092)  & 0.2046  & (0.2127)  & 0.4514  & (0.4450)  & 0.4582  & (0.4536)  \\ 
Sigmoid  & \textbf{0.1376}  & (0.1445)  & 0.1509  & (0.1620)  & 0.4382  & (0.4381)  & 0.4896  & (0.4897)  \\ 
Step  & 0.5039  & (0.5028)  & \textbf{0.1483}  & (0.1610)  & 0.4355  & (0.4343)  & 0.4854  & (0.4869)  \\ 
\hline & & & & & & & &  \vspace{-0.85cm}\\
\end{tabular}
\end{sc}
\end{small}
\end{center}
\caption{The same experiments as in Table~\ref{tab:Hetero_online_centered_1} but without centering. See Table~\ref{tab:Hetero_online_uncentered_2} and \ref{tab:Hetero_online_uncentered_3} in Appendix~\ref{appendix:hebbian_descent_additional_results} for more results.
\vspace{2.95cm}
} 
\label{tab:Hetero_online_uncentered_1}
\end{table}

\subsubsection{On the Advantage of Saturating Activation Functions in Online Learning}\label{sec:on_the_advantage_of_saturating_activation_functions_in_online_learning}

The advantage of activation functions with restricted output values is that their values cannot 'overshoot', meaning that even extreme weight changes will lead to reasonable output values, which seems to be crucial for online learning. 
This can best be seen from the sensitivity of the network with respect to the learning rate. Figure~\ref{fig:learning_rate_compare_lin_sig} shows the performance of the last 20 patterns for (a) linear networks and (b) networks with sigmoid activation functions that are trained to associate 100 binary random patterns (\emph{RAND}) with another 100 binary random patterns (\emph{RAND}) using different learning rates.
\begin{figure}[t]
\begin{center}
\subfigure[]{
\includegraphics[scale=0.425, trim=21 0 39 0, clip]{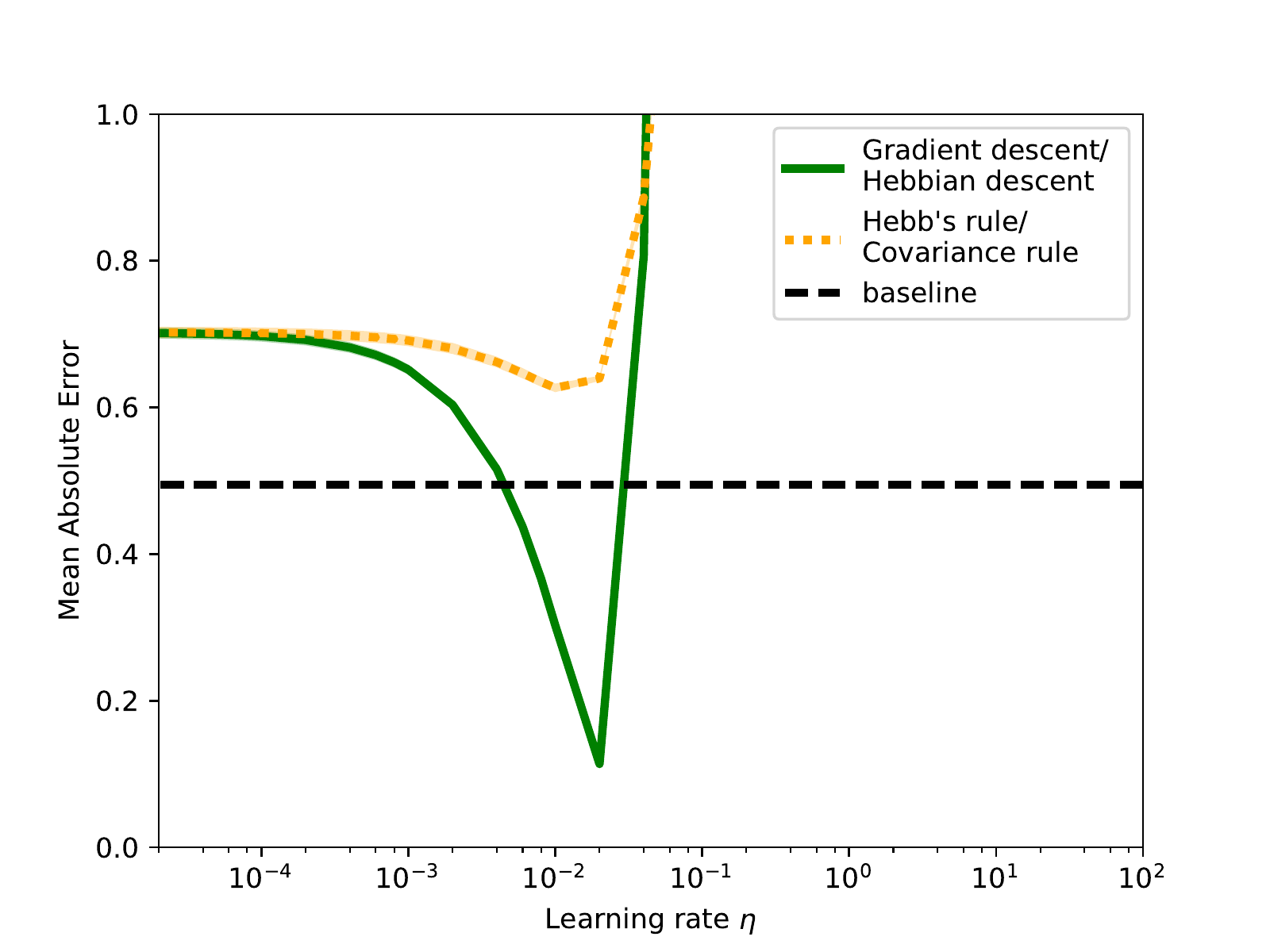}}
\subfigure[]{
\includegraphics[scale=0.425, trim=34 0 38 0, clip]{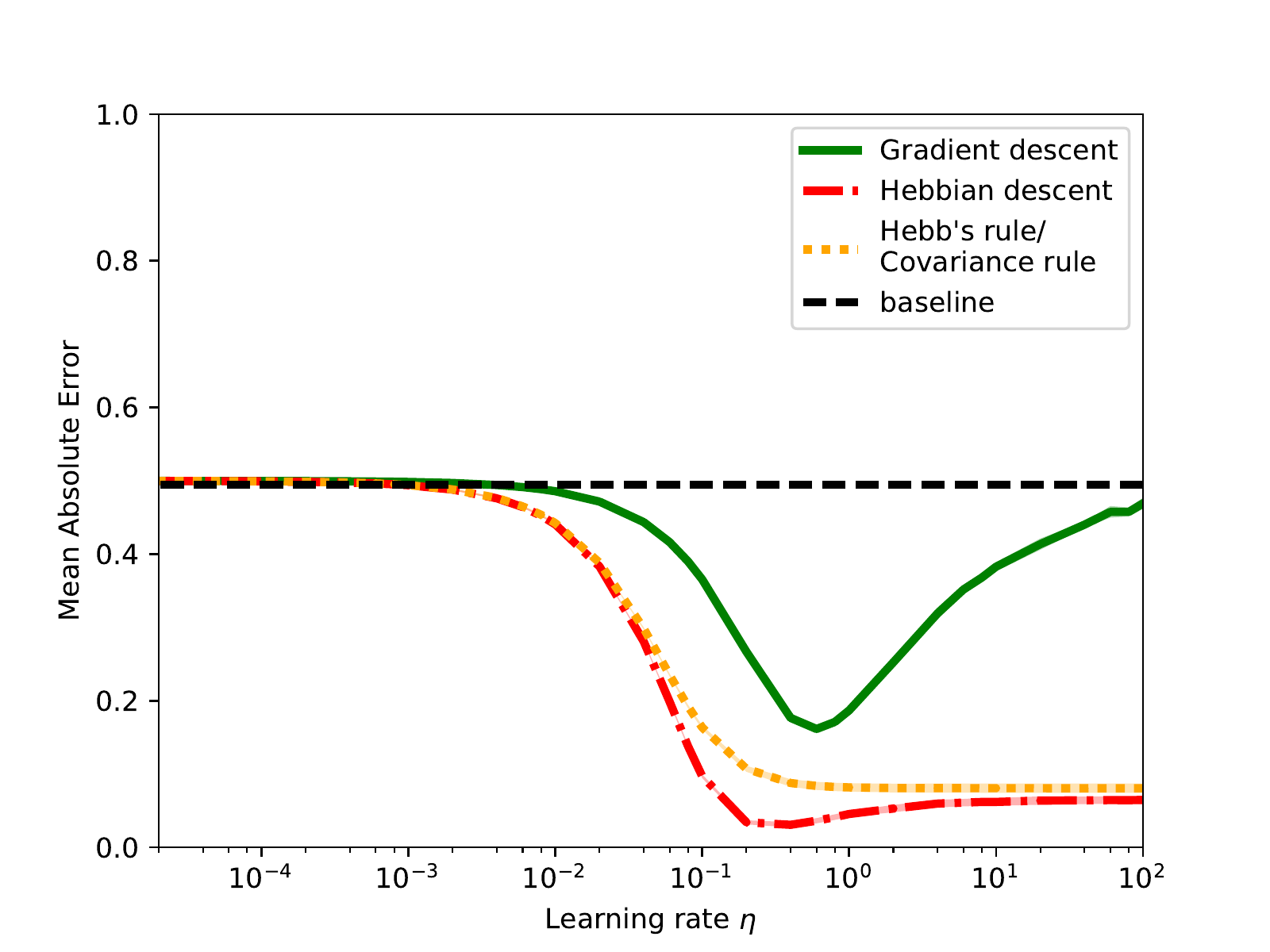}}
\caption{Online learning performance of the four different update rules with respect to the learning rate. 
The centered networks have been trained to associate 100 binary random patterns (\emph{RAND}) with another 100 binary random patterns (\emph{RAND}) with (a) linear, and (b) sigmoid activation function. 
The curves represent the average MAE for the last 20 patterns for the corresponding learning rate averaged over 10 trials. 
The standard deviation is also shown, but the values are too small to be visible without zooming in.
The baseline represents the performance of a network that independently of the input always returns the mean of the output patterns.}
\label{fig:learning_rate_compare_lin_sig}
\end{center}
\end{figure}
In the linear case it is very crucial for all methods to choose the right learning rate as can be seen from the very sharp optimum. 
If the learning rate is chosen to big the error increases exponentially. Also notice that Hebb's rule / covariance rule does not even get close to baseline performance.
When using a sigmoid activation function, however, all methods perform significantly better than baseline. 
While gradient descent has still a rather sharp optimum, the performance of the other methods do not change significantly for a learning rate above a certain threshold.
This is a very useful property as one can simply select a large learning rate, instead of performing a grid search, to achieve a performance close to optimum.
Qualitatively the same picture can be seen for the other datasets used in the experiments for Table~\ref{tab:Hetero_online_centered_1}.

\begin{table}[htbp]
\setlength{\tabcolsep}{2pt}
\begin{center}
\begin{small}
\begin{sc}
%\begin{tabular}{l@{\hskip 0.17in} r@{\hskip 0.02in} r@{\hskip 0.1in} r@{\hskip 0.01in} r@{\hskip 0.17in} r@{\hskip 0.01in} r@{\hskip 0.1in} r@{\hskip 0.01in} r@{\hskip 0.02in} }
\begin{tabular}{l@{\hskip 0.2in} r@{\hskip 0.02in} r@{\hskip 0.14in} r@{\hskip 0.02in} r@{\hskip 0.14in} r@{\hskip 0.02in} r@{\hskip 0.14in} r@{\hskip 0.02in} r }
\hline
\abovespace\belowspace
      & \multicolumn{4}{c}{\hskip -0.16in Gradient Descent} &   \multicolumn{4}{c}{\hskip -0.16in Hebbian-Descent} \\ 
  $\vect \phi$ & \multicolumn{2}{c}{\hskip -0.16in \footnotesize{fixed offsets}} &   \multicolumn{2}{c}{\hskip -0.16in  \footnotesize{adaptive offsets}} &  \multicolumn{2}{c}{\hskip -0.16in \footnotesize{fixed offsets}} &  \multicolumn{2}{c}{\hskip -0.05in \footnotesize{adaptive offsets}} \\ 
\hline & & & & & & & &\vspace{-0.3cm}\\ 
\multicolumn{3}{l}{{\emph{RAND$\,\,\rightarrow\,\,$ADULT (0.1229)}}}\\ 
 & & & & & & & &  \vspace{-0.35cm}\\ 
Linear  & \textbf{0.0793}  & (0.1829)  & 0.0820  & (0.1901)  & \textbf{0.0793}  & (0.1829)  & 0.0820  & (0.1901)  \\ 
Rectifier  & 0.0667  & (0.0760)  & 0.0639  & (0.0709)  & 0.0321  & (0.0553)  & \textbf{0.0285}  & (0.0516)  \\ 
ExpLin  & 0.1411  & (0.1979)  & 0.1475  & (0.2042)  & 0.0892  & (0.1692)  & \textbf{0.0877}  & (0.1734)  \\ 
Sigmoid  & 0.0668  & (0.0651)  & 0.0706  & (0.0625)  & 0.0120  & (0.0459)  & \textbf{0.0103}  & (0.0420)  \\ 
Step  & 0.4997  & (0.5008)  & 0.5026  & (0.5018)  & 0.0230  & (0.0680)  & \textbf{0.0157}  & (0.0450)  \\ 
\hline & & & & & & & &\vspace{-0.3cm}\\ 
\multicolumn{3}{l}{{\emph{RANDN$\,\,\rightarrow\,\,$CONNECT (0.1683)}}}\\ 
 & & & & & & & &  \vspace{-0.35cm}\\ 
Linear  & 0.1215  & (0.1479)  & \textbf{0.1166}  & (0.1450)  & 0.1215  & (0.1479)  & \textbf{0.1166}  & (0.1450)  \\ 
Rectifier  & 0.0992  & (0.1138)  & 0.0972  & (0.1122)  & 0.0928  & (0.1135)  & \textbf{0.0888}  & (0.1103)  \\ 
ExpLin  & 0.1218  & (0.1466)  & 0.1172  & (0.1437)  & 0.1209  & (0.1466)  & \textbf{0.1161}  & (0.1437)  \\ 
Sigmoid  & 0.1091  & (0.1099)  & 0.1082  & (0.1098)  & 0.0487  & (0.0728)  & \textbf{0.0442}  & (0.0707)  \\ 
Step  & 0.5014  & (0.5015)  & 0.5006  & (0.5008)  & 0.0573  & (0.0839)  & \textbf{0.0516}  & (0.0780)  \\ 
\hline & & & & & & & &\vspace{-0.3cm}\\ 
\multicolumn{3}{l}{{\emph{ADULT$\,\,\rightarrow\,\,$MNIST (0.1473)}}}\\ 
 & & & & & & & &  \vspace{-0.35cm}\\ 
Linear  & 0.1063  & (0.1531)  & \textbf{0.1059}  & (0.1577)  & 0.1063  & (0.1531)  & \textbf{0.1059}  & (0.1577)  \\ 
Rectifier  & 0.0891  & (0.1134)  & 0.1137  & (0.1236)  & 0.0778  & (0.1217)  & \textbf{0.0761}  & (0.1197)  \\ 
ExpLin  & 0.1061  & (0.1594)  & 0.1053  & (0.1538)  & 0.1060  & (0.1521)  & \textbf{0.1050}  & (0.1552)  \\ 
Sigmoid  & 0.1155  & (0.1265)  & 0.1219  & (0.1317)  & 0.0762  & (0.1247)  & \textbf{0.0706}  & (0.1182)  \\ 
Step  & 0.5006  & (0.4998)  & 0.4970  & (0.4967)  & 0.0907  & (0.1325)  & \textbf{0.0883}  & (0.1293)  \\ 
\hline & & & & & & & &\vspace{-0.3cm}\\ 
\multicolumn{3}{l}{{\emph{CONNECT$\,\,\rightarrow\,\,$CIFAR (0.167)}}}\\ 
 & & & & & & & &  \vspace{-0.35cm}\\ 
Linear  & 0.1238  & (0.1803)  & \textbf{0.1204}  & (0.1807)  & 0.1238  & (0.1803)  & \textbf{0.1204}  & (0.1807)  \\ 
Rectifier  & 0.1270  & (0.1886)  & 0.1514  & (0.2094)  & 0.1237  & (0.1802)  & \textbf{0.1205}  & (0.1803)  \\ 
ExpLin  & 0.1238  & (0.1803)  & \textbf{0.1204}  & (0.1805)  & 0.1238  & (0.1802)  & \textbf{0.1204}  & (0.1807)  \\ 
Sigmoid  & 0.1253  & (0.1682)  & \textbf{0.1236}  & (0.1698)  & 0.1264  & (0.1974)  & 0.1261  & (0.1679)  \\ 
\hline & & & & & & & &\vspace{-0.3cm}\\ 
\multicolumn{3}{l}{{\emph{MNIST$\,\,\rightarrow\,\,$CONNECT (0.1683)}}}\\ 
 & & & & & & & &  \vspace{-0.35cm}\\ 
Linear  & 0.1830  & (0.2659)  & \textbf{0.1522}  & (0.2053)  & 0.1830  & (0.2659)  & \textbf{0.1522}  & (0.2053)  \\ 
Rectifier  & 0.1911  & (0.2339)  & 0.1594  & (0.1751)  & 0.1042  & (0.1668)  & \textbf{0.0927}  & (0.1511)  \\ 
ExpLin  & 0.1671  & (0.2435)  & 0.1584  & (0.2045)  & 0.1695  & (0.2507)  & \textbf{0.1520}  & (0.2019)  \\ 
Sigmoid  & 0.1242  & (0.1434)  & 0.1321  & (0.1378)  & 0.0552  & (0.1218)  & \textbf{0.0441}  & (0.0955)  \\ 
Step  & 0.4982  & (0.4992)  & 0.4968  & (0.4971)  & 0.0672  & (0.1329)  & \textbf{0.0529}  & (0.1010)  \\ 
\hline & & & & & & & &\vspace{-0.3cm}\\ 
\multicolumn{3}{l}{{\emph{CIFAR$\,\,\rightarrow\,\,$CONNECT (0.1683)}}}\\ 
 & & & & & & & &  \vspace{-0.35cm}\\ 
Linear  & 0.2187  & (0.2576)  & \textbf{0.1783}  & (0.2184)  & 0.2187  & (0.2576)  & \textbf{0.1783}  & (0.2184)  \\ 
Rectifier  & 0.2348  & (0.2636)  & 0.1742  & (0.2013)  & 0.1570  & (0.1964)  & \textbf{0.1311}  & (0.1608)  \\ 
ExpLin  & 0.2239  & (0.2620)  & 0.1774  & (0.2147)  & 0.2161  & (0.2536)  & \textbf{0.1757}  & (0.2143)  \\ 
Sigmoid  & 0.1614  & (0.1852)  & 0.1208  & (0.1302)  & 0.1108  & (0.1633)  & \textbf{0.0811}  & (0.1247)  \\ 
Step  & 0.5009  & (0.4998)  & 0.5021  & (0.5013)  & 0.1186  & (0.1707)  & \textbf{0.0877}  & (0.1286)  \\ 
\end{tabular}
\end{sc}
\end{small}
\end{center}
\caption{The same experiments as in Table~\ref{tab:Hetero_online_centered_1} but with adapted offset values. The offset values were initialized to 0.5 and updated with a shifting factor of 0.05 after each parameter update.
} 
\label{tab:Hetero_online_centered_shift}
\end{table}

\subsubsection{Hetero-Associative Online Learning with Adaptive Offsets}

In the previous experiments the offsets were fixed to the data mean, which was necessary for a fair comparison of Hebb's rule and the covariance rule with gradient descent and Hebbian-descent.
We now want to show that slowly updated offsets lead to a similar or often even slightly better performance for Hebbian-descent than using fixed offsets.
We therefore performed the same experiments as performed for Table~\ref{tab:Hetero_online_uncentered_1}, \emph{i.e.}\ online hetero-association with gradient descent and Hebbian-descent in centered networks but with updated offsets. 
The offset values were initialized to 0.5 and updated by an exponentially moving average with a shifting factor of 0.05 after each parameter update.
The shifting factor should be chosen small enough, so that the offsets represent the temporary mean over sufficiently  many past time steps and large enough, so that the offsets can adapt quickly to changes in the temporary mean.
We found empirically that values between 0.01\,-\,0.05 work well.
The results are shown in Table~\ref{tab:Hetero_online_centered_shift} illustrating that a similar or slightly better performance is achieved when updating the offsets.
Figure~\ref{fig:heteroassociative_online_updated_offsets_scatter} visualizes the effect of updated offsets compared to fixed offsets. 
\begin{figure}[t]
\begin{center}
\subfigure[]{
\includegraphics[scale=0.415, trim=21 0 37 0, clip]{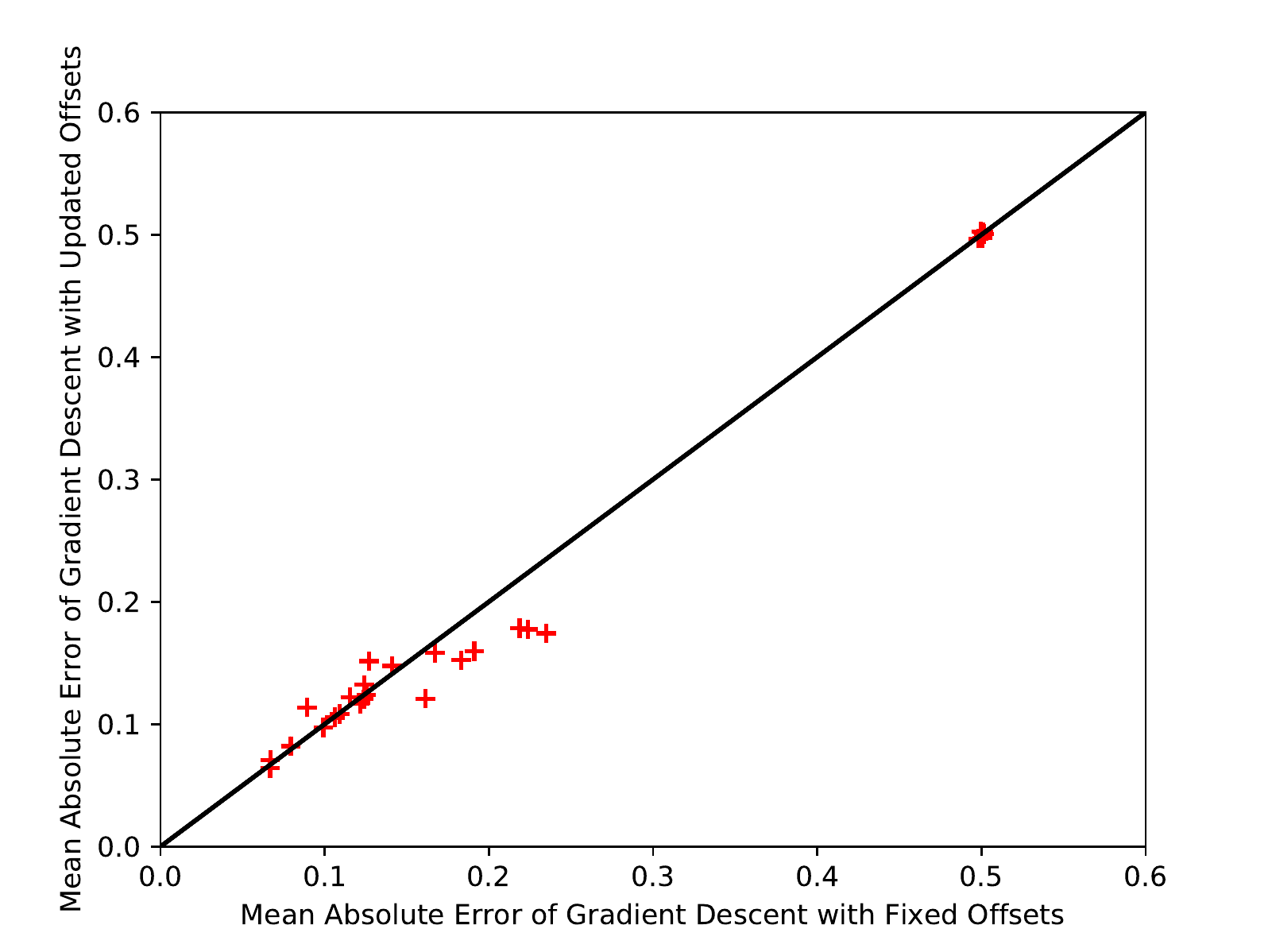}\label{fig:heteroassociative_updated_offsets_gradient_descent}}
\subfigure[]{
\includegraphics[scale=0.415, trim=21 0 37 0, clip]{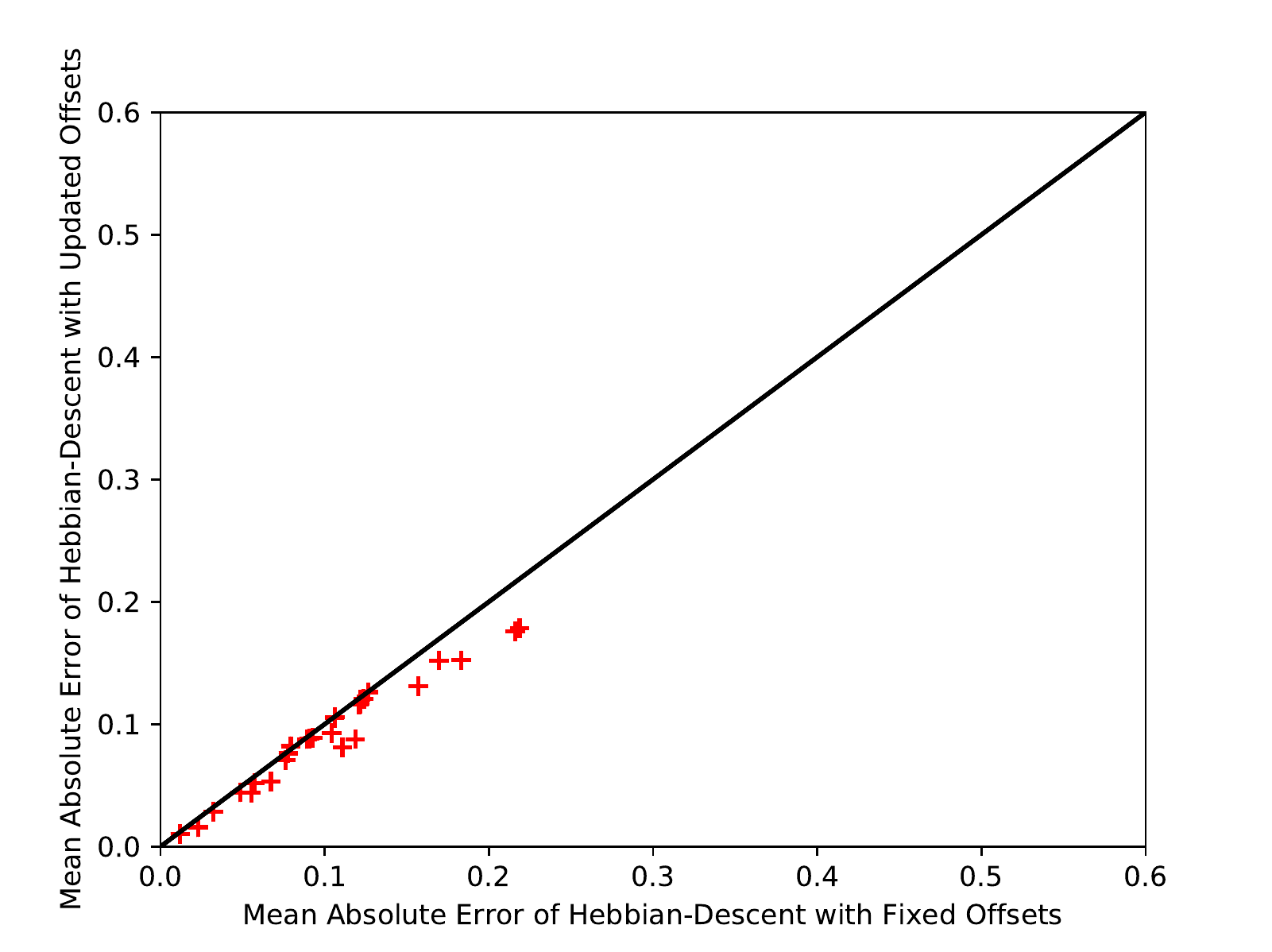}\label{fig:heteroassociative_updated_offsets_hebbian_descent}}
\caption{Comparison of fixed with updated offsets in online hetero-association with (a) gradient descent and (b) Hebbian-descent. 
Each cross represents the mean absolute error of the last 20 patterns averaged over 10 trials for one experiment in Table~\ref{tab:Hetero_online_centered_shift}.}
\label{fig:heteroassociative_online_updated_offsets_scatter}
\end{center}
\end{figure} 
Figure~\ref{fig:heteroassociative_updated_offsets_gradient_descent} shows the results for gradient descent where points lie above and below the diagonal so that there is no clear tendency for fixed or updated offsets. 
Figure~\ref{fig:heteroassociative_updated_offsets_hebbian_descent} shows the results for Hebbian-descent where points lie either on or below the diagonal so that Hebbian-descent seems to profit slightly from updated offsets.

From the perspective of online learning the use of an exponential moving average for updating the offsets makes sense as it leads to an exponentially weighted mean approximation where more recent patterns have a stronger influence. 
This can be interpreted as a forgetting mechanism that allows the offsets to represent the mean of the most recent patterns, in which we are actually interested in instead of the mean of all previously seen patterns.
Even if the input patterns are sampled i.i.d. the mean over the most recent patterns often differs from the mean over all patterns, which can be interpreted as non-stationarity.
To show the effect of non-stationary input distributions more clearly we trained networks using Hebbian-descent with and without updating the offsets to associate 100 \emph{MNIST} patterns that are ordered by their label values as shown in Figure~\ref{fig:moving_MNIST_a} with 100 randomly selected \emph{MNIST} patterns. 
\begin{figure}[t]
\begin{center}
\subfigure[]{
\includegraphics[scale=0.421, trim=100 0 70 0, clip]{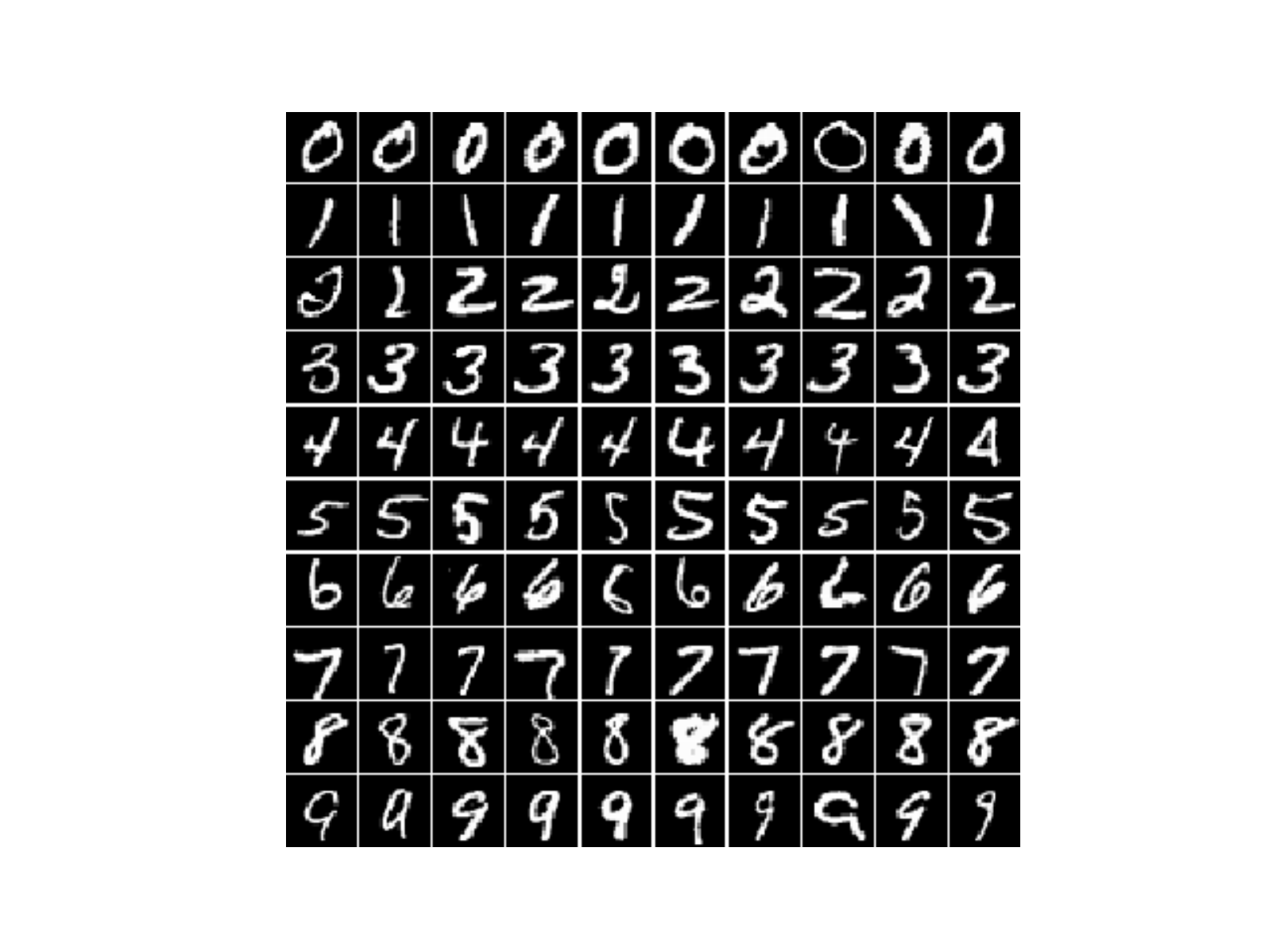}\label{fig:moving_MNIST_a}}
\subfigure[]{
\includegraphics[scale=0.421, trim=220 0 195 0, clip]{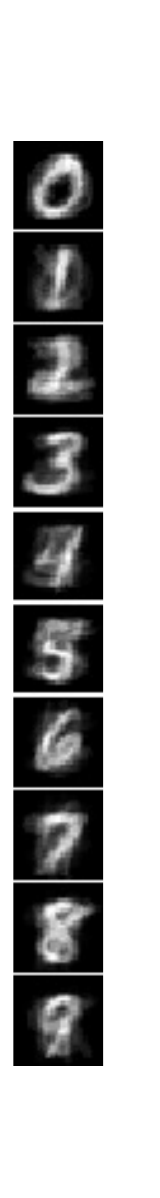}\label{fig:moving_MNIST_b}}
\subfigure[]{
\includegraphics[scale=0.421, trim=5 0 10 0, clip]{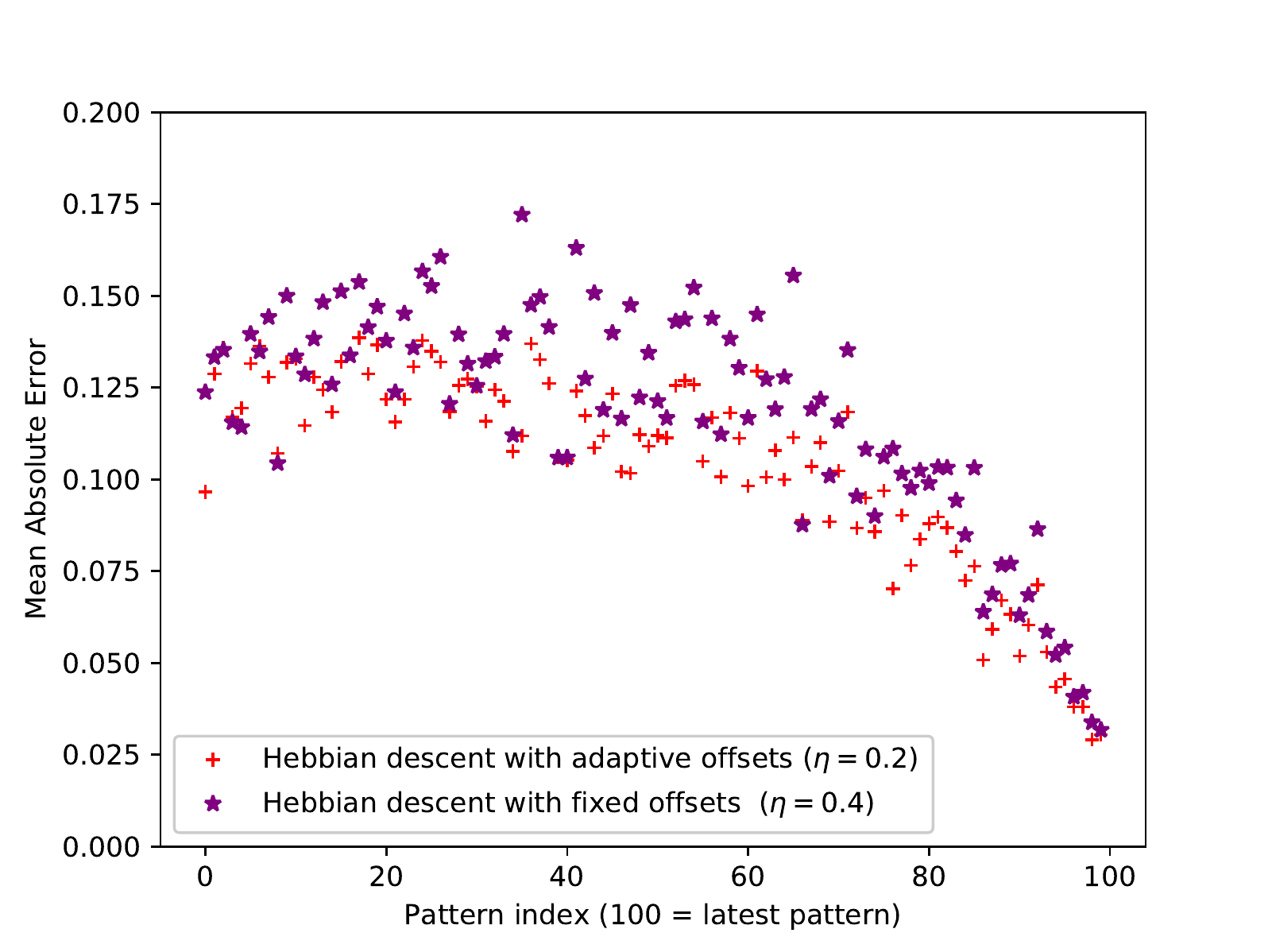}\label{fig:moving_MNIST_c}}
\caption{
Online learning performance of centered Hebbian descent with sigmoid units, when 100 \emph{MNIST} patterns in order have been associated with 100 randomly selected \emph{MNIST} patterns. (a) input data consisting of 100 \emph{MNIST} patterns ordered from first pattern (top left) to the last pattern (bottom right). (b) Offset patterns after every 10th update step. (c) Mean absolute error  of Hebbian-descent with and without updated offsets averaged over 10 trials.
The shifting factor was set to 0.1 and the learning rate $\eta$ was chosen for each method individually, so that the performance over the last 10 patterns is best, but optimizing on all or only the last pattern also leads to very similar results. Notice that the limit of the y-axis is 0.2 instead of 1.0 as for the other plots
}
\label{fig:moving_MNIST}
\end{center}
\end{figure}
Figure~\ref{fig:moving_MNIST_b} shows the evolution of the offset values when an exponentially moving average with shifting factor 0.1 is used.
The offsets have a clear tendency to represent the mean of the last 10 patterns given by the corresponding row in Figure~\ref{fig:moving_MNIST_a}.
Figure~\ref{fig:moving_MNIST_c} shows the performance of Hebbian-descent with fixed and updated offsets averaged over 10 trials. 
Updating the offsets leads to a better performance on the last 10 patterns, 
which is 0.0598 with updating the offsets compared to 0.0702 without updating them. 
But also the performance on all patterns is better when updating the offets, which is 0.1040 compared to 0.1179 without updating the offsets. 

\subsubsection{Hetero-Associative Online Learning with Weight Decay}\label{sec:HD_weight_decay_experiments}

The ability of the network to forget patterns becomes more and more important as more and more patterns are stored in the network.
To illustrate this effect we trained centered networks on 1000 rather than 100 patterns, which is way beyond the capacity of the network. 
Figure~\ref{fig:long_training_rand_adult} shows the results  when 1000 binary random patterns have been associated with (a) another 1000 binary random patterns and (b) with 1000 patterns of the \emph{ADULT} dataset. 
\begin{figure}[t]
\begin{center}
\subfigure[]{
\includegraphics[scale=0.421, trim=21 0 34 0, clip]{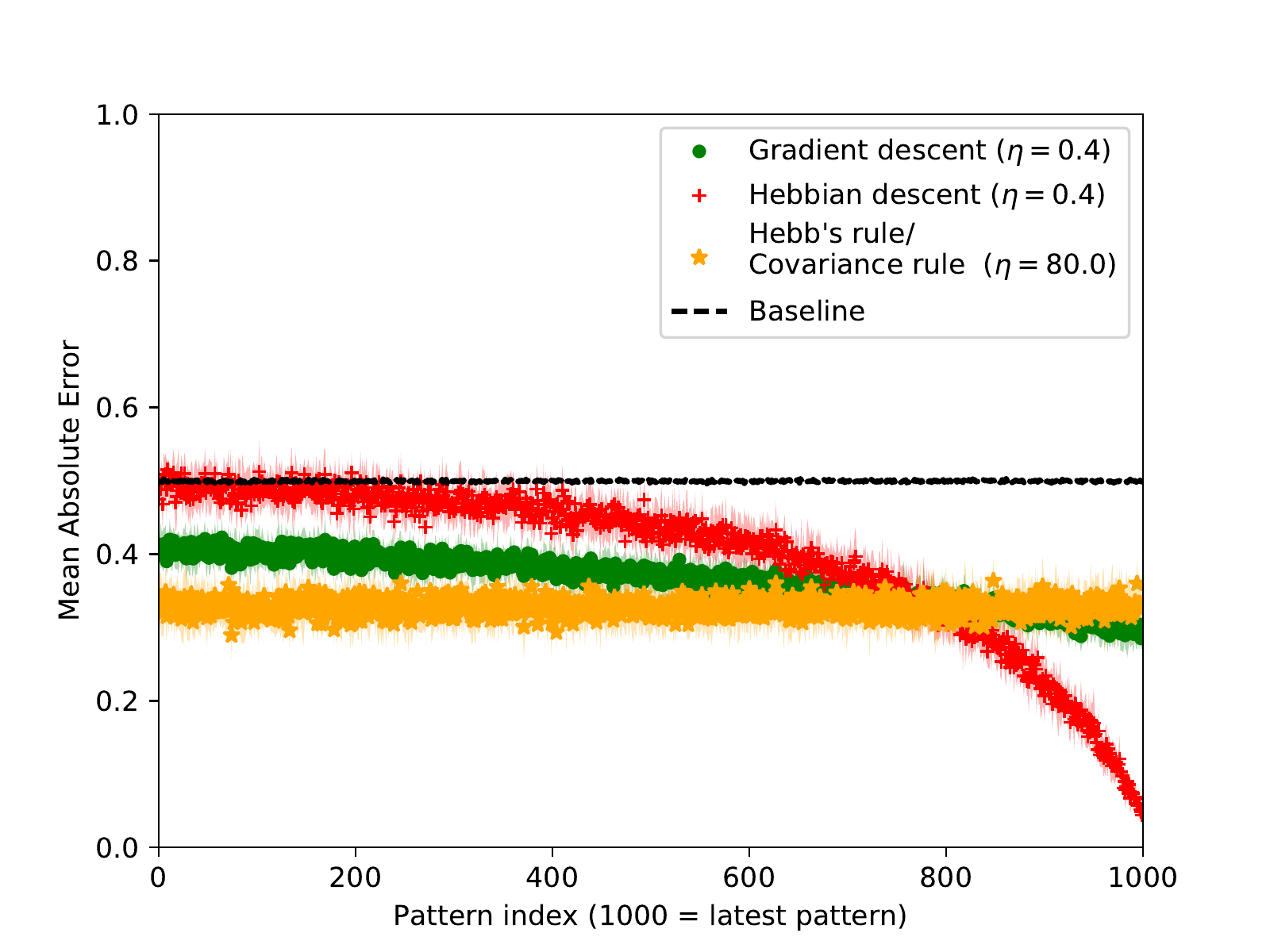}}
\subfigure[]{
\includegraphics[scale=0.421, trim=34 0 34 0, clip]{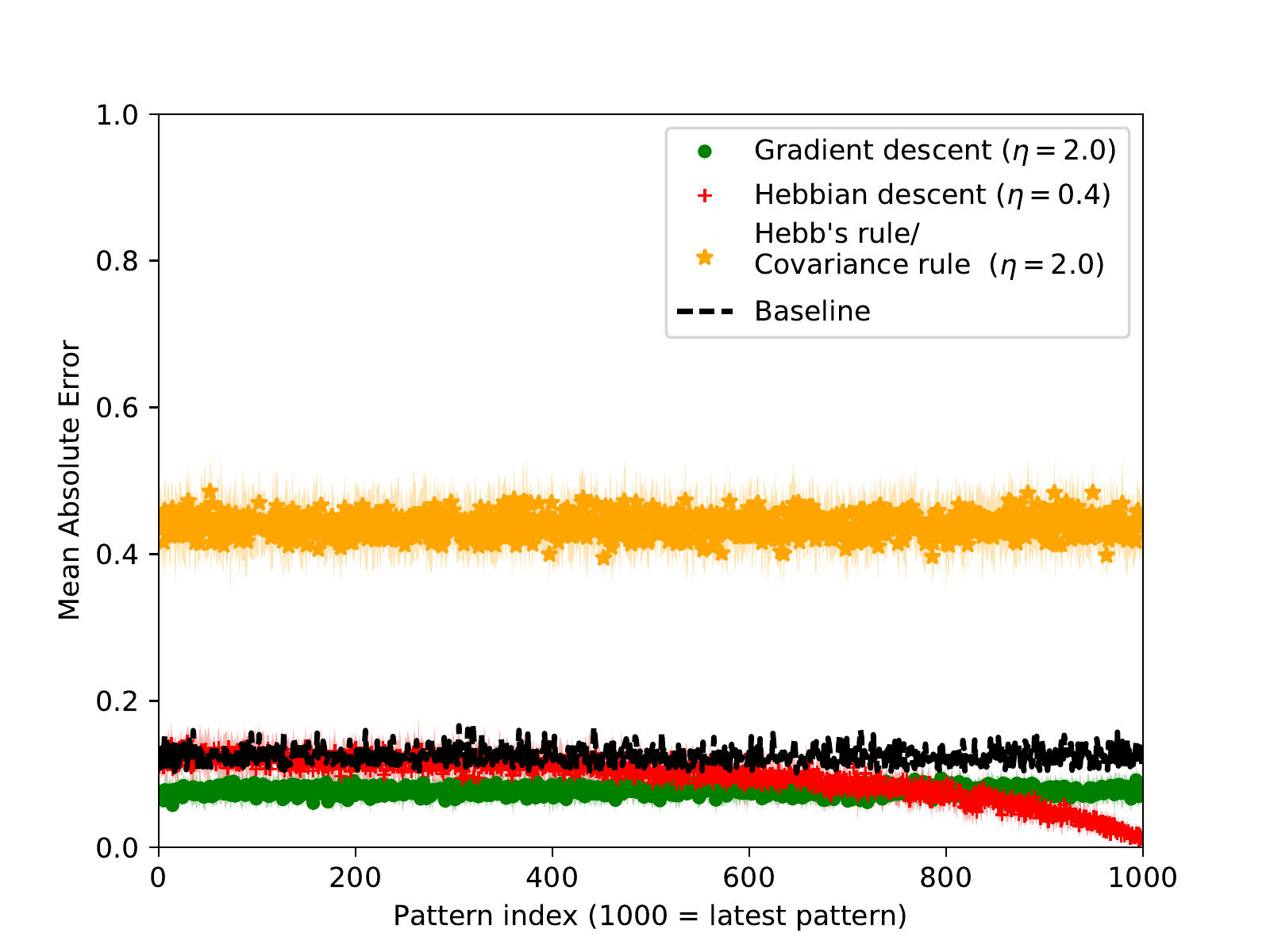}}
\caption{
Online learning performance of the four different update rules with centering and with sigmoid units, when 1000 binary random patterns have been associated with (a) another 1000 binary random patterns, and (b) 1000  patterns of the \emph{ADULT} dataset, one pattern pair at a time. 
The Mean absolute error and the corresponding standard deviation over 10 trials is plotted for each pattern separately. 
The learning rate $\eta$ was chosen for each method individually, so that the performance over the last 20 patterns is best, but optimizing on all or only the last pattern also leads to very similar results.
The baseline represents the performance of a network that independently of the input always returns the mean of the output patterns.
}
\label{fig:long_training_rand_adult}
\end{center}
\end{figure} 
Both plots show that the errors for Hebb's rule, the covariance rule, and gradient descent simply increase for all patterns leading to a very bad performance overall. 
Hebbian-descent in contrast still shows a power-law forgetting curve where more recent patterns are represented better than older once. 
When zooming in on the last 100 patterns the resulting curves for Hebbian-descent look very similar to the curves when only 100 patterns have been stored as shown in Figure~\ref{fig:rand_rand_sigmoid_true_online} and \ref{fig:adult_rand_sigmoid_true_online}, respectively.
Thus only Hebbian-descent allows to have always roughly the same performance on the latest patterns independent of the number of patterns that have been stored previously.
 
While Hebbian-descent controls the amount of forgetting automatically, the most common way to implement it manually is to use a weight decay term, which is independent of the learning rule.
It removes a certain proportion of the current weight matrix in each update step and therefore introduces an additional hyperparameter that controls the speed of forgetting.
To investigate the effect of a weight decay on the four different methods we trained networks to associate 100 binary random patterns (\emph{RAND}) with another 100 binary random patterns (\emph{RAND}). 
A simultaneous grid search over the learning rate and weight decay parameter was performed, so that, just as in the previous experiments, the performance on the last 20 patterns was best. 
The results are shown in Figure~\ref{fig:rand_adult_weightdecay} (a) illustrating that all methods now perform significantly better on the latest patterns and that the covariance rule performs even better than Hebbian-descent.
\begin{figure}[t]
\begin{center}
\subfigure[]{
\includegraphics[scale=0.425, trim=21 0 38 0, clip]{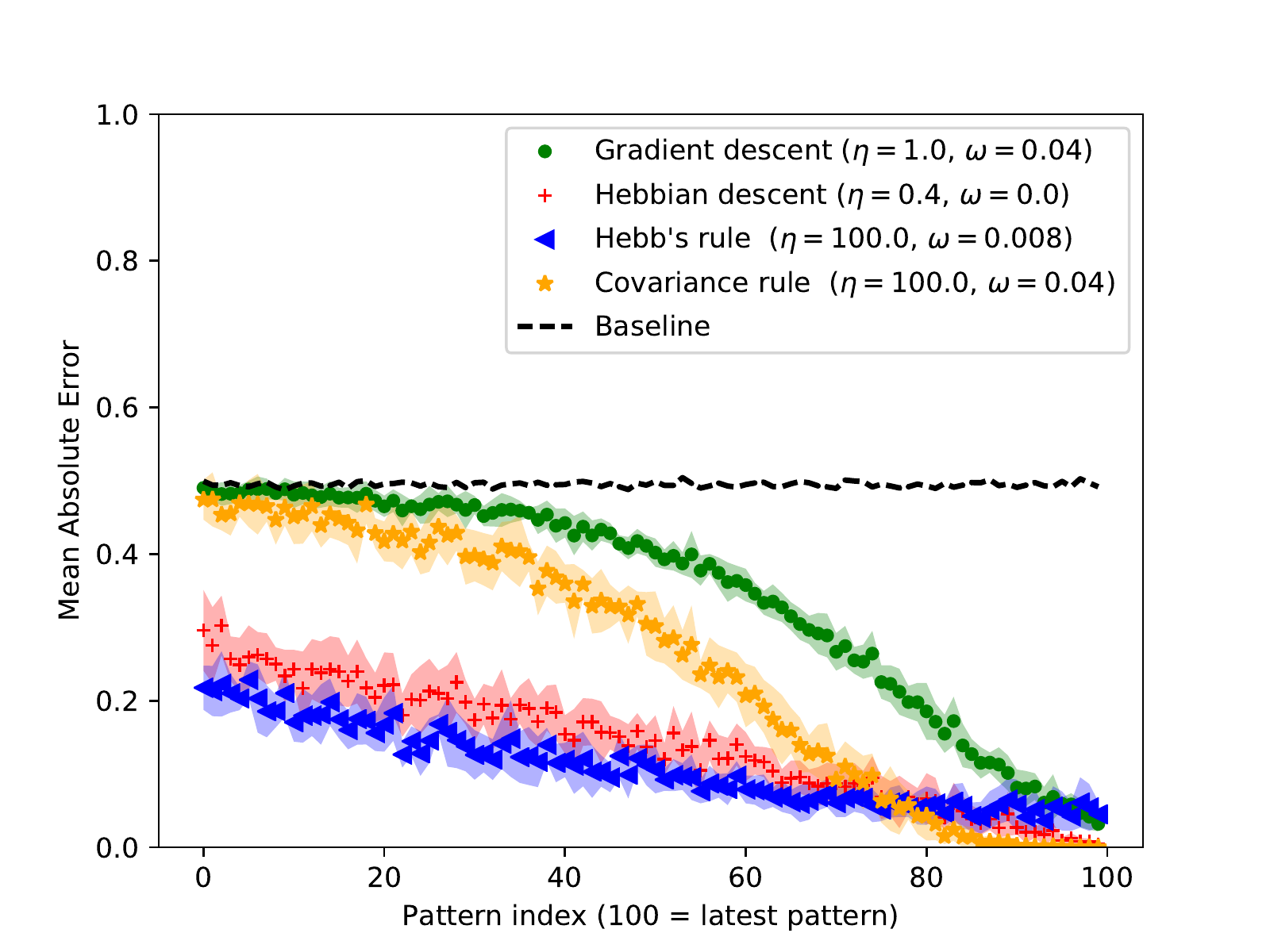}}
\subfigure[]{
\includegraphics[scale=0.425, trim=34 0 37 0, clip]{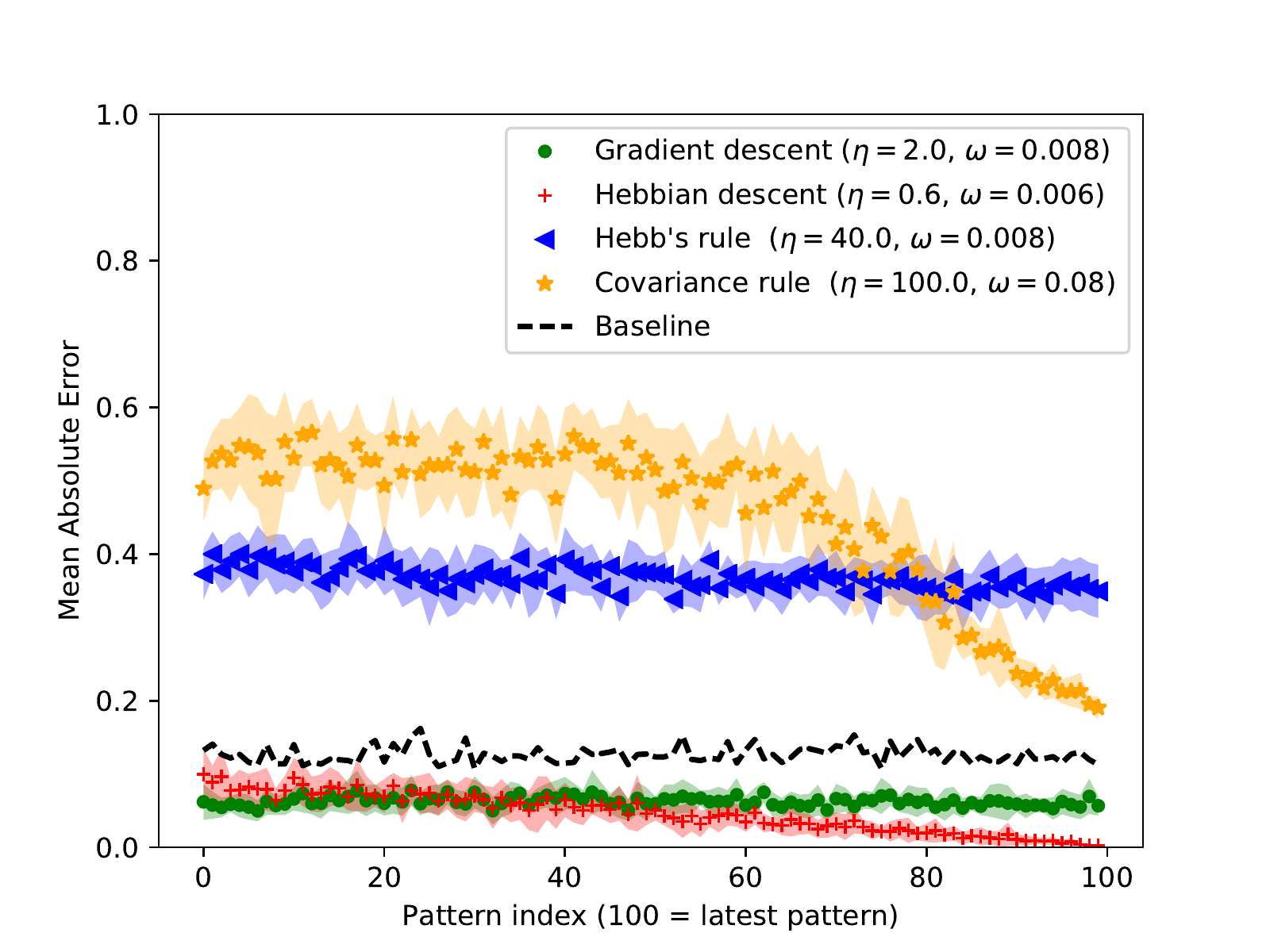}}
\caption{
Online learning performance of the four different update rules with weight decay, centering, and sigmoid units, when (a) 100 binary random patterns (\emph{RAND}) have been associated with 100 patterns of the \emph{ADULT} dataset and (b) 100 patterns of the \emph{ADULT} dataset have been associated with 100 binary random patterns (\emph{RAND}). 
The Mean absolute error and the corresponding standard deviation over 10 trials is plotted for each pattern separately. 
The learning rate $\eta$ and weight decay $\omega$ was chosen for each method individually via grid search, so that the performance over the last 20 patterns was best.
The baseline represents the performance of a network that independently of the input always returns the mean of the output patterns.
Notice, that when using a weight decay Hebb's rule and the covariance rule are not equivalent anymore in case of a centered network.
}
\label{fig:rand_adult_weightdecay}
\end{center}
\end{figure} 
Notice that the equivalence of Hebb's rule and the covariance rule does not hold anymore when a weight decay term is used as can be seen from Equation~\eqref{eqn:single_cov}.
The covariance rule and gradient descent still have a rather steep forgetting curve such that after 100 patterns the performance is already on baseline performance.
Worth mentioning that the optimal weight decay parameter for Hebbian-descent is 0 for this experiment, which indicates that the implicit mechanism of forgetting does not improve with weight decay. 
We also performed the same experiments of associating 100 binary random patterns (\emph{RAND}) with 100 patterns of the \emph{ADULT} dataset. 
The results are shown in Figure~\ref{fig:rand_adult_weightdecay} (b) illustrating that in this case weight decay does not allow for a better performance of neither Hebb's rule, covariance rule, or gradient descent than Hebbian-descent.
We performed many more experiments with other datasets and activation functions and the results are shown in Table~\ref{tab:centered_weight_decay}. 
% Updated results
\begin{table}[htbp]
\setlength{\tabcolsep}{2pt}
\begin{center}
\begin{small}
\begin{sc}
\begin{tabular}{l@{\hskip 0.2in} r@{\hskip 0.02in} r@{\hskip 0.14in} r@{\hskip 0.02in} r@{\hskip 0.14in} r@{\hskip 0.02in} r@{\hskip 0.14in} r@{\hskip 0.02in} r }
\hline
\abovespace\belowspace
  $\vect \phi$ & \multicolumn{2}{c}{\hskip -0.16in Grad. Descent} &   \multicolumn{2}{c}{\hskip -0.16in Hebb. Descent} &  \multicolumn{2}{c}{\hskip -0.16in Hebb rule} &  \multicolumn{2}{c}{\hskip -0.16in Cov. Rule} \\ 
\hline & & & & & & & &\vspace{-0.3cm}\\ 
\multicolumn{3}{l}{{\emph{RAND$\,\,\rightarrow\,\,$RAND (0.4946)}}}\\ 
 & & & & & & & &  \vspace{-0.35cm}\\ 
LINEAR  & \textbf{0.0988}  & (0.2525)  & \textbf{0.0988}  & (0.2525)  & 0.2913  & (0.5212)  & 0.5027  & (0.5247)  \\ 
SIGMOID  & 0.1003  & (0.3452)  & 0.0307  & (0.1411)  & 0.0516  & (0.1097)  & \textbf{0.0090}  & (0.2572)  \\ 
\hline & & & & & & & &\vspace{-0.3cm}\\ 
\multicolumn{3}{l}{{\emph{RAND$\,\,\rightarrow\,\,$ADULT (0.1229)}}}\\ 
 & & & & & & & &  \vspace{-0.35cm}\\
RECTIFIER  & 0.0529  & (0.0771)  & \textbf{0.0232}  & (0.0576)  & 0.0560  & (0.1013)  & 0.0589  & (0.1031)  \\  
SIGMOID  & 0.0592  & (0.0628)  & \textbf{0.0115}  & (0.0470)  & 0.3541  & (0.3680)  & 0.2569  & (0.4538)  \\ 
\hline & & & & & & & &\vspace{-0.3cm}\\ 
\multicolumn{3}{l}{{\emph{ADULT$\,\,\rightarrow\,\,$RAND (0.4946)}}}\\ 
 & & & & & & & &  \vspace{-0.35cm}\\ 
RECTIFIER  & 0.2466  & (0.4126)  & \textbf{0.2246}  & (0.3814)  & 0.3196  & (0.4161)  & 0.2889  & (0.4252)  \\ 
SIGMOID  & 0.2532  & (0.3952)  & 0.1623  & (0.3415)  & 0.2323  & (0.3530)  & \textbf{0.1541}  & (0.3712)  \\ 
\hline & & & & & & & &\vspace{-0.3cm}\\ 
\multicolumn{3}{l}{{\emph{ADULT$\,\,\rightarrow\,\,$MNIST (0.1473)}}}\\ 
 & & & & & & & &  \vspace{-0.35cm}\\ 
EXPLIN  & \textbf{0.0948}  & (0.1462)  & 0.0951  & (0.1468)  & 0.1330  & (0.1730)  & 0.1406  & (0.1382)  \\ 
SIGMOID  & 0.1094  & (0.1375)  & \textbf{0.0762}  & (0.1247)  & 0.4166  & (0.4375)  & 0.3632  & (0.4577)  \\ 
\hline & & & & & & & &\vspace{-0.3cm}\\ 
\multicolumn{3}{l}{{\emph{CONNECT$\,\,\rightarrow\,\,$CIFAR (0.167)}}}\\ 
 & & & & & & & &  \vspace{-0.35cm}\\ 
 LINEAR  & \textbf{0.1198}  & (0.1899)  & \textbf{0.1198}  & (0.1899)  & 0.3425  & (0.5296)  & 0.4936  & (0.4990)  \\ 
SIGMOID  & 0.1232  & (0.1733)  & \textbf{0.1230}  & (0.1957)  & 0.1714  & (0.1765)  & 0.1358  & (0.1954)  \\ 
\hline & & & & & & & &\vspace{-0.3cm}\\ 
\multicolumn{3}{l}{{\emph{CONNECT$\,\,\rightarrow\,\,$MNIST (0.1473)}}}\\ 
 & & & & & & & &  \vspace{-0.35cm}\\ 
EXPLIN  & \textbf{0.0904}  & (0.1484)  & 0.0911  & (0.1494)  & 0.1327  & (0.1756)  & 0.1406  & (0.1382)  \\ 
SIGMOID  & 0.1080  & (0.1376)  & \textbf{0.0696}  & (0.1239)  & 0.4043  & (0.4336)  & 0.3567  & (0.4544)  \\ 
\hline & & & & & & & &\vspace{-0.3cm}\\ 
\multicolumn{3}{l}{{\emph{MNIST$\,\,\rightarrow\,\,$RAND (0.4946)}}}\\ 
 & & & & & & & &  \vspace{-0.35cm}\\ 
SIGMOID  & 0.2374  & (0.3802)  & \textbf{0.1524}  & (0.2951)  & 0.2304  & (0.3444)  & 0.1630  & (0.3948)  \\ 
STEP  & 0.5012  & (0.4999)  & 0.1675  & (0.3051)  & 0.2301  & (0.3440)  & \textbf{0.1629}  & (0.3949)  \\ 
\hline & & & & & & & &\vspace{-0.3cm}\\ 
\multicolumn{3}{l}{{\emph{MNIST$\,\,\rightarrow\,\,$CONNECT (0.1683)}}}\\ 
 & & & & & & & &  \vspace{-0.35cm}\\ 
RECTIFIER  & 0.1190  & (0.1754)  & \textbf{0.0873}  & (0.1518)  & 0.2032  & (0.2941)  & 0.2802  & (0.3345)  \\ 
SIGMOID  & 0.1002  & (0.1344)  & \textbf{0.0546}  & (0.1230)  & 0.3243  & (0.4359)  & 0.2903  & (0.4457)  \\ 
\hline & & & & & & & &\vspace{-0.3cm}\\ 
\multicolumn{3}{l}{{\emph{RANDN$\,\,\rightarrow\,\,$CONNECT (0.1683)}}}\\ 
 & & & & & & & &  \vspace{-0.35cm}\\ 
SIGMOID  & 0.1073  & (0.1259)  & \textbf{0.0487}  & (0.0728)  & 0.2389  & (0.3425)  & 0.2117  & (0.4151)  \\ 
STEP  & 0.3333  & (0.3333)  & \textbf{0.0563}  & (0.0942)  & 0.2374  & (0.3425)  & 0.1913  & (0.4183)  \\ 
\hline & & & & & & & &\vspace{-0.3cm}\\ 
\multicolumn{3}{l}{{\emph{CIFAR$\,\,\rightarrow\,\,$CONNECT (0.1683)}}}\\ 
 & & & & & & & &  \vspace{-0.35cm}\\ 
EXPLIN  & \textbf{0.1643}  & (0.2050)  & 0.1650  & (0.2044)  & 0.3050  & (0.3699)  & 0.3334  & (0.3334)  \\ 
SIGMOID  & 0.1243  & (0.1271)  & \textbf{0.0939}  & (0.1532)  & 0.3923  & (0.4651)  & 0.3650  & (0.4709)  \\ 
\hline & & & & & & & &\vspace{-0.3cm}\\ 
\multicolumn{3}{l}{{\emph{CIFAR$\,\,\rightarrow\,\,$ADULT (0.1229)}}}\\ 
 & & & & & & & &  \vspace{-0.35cm}\\ 
LINEAR  & \textbf{0.0999}  & (0.1225)  & \textbf{0.0999}  & (0.1225)  & 0.1125  & (0.1121)  & 0.1125  & (0.1121)  \\ 
SIGMOID  & 0.0843  & (0.0883)  & \textbf{0.0713}  & (0.1107)  & 0.4431  & (0.4625)  & 0.3977  & (0.4817)  \\ 
\hline & & & & & & & &\vspace{-0.3cm}\\ 
\multicolumn{3}{l}{{\emph{RANDN$\,\,\rightarrow\,\,$RANDN (0.0944)}}}\\ 
 & & & & & & & &  \vspace{-0.35cm}\\ 
SIGMOID  & 0.0539  & (0.0731)  & \textbf{0.0537}  & (0.0732)  & 0.0571  & (0.0758)  & 0.0540  & (0.0855)  \\ 
\hline & & & & & & & &  \vspace{-0.85cm}\\
\end{tabular}
\end{sc}
\end{small}
\end{center}
\caption{The same experiments as in Table~\ref{tab:Hetero_online_centered_1} but with weight decay. We only performed experiments for two activations function per dataset as the grid search over 700 hyperparameter combinations is computational expensive and the performed experiments clearly show that an weight decay term is not an alternative to Hebbian-descent. 
} 
\label{tab:centered_weight_decay}
\end{table}
In most cases Hebbian-descent performs significantly better than the other methods and in the remaining cases its performance is close to the best.
Gradient descent either reaches a similar performance as Hebbian-descent or performs significantly worse.
We did not expect weight decay to allow gradient descent to become better than Hebbian-descent as it does not solve the problem of very small or zero values of the derivative of the activation function.
While Hebb's rule is always significantly worse than all other methods, the covariance rule reaches a very good performance in case of binary random input and output patterns as shown in Figure~\ref{fig:rand_adult_weightdecay} (b). 
When the output data is binary random the covariance rule reaches a similar performance as Hebbian-descent, but in all other cases the performance is significantly worse.
As the problem with correlation is still present for Hebb's rule and covariance rule, we did not expect that a weight decay leads to a better performance than Hebbian-descent.
In Figure~\ref{fig:heteroassociative_online_weight_decay_scatter} we compare the performance of all methods with and without weight decay.
\begin{figure}[t]
\begin{center}
\subfigure[]{
\includegraphics[scale=0.405, trim=21 10 37 10, clip]{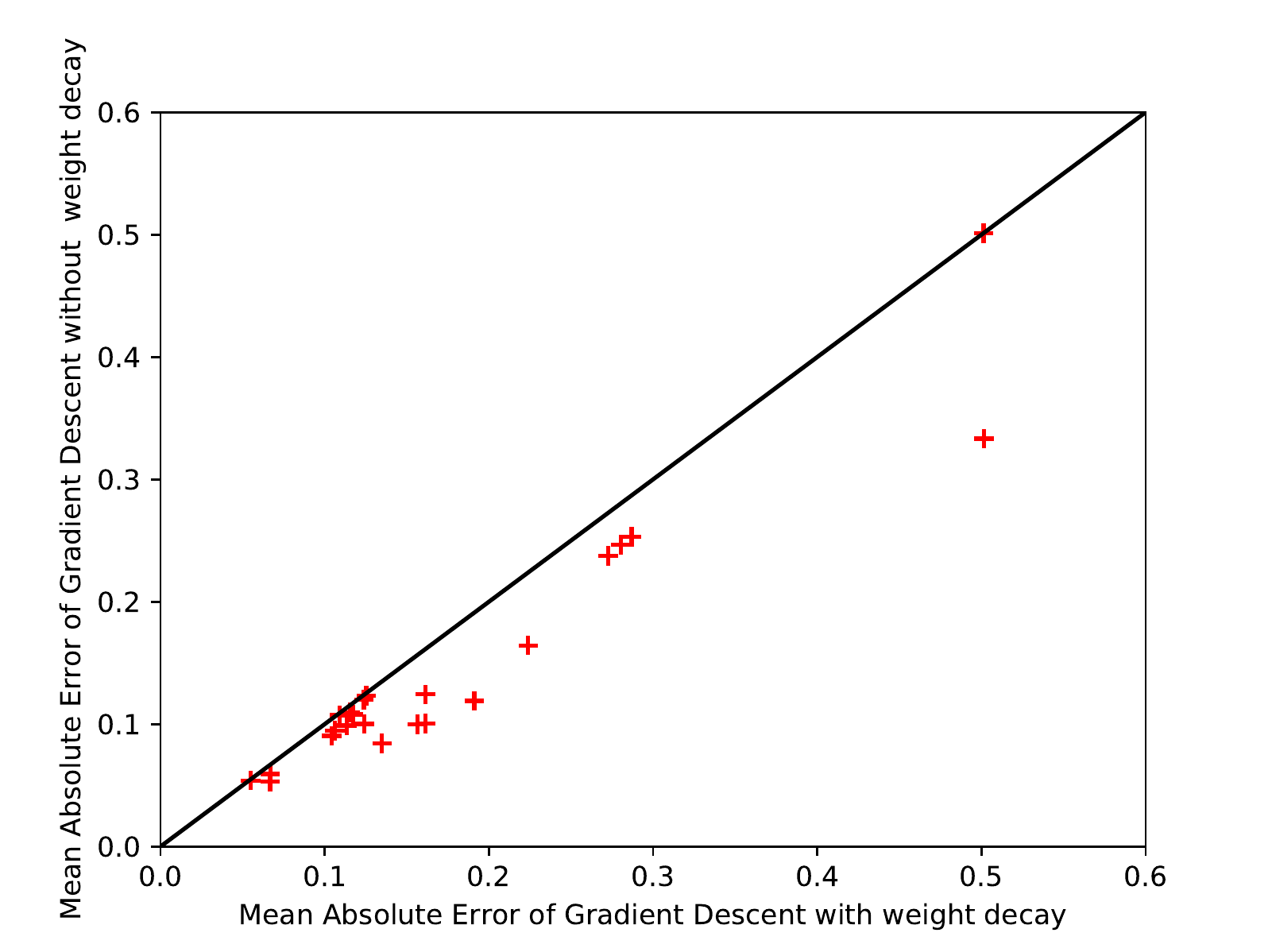}\label{fig:weight_decay_gradient_descent}}
\subfigure[]{
\includegraphics[scale=0.405, trim=21 10 37 10, clip]{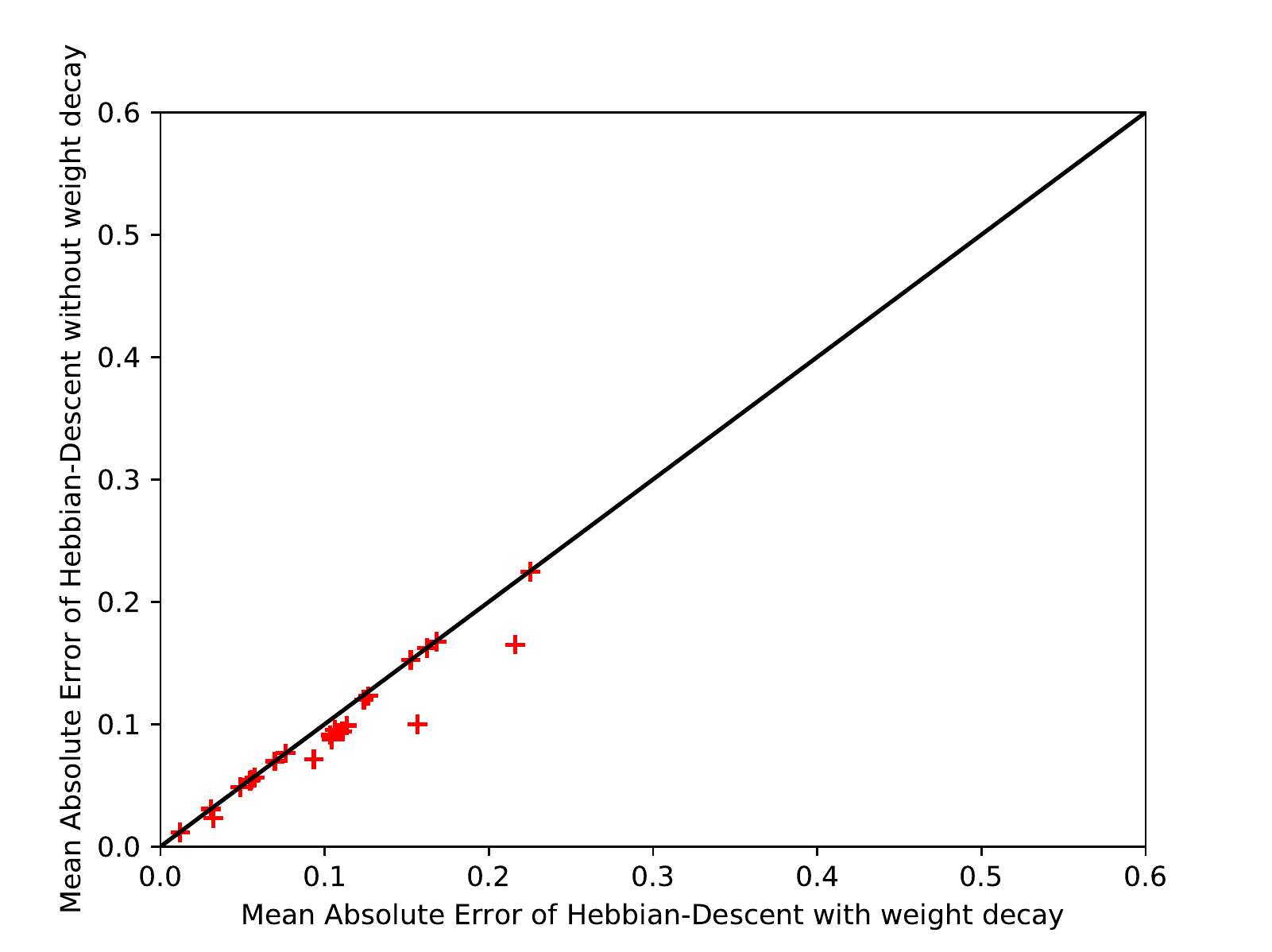}\label{fig:weight_decay_hebbian_descent}}
\subfigure[]{
\includegraphics[scale=0.405, trim=21 10 37 15, clip]{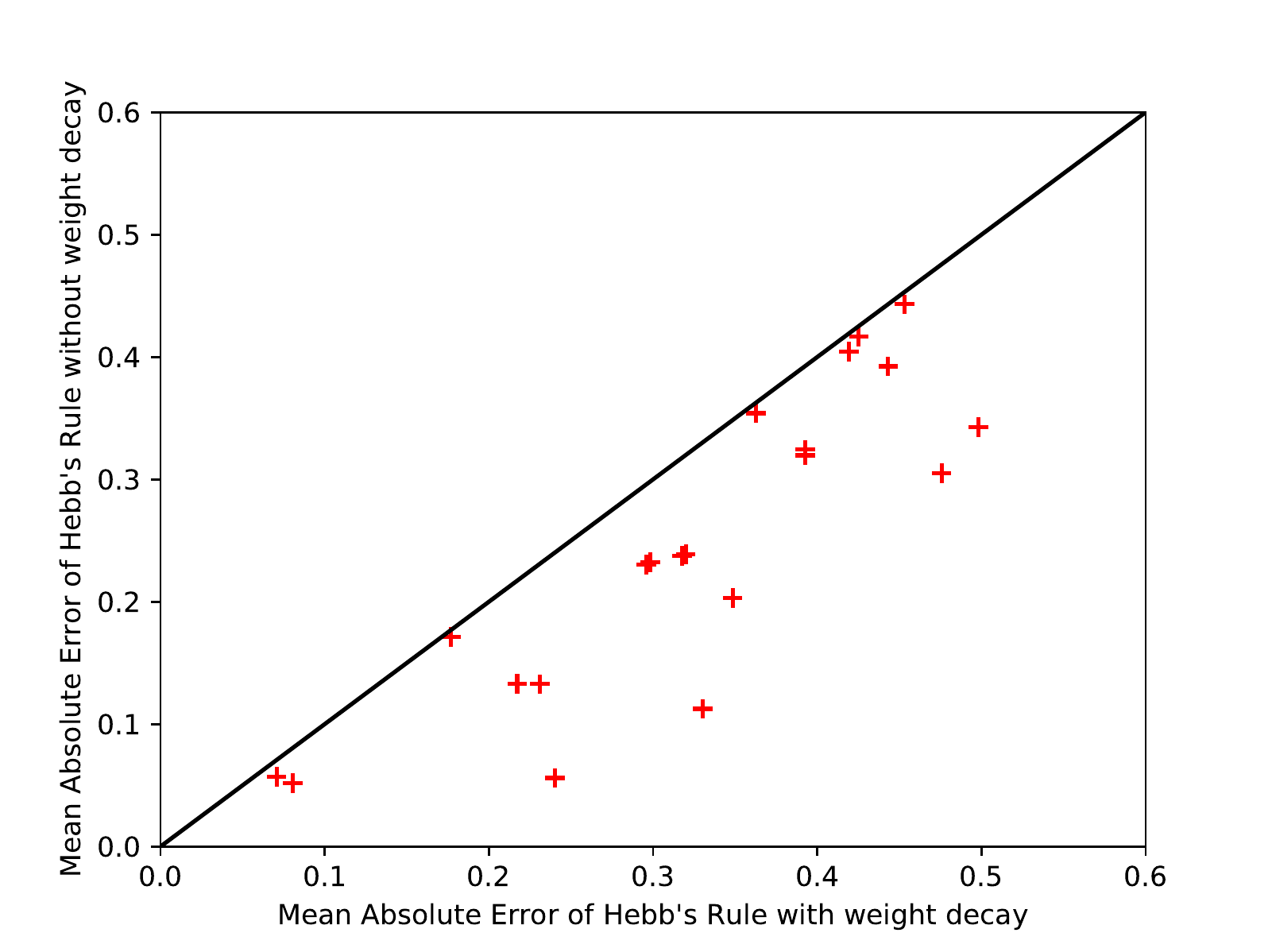}\label{fig:weight_decay_hebb_rule}}
\subfigure[]{
\includegraphics[scale=0.405, trim=21 10 37 15, clip]{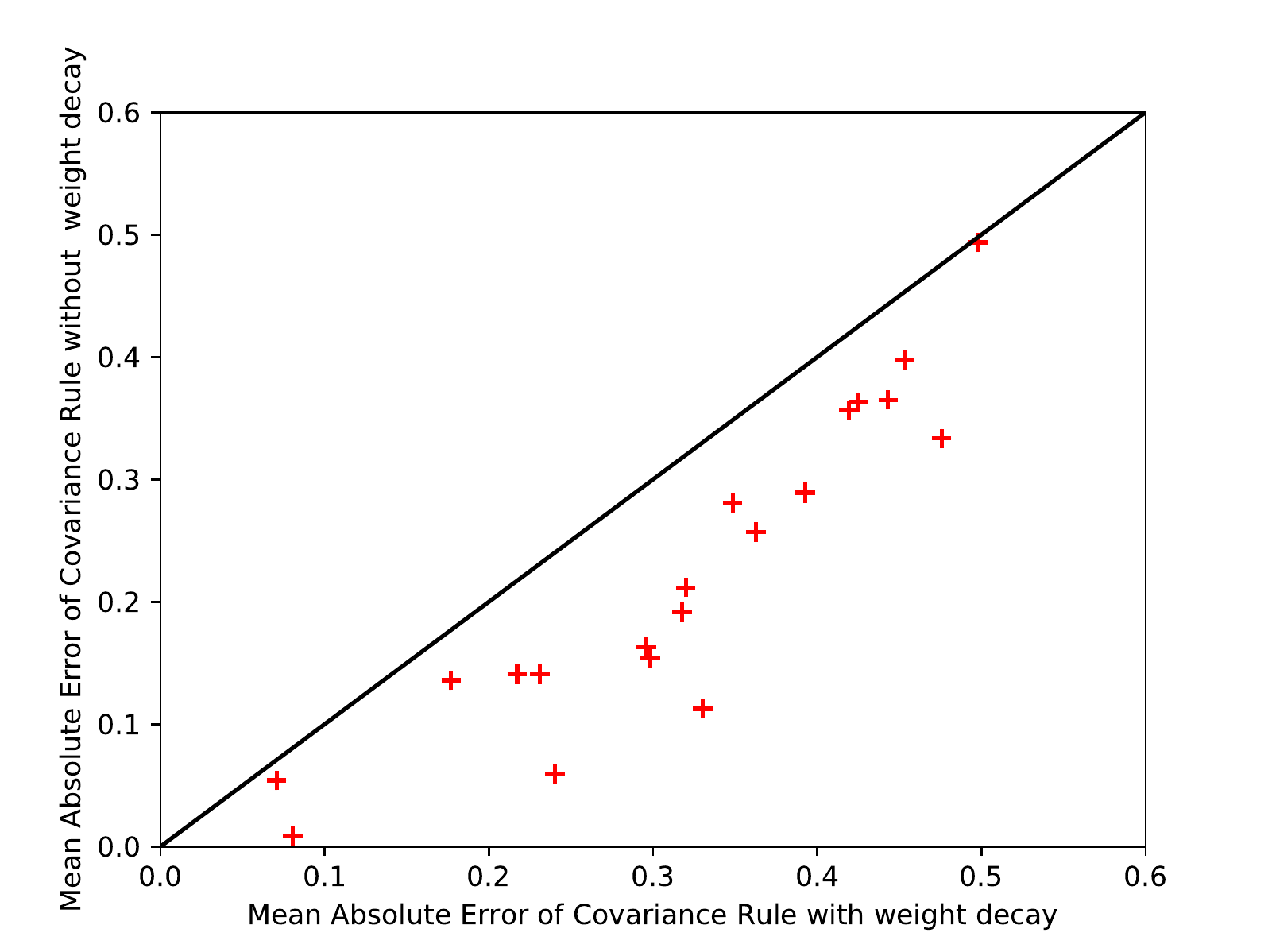}\label{fig:weight_decay_cov_rule}}
\caption{Comparison of online heteroassociation with and without weight decay when using (a) gradient descent, (b) Hebbian-descent, (c) Hebb's rule, and (d) covariance rule.
Each cross represents the mean absolute error of the last 20 patterns averaged over 10 trials for one experiment in Table~\ref{tab:centered_weight_decay}. }
\label{fig:heteroassociative_online_weight_decay_scatter}
\end{center}
\end{figure} 
All methods profit from using a weight decay as all points lie below the diagonal. 
It it obvious to see that Hebbian-descent has the best performance on average, closely followed by gradient descent, as the points are located much closer to the origin compared to the other methods.
Gradient descent profits in most cases from a weight decay as shown in Figure~\ref{fig:weight_decay_gradient_descent} although the improvement for Hebb's rule and the covariance rule is much higher as shown in Figure~\ref{fig:weight_decay_hebb_rule} and Figure~\ref{fig:weight_decay_cov_rule}. 
Hebbian-descent has a similar performance with and without weight decay in most cases as shown in Figure~\ref{fig:weight_decay_hebbian_descent}.
In fact in most of these cases the optimal weight decay is even zero or at least very small (not shown).
In case of unbounded activation functions such as linear, rectifier, and exponential linear units a weight decay term often has a slightly positive effect by keeping the weights from growing to large ('overshooting'). 
The use of a small weight decay term when using an unbounded activation function seems therefore a good idea in general.

\subsubsection{Hetero-Associative Multi-Epoch Learning}\label{sec:HD_multi_epoch_experiments}

In contrast to Hebb's rule and the covariance rule, gradient descent and Hebbian-descent take advantage of seeing training patterns several times. 
Repetitions help to stabilize memories, thinking of humans learning vocabularies for example.
This is a clear advantage as the methods can further improve the network's performance over time.
We used the same training setup as in the online-learning experiments to train networks using gradient descent and Hebbian-descent, but this time for 100 epochs instead of one. 
Or in other words we continued training for another 99 epochs to investigate how much the networks can improve their performance when each pattern is presented several times. 
The results for centered and uncentered networks are shown in Table~\ref{tab:Hetero_offline_1} as well as Table~\ref{tab:Hetero_offline_2} and Table~\ref{tab:Hetero_offline_3} in Appendix~\ref{appendix:hebbian_descent_additional_results}, where the number in each entry represents the MAE averaged over all patterns and trials followed by the corresponding standard deviation in parenthesis.
\begin{table}[htbp]
\setlength{\tabcolsep}{2pt}
\begin{center}
\begin{small}
\begin{sc}
\begin{tabular}{l@{\hskip 0.17in} r@{\hskip 0.02in} r@{\hskip 0.1in} r@{\hskip 0.01in} r@{\hskip 0.17in} r@{\hskip 0.01in} r@{\hskip 0.1in} r@{\hskip 0.01in} r@{\hskip 0.02in} }
\hline
\abovespace\belowspace
      & \multicolumn{4}{c}{\hskip -0.16in Gradient Descent} &   \multicolumn{4}{c}{\hskip -0.16in Hebbian-Descent} \\ 
  $\vect \phi$ & \multicolumn{2}{c}{\hskip -0.16in \footnotesize{Centered}} &   \multicolumn{2}{c}{\hskip -0.16in \footnotesize{Uncentered}} &  \multicolumn{2}{c}{\hskip -0.16in \footnotesize{Centered}} &  \multicolumn{2}{c}{\hskip -0.05in \footnotesize{Uncentered}} \\ 
\hline & & & & & & & &\vspace{-0.3cm}\\ 
\multicolumn{3}{l}{{\emph{RAND$\,\,\rightarrow\,\,$ADULT (0.1229)}}}\\ 
 & & & & & & & &  \vspace{-0.35cm}\\ 
Linear  &  \textbf{0.0000} & $\pm$ 0.0000  & 0.0237  & $\pm$ 0.0002  & \textbf{0.0000} & $\pm$ 0.0000 & 0.0237   & $\pm$ 0.0002  \\ 
ExpLin  &  \textbf{0.0000} & $\pm$ 0.0000  & 0.0264  & $\pm$ 0.0003  & \textbf{0.0000} & $\pm$ 0.0000 & 0.0240   & $\pm$ 0.0002  \\ 
Rectifier  &  0.0429  & $\pm$ 0.0006  & 0.0567  & $\pm$ 0.0013  & \textbf{0.0000} & $\pm$ 0.0000 & 0.0016   & $\pm$ 0.0002  \\ 
Sigmoid  &  0.0070  & $\pm$ 0.0001  & 0.0062  & $\pm$ 0.0003  & \textbf{0.0000} & $\pm$ 0.0000 & \textbf{0.0000} & $\pm$ 0.0000  \\ 
Step  &  0.4997  & $\pm$ 0.0022  & 0.4972  & $\pm$ 0.0194  & \textbf{0.0000} & $\pm$ 0.0000 & \textbf{0.0000} & $\pm$ 0.0000  \\ 
\hline & & & & & & & &\vspace{-0.3cm}\\ 
\multicolumn{3}{l}{{\emph{RANDN$\,\,\rightarrow\,\,$CONNECT (0.1683)}}}\\ 
 & & & & & & & &  \vspace{-0.35cm}\\ 
Linear  &  \textbf{0.0006} & $\pm$ 0.0000  & 0.1013  & $\pm$ 0.0007  & \textbf{0.0006} & $\pm$ 0.0000 & 0.1013   & $\pm$ 0.0007  \\ 
ExpLin  &  0.0007  & $\pm$ 0.0000  & 0.0999  & $\pm$ 0.0007  & \textbf{0.0006} & $\pm$ 0.0000 & 0.0997   & $\pm$ 0.0007  \\ 
Rectifier  &  0.0141  & $\pm$ 0.0010  & 0.1503  & $\pm$ 0.0178  & \textbf{0.0001} & $\pm$ 0.0000 & 0.0585   & $\pm$ 0.0005  \\ 
Sigmoid  &  0.0077  & $\pm$ 0.0001  & 0.0333  & $\pm$ 0.0015  & \textbf{0.0000} & $\pm$ 0.0000 & 0.0023   & $\pm$ 0.0015  \\ 
Step  &  0.4990  & $\pm$ 0.0010  & 0.4922  & $\pm$ 0.0381  & \textbf{0.0000} & $\pm$ 0.0000 & 0.0001   & $\pm$ 0.0001  \\ 
\hline & & & & & & & &\vspace{-0.3cm}\\ 
\multicolumn{3}{l}{{\emph{ADULT$\,\,\rightarrow\,\,$MNIST (0.1473)}}}\\ 
 & & & & & & & &  \vspace{-0.35cm}\\ 
Linear  &  \textbf{0.0944} & $\pm$ 0.0000  & 0.1038  & $\pm$ 0.0000  & \textbf{0.0944} & $\pm$ 0.0000 & 0.1038   & $\pm$ 0.0000  \\ 
ExpLin  &  0.0942  & $\pm$ 0.0000  & 0.1035  & $\pm$ 0.0000  & \textbf{0.0940} & $\pm$ 0.0000 & 0.1034   & $\pm$ 0.0000  \\ 
Rectifier  &  \textbf{0.0631} & $\pm$ 0.0003  & 0.0781  & $\pm$ 0.0005  & 0.0651   & $\pm$ 0.0000 & 0.0764   & $\pm$ 0.0000  \\ 
Sigmoid  &  \textbf{0.0426} & $\pm$ 0.0001  & 0.0551  & $\pm$ 0.0002  & 0.0581   & $\pm$ 0.0000 & 0.0701   & $\pm$ 0.0000  \\ 
Step  &  0.4995  & $\pm$ 0.0013  & 0.5008  & $\pm$ 0.0066  & \textbf{0.0744} & $\pm$ 0.0003 & 0.0881   & $\pm$ 0.0007  \\ 
\hline & & & & & & & &\vspace{-0.3cm}\\ 
\multicolumn{3}{l}{{\emph{CONNECT$\,\,\rightarrow\,\,$CIFAR (0.167)}}}\\ 
 & & & & & & & &  \vspace{-0.35cm}\\ 
Linear  &  \textbf{0.1125} & $\pm$ 0.0000  & 0.1300  & $\pm$ 0.0000  & \textbf{0.1125} & $\pm$ 0.0000 & 0.1300   & $\pm$ 0.0000  \\ 
ExpLin  &  \textbf{0.1125} & $\pm$ 0.0000  & 0.1300  & $\pm$ 0.0000  & \textbf{0.1125} & $\pm$ 0.0000 & 0.1300   & $\pm$ 0.0000  \\ 
Rectifier  &  \textbf{0.1123} & $\pm$ 0.0000  & 0.1601  & $\pm$ 0.0040  & 0.1124   & $\pm$ 0.0000 & 0.1300   & $\pm$ 0.0000  \\ 
Sigmoid  &  \textbf{0.1084} & $\pm$ 0.0000  & 0.1284  & $\pm$ 0.0000  & 0.1109   & $\pm$ 0.0000 & 0.1290   & $\pm$ 0.0000  \\ 
\hline & & & & & & & &\vspace{-0.3cm}\\ 
\multicolumn{3}{l}{{\emph{MNIST$\,\,\rightarrow\,\,$CONNECT (0.1683)}}}\\ 
 & & & & & & & &  \vspace{-0.35cm}\\ 
Linear  &  \textbf{0.0003} & $\pm$ 0.0000  & 0.0430  & $\pm$ 0.0006  & \textbf{0.0003} & $\pm$ 0.0000 & 0.0430   & $\pm$ 0.0006  \\ 
ExpLin  &  \textbf{0.0003} & $\pm$ 0.0000  & 0.0424  & $\pm$ 0.0006  & \textbf{0.0003} & $\pm$ 0.0000 & 0.0424   & $\pm$ 0.0006  \\ 
Rectifier  &  0.0359  & $\pm$ 0.0017  & 0.0526  & $\pm$ 0.0026  & \textbf{0.0000} & $\pm$ 0.0000 & 0.0145   & $\pm$ 0.0005  \\ 
Sigmoid  &  0.0151  & $\pm$ 0.0003  & 0.0205  & $\pm$ 0.0003  & \textbf{0.0000} & $\pm$ 0.0000 & \textbf{0.0000} & $\pm$ 0.0000  \\ 
Step  &  0.4990  & $\pm$ 0.0050  & 0.5142  & $\pm$ 0.0145  & \textbf{0.0000} & $\pm$ 0.0000 & \textbf{0.0000} & $\pm$ 0.0000  \\ 
\hline & & & & & & & &\vspace{-0.3cm}\\ 
\multicolumn{3}{l}{{\emph{CIFAR$\,\,\rightarrow\,\,$CONNECT (0.1683)}}}\\ 
 & & & & & & & &  \vspace{-0.35cm}\\ 
Linear  &  \textbf{0.0201} & $\pm$ 0.0002  & 0.1281  & $\pm$ 0.0009  & \textbf{0.0201} & $\pm$ 0.0002 & 0.1281   & $\pm$ 0.0009  \\ 
ExpLin  &  \textbf{0.0197} & $\pm$ 0.0002  & 0.1271  & $\pm$ 0.0009  & 0.0198   & $\pm$ 0.0002 & 0.1266   & $\pm$ 0.0009  \\ 
Rectifier  &  0.0319  & $\pm$ 0.0010  & 0.1326  & $\pm$ 0.0102  & \textbf{0.0074} & $\pm$ 0.0001 & 0.0885   & $\pm$ 0.0007  \\ 
Sigmoid  &  0.0158  & $\pm$ 0.0006  & 0.0642  & $\pm$ 0.0034  & \textbf{0.0000} & $\pm$ 0.0000 & 0.0311   & $\pm$ 0.0036  \\ 
Step  &  0.4993  & $\pm$ 0.0032  & 0.5165  & $\pm$ 0.0138  & \textbf{0.0000} & $\pm$ 0.0000 & 0.0313   & $\pm$ 0.0042  \\ 
\hline & & & & & & & &\vspace{-0.3cm}\\ 
\end{tabular}
\end{sc}
\end{small}
\end{center}
\caption{Performance of hetero-association for several epochs of training with and without centering for various activation functions and datasets. 
The task is to associate 100 patterns of one dataset with 100 patterns of another dataset and training is performed for 100 epochs again with a batch size of one.
Each experiment was repeated 10 times and the optimal learning rate was determined via grid search for each method separately, so that the performance over all patterns is best.
The first number in each entry shows the MAE for all patterns with the corresponding standard deviation over the 10 trials. See Table~\ref{tab:Hetero_offline_2} and \ref{tab:Hetero_offline_3} in Appendix~\ref{appendix:hebbian_descent_additional_results} for more results.
} 
\label{tab:Hetero_offline_1}
\end{table}
In all cases, whether centered or not, both methods improve significantly compared to when training only a single epoch (See Table~\ref{tab:Hetero_online_centered_1} and Table~\ref{tab:Hetero_online_uncentered_1} for comparison).
Furthermore, both methods reach better results when centering is used and networks trained with centered Hebbian-descent reach more often even a zero error value.
In most cases Hebbian-descent reaches better values than gradient descent especially when the association can be learned almost perfectly as in the case of \emph{RAND}$\,\,\rightarrow\,\,$\emph{ADULT}, \emph{RANDN}$\,\,\rightarrow\,\,$\emph{CONNECT}, or \emph{MNIST}$\,\,\rightarrow\,\,$\emph{CONNECT} for example. 
Both algorithms have a very similar performance for exponential linear units on all datasets, but the performance of rectifier, sigmoid, or step function is more often better when using Hebbian-descent. 
In case of the rectifier as activation function gradient descent gets often trapped in rather bad local optima, which is a direct consequence of the rectifier having a zero derivative for negative input values. 
As the derivative of the step function is constantly zero only Hebbian-descent can deal with the step-function as nonlinearity.
To emphasize the better performance of Hebbian-descent compared to gradient descent we plotted all results of both algorithms against each other either with centering shown in Figure~\ref{fig:heteroassociative_online_multi_epoch_centered} or without centering shown in Figure~\ref{fig:heteroassociative_online_multi_epoch_uncentered}.
\begin{figure}[t]
\begin{center}
\subfigure[]{
\includegraphics[scale=0.405, trim=21 10 37 10, clip]{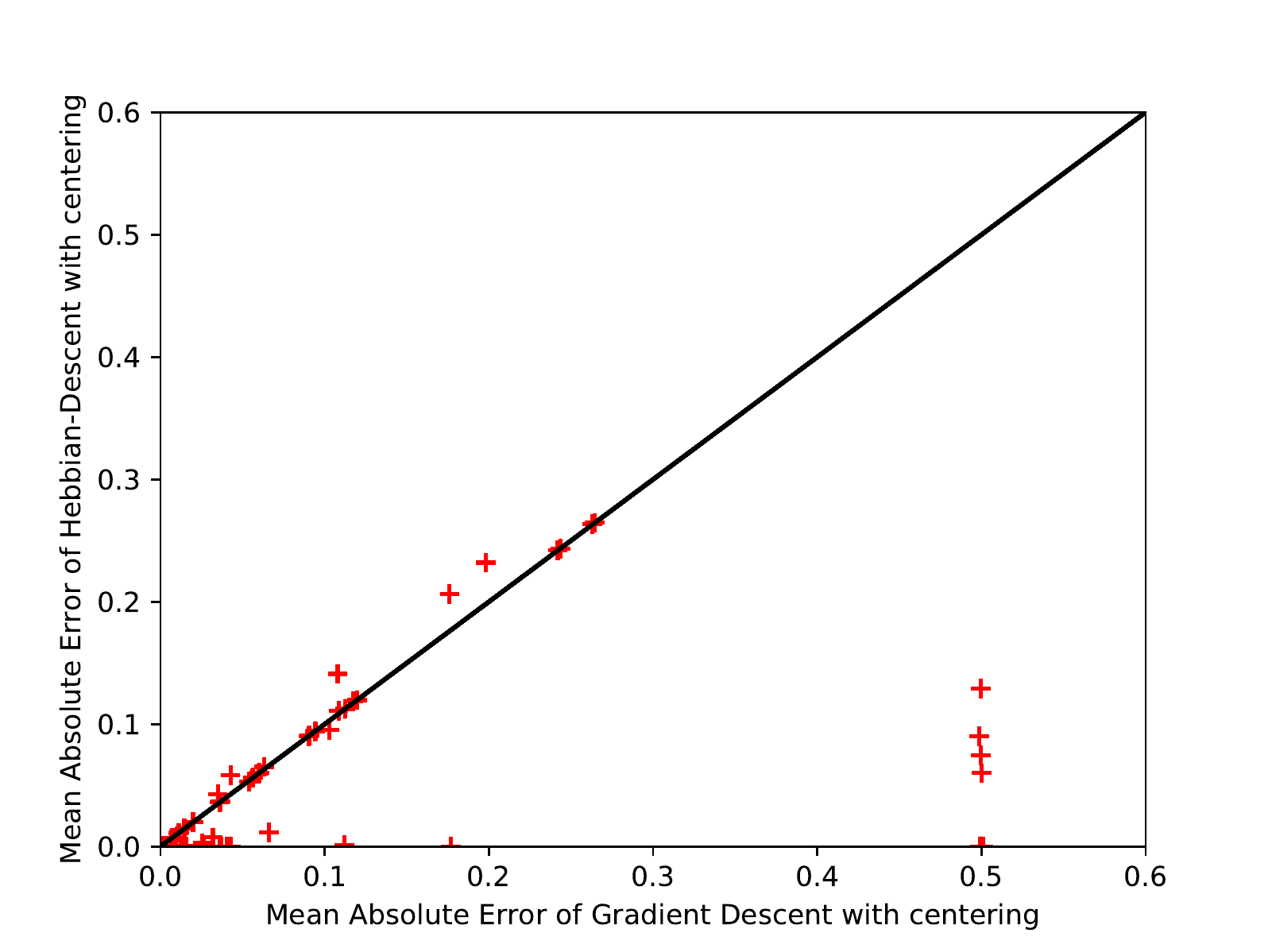}\label{fig:heteroassociative_online_multi_epoch_centered}}
\subfigure[]{
\includegraphics[scale=0.405, trim=21 10 37 10, clip]{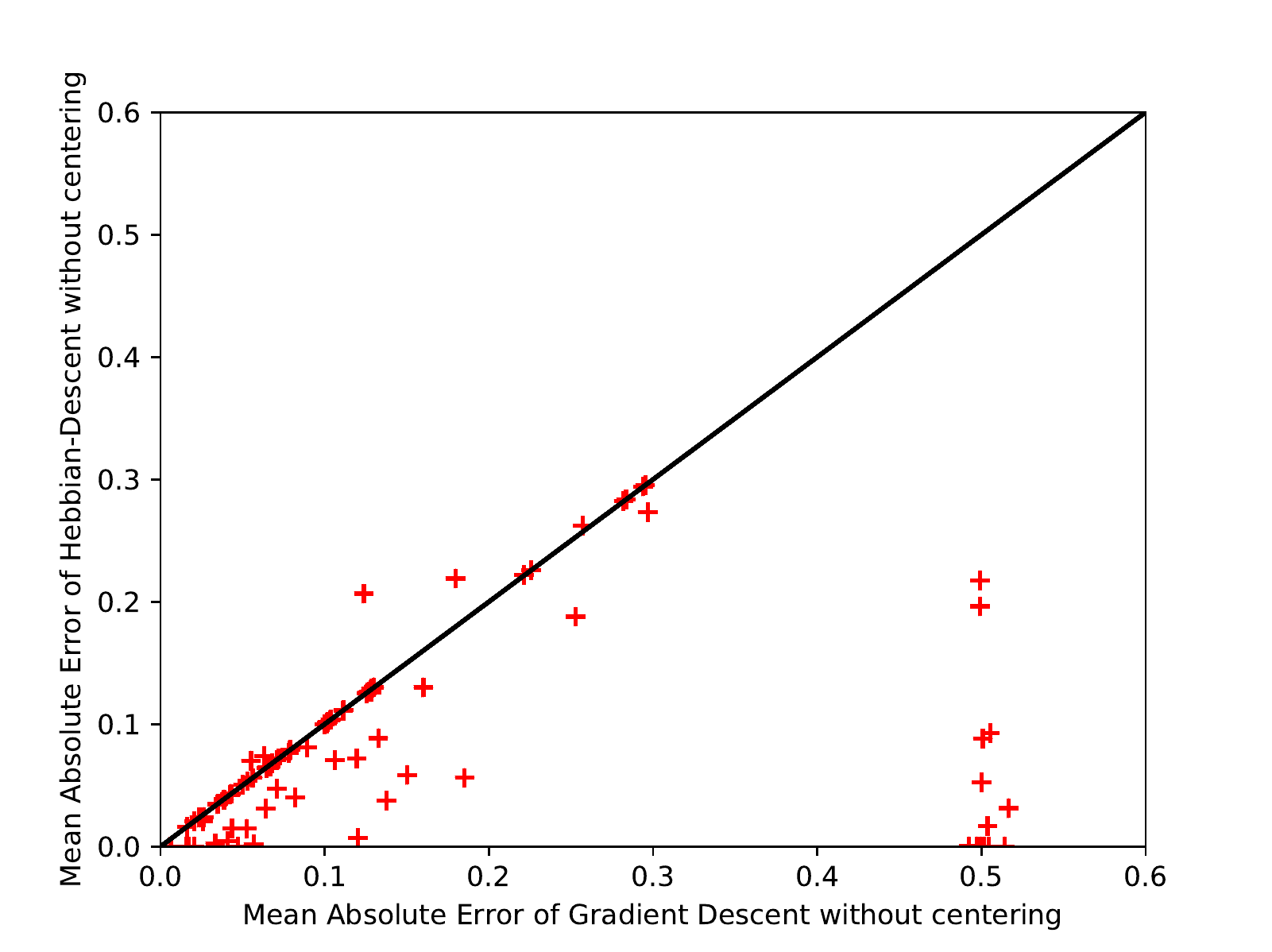}\label{fig:heteroassociative_online_multi_epoch_uncentered}}
\caption{Comparison of gradient descent and Hebbian-descent in multi epoch heteroassociation (a) with centering, and (b) without centering.
Each cross represents the mean absolute error on all patterns averaged over 10 trials for one experiment in Table~\ref{tab:Hetero_offline_1}. }
\label{fig:heteroassociative_online_multi_epoch_scatter}
\end{center}
\end{figure} 
As the points in both plots lie more often below the diagonal there is a tendency of Hebbian-descent to perform better than gradient descent, which is even more prominent without centering.

To summarize, both methods reach similar results, but only Hebbian-descent reaches good results also in online learning, \emph{i.e.}\ seeing every pattern only once. 
Similar results (not shown) are also achieved when mini-batch learning (\emph{e.g.}\ 10) is used instead of a batch size of one.

\subsection{Classification}

Another learning scenario that is of interest especially in machine learning, is that of classification, \emph{i.e.}\ each input pattern is associated with a corresponding label.
As there are usually many fewer labels than input patterns a lot of patterns get associated with the same label, which makes the difference to the previous experiments where two output patterns might have been similar but never equal.
The most common representation when performing classification with neural networks is that of `one-hot-encoding', \emph{i.e.}\ each output neuron corresponds to one of the dataset's classes. 
The output patterns then have a value of one for the neuron that corresponds to the class of the current input pattern and zero everywhere else.
Depending on the number of labels this usually leads to a network with much lower output than input dimensionality.
During training the network then learns to predict the label for a given input pattern, which is considered to be correct if the most active neuron is the one that represents the true class of the current input pattern and is incorrect otherwise.

We investigated the classification performance of networks trained on the four real-world datasets using the four different update rules. Various activation functions were used and each network was trained for 100 epochs, with a batch size of 100 and the mean squared error loss. 
As we are interested in the classification performance on the test set we investigate the multi-epoch mini-batch learning performance and discuss the online learning performance only briefly at the end of this section.
\begin{table}[htbp]
\setlength{\tabcolsep}{2pt}
\begin{center}
\begin{small}
\begin{sc}
\begin{tabular}{l@{\hskip 0.2in} r@{\hskip 0.02in} r@{\hskip 0.14in} r@{\hskip 0.02in} r@{\hskip 0.14in} r@{\hskip 0.02in} r@{\hskip 0.14in} r@{\hskip 0.02in} r }
\hline
\abovespace\belowspace
  $\vect \phi$ & \multicolumn{2}{c}{\hskip -0.16in Grad. Descent} &   \multicolumn{2}{c}{\hskip -0.16in Hebb. Descent} &  \multicolumn{2}{c}{\hskip -0.16in Hebb rule} &  \multicolumn{2}{c}{\hskip -0.16in Cov. Rule} \\ 
\hline & & & & & & & &\vspace{-0.3cm}\\ 
\multicolumn{3}{l}{{\emph{ADULT (0.2399)}}}\\ 
 & & & & & & & &  \vspace{-0.35cm}\\ 
Linear  & \textbf{0.1560}  & $\pm$0.0004  & \textbf{0.1560}  & $\pm$0.0004  & 0.2870  & $\pm$0.0002  & 0.2870  & $\pm$0.0002  \\ 
Rectifier  & 0.1546  & $\pm$0.0003  & \textbf{0.1544}  & $\pm$0.0005  & 0.2854  & $\pm$0.0033  & 0.2854  & $\pm$0.0033  \\ 
ExpLin  & \textbf{0.1558}  & $\pm$0.0004  & 0.1560  & $\pm$0.0004  & 0.2870  & $\pm$0.0002  & 0.2870  & $\pm$0.0002  \\ 
Sigmoid  & 0.1546  & $\pm$0.0002  & \textbf{0.1543}  & $\pm$0.0002  & 0.2870  & $\pm$0.0002  & 0.2870  & $\pm$0.0002  \\ 
Softmax  & 0.1545  & $\pm$0.0004  & \textbf{0.1544}  & $\pm$0.0004  & 0.2870  & $\pm$0.0002  & 0.2870  & $\pm$0.0002  \\ 
Step  & 0.3528  & $\pm$0.0377  & \textbf{0.1728}  & $\pm$0.0024  & 0.2664  & $\pm$0.0181  & 0.2664  & $\pm$0.0181  \\ 
\hline & & & & & & & &\vspace{-0.3cm}\\ 
\multicolumn{3}{l}{{\emph{CONNECT (0.3409)}}}\\ 
 & & & & & & & &  \vspace{-0.35cm}\\ 
Linear  & \textbf{0.2454}  & $\pm$0.0001  & \textbf{0.2454}  & $\pm$0.0001  & 0.4354  & $\pm$0.0004  & 0.4354  & $\pm$0.0004  \\ 
Rectifier  & \textbf{0.2434}  & $\pm$0.0001  & 0.2443  & $\pm$0.0000  & 0.4354  & $\pm$0.0004  & 0.4354  & $\pm$0.0004  \\ 
ExpLin  & \textbf{0.2450}  & $\pm$0.0001  & 0.2452  & $\pm$0.0001  & 0.4354  & $\pm$0.0004  & 0.4354  & $\pm$0.0004  \\ 
Sigmoid  & \textbf{0.2406}  & $\pm$0.0000  & 0.2432  & $\pm$0.0000  & 0.4353  & $\pm$0.0004  & 0.4353  & $\pm$0.0004  \\ 
Softmax  & \textbf{0.2410}  & $\pm$0.0001  & 0.2435  & $\pm$0.0000  & 0.4354  & $\pm$0.0004  & 0.4354  & $\pm$0.0004  \\ 
Step  & 0.7347  & $\pm$0.0147  & \textbf{0.3092}  & $\pm$0.0214  & 0.5541  & $\pm$0.0058  & 0.5541  & $\pm$0.0058  \\ 
\hline & & & & & & & &\vspace{-0.3cm}\\ 
\multicolumn{3}{l}{{\emph{MNIST (0.8865)}}}\\ 
 & & & & & & & &  \vspace{-0.35cm}\\ 
Linear  & \textbf{0.1410}  & $\pm$0.0011  & \textbf{0.1410}  & $\pm$0.0011  & 0.2559  & $\pm$0.0000  & 0.2559  & $\pm$0.0000  \\ 
Rectifier  & \textbf{0.0812}  & $\pm$0.0004  & 0.0836  & $\pm$0.0008  & 0.2559  & $\pm$0.0000  & 0.2559  & $\pm$0.0000  \\ 
ExpLin  & \textbf{0.1361}  & $\pm$0.0012  & 0.1395  & $\pm$0.0011  & 0.2559  & $\pm$0.0003  & 0.2559  & $\pm$0.0003  \\ 
Sigmoid  & \textbf{0.0764}  & $\pm$0.0002  & 0.0800  & $\pm$0.0002  & 0.2564  & $\pm$0.0029  & 0.2564  & $\pm$0.0029  \\ 
Softmax  & \textbf{0.0689}  & $\pm$0.0002  & 0.0742  & $\pm$0.0002  & 0.2559  & $\pm$0.0000  & 0.2559  & $\pm$0.0000  \\ 
Step  & 0.9182  & $\pm$0.0326  & \textbf{0.1746}  & $\pm$0.0025  & 0.7267  & $\pm$0.0013  & 0.7267  & $\pm$0.0013  \\ 
\hline & & & & & & & &\vspace{-0.3cm}\\ 
\multicolumn{3}{l}{{\emph{CIFAR (0.9)}}}\\ 
 & & & & & & & &  \vspace{-0.35cm}\\ 
Linear  & \textbf{0.7401}  & $\pm$0.0015  & \textbf{0.7401}  & $\pm$0.0015  & 0.7663  & $\pm$0.0003  & 0.7663  & $\pm$0.0003  \\ 
Rectifier  & \textbf{0.7047}  & $\pm$0.0022  & 0.7216  & $\pm$0.0024  & 0.7694  & $\pm$0.0009  & 0.7694  & $\pm$0.0009  \\ 
ExpLin  & \textbf{0.7254}  & $\pm$0.0023  & 0.7258  & $\pm$0.0022  & 0.7694  & $\pm$0.0009  & 0.7694  & $\pm$0.0009  \\ 
Sigmoid  & \textbf{0.6842}  & $\pm$0.0021  & 0.6949  & $\pm$0.0011  & 0.7695  & $\pm$0.0015  & 0.7695  & $\pm$0.0015  \\ 
Softmax  & \textbf{0.6803}  & $\pm$0.0018  & 0.6959  & $\pm$0.0020  & 0.7694  & $\pm$0.0009  & 0.7694  & $\pm$0.0009  \\ 
Step  & 0.8964  & $\pm$0.0178  & \textbf{0.8186}  & $\pm$0.0061  & 0.8561  & $\pm$0.0043  & 0.8561  & $\pm$0.0043  \\ 
\hline & & & & & & & &\vspace{-0.3cm}\\ 
\end{tabular}
\end{sc}
\end{small}
\end{center}
\caption{Classification error when performing mini-batch learning with the four different training methods. 
The networks were trained for 100 epochs with centering, various activation functions, and datasets. 
Each experiment was repeated 10 times and the optimal learning rate was determined for each method separately, so that the performance over all test patterns was best.
The first number in each entry shows the average classification error followed by the corresponding standard deviation over 10 trials. 
The best result is indicated in bold.
The baseline given for each dataset in parenthesis behind its name represent the classification error when the most frequent label in the test data is always returned.
Notice, that the first label in the \emph{ADULT} dataset is roughly 3 times more frequent than the second label for example.
} 
\label{tab:label_learning_centered}
\end{table}
The results for centered networks are shown in Table~\ref{tab:label_learning_centered} where the number in each entry is the average classification error on the test data followed by the corresponding standard deviation over the trials. 
The baseline, given in parenthesis behind the datasets' names, is the classification error when the most frequent label in the test data would have always been returned by the network.
As expected Hebb's rule and the covariance rule perform significantly worse than gradient descent and Hebbian-descent as they do not profit from seeing patterns several times.
Gradient descent often reaches slightly better performances than Hebbian-descent except when using the step function, of course. 
We checked that the networks did not overfit to the training data, such that the better performance of gradient descent it is not just due to a better generalization but due to a better optimum that has been found.

In contrast to the previous experiments a benefit of centering in case of classification is not noticeable.
This is somewhat surprising as in case of a softmax activation function Hebbian-descent corresponds to gradient descent with cross-entropy loss, which is usually considered to be the better loss function for softmax units.
Notice, however, that the cross entropy, and also Hebbian-descent in general, usually converge initially faster and are more robust with respect to the choice of the learning rate as shown for online learning with sigmoid units in Figure~\ref{fig:learning_rate_compare_lin_sig}(b).
%It is thus more likely to find a good solution with Hebbian-descent when the optimal learning rate is just guessed instead of performing a grid search.
Furthermore, things change when training deep neural networks where in our experience the cross entropy loss, and also the Hebbian-descent loss in general, perform better than gradient descent with mean squared error loss.
As the learning rate usually has to be magnitudes smaller in deep networks compared to single layer networks, the robustness of Hebbian-descent becomes more important.
Furthermore, as discussed in the introduction the derivative of the activation function in the output layer is part of the vanishing gradient problem in deep neural networks since the error term gets back-propagated through the entire network.
It seems that a non-vanishing error term is more important than a similar performance on all patterns, such that the Hebbian-descent losses (see Table~\ref{tab:list_HD_loss_functions}) perform better than the mean squared error loss in our experience.
However, for single layer networks where the vanishing gradient problem is not present one should consider using the mean squared error loss instead of the cross entropy.

We also performed the same experiments as in Table~\ref{tab:label_learning_centered} but without centering. 
The comparison of gradient descent with Hebbian-desdcent is visualized for centered networks in Figure~\ref{fig:label_learning_centered} and for uncentered networks in Figure~\ref{fig:label_learning_uncentered}.
\begin{figure}[t]
\begin{center}
\subfigure[]{
\includegraphics[scale=0.405, trim=21 10 37 10, clip]{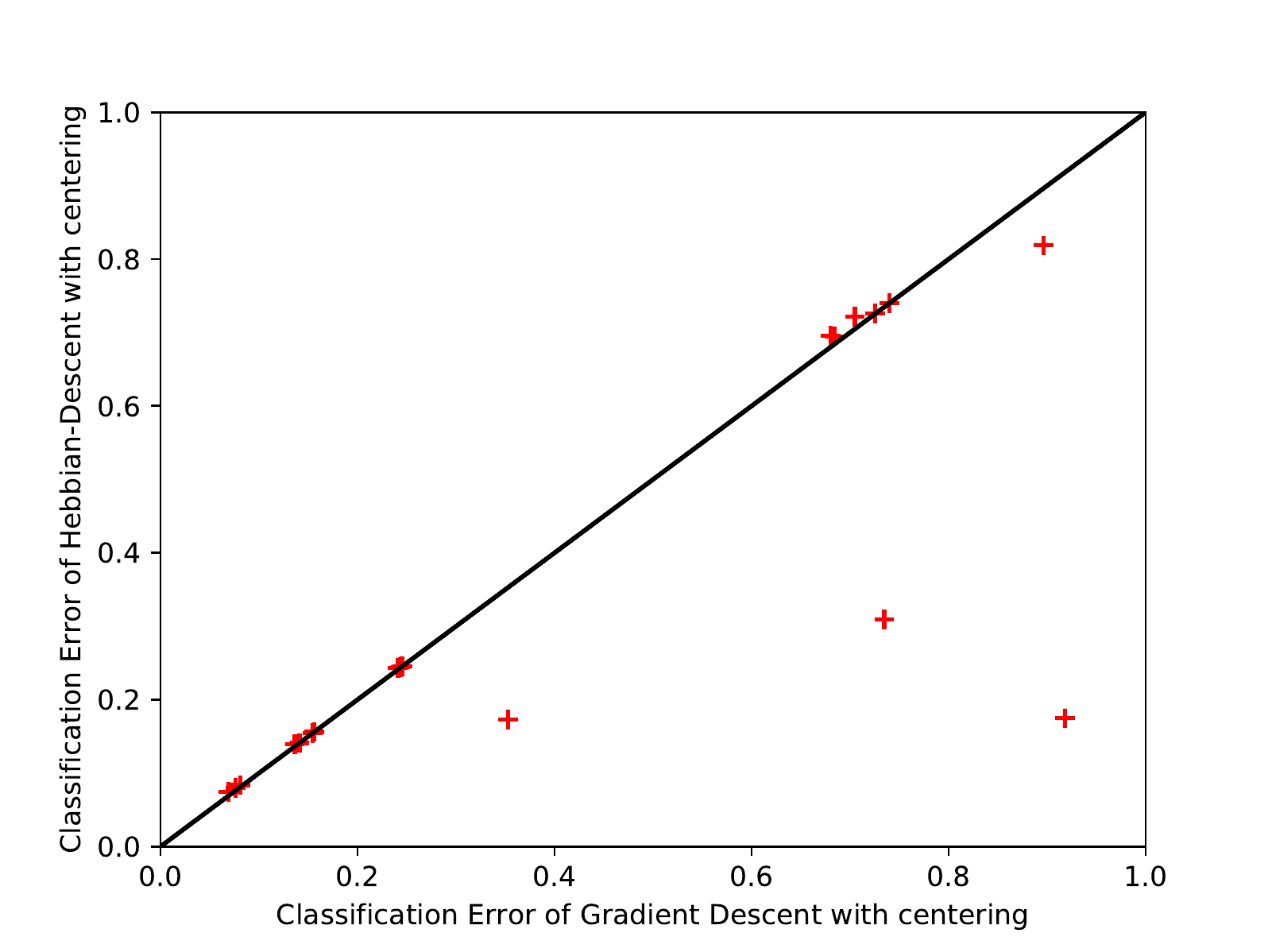}\label{fig:label_learning_centered}}
\subfigure[]{
\includegraphics[scale=0.405, trim=21 10 37 10, clip]{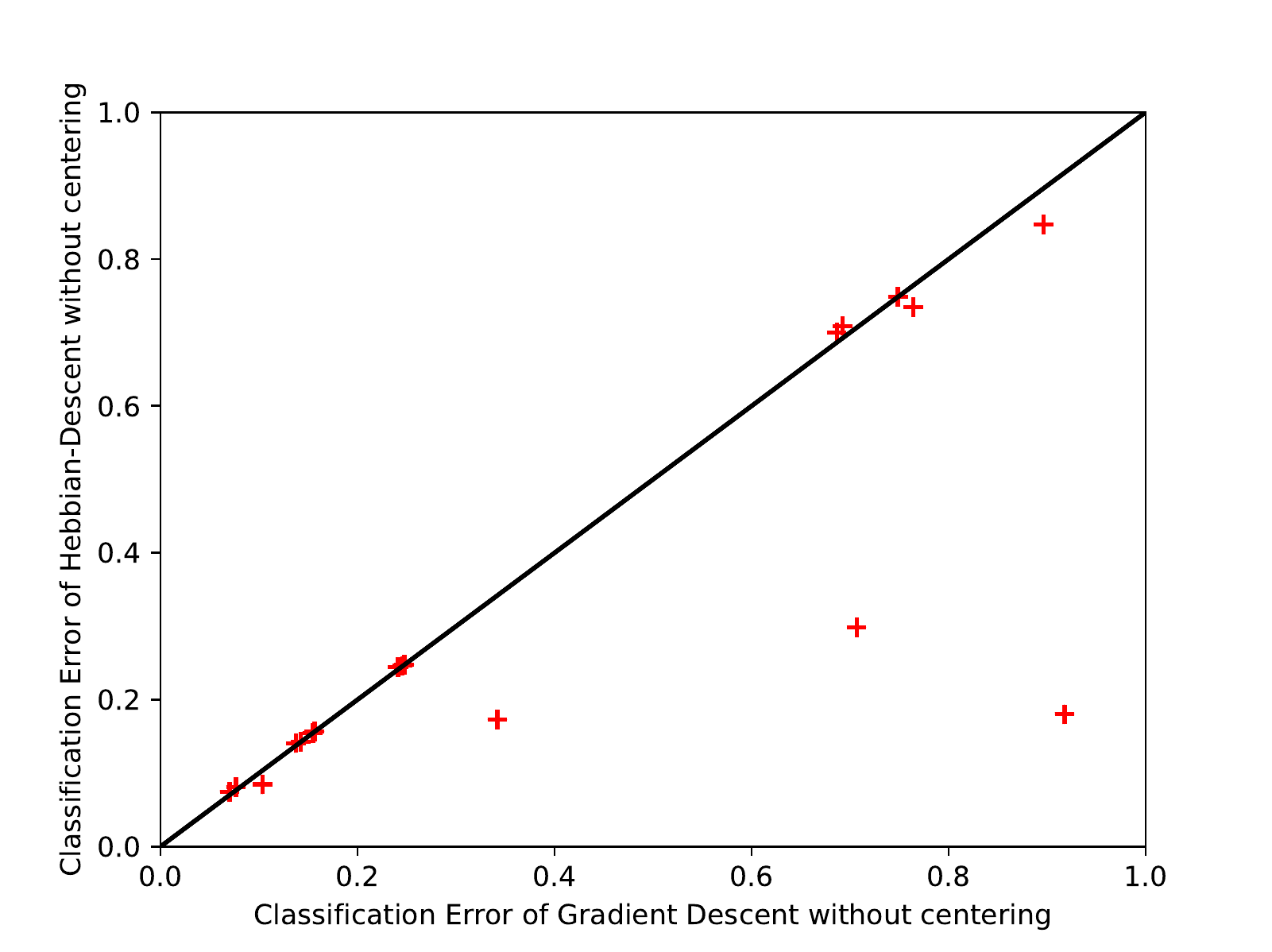}\label{fig:label_learning_uncentered}}
\caption{Comparison of gradient descent and Hebbian-descent in classification (a)~with centering, and (b) without centering.
Each cross represents the average classification error on all patterns averaged over 10 trials for one experiment in Table~\ref{tab:label_learning_centered}. }
\label{fig:label_learning_scatter}
\end{center}
\end{figure} 
Except for the results for the step function all points in both plots lie roughly on the diagonal, illustrating that, independent of whether centering is used or not, there is no significant difference between Hebbian-descent and gradient descent. 
Furthermore, as both plots are very similar there is no significant difference between centered and uncentered networks (for details compare Table~\ref{tab:label_learning_centered} with Table~\ref{tab:label_learning_uncentered} in Appendix~\ref{appendix:hebbian_descent_additional_results}). 

To show that the better performance of Hebbian-descent on more recently seen patterns can also be observed when performing classification we trained a centered network on the MNIST dataset with sigmoid units for one epoch with a batch size of one.
The learning rate was chosen again, so that the performance on the last 20 patterns was best.
The mean absolute error and the classification error for the last 20 patterns are shown in Figure~\ref{fig:online_label_learning_MNIST} (a) and (b), respectively. 
\begin{figure}[t]
\begin{center}
\subfigure[]{
\includegraphics[scale=0.4075, trim=21 0 27 0, clip]{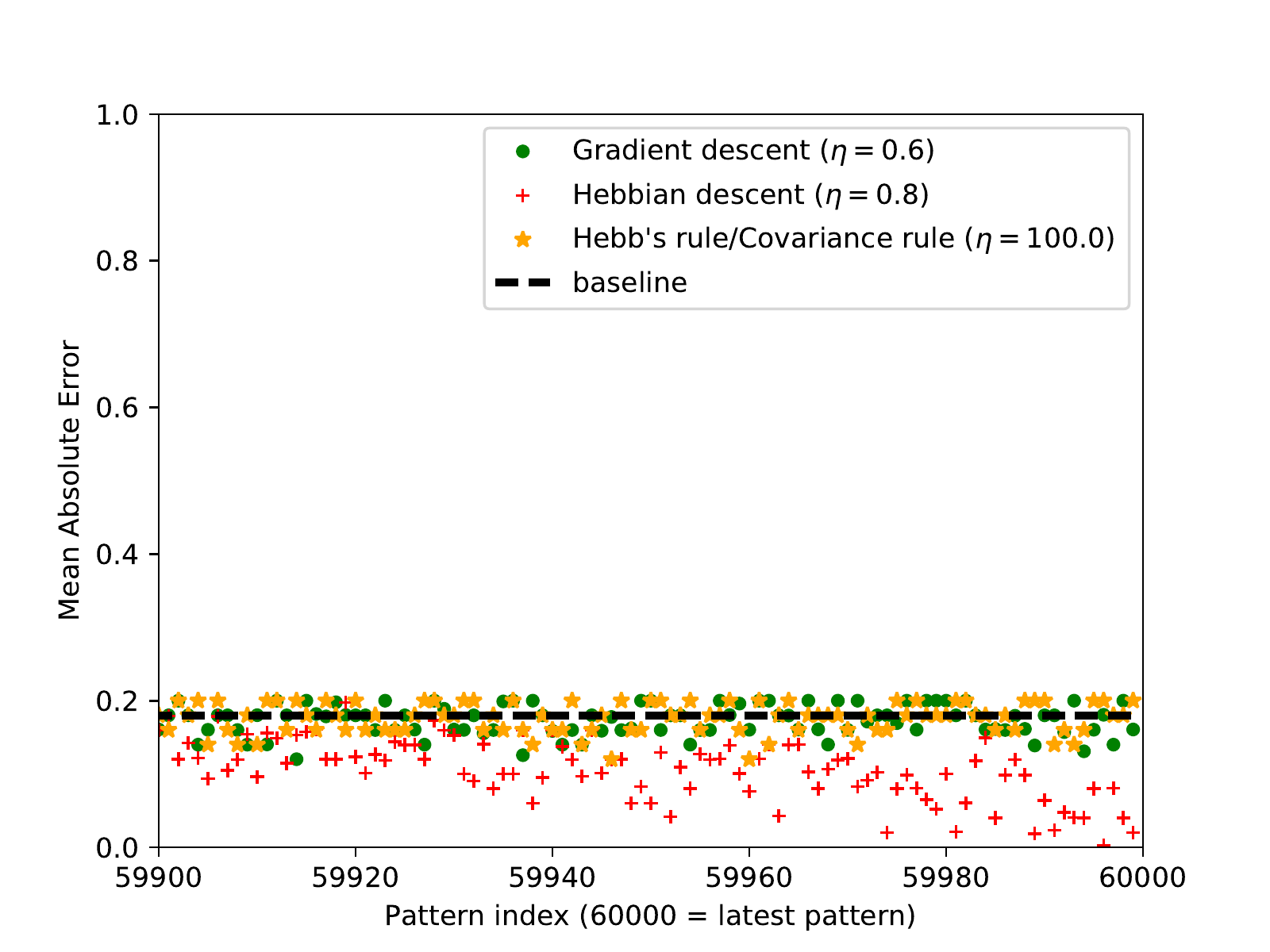}}
\subfigure[]{
\includegraphics[scale=0.4075, trim=21 0 27 0, clip]{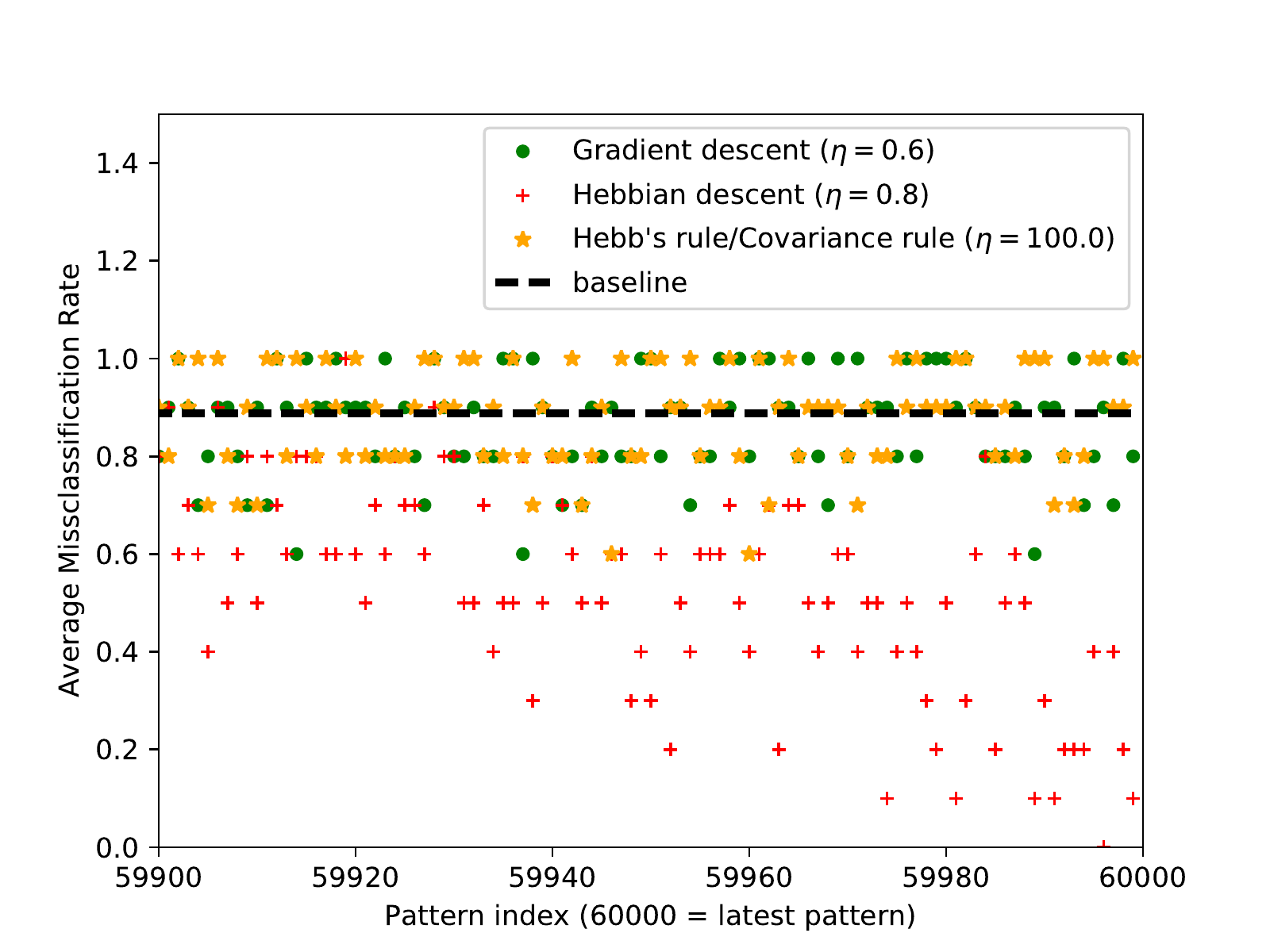}}
\caption{Online-learning classification performance of a network trained on the \emph{MNIST} dataset with the four different update rules, centering, and softmax units. 
The points show the (a) MAE, and (b) classification error for the last 100 patterns averaged over 10 trials.
The learning rate $\eta$ was chosen for each method individually, so that the performance over all 60,000 training patterns was best, but the results when choosing only the most recent patterns are very similar. 
The baseline represents (a) the MAE of a network that independently of the input always returns the mean of the labels and (b) the classification performance when the most frequent label in the training data is always returned, which in case of \emph{MNIST} is roughly 89\%.  }
\label{fig:online_label_learning_MNIST}
\end{center}
\end{figure} 
As shown in the other experiments only Hebbian-descent has a significantly better performance on the most recent patterns  in terms of both, the mean absolute error and the classification error.

\subsubsection{Outlook on Hebbian-descent in Deep Neural Networks}

The analysis of Hebbian-descent and the Hebbian-descent loss in deep networks is beyond the scope of this publication but an interesting research direction for future work.
To give an example apart from the cross entropy loss, where the Hebbian-descent loss is probably beneficial, consider the task of value-function approximation in reinforcement learning such as the famous DQN network~\citep{MnihKavukcuogluEtAl-2015}.
The output neurons are linear in this case although the output values are non-negative. 
The rectifier would thus be a more reasonable choice, not least because linear outputs always performed worse in all of our experiments. 
However, deep reinforcement learning suffers from high variance and when rectifiers are used as output units they can, especially in an early phase of training, get pushed into the negative regime for all inputs. 
While these units never get active again with gradient descent, so that the network might get trapped in a bad local optimum, Hebbian-decent in contrast does not have this problem at all.

\subsection{Auto-Associative Learning}\label{sec:HD_auto_associative_experiments}

Finally, we also performed experiments on auto-associative learning, namely we trained centered auto-encoder networks using either gradient descent or Hebbian-descent on the four real-world datasets. 
As the task is to find a compressed representation of the entired dataset we investigate the multi epoch mini batch learning performance.
Each experiment was repeated ten times and the centered networks were trained for 100 epochs with mean squared error loss, a batch size of 100, and a shifting factor for the hidden offsets of 0.01.
The optimal learning rate was determined via grid search for each method separately, so that the performance on all test patterns was best. 
We used a variety of activation functions for the hidden layer and either a linear or sigmoid activation function for the output layer.
The results are shown in Table~\ref{tab:auto_associative} where the first number represents the mean absolute error averaged over ten trials followed by the corresponding standard deviation, the optimal learning rate in parenthesis, and the average activity of the hidden units also in parenthesis.
\begin{table}[htbp]
\setlength{\tabcolsep}{2pt}
\begin{center}
\begin{small}
\begin{sc}
\begin{tabular}{l@{\hskip 0.2in} r@{\hskip 0.02in} r@{\hskip 0.02in} r@{\hskip 0.02in} r@{\hskip 0.20in} r@{\hskip 0.02in} r@{\hskip 0.02in} r@{\hskip 0.02in}r }
\hline
\abovespace\belowspace
  $\vect \phi$ & \multicolumn{4}{c}{\hskip -0.16in Grad. Descent} & \multicolumn{4}{c}{\hskip -0.16in Hebb. Descent}  \\ 
\hline & & & & & &  \vspace{-0.3cm}\\
\multicolumn{4}{l}{{\hspace{-0.1cm} \emph{ADULT (0.1248) with Linear Output}}} \\ 
 & & & & & & \vspace{-0.35cm}\\
Linear  &  0.0330  & $\pm$ 0.0002  & (0.06) & (0.064) & \textbf{0.0330} & $\pm$ 0.0003 & (0.06) & (-0.000)  \\ 
Rectifier  &  \textbf{0.0332} & $\pm$ 0.0002  & (0.08) & (0.864) & 0.0385   & $\pm$ 0.0019 & (0.20) & (0.209)  \\ 
ExpLin  &  0.0336  & $\pm$ 0.0002  & (0.10) & (0.513) & \textbf{0.0328} & $\pm$ 0.0005 & (0.20) & (0.027)  \\ 
Sigmoid  &  0.0364  & $\pm$ 0.0002  & (0.80) & (0.498) & \textbf{0.0348} & $\pm$ 0.0005 & (0.80) & (0.500)  \\ 
Step  &  \textbf{0.0452} & $\pm$ 0.0037  & (0.60) & (0.491) & \textbf{0.0452}   & $\pm$ 0.0037 & (0.60) & (0.491)  \\ 
\hline & & & &  & &  \vspace{-0.3cm}\\
\multicolumn{4}{l}{{\hspace{-0.1cm} \emph{ADULT (0.1248) with Sigmoid Output}}}\\ 
 & & & & & & \vspace{-0.35cm}\\
Linear  &  0.0032  & $\pm$ 0.0001  & (2.00) & (-1.847) & \textbf{0.0022} & $\pm$ 0.0001 & (1.00) & (0.001)  \\ 
Rectifier  &  0.0083  & $\pm$ 0.0008  & (4.00) & (2.457) & \textbf{0.0078} & $\pm$ 0.0007 & (2.00) & (1.432)  \\ 
ExpLin  &  0.0070  & $\pm$ 0.0003  & (4.00) & (1.070) & \textbf{0.0049} & $\pm$ 0.0004 & (2.00) & (0.870)  \\ 
Sigmoid  &  \textbf{0.0042} & $\pm$ 0.0004  & (80.00) & (0.503) & 0.0079   & $\pm$ 0.0010 & (100.00) & (0.498)  \\ 
Step  &  \textbf{0.0060} & $\pm$ 0.0008  & (100.00) & (0.499) & 0.0064   & $\pm$ 0.0013 & (100.00) & (0.498)  \\ 
\hline & & & &  & & \vspace{-0.3cm}\\
\multicolumn{4}{l}{{\hspace{-0.1cm} \emph{CONNECT (0.1717) with Linear Output}}}\\ 
 & & & & & & \vspace{-0.35cm}\\
Linear  &  0.0515  & $\pm$ 0.0014  & (0.04) & (-0.032) & \textbf{0.0496} & $\pm$ 0.0008 & (0.04) & (-0.001)  \\ 
Rectifier  &  \textbf{0.0501} & $\pm$ 0.0006  & (0.06) & (0.976) & 0.0659   & $\pm$ 0.0008 & (0.10) & (0.247)  \\ 
ExpLin  &  0.0500  & $\pm$ 0.0006  & (0.10) & (0.716) & \textbf{0.0495} & $\pm$ 0.0010 & (0.10) & (0.039)  \\ 
Sigmoid  &  0.0511  & $\pm$ 0.0007  & (0.60) & (0.504) & \textbf{0.0489} & $\pm$ 0.0003 & (0.60) & (0.500)  \\ 
Step  &  \textbf{0.0777} & $\pm$ 0.0012  & (0.10) & (0.498) & \textbf{0.0777}   & $\pm$ 0.0012 & (0.10) & (0.498)  \\ 
\hline & & & &  & & \vspace{-0.3cm}\\
\multicolumn{4}{l}{{\hspace{-0.1cm} \emph{CONNECT (0.1717) with Sigmoid Output }}}\\ 
 & & & & & & \vspace{-0.35cm}\\
Linear  &  0.0009  & $\pm$ 0.0001  & (2.00) & (-2.852) & \textbf{0.0005} & $\pm$ 0.0000 & (1.00) & (-0.002)  \\ 
Rectifier  &  0.0022  & $\pm$ 0.0001  & (4.00) & (2.761) & \textbf{0.0015} & $\pm$ 0.0004 & (2.00) & (2.308)  \\ 
ExpLin  &  0.0019  & $\pm$ 0.0002  & (4.00) & (2.045) & \textbf{0.0010} & $\pm$ 0.0002 & (2.00) & (1.317)  \\ 
Sigmoid  &  \textbf{0.0020} & $\pm$ 0.0002  & (60.00) & (0.500) & 0.0090   & $\pm$ 0.0008 & (100.00) & (0.497)  \\ 
Step  &  \textbf{0.0077} & $\pm$ 0.0005  & (100.00) & (0.497) & 0.0108   & $\pm$ 0.0012 & (100.00) & (0.497)  \\ 
\hline & & & & & & \vspace{-0.3cm}\\
\multicolumn{4}{l}{{\hspace{-0.1cm} \emph{MNIST (0.1496) with Sigmoid Output}}}\\ 
 & & & & & & \vspace{-0.35cm}\\
Linear  &  0.0136  & $\pm$ 0.0000  & (0.40) & (-0.932) & \textbf{0.0115} & $\pm$ 0.0000 & (0.06) & (-0.020)  \\ 
Rectifier  &  \textbf{0.0150} & $\pm$ 0.0001  & (1.00) & (2.230) & 0.0172   & $\pm$ 0.0001 & (0.10) & (0.881)  \\ 
ExpLin  &  \textbf{0.0149} & $\pm$ 0.0001  & (0.80) & (1.269) & 0.0168   & $\pm$ 0.0001 & (0.10) & (0.437)  \\ 
Sigmoid  &  \textbf{0.0200} & $\pm$ 0.0001  & (10.00) & (0.498) & 0.0261   & $\pm$ 0.0001 & (0.20) & (0.500)  \\ 
Step  &  0.0347  & $\pm$ 0.0001  & (8.00) & (0.499) & \textbf{0.0346} & $\pm$ 0.0001 & (0.20) & (0.497)  \\ 
 \hline & & & & & &  \vspace{-0.3cm}\\
\multicolumn{4}{l}{{\hspace{-0.1cm} \emph{CIFAR (0.1955) with Linear Output}}}\\ 
 & & & &  & & \vspace{-0.35cm}\\
 Linear  &  0.0285  & $\pm$ 0.0000  & (0.01) & (-0.206) & \textbf{0.0282} & $\pm$ 0.0000 & (0.01) & (-0.006)  \\ 
Rectifier  &  \textbf{0.0293} & $\pm$ 0.0001  & (0.02) & (0.821) & 0.0439   & $\pm$ 0.0002 & (0.02) & (0.237)  \\ 
ExpLin  &  \textbf{0.0293} & $\pm$ 0.0001  & (0.01) & (0.487) & 0.0373   & $\pm$ 0.0000 & (0.02) & (0.049)  \\ 
Sigmoid  &  \textbf{0.0349} & $\pm$ 0.0000  & (0.10) & (0.499) & 0.0380   & $\pm$ 0.0000 & (0.10) & (0.500)  \\ 
Step  &  \textbf{0.0974} & $\pm$ 0.0001  & (0.01) & (0.500) & \textbf{0.0974}   & $\pm$ 0.0001 & (0.01) & (0.500)  \\ 
\hline & & & &  & &  \vspace{-0.85cm}\\
\end{tabular}
\end{sc}
\end{small}
\end{center}
\caption{Performance for centered auto-encoder networks trained with either gradient descent or Hebbian-descent, for various activation functions and datasets. 
The networks are trained for 100 epochs with a batch size of 100. Each experiment is repeated 10 times and the optimal learning rate is determined via grid search for each method separately, so that the performance on all test patterns is best. 
The first number in each entry shows the MAE averaged over ten trials followed by the corresponding standard deviation, the optimal learning rate, and the average hidden activity in parenthesis. 
The best result is indicated in bold.
The baseline given for each dataset in parenthesis behind its name represent the mean absolute error between the input data and its mean.
} 
\label{tab:auto_associative}
\end{table}
The performances of Hebbian-descent or gradient descent are comparable in most cases, but when gradient descent performs better than Hebbian-descent the difference is usually more distinct compared to when Hebbian-descent has the better performance, which can also be seen from scatter plot in Figure~\ref{fig:auto_association_compare_performance}. 
\begin{figure}[t]
\begin{center}
\subfigure[]{
\includegraphics[scale=0.405, trim=11 5 27 10, clip]{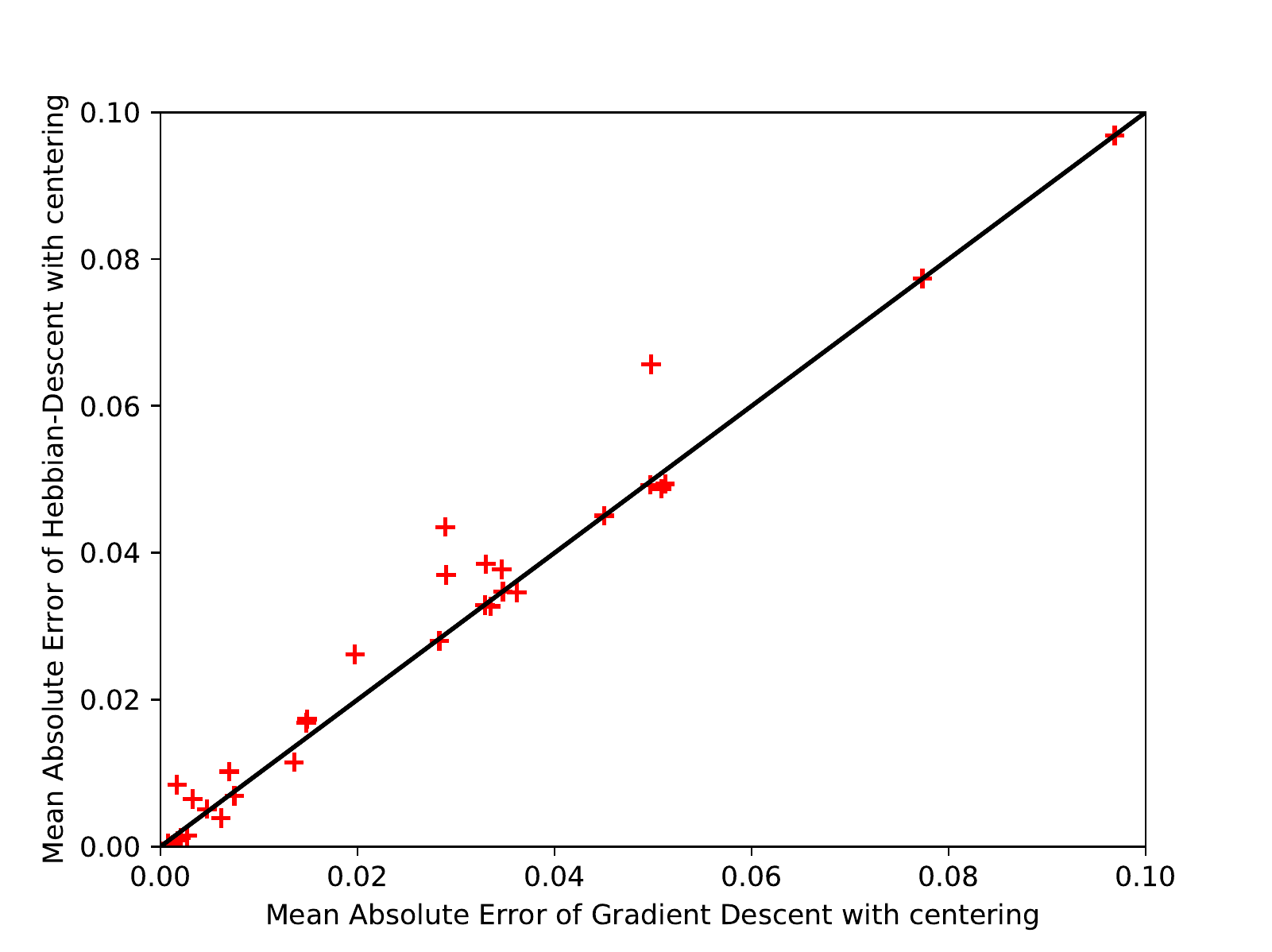}\label{fig:auto_association_compare_performance}}
\subfigure[]{
\includegraphics[scale=0.405, trim=21 5 37 10, clip]{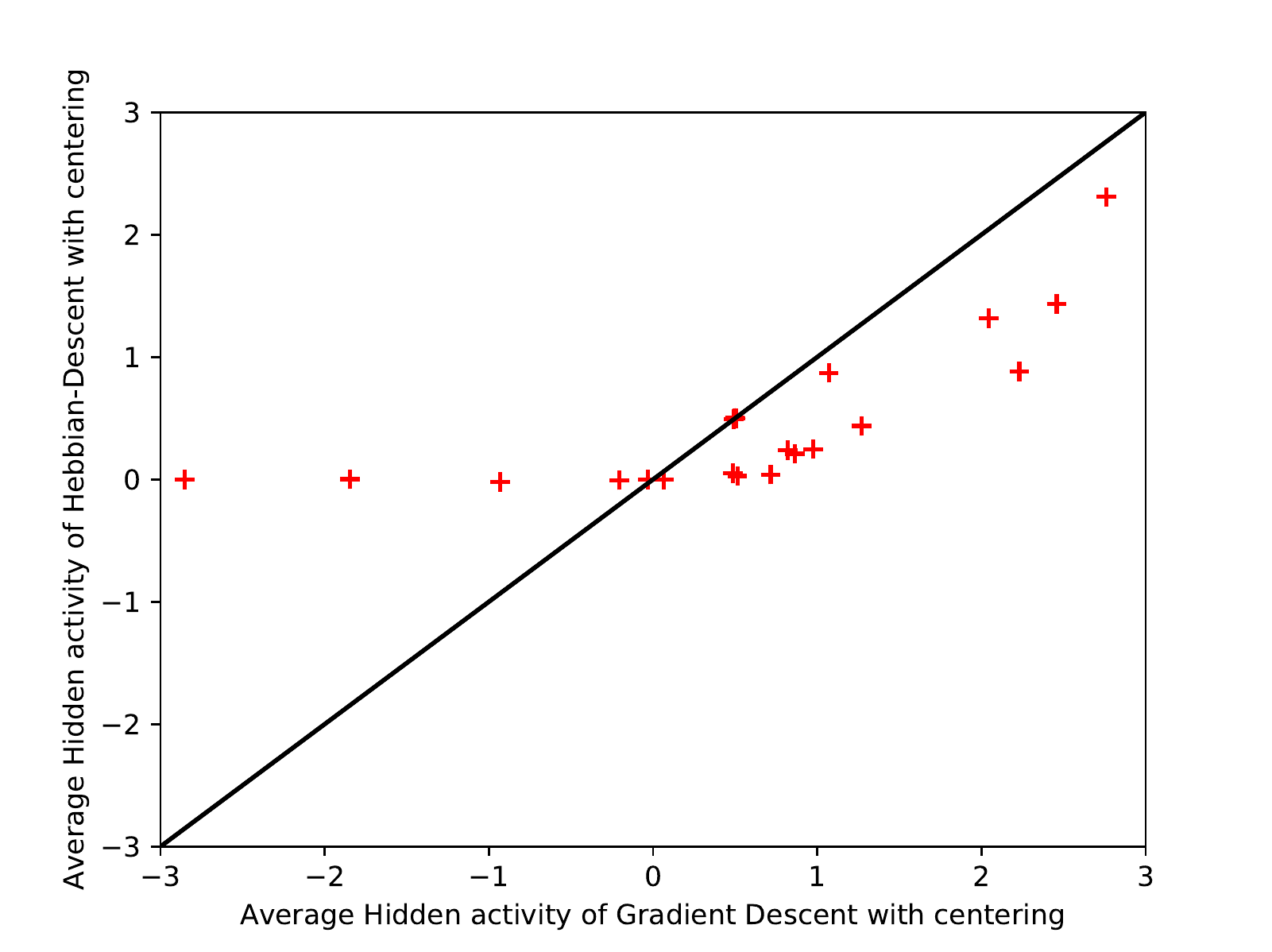}\label{fig:auto_association_compare_activity}}
\caption{Comparison of (a) mean absolute error, and (b) average hidden activity of gradient descent and Hebbian-descent in auto-associative learning.
Each cross represents the result for one experiment in Table~\ref{tab:auto_associative}. }
\label{fig:auto_association_compare}
\end{center}
\end{figure} 
As an example see the \emph{CIFAR} dataset in combination with rectifier or exponential linear units or the \emph{CONNECT} dataset with rectifier units for example.
The results on the training data show qualitatively the same picture and we verified that the networks have not over-fitted to the training data, and that the methods actually converge.
Notice, that we do not show results for uncentered networks as the benefit of centering in auto-encoder networks has already been shown by~\citet{Melchior2016}.

We also performed experiments with Hebb's rule / covariance rule, but the results are not shown as they are always much worse than that of gradient descent and Hebbian-descent.
The performance when using sigmoid hidden units and linear output units for example is 0.1128, 0.3335, 0.1340, and 0.4835 for \emph{ADULT}, \emph{CONNECT}, \emph{MNIST}, and \emph{CIFAR} dataset, respectively. 
In all cases the results are either close to the corresponding baseline or even much worse.

\subsubsection{Auto-Associative Hebbian-Descent Learns a Homogeneous Hidden Representation}\label{sec:HD_homo_hidden_experiments}

An interesting observation is that the average activity of hidden linear units as shown in Table~\ref{tab:auto_associative} is approximately zero when using Hebbian-descent, whereas it is significantly negative when using gradient descent except for first and third example.
This can best be seen from the scatter plot in Figure~\ref{fig:auto_association_compare_activity}, which also illustrates that Hebbian-descent in contrast to gradient descent never leads to a negative average hidden activity.
Furthermore, for rectifier and exponential linear units Hebbian-descent learns a hidden distribution with much lower activity compared to gradient descent in most cases, while for Sigmoid and Step units the average activity is approximately 0.5 for both methods.
To get a more detailed picture of the learned hidden representations we investigated the units' individual mean values and the corresponding standard deviation averaged over all test patterns but for single trials. 
Through all experiments the average activity of all hidden units was approximately the same when the network had been trained with Hebbian-descent, whereas it varied significantly with gradient descent. 
Figure~\ref{fig:hidden_activity_GD_HD} illustrates this for (a) the \emph{ADULT} dataset with sigmoid hidden units and linear output units and (b) the \emph{MNIST} dataset with rectifier hidden units and sigmoid output units.
\begin{figure}[t]
\begin{center}
\subfigure[]{
\includegraphics[scale=0.425, trim=21 0 38 0, clip]{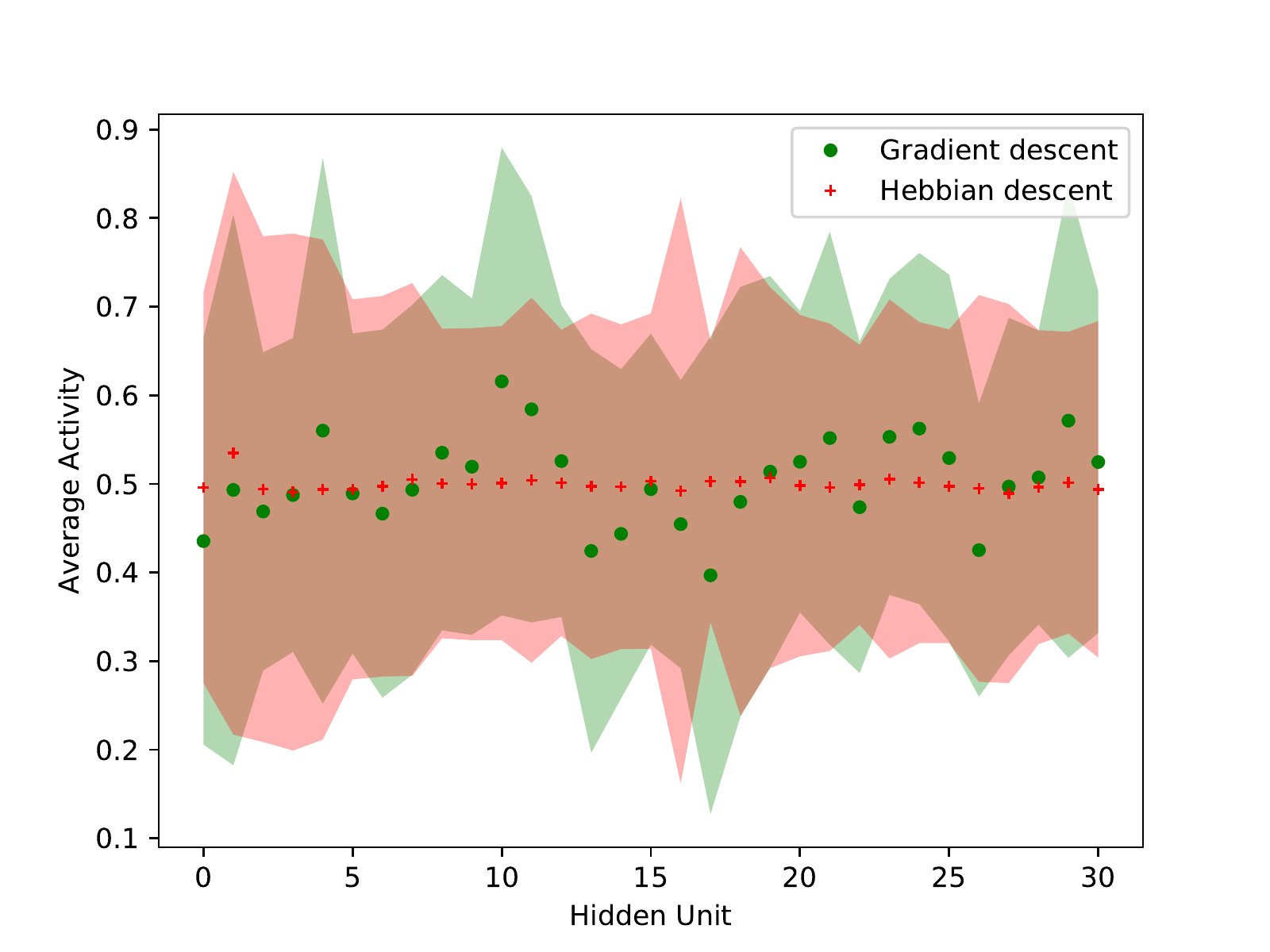}}
\subfigure[]{
\includegraphics[scale=0.425, trim=35 0 38 0, clip]{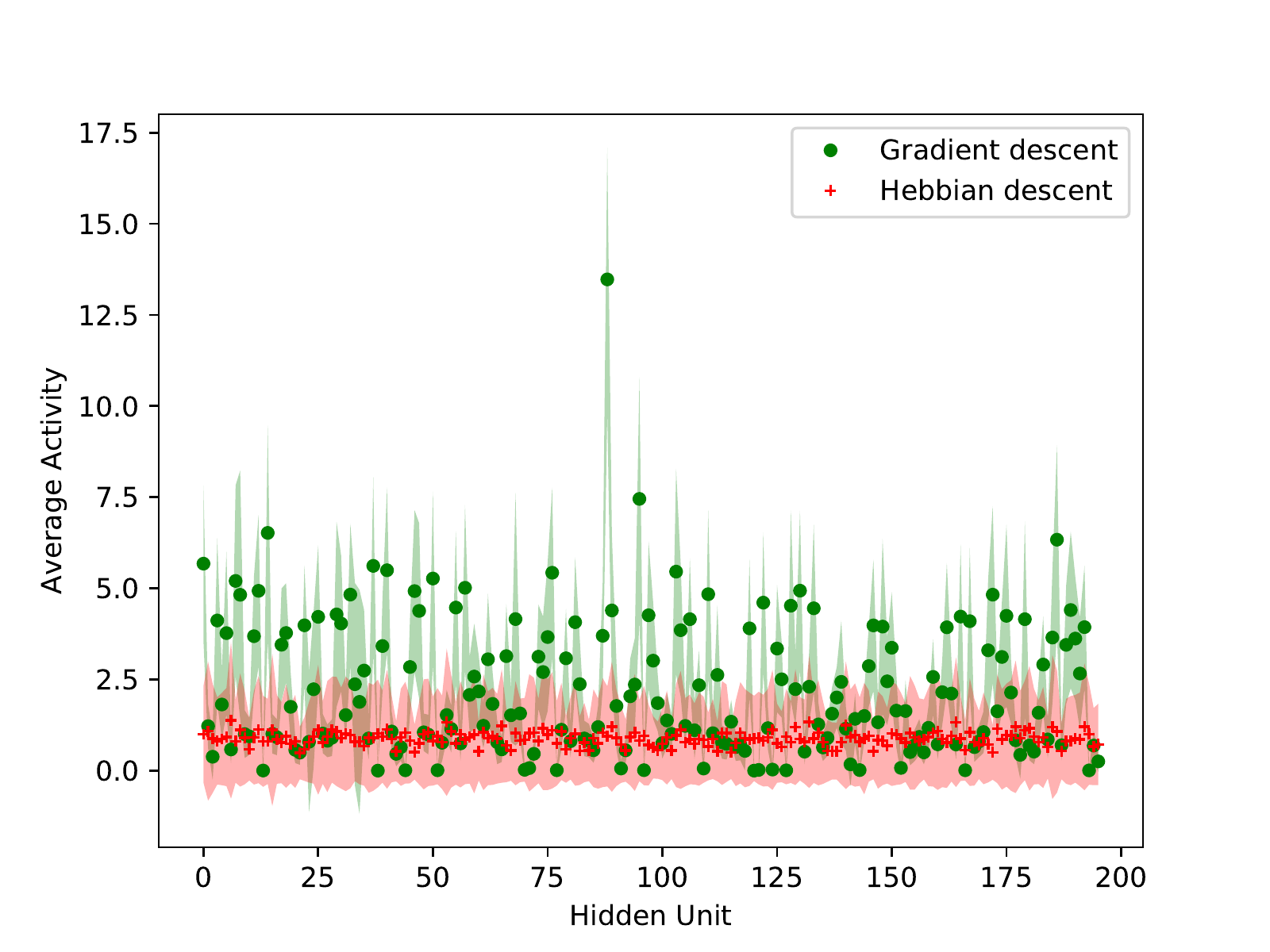}}
\caption{Average hidden activity with corresponding standard deviation of a network trained either with gradient descent or Hebbian-descent on (a) \emph{ADULT} dataset with sigmoid hidden units and linear output units and (b) \emph{MNIST} dataset with rectifier hidden units and sigmoid output units.}
\label{fig:hidden_activity_GD_HD}
\end{center}
\end{figure} 
Thus Hebbian-descent covers the hidden or latent space homogeneously, which leads to an equal importance of all hidden units and in case of linear units even leads to a mean free hidden representation (See Table~\ref{tab:auto_associative}).
A homogeneously covered latent space is desirable from an efficient coding perspective as it has maximal entropy, but it does not necessarily imply a smaller reconstruction error.
Notice, that we did not use the update for the encoder bias (Equation~\eqref{eqn:update_auto_hebbian_descent_unsup_b_1}), which can control the activation level of the hidden representation if desired. 

\section{Conclusion}\label{sec:conclusion}

In this work we have proposed Hebbian-descent as a biologically plausible learning rule for hetero-associative as well as auto-associative learning in artificial neural networks.
It can be used as a replacement for gradient descent as well as Hebbian learning as it inherits their advantages while not suffering from their disadvantages.
In Section~\ref{sec:HD_limit_cov_rule} we argued that the major drawbacks of Hebbian learning are the problem of dealing with correlated input data and that it does not profit from seeing training patterns several times.
For gradient descent we identify the derivative of the activation function as problematic as it is not biologically plausible and can lead to a vanishing error-term that prevents efficient online learning (Section~\ref{sec:HD_problems_gradient_descent}).
Hebbian-descent addresses these problems by getting rid of the activation function's derivative leading to a biologically plausible update rule that is provably convergent, does not suffer from the vanishing error term problem, can deal with correlated data, profits form seeing patterns several times, and enables successful online learning when centering is used.\\\\
In the case of hetero-associative learning we have shown analytically that:
\begin{itemize}
\item[1.] The Hebbian-descent update can generally be understood as gradient descent update where the derivative of the activation function of the output layer is removed (Section~\ref{sec:HD_heteroasso_HD} and Section~\ref{sec:HD_gradient_descent_update}).
\item[2.] In case of a strictly positive derivative of the activation function Hebbian-descent leads to the same update rule as gradient descent but for a different loss function (Section~\ref{sec:HD_Hebbian_descent_loss}).
\item[3.] In this case Hebbian-descent is a particular form of gradient descent and thus inherits its convergence properties (Section~\ref{sec:HD_Hebbian_descent_covergence}).
\item[4.] A well known case of this loss is the cross entropy in combination with softmax output units, in which case the derivative of the activation function vanishes and the error term is just the difference between the network output and the desired output value. Hebbian-descent generalizes this idea for arbitrary activation functions and error terms (Section~\ref{sec:HD_the_error_term_perspective} and Appendix~\ref{appendix:list_of_HD_losses}).
\end{itemize}
Furthermore, in case of the mean squared error loss we have shown in particular that:
\begin{itemize}
\item[5.] Hebbian-descent can be understood as the difference of a supervised and an unsupervised Hebb learning step (Section~\ref{sec:HD_heteroasso_HD} and Section~\ref{sec:HD_hebb}).
\item[6.] It can also be understood as contrastive learning where the hidden units are clamped to the same values in positive and negative phase (Section~\ref{sec:HD_contrastive_clamping}).
\item[7.] For an invertible and integrable activation function Hebbian-descent optimizes a generalized linear model (Section~\ref{sec:HD_glm}).
\item[8.] As a consequence the Hebbian-descent loss can be seen as the general log-likelihood loss (Section~\ref{sec:HD_HD_loss_generalized}).
\end{itemize}
Our empirical results suggests that:
\begin{itemize}
\item[9.] All update rules considered in this work profit from centering (Section~\ref{sec:hetero_ssociative_online_learning_with_centering}).
\item[10.] Hebbian-descent outperforms Hebb's rule and the covariance rule in general (Section~\ref{sec:HD_hetero_associative_learning}).
\item[11.] In batch or mini-batch learning for several epochs it has a similar performance as gradient decent (Section~\ref{sec:HD_multi_epoch_experiments}).
\item[12.] It performs significantly better in online / one-shot learning than all the other update rules (Section~\ref{sec:hetero_ssociative_online_learning_with_centering}).
\item[13.] Hebbian-descent even converges when the derivative of the activation function is non-negative, then often outperforms gradient decent, and can even be used with non-linearities like the step function (Section~\ref{sec:HD_results}).
\item[14.] Only Hebbian-descent with centering shows an inherent and plausible curve of forgetting that follows a power law distribution. No additional forgetting mechanism like a weight decay term is required (Section~\ref{sec:HD_weight_decay_experiments}).
\end{itemize}
In case of auto-associative learning we have shown that:
\begin{itemize}
\item[15.] Hebbian-descent corresponds to a one step mean field contrastive divergence update and and is thus related to contrastive learning as used for restricted Boltzmann machines or Hopfield networks for example (Section~\ref{sec:HD_auto_asso_from_the_perspective_of_contrastive_learning}).
\item[16.] The update rule is also related to gradient-decent as it corresponds to the encoder related part of the gradient of an auto encoder (Section~\ref{sec:HD_from_the_perspective_of_gradient_descent}).
\item[17.] Furthermore, it can also be understood as non-linear Oja's / Sanger's rule (Section~\ref{sec:HD_from_the_perspective_of_oja_sanger}).
\item[18.] We have proven that this update rule does not correspond to the gradient of any objective function (Appendix~\ref{appendix:AA_HD_is_generally_not_the_gradient_of_any_objective_function}). However, this does not imply that the update rule does not converge and we have shown empirically that Hebbian-descent converges in terms of the reconstruction error (Section~\ref{sec:HD_auto_associative_experiments}).
\item[19.] It leads to a hidden representation where each hidden unit has almost the same activity on average such that the learned distribution implies a more efficient encoding of the input data (Section~\ref{sec:HD_homo_hidden_experiments}).
\item[20.] It has a similar or slightly worse performance than gradient descent, but is computationally less expensive, and biologically more plausible (Section~\ref{sec:HD_auto_associative_experiments}).
\end{itemize}
Besides the benefit for shallow networks that have been considered in this work, Hebbian-descent has also a direct impact on deep learning. 
The vanishing error term influences the lower layers in deep neural networks (Section~\ref{sec:HD_learning_rate} and Section~\ref{sec:HD_outlook_deep_auto_encoder}) and is thus part of the vanishing gradient problem in deep neural networks. 
Furthermore, the error-term perspective of Hebbian-descent, which treats the output activation functions and loss functions as a unity, allows to design and understand loss functions more intuitively (Section~\ref{sec:HD_the_error_term_perspective}). 
The empirical evaluation of Hebbian-descent losses (\emph{e.g.}\ Appendix~\ref{appendix:list_of_HD_losses}) in deep neural networks is therefore a promising future research direction.
Future work might also focus on Hebbian-decent learning in spiking neural networks since an implementation of Hebbian-descent through spike-timing-dependent plasticity is obvious.

\section*{Acknowledgement}

We would like to thank Dr. Amir Hossein Azizi and Dr. Mehdi Bayati for helpful discussions on Hebbian learning.

\bibliography{HebbianDescent}
\bibliographystyle{apalike}
%\nocite{*}
\clearpage

\begin{appendices}

\section{List of Hebbian-Descent Loss Functions}\label{appendix:list_of_HD_losses}

The following table lists some Hebbian-descent loss functions for different activation functions and error terms. As shown in Section~\ref{sec:HD_Hebbian_descent_loss}, using these loss functions with gradient descent leads to the Hebbian-descent update. For a single weight update we have
$
\delta_{_{GD}} w_{ij} \stackrel{\eqref{eqn:update_general_gradient_descent_sup_w_2}}{=}  -\eta \big(x_i-\mu_i \big) 
\frac{\partial \mathcal{L}_{_{HD}}(t_j, h_j)}{ \partial h_j} \phi'(a_j)
\stackrel{\eqref{eqn:update_general_gradient_descent_sup_w_2}}{=} -\eta \big(x_i-\mu_i \big) 
\frac{\mathcal{E}(t_j, h_j)}{\phi'(a_j)}\phi'(a_j)  \\
\stackrel{\eqref{eqn:update_general_gradient_descent_sup_w_2}}{=} -\eta \big(x_i-\mu_i \big) 
\mathcal{E}(t_j, h_j) \stackrel{\eqref{eqn:update_general_hebbian_descent_sup_w_1}}{=} \delta_{_{HD}} w_{ij} 
$.
\begin{table}[!htbp]
\centering
\begin{tabular}{l|l|l}
\hline \hline& & \vspace{-0.3cm}\\
$ h_j = \phi(a_j)$ &  $\frac{ \mathcal{E}( t_j, h_j)}{ \phi'(a_j)} = \frac{h_j -t_j}{ \phi'(a_j)}$  &  $\mathcal{L}_{_{HD}}(t_j, h_j) = \int \frac{ \mathcal{E}(t_j, h_j)}{\phi'(a_j)} \,\mathrm{d} h_j$  \\
& & \vspace{-0.3cm}\\
\hline \hline& & \vspace{-0.3cm}\\
Linear & \multirow{2}{*}{} & Squared error loss \\
$a_j$ & $h_j - t_j$ & $\frac{1}{2}\left(h_j - t_j\right)^2$\\   
 & & \vspace{-0.3cm}\\
\hline & & \vspace{-0.3cm}\\
Sigmoid \scriptsize & \multirow{2}{*}{} & Cross entropy loss   \\
$\frac{1}{1+\exp(-a_j)}$ & $\frac{h_j - t_j}{h_j  ( 1 - h_j)}$ & $-t_j \ln ({h_j})-(1-t_j)\ln(1-{h_j})$ \\ 
 & & \vspace{-0.3cm}\\
\hline & & \vspace{-0.3cm}\\
Softmax &  & Cross entropy loss   \\
 \scriptsize {\citep{Bridle-1990}} &  &    \\
$\frac{\exp(a_j)}{\sum_k \exp(a_k)}$ & $\frac{h_j - t_j}{h_j  ( \delta_{j,k} - h_k)}$ & $-t_j \ln ({h_j})-(1-t_j)\ln(1-{h_j})$ \\ 
 & & \vspace{-0.3cm}\\
\hline & & \vspace{-0.3cm}\\
Scaled Hyperbolic  &  &   \\
Tangent &  &   \\
$\alpha \tanh(a_j)$ & $\frac{h_j - t_j}{\alpha - \frac{1}{\alpha} h_j^2}$ & $-\frac{1}{2}(\alpha+t_j) \ln ({\alpha+h_j})-\frac{1}{2}(\alpha-t_j)\ln(\alpha-{h_j})$ \\     
 & & \vspace{-0.3cm}\\
\hline & & \vspace{-0.3cm}\\
Approx. Step \scriptsize{($\alpha \rightarrow 0$)} & \multirow{2}{*}{} & \multirow{2}{*}{} \\
%$\begin{cases} \alpha a_j &{}\\ \alpha a_j + \beta &{}\end{cases}$ 
%& $\begin{cases} \frac{h_j - t_j}{\alpha} &{}\\ \frac{h_j - t_j}{\alpha} &{}\end{cases}$ 
%& $\begin{cases} \frac{1}{2 \alpha}\left(h_j - t_j\right)^2 &{\text{for }}a_j<0\\ \frac{1}{2 \alpha}\left(h_j - t_j\right)^2 &{\text{for }}a_j\geq 0\end{cases}$  \\
%$\begin{cases} \alpha a_j &{\text{for }}a_j<0\\ \alpha a_j + \beta &{\text{for }}a_j\geq 0\end{cases}$ 
%& $\frac{h_j - t_j}{\alpha} $ 
%& $\frac{1}{2 \alpha}\left(h_j - t_j\right)^2$  \\
$\begin{cases} \alpha a_j &{}\\ \alpha a_j + \beta &{}\end{cases}$ 
& $\begin{cases}  &{}\\  &{} \end{cases} \hspace{-0.5cm}\frac{h_j - t_j}{\alpha}$ 
& $\begin{cases}  &{}\\  &{} \end{cases} \hspace{-0.5cm}\frac{1}{2 \alpha}\left(h_j - t_j\right)^2 
\hspace{0.3cm}
\begin{aligned}
{\text{for }}a_j<0\\
{\text{for }}a_j\geq 0
\end{aligned}$\\
 & & \vspace{-0.3cm}\\
\hline & & \vspace{-0.3cm}\\
Leaky Rectifier.  & & \\
\scriptsize {\citep{MaasHannunEtAl-2013}} & & \\
$\begin{cases} \alpha\,a_j &{}\\a_j &{}\end{cases}
$ & $\begin{cases} \frac{1}{\alpha} \left(h_j - t_j\right) &{}\\h_j - t_j &{}\end{cases}$ 
& $\begin{cases} \frac{1}{2\alpha}\left(h_j - t_j\right)^2 &{\text{for }}a_j<0\\\frac{1}{2}\left(h_j - t_j\right)^2 &{\text{for }}a_j\geq 0\end{cases}$  \\  
 & & \vspace{-0.3cm}\\
\hline & & \vspace{-0.3cm}\\
Scaled Exp. Linear & & \\
\scriptsize {\citep{KlambauerUnterthinerEtAl-2017}}& & \\
$\begin{cases} \lambda \alpha (\exp(a_j) - 1) &{}\\\lambda a_j &{}\end{cases}
$ & $\begin{cases} \frac{h_j - t_j}{h_j - \lambda \alpha} &{}\\ \frac{1}{\lambda}\left(h_j - t_j\right) &{}\end{cases}$ 
& $\begin{cases} h_j -\left(t_j + \lambda \alpha \right)\log(h_j + \lambda \alpha ) &{\text{for }}a_j<0\\\frac{1}{2 \lambda}\left(h_j - t_j\right)^2 &{\text{for }}a_j\geq 0\end{cases}$  \\
%& & \vspace{-0.3cm}\\ 
\hline & & \vspace{-0.3cm}\\
\end{tabular}

\end{table}
\begin{table}[!thpb]
\centering
\begin{tabular}{l|l|l}
% & & \vspace{-0.3cm}\\
%\hline & & \vspace{-0.3cm}\\
Inv. Sqrt. & & \\
\scriptsize {\citep{CarlileDelamarterEtAl-2017}}& & \\
$\begin{cases} \frac{a_j}{\sqrt{1+\alpha a_j^2}} &{}\\a_j &{}\end{cases}
$ & $\begin{cases} \frac{h_j - t_j}{\left(\alpha a_j^2+1\right)^\frac{3}{2}} &{}\\\left(h_j - t_j\right) &{}\end{cases}$ 
& $\begin{cases} 
\frac{1}{2} \left(\alpha \hat{a}_j^2+1\right)^\frac{3}{2}\left(h_j - t_j\right)^2
 &{\text{for }}a_j<0\\\frac{1}{2}\left(h_j - t_j\right)^2 &{\text{for }}a_j\geq 0\end{cases}$  \\
 \hline & & \vspace{-0.3cm}\\
SoftSign & & \\
\scriptsize {\citep{BergstraDesjardinsEtAl-2009}}& & \\
$\frac{a_j}{1+|a_j|}$ & $\frac{(h_j - t_j)}{(1+|a_j|)^{-2}}$ & $\frac{1}{2}(1+|\hat{a}_j|)^2\left(h_j - t_j\right)^2$\\
\hline & & \vspace{-0.3cm}\\
SoftPlus & & \\
\scriptsize {\citep{GlorotBordesEtAl-2011}} & & \\
$\log(1+\exp(a_j))$ & $\frac{(h_j - t_j)}{(1+\exp(-a_j)^{-1}}$ & $\frac{1}{2}(1+\exp(-\hat{a}_j)\left(h_j - t_j\right)^2$\\
 & & \vspace{-0.3cm}\\
\hline & & \vspace{-0.3cm}\\
Inv. Sqrt.  & & \\
\scriptsize {\citep{CarlileDelamarterEtAl-2017}} & & \\
$\frac{a_j}{\sqrt{1+\alpha a_j^2}}$ & $\frac{h_j - t_j}{\left(\alpha a_j^2+1\right)^\frac{3}{2}}$ & $\frac{1}{2} \left(\alpha \hat{a}_j^2+1\right)^\frac{3}{2}\left(h_j - t_j\right)^2$\\
& & \vspace{+0.5cm}\\
\hline \hline& & \vspace{-0.3cm}\\
 $ h_j = \phi(a_j)$ &  $\frac{ \mathcal{E}( t_j, h_j)}{ \phi'(a_j)} = $  &  $\mathcal{L}_{_{HD}}(t_j, h_j) = \int \frac{ \mathcal{E}(t_j, h_j)}{\phi'(a_j)} \,\mathrm{d} h_j$  \\
  &  $\frac{\alpha \tanh(\beta (h_j - t_j))}{ \phi'(a_j)}$  &   \\
\hline \hline& & \vspace{-0.2cm}\\
Linear & \multirow{2}{*}{} & A smooth version of the Huber loss \\
$a_j$ & $\alpha \tanh(\beta (h_j - t_j))$ & $\frac{\alpha}{\beta}  \ln(\cosh(\beta ( h_j- t_j)))$\\
& & \vspace{+0.5cm}\\
\hline \hline & & \vspace{-0.3cm}\\
$ h_j = \phi(a_j)$ & \scriptsize{$(\alpha > 0, \alpha \rightarrow 0)$} & $\mathcal{L}_{_{HD}}(t_j, h_j) = \int \frac{ \mathcal{E}(t_j, h_j)}{\phi'(a_j)} \,\mathrm{d} h_j$ \\
 & $\frac{ \mathcal{E}( t_j, h_j)}{ \phi'(a_j)} = \begin{cases} \frac{-t_j}{\phi'(a_j)} &{}\\ \frac{\alpha}{\phi'(a_j)} \end{cases}$ \\
\hline \hline& & \vspace{-0.3cm}\\
Linear & \scriptsize{$(\alpha > 0, \alpha \rightarrow 0)$} & 'Leaky' version of the Hinge loss  \\
$a_j$ 
& $\begin{cases} -t_j &{}\\ \alpha \end{cases}$ 
& $\begin{cases} 1-t_j h_j  &{\text{for }} t_j h_j<1 \\ \alpha h_j &{\text{for }} t_jh_j\geq 1 \end{cases}$  \\  
 & & \vspace{-0.3cm}\\
\hline & & \vspace{-0.3cm}\\
Sigmoid   & \scriptsize{$(\alpha > 0, \alpha \rightarrow 0)$} & \multirow{2}{*}{}   \\
$\frac{1}{1+\exp(-a_j)}$ 
& $\begin{cases} \frac{-t_j}{h_j  ( 1 - h_j)} &{}\\ \frac{\alpha}{h_j  ( 1 - h_j)}\end{cases}$ 
& $\begin{cases} -t_j \log(h_j) - \log(1-h_j) &{\text{for }} t_j h_j<1 \\\alpha \log(h_j) - \log(1-h_j) &{\text{for }} t_jh_j\geq 1 \end{cases}$  
\end{tabular}
\caption{List of Hebbian-descent loss functions (see Equation~\eqref{eqn:gradient_descent_loss_implementing_Hebbian_descent}) for various activation functions and error terms. For notational simplicity given for a single output, but for multiple outputs simply use $\mathcal{L}_{_{HD}}(\vect t, \vect h) = \sum_j \mathcal{L}_{_{HD}}(t_j, h_j)$. Notice that $\hat{a}_j$ needs to be treated as constant when the gradient of the loss function $\mathcal{L}_{_{HD}}(t_j, h_j)$ w.r.t. the model parameters is calculated.}
\label{tab:list_HD_loss_functions}
\end{table}

\section{Convergence of Stochastic Hebbian-Descent}\label{appendix:convergence_of_stochastic_hebbian_descent}

\begin{theorem} 
Stochastic Hebbian-descent converges for an infinitesimal step-size if the derivative of the activation function is strictly positive. 
\end{theorem}
\begin{proof}
We take for granted that stochastic gradient descent converges for an infinitesimal step-size~\citep{Robbins1985, saad1998online} and we know that the negative gradient is orthogonal to the loss equipotential-line, pointing locally in the direction of steepest loss-descent. Thus, any vector whose direction deviates less than 90 degrees from the negative gradient vector also points locally in a direction of a smaller loss value. The angle between two non-zero vectors is less than 90 degrees if their inner product is positive, more than 90 degrees if their inner product is negative, and exactly 90 degrees if the inner product is zero. 
As a consequence, an update rule converges for an infinitesimal step-size to an optimum of the loss function if the inner product between an update vector and the corresponding gradient vector is always (i) positive or (ii) zero if and only if both update vectors are zero.
The inner product between the Hebbian-descent update and the gradient descent update for an activation function with strictly positive derivative is given by
\begin{eqnarray}
\sum_i^N \underbrace{\Delta_{HD} \theta_i \Delta_{GD} \theta_i}_{\geq 0} &\stackrel{(\eqref{eqn:update_general_hebbian_descent_sup_w_2}, \eqref{eqn:update_general_hebbian_descent_sup_b_2},\eqref{eqn:update_general_gradient_descent_sup_w_1}, \eqref{eqn:update_general_gradient_descent_sup_b_1})}{=}&  \sum_i^N \underbrace{\Delta_{HD} \theta_i^2}_{\geq 0} \underbrace{\phi_i'(\cdot)}_{> 0} \,\,\,\,\,\,\,\geq\,\,\,\,\,\,\, 0,
\label{eqn:scalar_product_HD_GD}
\end{eqnarray}
where $\theta_i$ generally denotes one of the model parameters, \emph{i.e.}\ $w_{ij}$ or $b_j$.
The inner product as well as each of its summands is always greater or equal zero such that the inner product is greater zero if at least one summand is greater zero and can only get zero if all summands are zero.
Since we defined $\phi_i'(\cdot)$ to be strictly positive $\Delta_{HD} \theta_i^2$ fully determines whether a summand is positive or zero. (i) summands are positive if $\Delta_{HD} \theta_i^2$ is greater zero in which case $\Delta_{HD} \theta_i$ as well as $\Delta_{GD} \theta_i$ are both non-zero and have the same sign (Equation~\eqref{eqn:scalar_product_HD_GD}). (ii) summands are zero if $\Delta_{HD} \theta_i^2$ is zero in which case $\Delta_{HD} \theta_i$ as well as $\Delta_{GD} \theta_i = \Delta_{HD} \theta_i \phi_i'(\cdot)$ are both zero and thus both algorithms have converged. The two possible cases are listed in Table~\ref{tab:scalarproduct_cases}~(i) and (ii) and it follows that Hebbian-descent converges for a infinitesimal step-size and an activation function with strictly positive derivative.
\begin{table}[htbp]
\centering
\label{tab:scalarproduct_cases}
\begin{tabular}{clll}
\hspace{-0.2cm}(i) & $\Delta_{HD}\theta_i^2>0 \wedge \phi_i'(\cdot)>0 \Rightarrow \Delta_{HD} \theta_i \neq 0 \wedge \Delta_{GD} \theta_i \neq 0 \Rightarrow \Delta_{HD} \theta_i \Delta_{GD} \theta_i > 0$ \,\,\,\,\,\,\,\,\,\,\,\, \\
\hspace{-0.2cm}(ii) & $\Delta_{HD}\theta_i^2=0 \wedge \phi_i'(\cdot)>0 \Rightarrow \Delta_{HD} \theta_i = 0 \wedge \Delta_{GD} \theta_i = 0 \Rightarrow \Delta_{HD} \theta_i \Delta_{GD} \theta_i = 0$ \\
\hspace{-0.2cm}(iv) & $\Delta_{HD}\theta_i^2>0 \wedge \phi_i'(\cdot)=0 \Rightarrow \Delta_{HD} \theta_i \neq 0 \wedge \Delta_{GD} \theta_i = 0 \Rightarrow \Delta_{HD} \theta_i \Delta_{GD} \theta_i = 0$  \\
\hspace{-0.2cm}(iii) & $\Delta_{HD}\theta_i^2=0 \wedge \phi_i'(\cdot)=0 \Rightarrow \Delta_{HD} \theta_i = 0 \wedge \Delta_{GD} \theta_i = 0 \Rightarrow \Delta_{HD} \theta_i \Delta_{GD} \theta_i = 0$ 
\end{tabular}
\caption{Relationship of the Hebbian-descent update values and gradient descent update values for (i)-(ii) a strictly positive derivative of the activation function, and (i)-(iv) a positive derivative of the activation function. }
\end{table}
\end{proof}
For several data-points, gradient-descent as well as Hebbian-descent are just the average over the single updates.
The joint optima for all data-points gradient descent and Hebbian-descent converge to does thus not have to be the same anymore as shown by Equation~\eqref{eqn:gradient_descent_loss_implementing_Hebbian_descent}.

\section{Auto-Associative Hebbian-Descent is Generally not the Gradient of any Objective Function}\label{appendix:AA_HD_is_generally_not_the_gradient_of_any_objective_function}

We have shown in Section~\ref{sec:HD_Hebbian_descent_loss} that the Hebbian-descent update is the gradient of the corresponding Hebbian-descent loss(Equation~\ref{eqn:gradient_descent_loss_implementing_Hebbian_descent}).
Here we proof that the auto-associative Hebbian-descent updates (Equation~\eqref{eqn:update_auto_hebbian_descent_unsup_w_1} and \eqref{eqn:update_auto_hebbian_descent_unsup_c_1}) is not the gradient of any function following the proof of \citet{SutskeverTieleman-2010} who showed that the Contrastive Divergence update is not the gradient of any function.

From Schwarz's theorem we know that if our network function $\mathcal{L}(\vect t, \vect x | \vect \theta) \colon \mathbb {R} ^{n}\to \mathbb {R} $ has continuous second partial derivatives at given target $\vect t$ and input $\vect x$ then 
\begin{eqnarray}
\frac{\partial^2 \mathcal{L}(\vect t, \vect x | \vect \theta)}{ \partial  \theta_i \partial \theta_j} &=&
\frac{\partial^2 \mathcal{L}(\vect t, \vect x | \vect \theta)}{ \partial  \theta_j \partial \theta_i},\\
\iff \frac{\partial}{ \partial \theta_i} \left(\frac{\partial \mathcal{L}(\vect t, \vect x | \vect \theta)}{ \partial  \theta_j}   \right) &=& 
\frac{\partial}{ \partial \theta_j} \left(\frac{\partial \mathcal{L}(\vect t, \vect x | \vect \theta)}{ \partial  \theta_i}   \right),\label{eqn:schwarz_theorem}
\end{eqnarray}
telling that the order in which the partial derivatives are taken does not matter or that the partial derivatives  commute leading to a symmetric Hessian matrix. 

Now, suppose there exists a loss function $\mathcal{\hat L}(\vect t, \vect x | \vect \theta)$ for which the Hebbian-descent update $\Delta_{_{HD}} \vect \theta$ is the gradient. Then the following equivalence has to hold in general.
\begin{eqnarray}
\frac{\partial}{ \partial \theta_i} \left(\frac{\partial \mathcal{\hat L}(\vect t, \vect x | \vect \theta)}{ \partial  \theta_j}   \right) &\stackrel{\eqref{eqn:schwarz_theorem}}{=}& 
\frac{\partial}{ \partial \theta_j} \left(\frac{\partial \mathcal{\hat L}(\vect t, \vect x | \vect \theta)}{ \partial  \theta_i}   \right), \\
\iff \frac{\partial}{ \partial \theta_i} \left( \Delta_{_{HD}}  \theta_j  \right) &=& 
\frac{\partial}{ \partial \theta_j} \left(\Delta_{_{HD}}  \theta_i  \right).
\end{eqnarray}
A single counter example is therefore sufficient to proof that in general there exists not function for which the Hebbian-descent update is the corresponding gradient.
Now \emph{w.l.o.g.} consider a simple linear auto-encoder with squared error loss, two input dimensions, only one hidden unit, and where bias and offset values are fixed to zero. The corresponding variables and Hebbian-descent updates are thus given by
\begin{eqnarray}
h &\stackrel{\eqref{eqn:neuron_second_layer3}}{=}&  \sum_i w_i x_i = w_0 x_0 + w_1 x_1,\label{eqn:proof_no_loss_h}\\
o_0 &\stackrel{\eqref{eqn:neuron_second_layer3}}{=}&  w_0 h =  w_{0} \sum_i w_i x_i = w_0^2 x_0 + w_{0}  w_1 x_1,\label{eqn:proof_no_loss_o0}\\
o_1 &\stackrel{\eqref{eqn:neuron_second_layer3}}{=}&   w_1 h =  w_{1} \sum_i w_i x_i = w_0  w_1 x_0 + w_1^2 x_1,\label{eqn:proof_no_loss_o1}\\
\Delta_{_{HD}} w_0 &\stackrel{\eqref{eqn:update_auto_hebbian_descent_unsup_w_1}}{=}&   (o_0 - x_0)h \stackrel{(\ref{eqn:proof_no_loss_h},\ref{eqn:proof_no_loss_o0})}{=} (w_0^2 x_0 + w_{0}  w_1 x_1 - x_0)(w_0 x_0 + w_1 x_1)\label{eqn:proof_no_loss_deltaw0}\\
\Delta_{_{HD}} w_1 &\stackrel{\eqref{eqn:update_auto_hebbian_descent_unsup_w_1}}{=}&   (o_1 - x_1)h \stackrel{(\ref{eqn:proof_no_loss_h},\ref{eqn:proof_no_loss_o1})}{=} (w_0  w_1 x_0 + w_1^2 x_1 - x_1)(w_0 x_0 + w_1 x_1),\label{eqn:proof_no_loss_deltaw1}
\end{eqnarray}
where the Hebbian-descent updates as given by Equation~\eqref{eqn:update_auto_hebbian_descent_unsup_w_1} and \eqref{eqn:update_auto_hebbian_descent_unsup_c_1} are multiplied by -1 since they minimize the corresponding loss and \emph{w.l.o.g.} we set $\eta=1$. Now it is straight forward to show that the order of taking the partial derivatives matters since
\begin{eqnarray}
&&\frac{\partial}{ \partial w_1} \left( \Delta_{_{HD}}  w_0  \right) \\&\stackrel{\eqref{eqn:proof_no_loss_deltaw0}}{=}& \frac{\partial}{ \partial w_1} (w_0^2 x_0 + w_{0}  w_1 x_1 - x_0)(w_0 x_0 + w_1 x_1)\\
%&=& \frac{\partial}{ \partial w_1} 
%w_0^2 x_0 (w_0 x_0 + w_1 x_1) 
%+ w_{0}  w_1 x_1 (w_0 x_0 + w_1 x_1) 
%- x_0 (w_0 x_0 + w_1 x_1)\\
%&=& \frac{\partial}{ \partial w_1} 
%w_0^3 x_0^2 
%+ w_0^2  w_1 x_0 x_1
%- w_0 x_0^2 
%+  w_0^2  w_1 x_0  x_1 
%+ w_{0}  w_1^2 x_1^2
%- w_1 x_0  x_1\\
&=& \frac{\partial}{ \partial w_1} 
 2 w_0^2  w_1 x_0 x_1
+ w_{0}  w_1^2 x_1^2
- w_1 x_0  x_1
+ w_0^3 x_0^2
- w_0 x_0^2  \\
&=& 
 2 w_0^2  x_0 x_1
+ 2 w_{0}  w_1 x_1^2
- x_0  x_1\\
&\neq & \nonumber\\
&=& 
 2  w_1^2 x_0 x_1
+  2 w_0  w_1 x_0^2
- x_0 x_1 \\
&=& \frac{\partial}{ \partial w_0} 
  w_0^2  w_1 x_0^2
+ 2 w_0 w_1^2 x_0 x_1
+ w_1^3 x_1^2 
- w_0 x_0 x_1 
+ w_1 x_1^2\\
%&=& \frac{\partial}{ \partial w_0} 
% w_0  w_1 x_0 w_0 x_0 
% + w_0  w_1 x_0 w_1 x_1
%+  w_1^2 x_1 w_0 x_0 
%+ w_1^2 x_1 w_1 x_1 
%- x_1 w_0 x_0 
%+ x_1 w_1 x_1\\
%&=& \frac{\partial}{ \partial w_0} 
%w_0  w_1 x_0 (w_0 x_0 + w_1 x_1) + w_1^2 x_1 (w_0 x_0 + w_1 x_1) - x_1(w_0 x_0 + w_1 %x_1)\\
&\stackrel{\eqref{eqn:proof_no_loss_deltaw1}}{=}& \frac{\partial}{ \partial w_1} (w_0  w_1 x_0 + w_1^2 x_1 -x_1)(w_0 x_0 + w_1 x_1)\\
&=&\frac{\partial}{ \partial w_0} \left( \Delta_{_{HD}}  w_1  \right). 
\end{eqnarray}
Thus, the partial derivatives do not commute proving that the auto-associative Hebbian-descent updates are not the gradient of any function in general.

It is straight forward to see that using only the 'decoder related part of the gradient' when training deep auto encoder networks is also not the gradient of any objective function, as the updates for the highest layer are equivalent to the single layer auto encoder updates as shown above.
\qed

\section{Additional Results}\label{appendix:hebbian_descent_additional_results}

\begin{table}[htbp]
\setlength{\tabcolsep}{2pt}
\begin{center}
\begin{small}
\begin{sc}
\begin{tabular}{l@{\hskip 0.2in} r@{\hskip 0.02in} r@{\hskip 0.14in} r@{\hskip 0.02in} r@{\hskip 0.14in} r@{\hskip 0.02in} r@{\hskip 0.14in} r@{\hskip 0.02in} r }
\hline
\abovespace\belowspace
  $\vect \phi$ & \multicolumn{2}{c}{\hskip -0.16in Grad. Descent} &   \multicolumn{2}{c}{\hskip -0.16in Hebb. Descent} &  \multicolumn{2}{c}{\hskip -0.16in Hebb rule} &  \multicolumn{2}{c}{\hskip -0.16in Cov. Rule} \\ 
\hline & & & & & & & &\vspace{-0.3cm}\\ 
\multicolumn{3}{l}{{\emph{RAND$\,\,\rightarrow\,\,$RAND (0.4946)}}}\\ 
 & & & & & & & &  \vspace{-0.35cm}\\ 
Linear  & \textbf{0.1135}  & (0.2620)  & \textbf{0.1135}  & (0.2620)  & 0.6267  & (0.6293)  & 0.6267  & (0.6293)  \\ 
ExpLin  & 0.1579  & (0.2640)  & \textbf{0.1135}  & (0.2491)  & 0.5324  & (0.5342)  & 0.5324  & (0.5342)  \\ 
Rectifier  & 0.2657  & (0.3067)  & \textbf{0.1130}  & (0.2083)  & 0.3239  & (0.3237)  & 0.3239  & (0.3237)  \\ 
Sigmoid  & 0.1614  & (0.1654)  & \textbf{0.0307}  & (0.1411)  & 0.0806  & (0.0798)  & 0.0806  & (0.0798)  \\ 
Step  & 0.5012  & (0.5009)  & \textbf{0.0630}  & (0.1712)  & 0.0806  & (0.0799)  & 0.0806  & (0.0799)  \\ 
\hline & & & & & & & &\vspace{-0.3cm}\\ 
\multicolumn{3}{l}{{\emph{RANDN$\,\,\rightarrow\,\,$RANDN (0.0944)}}}\\ 
 & & & & & & & &  \vspace{-0.35cm}\\ 
Linear  & \textbf{0.0694}  & (0.0880)  & \textbf{0.0694}  & (0.0880)  & 0.5448  & (0.5438)  & 0.5448  & (0.5438)  \\ 
ExpLin  & \textbf{0.0694}  & (0.0880)  & \textbf{0.0694}  & (0.0880)  & 0.5233  & (0.5223)  & 0.5233  & (0.5223)  \\ 
Rectifier  & 0.0699  & (0.0906)  & \textbf{0.0694}  & (0.0880)  & 0.3894  & (0.3879)  & 0.3894  & (0.3879)  \\ 
Sigmoid  & 0.0548  & (0.0689)  & \textbf{0.0546}  & (0.0694)  & 0.0707  & (0.0701)  & 0.0707  & (0.0701)  \\ 
\hline & & & & & & & &\vspace{-0.3cm}\\ 
\multicolumn{3}{l}{{\emph{CONNECT$\,\,\rightarrow\,\,$RAND (0.4946)}}}\\ 
 & & & & & & & &  \vspace{-0.35cm}\\ 
Linear  & \textbf{0.2557}  & (0.3874)  & \textbf{0.2557}  & (0.3874)  & 0.5834  & (0.5831)  & 0.5834  & (0.5831)  \\ 
ExpLin  & 0.2512  & (0.4095)  & \textbf{0.2511}  & (0.4142)  & 0.5586  & (0.5583)  & 0.5586  & (0.5583)  \\ 
Rectifier  & 0.2898  & (0.3789)  & \textbf{0.2214}  & (0.3805)  & 0.3924  & (0.3938)  & 0.3924  & (0.3938)  \\ 
Sigmoid  & 0.2850  & (0.3344)  & \textbf{0.1598}  & (0.3357)  & 0.2986  & (0.3023)  & 0.2986  & (0.3023)  \\ 
Step  & 0.5016  & (0.5002)  & \textbf{0.1744}  & (0.3405)  & 0.2984  & (0.3024)  & 0.2984  & (0.3024)  \\ 
\hline & & & & & & & &\vspace{-0.3cm}\\ 
\multicolumn{3}{l}{{\emph{CONNECT$\,\,\rightarrow\,\,$RANDN (0.0944)}}}\\ 
 & & & & & & & &  \vspace{-0.35cm}\\ 
Linear  & \textbf{0.0761}  & (0.1141)  & \textbf{0.0761}  & (0.1141)  & 0.4786  & (0.4777)  & 0.4786  & (0.4777)  \\ 
ExpLin  & 0.0763  & (0.1147)  & \textbf{0.0761}  & (0.1141)  & 0.4627  & (0.4619)  & 0.4627  & (0.4619)  \\ 
Rectifier  & 0.0930  & (0.1533)  & \textbf{0.0765}  & (0.1150)  & 0.3483  & (0.3481)  & 0.3483  & (0.3481)  \\ 
Sigmoid  & \textbf{0.0599}  & (0.0879)  & 0.0630  & (0.0856)  & 0.1015  & (0.1015)  & 0.1015  & (0.1015)  \\ 
\hline & & & & & & & &\vspace{-0.3cm}\\ 
\multicolumn{3}{l}{{\emph{CONNECT$\,\,\rightarrow\,\,$MNIST (0.1473)}}}\\ 
 & & & & & & & &  \vspace{-0.35cm}\\ 
Linear  & \textbf{0.1037}  & (0.1569)  & \textbf{0.1037}  & (0.1569)  & 0.2371  & (0.2352)  & 0.2371  & (0.2352)  \\ 
ExpLin  & 0.1041  & (0.1548)  & \textbf{0.1034}  & (0.1557)  & 0.2310  & (0.2290)  & 0.2310  & (0.2290)  \\ 
Rectifier  & 0.0895  & (0.1136)  & \textbf{0.0707}  & (0.1215)  & 0.1572  & (0.1557)  & 0.1572  & (0.1557)  \\ 
Sigmoid  & 0.1174  & (0.1224)  & \textbf{0.0696}  & (0.1239)  & 0.4192  & (0.4210)  & 0.4192  & (0.4210)  \\ 
Step  & 0.4999  & (0.4996)  & \textbf{0.0848}  & (0.1337)  & 0.4193  & (0.4206)  & 0.4193  & (0.4206)  \\ 
\hline & & & & & & & &\vspace{-0.3cm}\\ 
\multicolumn{3}{l}{{\emph{MNIST$\,\,\rightarrow\,\,$RAND (0.4946)}}}\\ 
 & & & & & & & &  \vspace{-0.35cm}\\ 
Linear  & \textbf{0.3010}  & (0.4272)  & \textbf{0.3010}  & (0.4272)  & 0.6092  & (0.6051)  & 0.6092  & (0.6051)  \\ 
ExpLin  & \textbf{0.2813}  & (0.4019)  & 0.2832  & (0.4099)  & 0.5729  & (0.5687)  & 0.5729  & (0.5687)  \\ 
Rectifier  & 0.3272  & (0.3904)  & \textbf{0.2289}  & (0.3481)  & 0.4072  & (0.4021)  & 0.4072  & (0.4021)  \\ 
Sigmoid  & 0.2727  & (0.3170)  & \textbf{0.1524}  & (0.2951)  & 0.2959  & (0.2923)  & 0.2959  & (0.2923)  \\ 
Step  & 0.5012  & (0.4999)  & \textbf{0.1680}  & (0.3026)  & 0.2958  & (0.2922)  & 0.2958  & (0.2922)  \\ 
\hline & & & & & & & &\vspace{-0.3cm}\\ 
\multicolumn{3}{l}{{\emph{MNIST$\,\,\rightarrow\,\,$RANDN (0.0944)}}}\\ 
 & & & & & & & &  \vspace{-0.35cm}\\ 
Linear  & \textbf{0.1163}  & (0.1759)  & \textbf{0.1163}  & (0.1759)  & 0.4963  & (0.4961)  & 0.4963  & (0.4961)  \\ 
ExpLin  & 0.1151  & (0.1710)  & \textbf{0.1150}  & (0.1728)  & 0.4722  & (0.4723)  & 0.4722  & (0.4723)  \\ 
Rectifier  & 0.2488  & (0.2881)  & \textbf{0.1164}  & (0.1695)  & 0.3497  & (0.3506)  & 0.3497  & (0.3506)  \\ 
Sigmoid  & 0.0585  & (0.0815)  & \textbf{0.0549}  & (0.0838)  & 0.1075  & (0.1077)  & 0.1075  & (0.1077)  \\ 
\hline & & & & & & & & 
\end{tabular}
\end{sc}
\end{small}
\end{center}
\caption{Online learning performance of hetero-association with centering for various activation functions and datasets. Supplementing Table~\ref{tab:Hetero_online_centered_1}.
} 
\label{tab:Hetero_online_centered_2}
\end{table}

\begin{table}[htbp]
\setlength{\tabcolsep}{2pt}
\begin{center}
\begin{small}
\begin{sc}
\begin{tabular}{l@{\hskip 0.2in} r@{\hskip 0.02in} r@{\hskip 0.14in} r@{\hskip 0.02in} r@{\hskip 0.14in} r@{\hskip 0.02in} r@{\hskip 0.14in} r@{\hskip 0.02in} r }
\hline
\abovespace\belowspace
  $\vect \phi$ & \multicolumn{2}{c}{\hskip -0.16in Grad. Descent} &   \multicolumn{2}{c}{\hskip -0.16in Hebb. Descent} &  \multicolumn{2}{c}{\hskip -0.16in Hebb rule} &  \multicolumn{2}{c}{\hskip -0.16in Cov. Rule} \\ 
\hline & & & & & & & &\vspace{-0.3cm}\\ 
\multicolumn{3}{l}{{\emph{MNIST$\,\,\rightarrow\,\,$ADULT (0.1229)}}}\\ 
 & & & & & & & &  \vspace{-0.35cm}\\ 
Linear  & \textbf{0.1412}  & (0.1810)  & \textbf{0.1412}  & (0.1810)  & 0.3503  & (0.3497)  & 0.3503  & (0.3497)  \\ 
ExpLin  & \textbf{0.1333}  & (0.1929)  & 0.1360  & (0.1988)  & 0.3232  & (0.3226)  & 0.3232  & (0.3226)  \\ 
Rectifier  & 0.0812  & (0.1013)  & \textbf{0.0553}  & (0.0934)  & 0.2018  & (0.2010)  & 0.2018  & (0.2010)  \\ 
Sigmoid  & 0.1006  & (0.1151)  & \textbf{0.0417}  & (0.0907)  & 0.4256  & (0.4206)  & 0.4256  & (0.4206)  \\ 
Step  & 0.4993  & (0.5001)  & \textbf{0.0532}  & (0.1017)  & 0.4254  & (0.4205)  & 0.4254  & (0.4205)  \\ 
\hline & & & & & & & &\vspace{-0.3cm}\\ 
\multicolumn{3}{l}{{\emph{ADULT$\,\,\rightarrow\,\,$RAND (0.4946)}}}\\ 
 & & & & & & & &  \vspace{-0.35cm}\\ 
Linear  & \textbf{0.2517}  & (0.4104)  & \textbf{0.2517}  & (0.4104)  & 0.5738  & (0.5706)  & 0.5738  & (0.5706)  \\ 
ExpLin  & 0.2494  & (0.4022)  & \textbf{0.2484}  & (0.4060)  & 0.5537  & (0.5508)  & 0.5537  & (0.5508)  \\ 
Rectifier  & 0.2802  & (0.3764)  & \textbf{0.2252}  & (0.3782)  & 0.3925  & (0.3910)  & 0.3925  & (0.3910)  \\ 
Sigmoid  & 0.2869  & (0.3444)  & \textbf{0.1624}  & (0.3377)  & 0.2982  & (0.3008)  & 0.2982  & (0.3008)  \\ 
Step  & 0.5018  & (0.4983)  & \textbf{0.1792}  & (0.3418)  & 0.2977  & (0.3014)  & 0.2977  & (0.3014)  \\ 
\hline & & & & & & & &\vspace{-0.3cm}\\ 
\multicolumn{3}{l}{{\emph{ADULT$\,\,\rightarrow\,\,$RANDN (0.0944)}}}\\ 
 & & & & & & & &  \vspace{-0.35cm}\\ 
Linear  & \textbf{0.0738}  & (0.1080)  & \textbf{0.0738}  & (0.1080)  & 0.4713  & (0.4727)  & 0.4713  & (0.4727)  \\ 
ExpLin  & 0.0739  & (0.1083)  & \textbf{0.0738}  & (0.1080)  & 0.4594  & (0.4603)  & 0.4594  & (0.4603)  \\ 
Rectifier  & 0.0810  & (0.1277)  & \textbf{0.0740}  & (0.1084)  & 0.3447  & (0.3449)  & 0.3447  & (0.3449)  \\ 
Sigmoid  & 0.0657  & (0.1006)  & \textbf{0.0616}  & (0.0935)  & 0.0987  & (0.0985)  & 0.0987  & (0.0985)  \\ 
\hline & & & & & & & &\vspace{-0.3cm}\\ 
\multicolumn{3}{l}{{\emph{ADULT$\,\,\rightarrow\,\,$CIFAR (0.167)}}}\\ 
 & & & & & & & &  \vspace{-0.35cm}\\ 
Linear  & \textbf{0.1257}  & (0.1811)  & \textbf{0.1257}  & (0.1811)  & 0.4956  & (0.5008)  & 0.4956  & (0.5008)  \\ 
ExpLin  & \textbf{0.1256}  & (0.1810)  & 0.1257  & (0.1810)  & 0.4855  & (0.4917)  & 0.4855  & (0.4917)  \\ 
Rectifier  & 0.1262  & (0.1838)  & \textbf{0.1256}  & (0.1809)  & 0.3635  & (0.3696)  & 0.3635  & (0.3696)  \\ 
Sigmoid  & 0.1252  & (0.1855)  & \textbf{0.1212}  & (0.1811)  & 0.1770  & (0.1774)  & 0.1770  & (0.1774)  \\ 
\hline & & & & & & & &\vspace{-0.3cm}\\ 
\multicolumn{3}{l}{{\emph{CIFAR$\,\,\rightarrow\,\,$RAND (0.4946)}}}\\ 
 & & & & & & & &  \vspace{-0.35cm}\\ 
Linear  & \textbf{0.3833}  & (0.4468)  & \textbf{0.3833}  & (0.4468)  & 0.6169  & (0.6163)  & 0.6169  & (0.6163)  \\ 
ExpLin  & 0.3818  & (0.4468)  & \textbf{0.3731}  & (0.4381)  & 0.5916  & (0.5920)  & 0.5916  & (0.5920)  \\ 
Rectifier  & 0.3935  & (0.4370)  & \textbf{0.3423}  & (0.4015)  & 0.4606  & (0.4613)  & 0.4606  & (0.4613)  \\ 
Sigmoid  & 0.3299  & (0.3759)  & \textbf{0.2405}  & (0.3460)  & 0.3712  & (0.3770)  & 0.3712  & (0.3770)  \\ 
Step  & 0.5054  & (0.5015)  & \textbf{0.2483}  & (0.3516)  & 0.3710  & (0.3771)  & 0.3710  & (0.3771)  \\ 
\hline & & & & & & & &\vspace{-0.3cm}\\ 
\multicolumn{3}{l}{{\emph{CIFAR$\,\,\rightarrow\,\,$RANDN (0.0944)}}}\\ 
 & & & & & & & &  \vspace{-0.35cm}\\ 
Linear  & \textbf{0.1786}  & (0.2098)  & \textbf{0.1786}  & (0.2098)  & 0.4914  & (0.4925)  & 0.4914  & (0.4925)  \\ 
ExpLin  & 0.1826  & (0.2147)  & \textbf{0.1796}  & (0.2108)  & 0.4716  & (0.4724)  & 0.4716  & (0.4724)  \\ 
Rectifier  & 0.2904  & (0.3137)  & \textbf{0.1751}  & (0.2052)  & 0.3649  & (0.3658)  & 0.3649  & (0.3658)  \\ 
Sigmoid  & \textbf{0.0762}  & (0.0944)  & 0.0765  & (0.0925)  & 0.1110  & (0.1112)  & 0.1110  & (0.1112)  \\ 
\hline & & & & & & & &\vspace{-0.3cm}\\ 
\multicolumn{3}{l}{{\emph{CIFAR$\,\,\rightarrow\,\,$ADULT (0.1229)}}}\\ 
 & & & & & & & &  \vspace{-0.35cm}\\ 
Linear  & \textbf{0.1564}  & (0.1870)  & \textbf{0.1564}  & (0.1870)  & 0.3301  & (0.3252)  & 0.3301  & (0.3252)  \\ 
ExpLin  & 0.1572  & (0.1846)  & \textbf{0.1532}  & (0.1817)  & 0.3087  & (0.3041)  & 0.3087  & (0.3041)  \\ 
Rectifier  & 0.0987  & (0.1193)  & \textbf{0.0839}  & (0.1124)  & 0.2048  & (0.2012)  & 0.2048  & (0.2012)  \\ 
Sigmoid  & 0.1349  & (0.1554)  & \textbf{0.0933}  & (0.1348)  & 0.4531  & (0.4544)  & 0.4531  & (0.4544)  \\ 
Step  & 0.5055  & (0.5010)  & \textbf{0.0982}  & (0.1442)  & 0.4518  & (0.4543)  & 0.4518  & (0.4543)  \\  
\hline & & & & & & &  & 
\end{tabular}
\end{sc}
\end{small}
\end{center}
\caption{Online learning performance of hetero-association with centering for various activation functions and datasets. Supplementing Table~\ref{tab:Hetero_online_centered_1} and~\ref{tab:Hetero_online_centered_2}.
} 
\label{tab:Hetero_online_centered_3}
\end{table}

\begin{table}[htbp]
\setlength{\tabcolsep}{2pt}
\begin{center}
\begin{small}
\begin{sc}
\begin{tabular}{l@{\hskip 0.2in} r@{\hskip 0.02in} r@{\hskip 0.14in} r@{\hskip 0.02in} r@{\hskip 0.14in} r@{\hskip 0.02in} r@{\hskip 0.14in} r@{\hskip 0.02in} r }
\hline
\abovespace\belowspace
  $\vect \phi$ & \multicolumn{2}{c}{\hskip -0.16in Grad. Descent} &   \multicolumn{2}{c}{\hskip -0.16in Hebb. Descent} &  \multicolumn{2}{c}{\hskip -0.16in Hebb rule} &  \multicolumn{2}{c}{\hskip -0.16in Cov. Rule} \\ 
\hline & & & & & & & &\vspace{-0.3cm}\\ 
\multicolumn{3}{l}{{\emph{RAND$\,\,\rightarrow\,\,$RAND (0.4946)}}}\\ 
 & & & & & & & &  \vspace{-0.35cm}\\ 
Linear  & \textbf{0.4159}  & (0.4522)  & \textbf{0.4159}  & (0.4522)  & 0.6927  & (0.6958)  & 0.7389  & (0.7439)  \\ 
ExpLin  & 0.4225  & (0.4582)  & \textbf{0.4100}  & (0.4453)  & 0.6676  & (0.6708)  & 0.6325  & (0.6374)  \\ 
Rectifier  & 0.4267  & (0.4557)  & \textbf{0.3642}  & (0.4005)  & 0.5294  & (0.5325)  & 0.4009  & (0.4040)  \\ 
Sigmoid  & 0.3409  & (0.3730)  & 0.2674  & (0.3084)  & 0.4928  & (0.4938)  & \textbf{0.1667}  & (0.1671)  \\ 
Step  & 0.4984  & (0.5010)  & 0.2701  & (0.3142)  & 0.4922  & (0.4937)  & \textbf{0.1662}  & (0.1671)  \\ 
\hline & & & & & & & &\vspace{-0.3cm}\\ 
\multicolumn{3}{l}{{\emph{RANDN$\,\,\rightarrow\,\,$RANDN (0.0944)}}}\\ 
 & & & & & & & &  \vspace{-0.35cm}\\ 
Linear  & \textbf{0.1327}  & (0.1373)  & \textbf{0.1327}  & (0.1373)  & 0.3970  & (0.3964)  & 0.5989  & (0.5973)  \\ 
ExpLin  & \textbf{0.1327}  & (0.1373)  & \textbf{0.1327}  & (0.1373)  & 0.3898  & (0.3892)  & 0.5542  & (0.5527)  \\ 
Rectifier  & 0.2373  & (0.2420)  & \textbf{0.1327}  & (0.1373)  & 0.3579  & (0.3575)  & 0.4096  & (0.4086)  \\ 
Sigmoid  & 0.0978  & (0.1015)  & \textbf{0.0977}  & (0.1010)  & 0.1327  & (0.1330)  & 0.1323  & (0.1319)  \\ 
\hline & & & & & & & &\vspace{-0.3cm}\\ 
\multicolumn{3}{l}{{\emph{CONNECT$\,\,\rightarrow\,\,$RAND (0.4946)}}}\\ 
 & & & & & & & &  \vspace{-0.35cm}\\ 
Linear  & \textbf{0.4398}  & (0.4671)  & \textbf{0.4398}  & (0.4671)  & 0.5959  & (0.5972)  & 0.6716  & (0.6725)  \\ 
ExpLin  & 0.4395  & (0.4674)  & \textbf{0.4393}  & (0.4667)  & 0.5842  & (0.5842)  & 0.6133  & (0.6145)  \\ 
Rectifier  & 0.4406  & (0.4711)  & \textbf{0.4262}  & (0.4624)  & 0.5071  & (0.5060)  & 0.4285  & (0.4310)  \\ 
Sigmoid  & 0.4106  & (0.4477)  & 0.3857  & (0.4282)  & 0.4959  & (0.4979)  & \textbf{0.3076}  & (0.3127)  \\ 
Step  & 0.5000  & (0.5003)  & 0.3525  & (0.4148)  & 0.4945  & (0.4985)  & \textbf{0.3074}  & (0.3127)  \\ 
\hline & & & & & & & &\vspace{-0.3cm}\\ 
\multicolumn{3}{l}{{\emph{CONNECT$\,\,\rightarrow\,\,$RANDN (0.0944)}}}\\ 
 & & & & & & & &  \vspace{-0.35cm}\\ 
Linear  & \textbf{0.1605}  & (0.1720)  & \textbf{0.1605}  & (0.1720)  & 0.4361  & (0.4350)  & 0.5600  & (0.5593)  \\ 
ExpLin  & 0.1611  & (0.1727)  & \textbf{0.1606}  & (0.1721)  & 0.4197  & (0.4188)  & 0.5107  & (0.5105)  \\ 
Rectifier  & 0.2031  & (0.2160)  & \textbf{0.1610}  & (0.1725)  & 0.3519  & (0.3516)  & 0.3615  & (0.3620)  \\ 
Sigmoid  & 0.0980  & (0.1024)  & \textbf{0.0973}  & (0.1025)  & 0.1410  & (0.1406)  & 0.1326  & (0.1324)  \\ 
\hline & & & & & & & &\vspace{-0.3cm}\\ 
\multicolumn{3}{l}{{\emph{CONNECT$\,\,\rightarrow\,\,$MNIST (0.1473)}}}\\ 
 & & & & & & & &  \vspace{-0.35cm}\\ 
Linear  & \textbf{0.1893}  & (0.2005)  & \textbf{0.1893}  & (0.2005)  & 0.3145  & (0.3134)  & 0.3351  & (0.3337)  \\ 
ExpLin  & 0.1886  & (0.1996)  & \textbf{0.1880}  & (0.1990)  & 0.3007  & (0.2996)  & 0.3132  & (0.3120)  \\ 
Rectifier  & \textbf{0.1270}  & (0.1328)  & 0.1372  & (0.1462)  & 0.2201  & (0.2183)  & 0.1994  & (0.1986)  \\ 
Sigmoid  & 0.1404  & (0.1476)  & \textbf{0.1360}  & (0.1476)  & 0.4733  & (0.4738)  & 0.4280  & (0.4307)  \\ 
Step  & 0.5021  & (0.5028)  & \textbf{0.1466}  & (0.1592)  & 0.5049  & (0.5056)  & 0.4311  & (0.4336)  \\ 
\hline & & & & & & & &\vspace{-0.3cm}\\ 
\multicolumn{3}{l}{{\emph{MNIST$\,\,\rightarrow\,\,$RAND (0.4946)}}}\\ 
 & & & & & & & &  \vspace{-0.35cm}\\ 
Linear  & \textbf{0.3776}  & (0.4408)  & \textbf{0.3776}  & (0.4408)  & 0.5575  & (0.5559)  & 0.6586  & (0.6538)  \\ 
ExpLin  & 0.3767  & (0.4376)  & \textbf{0.3746}  & (0.4371)  & 0.5504  & (0.5488)  & 0.6175  & (0.6130)  \\ 
Rectifier  & 0.3950  & (0.4383)  & \textbf{0.3527}  & (0.4178)  & 0.5051  & (0.5011)  & 0.4581  & (0.4542)  \\ 
Sigmoid  & 0.3503  & (0.3932)  & \textbf{0.2786}  & (0.3581)  & 0.4949  & (0.4960)  & 0.3724  & (0.3700)  \\ 
Step  & 0.5017  & (0.4992)  & \textbf{0.2832}  & (0.3609)  & 0.4925  & (0.4962)  & 0.3720  & (0.3700)  \\ 
\hline & & & & & & & &\vspace{-0.3cm}\\ 
\multicolumn{3}{l}{{\emph{MNIST$\,\,\rightarrow\,\,$RANDN (0.0944)}}}\\ 
 & & & & & & & &  \vspace{-0.35cm}\\ 
Linear  & \textbf{0.1562}  & (0.1859)  & \textbf{0.1562}  & (0.1859)  & 0.3706  & (0.3706)  & 0.5280  & (0.5283)  \\ 
ExpLin  & 0.1574  & (0.1870)  & \textbf{0.1563}  & (0.1859)  & 0.3616  & (0.3619)  & 0.4918  & (0.4928)  \\ 
Rectifier  & 0.1891  & (0.2199)  & \textbf{0.1571}  & (0.1859)  & 0.3182  & (0.3195)  & 0.3585  & (0.3603)  \\ 
Sigmoid  & \textbf{0.0824}  & (0.0946)  & 0.0840  & (0.1019)  & 0.1281  & (0.1279)  & 0.1216  & (0.1217)  \\ 
\hline & & & & & & & &
\end{tabular}
\end{sc}
\end{small}
\end{center}
\caption{Online learning performance of hetero-association without centering for various activation functions and datasets. Supplementing Table~\ref{tab:Hetero_online_uncentered_1}.
} 
\label{tab:Hetero_online_uncentered_2}
\end{table}

\begin{table}[htbp]
\setlength{\tabcolsep}{2pt}
\begin{center}
\begin{small}
\begin{sc}
\begin{tabular}{l@{\hskip 0.2in} r@{\hskip 0.02in} r@{\hskip 0.14in} r@{\hskip 0.02in} r@{\hskip 0.14in} r@{\hskip 0.02in} r@{\hskip 0.14in} r@{\hskip 0.02in} r }
\hline
\abovespace\belowspace
  $\vect \phi$ & \multicolumn{2}{c}{\hskip -0.16in Grad. Descent} &   \multicolumn{2}{c}{\hskip -0.16in Hebb. Descent} &  \multicolumn{2}{c}{\hskip -0.16in Hebb rule} &  \multicolumn{2}{c}{\hskip -0.16in Cov. Rule} \\ 
\hline & & & & & & & &\vspace{-0.3cm}\\ 
\multicolumn{3}{l}{{\emph{MNIST$\,\,\rightarrow\,\,$ADULT (0.1229)}}}\\ 
 & & & & & & & &  \vspace{-0.35cm}\\ 
Linear  & \textbf{0.1955}  & (0.2214)  & \textbf{0.1955}  & (0.2214)  & 0.4044  & (0.4035)  & 0.4276  & (0.4268)  \\ 
ExpLin  & 0.1933  & (0.2202)  & \textbf{0.1885}  & (0.2180)  & 0.3719  & (0.3708)  & 0.3857  & (0.3847)  \\ 
Rectifier  & \textbf{0.0860}  & (0.0893)  & \textbf{0.0860}  & (0.1020)  & 0.2522  & (0.2515)  & 0.2445  & (0.2440)  \\ 
Sigmoid  & 0.0861  & (0.0874)  & \textbf{0.0673}  & (0.0901)  & 0.5003  & (0.5000)  & 0.4428  & (0.4394)  \\ 
Step  & 0.4973  & (0.4995)  & \textbf{0.0673}  & (0.0888)  & 0.4975  & (0.5004)  & 0.4422  & (0.4403)  \\ 
\hline & & & & & & & &\vspace{-0.3cm}\\ 
\multicolumn{3}{l}{{\emph{ADULT$\,\,\rightarrow\,\,$RAND (0.4946)}}}\\ 
 & & & & & & & &  \vspace{-0.35cm}\\ 
Linear  & \textbf{0.3606}  & (0.4208)  & \textbf{0.3606}  & (0.4208)  & 0.5043  & (0.5044)  & 0.6029  & (0.6001)  \\ 
ExpLin  & 0.3607  & (0.4201)  & \textbf{0.3600}  & (0.4200)  & 0.5040  & (0.5042)  & 0.5744  & (0.5720)  \\ 
Rectifier  & 0.3698  & (0.4242)  & \textbf{0.3360}  & (0.4165)  & 0.4953  & (0.4937)  & 0.4107  & (0.4085)  \\ 
Sigmoid  & 0.3574  & (0.4040)  & \textbf{0.2882}  & (0.3793)  & 0.4937  & (0.4958)  & 0.3169  & (0.3172)  \\ 
Step  & 0.5015  & (0.4985)  & \textbf{0.2948}  & (0.3834)  & 0.4938  & (0.4952)  & 0.3167  & (0.3172)  \\ 
\hline & & & & & & & &\vspace{-0.3cm}\\ 
\multicolumn{3}{l}{{\emph{ADULT$\,\,\rightarrow\,\,$RANDN (0.0944)}}}\\ 
 & & & & & & & &  \vspace{-0.35cm}\\ 
Linear  & \textbf{0.1117}  & (0.1330)  & \textbf{0.1117}  & (0.1330)  & 0.2718  & (0.2705)  & 0.4867  & (0.4878)  \\ 
ExpLin  & 0.1118  & (0.1332)  & \textbf{0.1117}  & (0.1330)  & 0.2715  & (0.2702)  & 0.4676  & (0.4685)  \\ 
Rectifier  & 0.1200  & (0.1436)  & \textbf{0.1119}  & (0.1333)  & 0.2662  & (0.2667)  & 0.3528  & (0.3537)  \\ 
Sigmoid  & \textbf{0.0829}  & (0.0935)  & \textbf{0.0829}  & (0.0974)  & 0.1140  & (0.1135)  & 0.1070  & (0.1065)  \\ 
\hline & & & & & & & &\vspace{-0.3cm}\\ 
\multicolumn{3}{l}{{\emph{ADULT$\,\,\rightarrow\,\,$CIFAR (0.167)}}}\\ 
 & & & & & & & &  \vspace{-0.35cm}\\ 
Linear  & \textbf{0.1727}  & (0.1994)  & \textbf{0.1727}  & (0.1994)  & 0.2623  & (0.2588)  & 0.4990  & (0.5037)  \\ 
ExpLin  & \textbf{0.1727}  & (0.1994)  & \textbf{0.1727}  & (0.1994)  & 0.2623  & (0.2588)  & 0.4908  & (0.4963)  \\ 
Rectifier  & 0.1761  & (0.2038)  & \textbf{0.1727}  & (0.1993)  & 0.2623  & (0.2588)  & 0.3648  & (0.3706)  \\ 
Sigmoid  & \textbf{0.1668}  & (0.1894)  & 0.1680  & (0.1979)  & 0.2010  & (0.1989)  & 0.1798  & (0.1804)  \\ 
\hline & & & & & & & &\vspace{-0.3cm}\\ 
\multicolumn{3}{l}{{\emph{CIFAR$\,\,\rightarrow\,\,$RAND (0.4946)}}}\\ 
 & & & & & & & &  \vspace{-0.35cm}\\ 
Linear  & \textbf{0.4762}  & (0.4958)  & \textbf{0.4762}  & (0.4958)  & 0.6921  & (0.6830)  & 0.7936  & (0.7840)  \\ 
ExpLin  & 0.4756  & (0.4961)  & \textbf{0.4754}  & (0.4951)  & 0.6673  & (0.6590)  & 0.7085  & (0.7024)  \\ 
Rectifier  & 0.4686  & (0.4895)  & \textbf{0.4639}  & (0.4849)  & 0.5530  & (0.5481)  & 0.5224  & (0.5206)  \\ 
Sigmoid  & 0.4322  & (0.4662)  & \textbf{0.4257}  & (0.4545)  & 0.4959  & (0.4985)  & 0.4638  & (0.4701)  \\ 
Step  & 0.5004  & (0.4998)  & \textbf{0.4252}  & (0.4535)  & 0.4932  & (0.4986)  & 0.4633  & (0.4703)  \\ 
\hline & & & & & & & &\vspace{-0.3cm}\\ 
\multicolumn{3}{l}{{\emph{CIFAR$\,\,\rightarrow\,\,$RANDN (0.0944)}}}\\ 
 & & & & & & & &  \vspace{-0.35cm}\\ 
Linear  & \textbf{0.2109}  & (0.2193)  & \textbf{0.2109}  & (0.2193)  & 0.5501  & (0.5538)  & 0.6704  & (0.6743)  \\ 
ExpLin  & 0.2120  & (0.2207)  & \textbf{0.2107}  & (0.2192)  & 0.5250  & (0.5285)  & 0.5881  & (0.5907)  \\ 
Rectifier  & 0.2645  & (0.2733)  & \textbf{0.2083}  & (0.2168)  & 0.4157  & (0.4178)  & 0.4052  & (0.4066)  \\ 
Sigmoid  & \textbf{0.1023}  & (0.1067)  & 0.1025  & (0.1064)  & 0.1697  & (0.1702)  & 0.1571  & (0.1577)  \\ 
\hline & & & & & & & &\vspace{-0.3cm}\\ 
\multicolumn{3}{l}{{\emph{CIFAR$\,\,\rightarrow\,\,$ADULT (0.1229)}}}\\ 
 & & & & & & & &  \vspace{-0.35cm}\\ 
Linear  & \textbf{0.2481}  & (0.2563)  & \textbf{0.2481}  & (0.2563)  & 0.6214  & (0.6139)  & 0.6419  & (0.6349)  \\ 
ExpLin  & 0.2391  & (0.2476)  & \textbf{0.2386}  & (0.2465)  & 0.5402  & (0.5347)  & 0.5463  & (0.5409)  \\ 
Rectifier  & \textbf{0.0944}  & (0.0966)  & 0.1125  & (0.1121)  & 0.3652  & (0.3606)  & 0.3555  & (0.3511)  \\ 
Sigmoid  & \textbf{0.0969}  & (0.0986)  & 0.1034  & (0.1107)  & 0.5044  & (0.5039)  & 0.4829  & (0.4847)  \\ 
Step  & 0.5130  & (0.5123)  & \textbf{0.1017}  & (0.1103)  & 0.5152  & (0.5145)  & 0.4865  & (0.4870)  \\ 
\hline & & & & & & & & 
\end{tabular}
\end{sc}
\end{small}
\end{center}
\caption{Online learning performance of hetero-association without centering for various activation functions and datasets. Supplementing Table~\ref{tab:Hetero_online_uncentered_1} and~\ref{tab:Hetero_online_uncentered_2}.
} 
\label{tab:Hetero_online_uncentered_3}
\end{table}

\begin{table}[htbp]
\setlength{\tabcolsep}{2pt}
\begin{center}
\begin{small}
\begin{sc}
\begin{tabular}{l@{\hskip 0.17in} r@{\hskip 0.02in} r@{\hskip 0.1in} r@{\hskip 0.01in} r@{\hskip 0.17in} r@{\hskip 0.01in} r@{\hskip 0.1in} r@{\hskip 0.01in} r@{\hskip 0.02in} }
\hline
\abovespace\belowspace
      & \multicolumn{4}{c}{\hskip -0.16in Gradient Descent} &   \multicolumn{4}{c}{\hskip -0.16in Hebbian-Descent} \\ 
  $\vect \phi$ & \multicolumn{2}{c}{\hskip -0.16in Centered} &   \multicolumn{2}{c}{\hskip -0.16in Uncentered} &  \multicolumn{2}{c}{\hskip -0.16in Centered} &  \multicolumn{2}{c}{\hskip -0.16in Uncentered} \\ 
\hline & & & & & & & &\vspace{-0.3cm}\\ 
\multicolumn{3}{l}{{\emph{RAND$\,\,\rightarrow\,\,$RAND (0.4946)}}}\\ 
 & & & & & & & &  \vspace{-0.35cm}\\ 
Linear  &  \textbf{0.0000} & $\pm$ 0.0000  & 0.0384  & $\pm$ 0.0005  & \textbf{0.0000} & $\pm$ 0.0000 & 0.0384   & $\pm$ 0.0005  \\ 
ExpLin  &  \textbf{0.0000} & $\pm$ 0.0000  & 0.0387  & $\pm$ 0.0005  & \textbf{0.0000} & $\pm$ 0.0000 & 0.0383   & $\pm$ 0.0005  \\ 
Rectifier  &  0.1768  & $\pm$ 0.0027  & 0.1202  & $\pm$ 0.0069  & \textbf{0.0000} & $\pm$ 0.0000 & 0.0068   & $\pm$ 0.0001  \\ 
Sigmoid  &  0.0156  & $\pm$ 0.0002  & 0.0159  & $\pm$ 0.0003  & \textbf{0.0000} & $\pm$ 0.0000 & \textbf{0.0000} & $\pm$ 0.0000  \\ 
Step  &  0.5002  & $\pm$ 0.0045  & 0.4996  & $\pm$ 0.0039  & \textbf{0.0000} & $\pm$ 0.0000 & \textbf{0.0000} & $\pm$ 0.0000  \\ 
\hline & & & & & & & &\vspace{-0.3cm}\\ 
\multicolumn{3}{l}{{\emph{RANDN$\,\,\rightarrow\,\,$RANDN (0.0944)}}}\\ 
 & & & & & & & &  \vspace{-0.35cm}\\ 
Linear  &  \textbf{0.0004} & $\pm$ 0.0000  & 0.0561  & $\pm$ 0.0005  & \textbf{0.0004} & $\pm$ 0.0000 & 0.0561   & $\pm$ 0.0005  \\ 
ExpLin  &  \textbf{0.0004} & $\pm$ 0.0000  & 0.0561  & $\pm$ 0.0005  & \textbf{0.0004} & $\pm$ 0.0000 & 0.0561   & $\pm$ 0.0005  \\ 
Rectifier  &  \textbf{0.0004} & $\pm$ 0.0000  & 0.1852  & $\pm$ 0.0149  & \textbf{0.0004} & $\pm$ 0.0000 & 0.0561   & $\pm$ 0.0005  \\ 
Sigmoid  &  0.0004  & $\pm$ 0.0000  & 0.0502  & $\pm$ 0.0001  & \textbf{0.0003} & $\pm$ 0.0000 & 0.0504   & $\pm$ 0.0001  \\ 
\hline & & & & & & & &\vspace{-0.3cm}\\ 
\multicolumn{3}{l}{{\emph{CONNECT$\,\,\rightarrow\,\,$RAND (0.4946)}}}\\ 
 & & & & & & & &  \vspace{-0.35cm}\\ 
Linear  &  \textbf{0.2435} & $\pm$ 0.0001  & 0.2952  & $\pm$ 0.0002  & \textbf{0.2435} & $\pm$ 0.0001 & 0.2952   & $\pm$ 0.0002  \\ 
ExpLin  &  \textbf{0.2419} & $\pm$ 0.0001  & 0.2940  & $\pm$ 0.0003  & 0.2422   & $\pm$ 0.0001 & 0.2942   & $\pm$ 0.0002  \\ 
Rectifier  &  \textbf{0.1760} & $\pm$ 0.0016  & 0.2967  & $\pm$ 0.0031  & 0.2064   & $\pm$ 0.0000 & 0.2733   & $\pm$ 0.0002  \\ 
Sigmoid  &  0.1029  & $\pm$ 0.0006  & 0.1796  & $\pm$ 0.0027  & \textbf{0.0951} & $\pm$ 0.0008 & 0.2190   & $\pm$ 0.0036  \\ 
Step  &  0.4987  & $\pm$ 0.0037  & 0.4990  & $\pm$ 0.0035  & \textbf{0.0900} & $\pm$ 0.0020 & 0.2175   & $\pm$ 0.0035  \\ 
\hline & & & & & & & &\vspace{-0.3cm}\\ 
\multicolumn{3}{l}{{\emph{CONNECT$\,\,\rightarrow\,\,$RANDN (0.0944)}}}\\ 
 & & & & & & & &  \vspace{-0.35cm}\\ 
Linear  &  \textbf{0.0563} & $\pm$ 0.0000  & 0.0707  & $\pm$ 0.0003  & \textbf{0.0563} & $\pm$ 0.0000 & 0.0707   & $\pm$ 0.0003  \\ 
ExpLin  &  \textbf{0.0563} & $\pm$ 0.0000  & 0.0707  & $\pm$ 0.0003  & \textbf{0.0563} & $\pm$ 0.0000 & 0.0707   & $\pm$ 0.0003  \\ 
Rectifier  &  0.0564  & $\pm$ 0.0000  & 0.1062  & $\pm$ 0.0078  & \textbf{0.0563} & $\pm$ 0.0000 & 0.0707   & $\pm$ 0.0003  \\ 
Sigmoid  &  \textbf{0.0561} & $\pm$ 0.0000  & 0.0680  & $\pm$ 0.0001  & \textbf{0.0561} & $\pm$ 0.0000 & 0.0681   & $\pm$ 0.0001  \\ 
\hline & & & & & & & &\vspace{-0.3cm}\\ 
\multicolumn{3}{l}{{\emph{CONNECT$\,\,\rightarrow\,\,$MNIST (0.1473)}}}\\ 
 & & & & & & & &  \vspace{-0.35cm}\\ 
Linear  &  \textbf{0.0908} & $\pm$ 0.0000  & 0.1116  & $\pm$ 0.0001  & \textbf{0.0908} & $\pm$ 0.0000 & 0.1116   & $\pm$ 0.0001  \\ 
ExpLin  &  \textbf{0.0901} & $\pm$ 0.0000  & 0.1112  & $\pm$ 0.0001  & \textbf{0.0901} & $\pm$ 0.0000 & 0.1110   & $\pm$ 0.0001  \\ 
Rectifier  &  0.0540  & $\pm$ 0.0005  & 0.0892  & $\pm$ 0.0012  & \textbf{0.0530} & $\pm$ 0.0000 & 0.0810   & $\pm$ 0.0000  \\ 
Sigmoid  &  \textbf{0.0349} & $\pm$ 0.0001  & 0.0631  & $\pm$ 0.0010  & 0.0425   & $\pm$ 0.0000 & 0.0741   & $\pm$ 0.0000  \\ 
Step  &  0.5001  & $\pm$ 0.0014  & 0.5052  & $\pm$ 0.0131  & \textbf{0.0600} & $\pm$ 0.0003 & 0.0926   & $\pm$ 0.0005  \\ 
\hline & & & & & & & &\vspace{-0.3cm}\\ 
\multicolumn{3}{l}{{\emph{MNIST$\,\,\rightarrow\,\,$RAND (0.4946)}}}\\ 
 & & & & & & & &  \vspace{-0.35cm}\\ 
Linear  &  \textbf{0.0113} & $\pm$ 0.0001  & 0.0790  & $\pm$ 0.0004  & \textbf{0.0113} & $\pm$ 0.0001 & 0.0790   & $\pm$ 0.0004  \\ 
ExpLin  &  \textbf{0.0109} & $\pm$ 0.0000  & 0.0784  & $\pm$ 0.0004  & 0.0111   & $\pm$ 0.0001 & 0.0784   & $\pm$ 0.0004  \\ 
Rectifier  &  0.1121  & $\pm$ 0.0020  & 0.0821  & $\pm$ 0.0023  & \textbf{0.0011} & $\pm$ 0.0000 & 0.0402   & $\pm$ 0.0004  \\ 
Sigmoid  &  0.0403  & $\pm$ 0.0004  & 0.0470  & $\pm$ 0.0007  & \textbf{0.0000} & $\pm$ 0.0000 & \textbf{0.0000} & $\pm$ 0.0000  \\ 
Step  &  0.5004  & $\pm$ 0.0028  & 0.5016  & $\pm$ 0.0031  & \textbf{0.0000} & $\pm$ 0.0000 & \textbf{0.0000} & $\pm$ 0.0000  \\ 
\hline & & & & & & & &\vspace{-0.3cm}\\ 
\multicolumn{3}{l}{{\emph{MNIST$\,\,\rightarrow\,\,$RANDN (0.0944)}}}\\ 
 & & & & & & & &  \vspace{-0.35cm}\\ 
Linear  &  \textbf{0.0002} & $\pm$ 0.0000  & 0.0206  & $\pm$ 0.0003  & \textbf{0.0002} & $\pm$ 0.0000 & 0.0206   & $\pm$ 0.0003  \\ 
ExpLin  &  \textbf{0.0002} & $\pm$ 0.0000  & 0.0207  & $\pm$ 0.0003  & \textbf{0.0002} & $\pm$ 0.0000 & 0.0206   & $\pm$ 0.0003  \\ 
Rectifier  &  0.0086  & $\pm$ 0.0011  & 0.0259  & $\pm$ 0.0016  & \textbf{0.0002} & $\pm$ 0.0000 & 0.0206   & $\pm$ 0.0003  \\ 
Sigmoid  &  0.0003  & $\pm$ 0.0000  & 0.0160  & $\pm$ 0.0001  & \textbf{0.0002} & $\pm$ 0.0000 & 0.0160   & $\pm$ 0.0001  \\ 
\hline & & & & & & & &\vspace{-0.3cm}\\ 
\end{tabular}
\end{sc}
\end{small}
\end{center}
\caption{Multi epoch learning performance of hetero-association with and without centering for various activation functions and datasets. Supplementing Table~\ref{tab:Hetero_offline_1}.}
\label{tab:Hetero_offline_2}
\end{table}

\begin{table}[htbp]
\setlength{\tabcolsep}{2pt}
\begin{center}
\begin{small}
\begin{sc}
\begin{tabular}{l@{\hskip 0.17in} r@{\hskip 0.02in} r@{\hskip 0.1in} r@{\hskip 0.01in} r@{\hskip 0.17in} r@{\hskip 0.01in} r@{\hskip 0.1in} r@{\hskip 0.01in} r@{\hskip 0.02in} }
\hline
\abovespace\belowspace
      & \multicolumn{4}{c}{\hskip -0.16in Gradient Descent} &   \multicolumn{4}{c}{\hskip -0.16in Hebbian-Descent} \\ 
  $\vect \phi$ & \multicolumn{2}{c}{\hskip -0.16in Centered} &   \multicolumn{2}{c}{\hskip -0.16in Uncentered} &  \multicolumn{2}{c}{\hskip -0.16in Centered} &  \multicolumn{2}{c}{\hskip -0.16in Uncentered} \\ 
\hline & & & & & & & &\vspace{-0.3cm}\\ 
\multicolumn{3}{l}{{\emph{MNIST$\,\,\rightarrow\,\,$ADULT (0.1229)}}}\\ 
 & & & & & & & &  \vspace{-0.35cm}\\ 
Linear  &  \textbf{0.0003} & $\pm$ 0.0000  & 0.0350  & $\pm$ 0.0007  & \textbf{0.0003} & $\pm$ 0.0000 & 0.0350   & $\pm$ 0.0007  \\ 
ExpLin  &  \textbf{0.0003} & $\pm$ 0.0000  & 0.0346  & $\pm$ 0.0007  & \textbf{0.0003} & $\pm$ 0.0000 & 0.0345   & $\pm$ 0.0007  \\ 
Rectifier  &  0.0263  & $\pm$ 0.0014  & 0.0408  & $\pm$ 0.0024  & \textbf{0.0000} & $\pm$ 0.0000 & 0.0045   & $\pm$ 0.0003  \\ 
Sigmoid  &  0.0124  & $\pm$ 0.0004  & 0.0170  & $\pm$ 0.0005  & \textbf{0.0000} & $\pm$ 0.0000 & \textbf{0.0000} & $\pm$ 0.0000  \\ 
Step  &  0.4989  & $\pm$ 0.0043  & 0.5043  & $\pm$ 0.0199  & \textbf{0.0000} & $\pm$ 0.0000 & \textbf{0.0000} & $\pm$ 0.0000  \\ 
\hline & & & & & & & &\vspace{-0.3cm}\\ 
\multicolumn{3}{l}{{\emph{ADULT$\,\,\rightarrow\,\,$RAND (0.4946)}}}\\ 
 & & & & & & & &  \vspace{-0.35cm}\\ 
Linear  &  \textbf{0.2645} & $\pm$ 0.0000  & 0.2835  & $\pm$ 0.0001  & \textbf{0.2645} & $\pm$ 0.0000 & 0.2835   & $\pm$ 0.0001  \\ 
ExpLin  &  \textbf{0.2629} & $\pm$ 0.0000  & 0.2820  & $\pm$ 0.0001  & 0.2634   & $\pm$ 0.0000 & 0.2824   & $\pm$ 0.0001  \\ 
Rectifier  &  \textbf{0.1980} & $\pm$ 0.0020  & 0.2571  & $\pm$ 0.0028  & 0.2319   & $\pm$ 0.0000 & 0.2621   & $\pm$ 0.0001  \\ 
Sigmoid  &  \textbf{0.1079} & $\pm$ 0.0002  & 0.1238  & $\pm$ 0.0012  & 0.1411   & $\pm$ 0.0014 & 0.2066   & $\pm$ 0.0000  \\ 
Step  &  0.4996  & $\pm$ 0.0026  & 0.4991  & $\pm$ 0.0019  & \textbf{0.1292} & $\pm$ 0.0025 & 0.1962   & $\pm$ 0.0017  \\ 
\hline & & & & & & & &\vspace{-0.3cm}\\ 
\multicolumn{3}{l}{{\emph{ADULT$\,\,\rightarrow\,\,$RANDN (0.0944)}}}\\ 
 & & & & & & & &  \vspace{-0.35cm}\\ 
Linear  &  \textbf{0.0602} & $\pm$ 0.0000  & 0.0650  & $\pm$ 0.0001  & \textbf{0.0602} & $\pm$ 0.0000 & 0.0650   & $\pm$ 0.0001  \\ 
ExpLin  &  \textbf{0.0602} & $\pm$ 0.0000  & 0.0650  & $\pm$ 0.0001  & \textbf{0.0602} & $\pm$ 0.0000 & 0.0650   & $\pm$ 0.0001  \\ 
Rectifier  &  \textbf{0.0602} & $\pm$ 0.0000  & 0.0667  & $\pm$ 0.0017  & \textbf{0.0602} & $\pm$ 0.0000 & 0.0650   & $\pm$ 0.0001  \\ 
Sigmoid  &  \textbf{0.0597} & $\pm$ 0.0000  & 0.0644  & $\pm$ 0.0000  & \textbf{0.0597} & $\pm$ 0.0000 & 0.0642   & $\pm$ 0.0000  \\ 
\hline & & & & & & & &\vspace{-0.3cm}\\ 
\multicolumn{3}{l}{{\emph{ADULT$\,\,\rightarrow\,\,$CIFAR (0.167)}}}\\ 
 & & & & & & & &  \vspace{-0.35cm}\\ 
Linear  &  \textbf{0.1197} & $\pm$ 0.0000  & 0.1260  & $\pm$ 0.0000  & \textbf{0.1197} & $\pm$ 0.0000 & 0.1260   & $\pm$ 0.0000  \\ 
ExpLin  &  \textbf{0.1197} & $\pm$ 0.0000  & 0.1260  & $\pm$ 0.0000  & \textbf{0.1197} & $\pm$ 0.0000 & 0.1260   & $\pm$ 0.0000  \\ 
Rectifier  &  \textbf{0.1196} & $\pm$ 0.0000  & 0.1284  & $\pm$ 0.0007  & \textbf{0.1196} & $\pm$ 0.0000 & 0.1260   & $\pm$ 0.0000  \\ 
Sigmoid  &  \textbf{0.1172} & $\pm$ 0.0000  & 0.1254  & $\pm$ 0.0000  & 0.1185   & $\pm$ 0.0000 & 0.1249   & $\pm$ 0.0000  \\ 
\hline & & & & & & & &\vspace{-0.3cm}\\ 
\multicolumn{3}{l}{{\emph{CIFAR$\,\,\rightarrow\,\,$RAND (0.4946)}}}\\ 
 & & & & & & & &  \vspace{-0.35cm}\\ 
Linear  &  \textbf{0.0366} & $\pm$ 0.0003  & 0.2258  & $\pm$ 0.0018  & \textbf{0.0366} & $\pm$ 0.0003 & 0.2258   & $\pm$ 0.0018  \\ 
ExpLin  &  \textbf{0.0359} & $\pm$ 0.0002  & 0.2214  & $\pm$ 0.0018  & 0.0361   & $\pm$ 0.0003 & 0.2218   & $\pm$ 0.0018  \\ 
Rectifier  &  0.0660  & $\pm$ 0.0013  & 0.2529  & $\pm$ 0.0056  & \textbf{0.0114} & $\pm$ 0.0001 & 0.1878   & $\pm$ 0.0014  \\ 
Sigmoid  &  0.0371  & $\pm$ 0.0006  & 0.1375  & $\pm$ 0.0012  & \textbf{0.0000} & $\pm$ 0.0000 & 0.0376   & $\pm$ 0.0007  \\ 
Step  &  0.5011  & $\pm$ 0.0039  & 0.5001  & $\pm$ 0.0028  & \textbf{0.0000} & $\pm$ 0.0000 & 0.0525   & $\pm$ 0.0034  \\ 
\hline & & & & & & & &\vspace{-0.3cm}\\ 
\multicolumn{3}{l}{{\emph{CIFAR$\,\,\rightarrow\,\,$RANDN (0.0944)}}}\\ 
 & & & & & & & &  \vspace{-0.35cm}\\ 
Linear  &  \textbf{0.0069} & $\pm$ 0.0002  & 0.0720  & $\pm$ 0.0013  & \textbf{0.0069} & $\pm$ 0.0002 & 0.0720   & $\pm$ 0.0013  \\ 
ExpLin  &  \textbf{0.0069} & $\pm$ 0.0002  & 0.0720  & $\pm$ 0.0013  & \textbf{0.0069} & $\pm$ 0.0002 & 0.0720   & $\pm$ 0.0013  \\ 
Rectifier  &  0.0093  & $\pm$ 0.0003  & 0.1195  & $\pm$ 0.0060  & \textbf{0.0069} & $\pm$ 0.0002 & 0.0720   & $\pm$ 0.0013  \\ 
Sigmoid  &  0.0064  & $\pm$ 0.0000  & 0.0530  & $\pm$ 0.0003  & \textbf{0.0061} & $\pm$ 0.0000 & 0.0535   & $\pm$ 0.0002  \\ 
\hline & & & & & & & &\vspace{-0.3cm}\\ 
\multicolumn{3}{l}{{\emph{CIFAR$\,\,\rightarrow\,\,$ADULT (0.1229)}}}\\ 
 & & & & & & & &  \vspace{-0.35cm}\\ 
Linear  &  \textbf{0.0148} & $\pm$ 0.0003  & 0.1019  & $\pm$ 0.0011  & \textbf{0.0148} & $\pm$ 0.0003 & 0.1019   & $\pm$ 0.0011  \\ 
ExpLin  &  \textbf{0.0145} & $\pm$ 0.0003  & 0.1013  & $\pm$ 0.0010  & 0.0146   & $\pm$ 0.0003 & 0.1006   & $\pm$ 0.0010  \\ 
Rectifier  &  0.0255  & $\pm$ 0.0008  & 0.0708  & $\pm$ 0.0029  & \textbf{0.0027} & $\pm$ 0.0001 & 0.0471   & $\pm$ 0.0006  \\ 
Sigmoid  &  0.0134  & $\pm$ 0.0004  & 0.0435  & $\pm$ 0.0018  & \textbf{0.0000} & $\pm$ 0.0000 & 0.0148   & $\pm$ 0.0005  \\ 
Step  &  0.4998  & $\pm$ 0.0040  & 0.5037  & $\pm$ 0.0296  & \textbf{0.0000} & $\pm$ 0.0000 & 0.0168   & $\pm$ 0.0030  \\ 
\hline & & & & & & & &\vspace{-0.3cm}\\ 
\end{tabular}
\end{sc}
\end{small}
\end{center}
\caption{Multi epoch learning performance of hetero-association with and without centering for various activation functions and datasets. Supplementing Table~\ref{tab:Hetero_offline_1} and \ref{tab:Hetero_offline_2}.}
\label{tab:Hetero_offline_3}
\end{table}

\begin{table}[t]
\setlength{\tabcolsep}{2pt}
\begin{center}
\begin{small}
\begin{sc}
\begin{tabular}{l@{\hskip 0.2in} r@{\hskip 0.02in} r@{\hskip 0.14in} r@{\hskip 0.02in} r@{\hskip 0.14in} r@{\hskip 0.02in} r@{\hskip 0.14in} r@{\hskip 0.02in} r }
\hline
\abovespace\belowspace
  $\vect \phi$ & \multicolumn{2}{c}{\hskip -0.16in Grad. Descent} &   \multicolumn{2}{c}{\hskip -0.16in Hebb. Descent} &  \multicolumn{2}{c}{\hskip -0.16in Hebb rule} &  \multicolumn{2}{c}{\hskip -0.16in Cov. Rule} \\ 
\hline & & & & & & & &\vspace{-0.3cm}\\ 
\multicolumn{3}{l}{{\emph{ADULT (0.2399)}}}\\ 
 & & & & & & & &  \vspace{-0.35cm}\\ 
Linear  & \textbf{0.1564}  & $\pm$0.0003  & \textbf{0.1564}  & $\pm$0.0003  & 0.2399  & $\pm$0.0000  & 0.2957  & $\pm$0.0046  \\ 
Rectifier  & 0.1549  & $\pm$0.0006  & \textbf{0.1545}  & $\pm$0.0005  & 0.2399  & $\pm$0.0000  & 0.2795  & $\pm$0.0348  \\ 
ExpLin  & \textbf{0.1561}  & $\pm$0.0003  & 0.1564  & $\pm$0.0003  & 0.2399  & $\pm$0.0000  & 0.2957  & $\pm$0.0046  \\ 
Sigmoid  & \textbf{0.1541}  & $\pm$0.0002  & 0.1542  & $\pm$0.0003  & 0.2399  & $\pm$0.0000  & 0.2957  & $\pm$0.0046  \\ 
Softmax  & \textbf{0.1541}  & $\pm$0.0002  & 0.1543  & $\pm$0.0002  & 0.2399  & $\pm$0.0000  & 0.2957  & $\pm$0.0046  \\ 
Step  & 0.3419  & $\pm$0.0907  & \textbf{0.1730}  & $\pm$0.0026  & 0.2398  & $\pm$0.0016  & 0.2556  & $\pm$0.0366  \\ 
\hline & & & & & & & &\vspace{-0.3cm}\\ 
\multicolumn{3}{l}{{\emph{CONNECT (0.3409)}}}\\ 
 & & & & & & & &  \vspace{-0.35cm}\\ 
Linear  & \textbf{0.2475}  & $\pm$0.0005  & \textbf{0.2475}  & $\pm$0.0005  & 0.3409  & $\pm$0.0000  & 0.4306  & $\pm$0.0106  \\ 
Rectifier  & \textbf{0.2456}  & $\pm$0.0004  & 0.2461  & $\pm$0.0003  & 0.3409  & $\pm$0.0000  & 0.4310  & $\pm$0.0043  \\ 
ExpLin  & \textbf{0.2473}  & $\pm$0.0005  & 0.2475  & $\pm$0.0005  & 0.3409  & $\pm$0.0000  & 0.4306  & $\pm$0.0106  \\ 
Sigmoid  & \textbf{0.2409}  & $\pm$0.0003  & 0.2445  & $\pm$0.0001  & 0.3409  & $\pm$0.0000  & 0.4306  & $\pm$0.0106  \\ 
Softmax  & \textbf{0.2416}  & $\pm$0.0001  & 0.2444  & $\pm$0.0001  & 0.3409  & $\pm$0.0000  & 0.4306  & $\pm$0.0106  \\ 
Step  & 0.7066  & $\pm$0.0802  & \textbf{0.2983}  & $\pm$0.0289  & 0.7537  & $\pm$0.0000  & 0.5197  & $\pm$0.0377  \\ 
\hline & & & & & & & &\vspace{-0.3cm}\\ 
\multicolumn{3}{l}{{\emph{MNIST (0.8895)}}}\\ 
 & & & & & & & &  \vspace{-0.35cm}\\ 
Linear  & \textbf{0.1423}  & $\pm$0.0012  & \textbf{0.1423}  & $\pm$0.0012  & 0.3191  & $\pm$0.0009  & 0.3536  & $\pm$0.0038  \\ 
Rectifier  & 0.1035  & $\pm$0.0379  & \textbf{0.0846}  & $\pm$0.0010  & 0.3191  & $\pm$0.0009  & 0.3536  & $\pm$0.0038  \\ 
ExpLin  & \textbf{0.1375}  & $\pm$0.0011  & 0.1407  & $\pm$0.0010  & 0.3192  & $\pm$0.0015  & 0.3536  & $\pm$0.0038  \\ 
Sigmoid  & \textbf{0.0766}  & $\pm$0.0002  & 0.0809  & $\pm$0.0001  & 0.3285  & $\pm$0.0117  & 0.3542  & $\pm$0.0074  \\ 
Softmax  & \textbf{0.0700}  & $\pm$0.0003  & 0.0741  & $\pm$0.0003  & 0.3191  & $\pm$0.0009  & 0.3536  & $\pm$0.0038  \\ 
Step  & 0.9176  & $\pm$0.0286  & \textbf{0.1801}  & $\pm$0.0042  & 0.9020  & $\pm$0.0000  & 0.7635  & $\pm$0.0001  \\ 
\hline & & & & & & & &\vspace{-0.3cm}\\ 
\multicolumn{3}{l}{{\emph{CIFAR (0.9)}}}\\ 
 & & & & & & & &  \vspace{-0.35cm}\\ 
Linear  & \textbf{0.7485}  & $\pm$0.0021  & \textbf{0.7485}  & $\pm$0.0021  & 0.7663  & $\pm$0.0003  & 0.7663  & $\pm$0.0003  \\ 
Rectifier  & 0.7642  & $\pm$0.0175  & \textbf{0.7347}  & $\pm$0.0032  & 0.8960  & $\pm$0.0046  & 0.8960  & $\pm$0.0046  \\ 
ExpLin  & \textbf{0.7483}  & $\pm$0.0023  & 0.7485  & $\pm$0.0024  & 0.7664  & $\pm$0.0019  & 0.7664  & $\pm$0.0019  \\ 
Sigmoid  & \textbf{0.6922}  & $\pm$0.0008  & 0.7090  & $\pm$0.0013  & 0.7694  & $\pm$0.0022  & 0.7694  & $\pm$0.0022  \\ 
Softmax  & \textbf{0.6863}  & $\pm$0.0015  & 0.7002  & $\pm$0.0014  & 0.8960  & $\pm$0.0046  & 0.8960  & $\pm$0.0046  \\ 
Step  & 0.8961  & $\pm$0.0074  & \textbf{0.8470}  & $\pm$0.0031  & 0.9000  & $\pm$0.0000  & 0.9000  & $\pm$0.0000  \\ 
\hline & & & & & & & &\vspace{-0.3cm}\\ 
\end{tabular}
\end{sc}
\end{small}
\end{center}
\caption{The same classification experiments as in Table~\ref{tab:label_learning_centered} but without centering.
} 
\vspace{5.5cm}
\label{tab:label_learning_uncentered}
\end{table}

\end{appendices}

\end{document}